\documentclass[11pt]{article}
\usepackage[utf8]{inputenc}
\usepackage[top=1in, left=1in, right=1in, bottom=1in]{geometry}

\usepackage[utf8]{inputenc} 
\usepackage[T1]{fontenc}    
\usepackage{booktabs}       
\usepackage{amsfonts}       
\usepackage{nicefrac}       
\usepackage{microtype}
\usepackage{xcolor}         

\usepackage{etoc}
\usepackage[colorlinks,linkcolor=red,filecolor=blue,citecolor=blue,urlcolor=blue]{hyperref}
\usepackage{booktabs}       
\usepackage{amsfonts}       
\usepackage{amsmath}
\usepackage{graphicx}
\usepackage{amsthm}
\usepackage{amssymb}
\usepackage{multirow}
\usepackage{boldline}
\usepackage{makecell}
\usepackage{algpseudocode}
\usepackage{algorithm}
\usepackage[round]{natbib}
\usepackage{amssymb}%
\usepackage{mathtools}
\usepackage{stmaryrd}
\usepackage[T1]{fontenc}
\usepackage{xargs}
\usepackage{xcolor}
\usepackage{wrapfig}
\usepackage{bm,bbm}

\newtheorem{Lemma}{Lemma}
\newtheorem{proposition}{Proposition}
\newtheorem{definition}{Definition}
\newtheorem{Theorem}{Theorem}

\newtheorem{Corollary}{Corollary}

\newtheorem{assumption}{Assumption}
\newtheorem{remark}{Remark}
\newtheorem*{Lemma*}{Lemma}
\newtheorem*{Theorem*}{Theorem}
\newtheorem*{Corollary*}{Corollary}

\usepackage[utf8]{inputenc}
\usepackage{chngcntr}
\usepackage{apptools}

\AtAppendix{\counterwithin{Lemma}{section}}
\AtAppendix{\counterwithin{Theorem}{section}}
\AtAppendix{\counterwithin{proposition}{section}}

\newcommand{\beq}{\begin{equation}}
\newcommand{\eeq}{\end{equation}}
\newcommand{\del}{\Delta}
\newcommand{\eqdef}{\mathrel{\mathop:}=}
\def\EE{\mathbb{E}}
\newcommand{\norm}[1]{\left\Vert #1 \right\Vert}

\begin{document}
\etocdepthtag.toc{mtchapter}
\etocsettagdepth{mtchapter}{subsection}
\etocsettagdepth{mtappendix}{none}

\title{\bf Analysis of Error Feedback in Federated\\ Non-Convex Optimization with Biased Compression}

\author{\vspace{0.3in}\\\\
\textbf{Xiaoyun Li}, \ \textbf{Ping Li} \\\\
LinkedIn Ads\\
700 Bellevue Way NE, Bellevue, WA 98004, USA\\\\
  \texttt{\{lixiaoyun996,pingli98\}@gmail.com}
}

\date{\vspace{0.5in}}
\maketitle

\begin{abstract}\vspace{0.1in}
\noindent In practical federated learning (FL) systems, e.g., wireless networks, the communication cost between the clients and the central server can often be a bottleneck. To reduce the communication cost, the paradigm of communication compression has become a popular strategy in the literature.  In this paper, we focus on  biased gradient compression techniques in non-convex FL problems. In the classical setting of distributed learning, the method of error feedback (EF) is a common technique to remedy the downsides of biased gradient compression. In this work, we study a compressed FL scheme equipped with error feedback, named Fed-EF. We further propose  two variants: Fed-EF-SGD and Fed-EF-AMS, depending on the choice of the global model optimizer.   We provide a generic theoretical analysis, which shows that directly applying biased compression in FL leads to a non-vanishing bias in the convergence rate. The proposed Fed-EF is able to match the convergence rate of the full-precision FL counterparts under data heterogeneity with a linear speedup w.r.t. the number of clients. Experiments are provided to empirically verify that Fed-EF achieves the same performance as the full-precision FL approach, at the substantially reduced communication cost.

\vspace{0.2in}

\noindent Moreover, we develop a new analysis of the EF under partial client participation, which is an important scenario in FL. We prove that under partial participation, the convergence rate of Fed-EF exhibits an extra slow-down factor due to a so-called ``stale error compensation'' effect. A numerical study is conducted to justify the intuitive impact of stale error accumulation on the norm convergence of Fed-EF under partial participation. Finally, we also demonstrate that incorporating the two-way compression in Fed-EF does not change the convergence results.

\vspace{0.2in}

\noindent In summary, our work conducts a thorough analysis of the error feedback in federated non-convex optimization. Our analysis with partial client participation also provides insights on a theoretical limitation of the error feedback mechanism, and possible directions for improvements.
\end{abstract}

\newpage

\section{Introduction}\label{sec:introduction}

The general framework of federated learning (FL) has seen numerous applications in, e.g., wireless communications (e.g., 5G or 6G),  Internet of Things (IoT), sensor networks, input method editor (IME), advertising, online visual object detection,  public health records~\citep{hard2018federated,yang2019federated,liu2020fedvision,rieke2020future,kairouz2021advances,khan2021federated,zhao2022communication_arxiv}. A centralized FL system includes multiple clients each with local data, and one central server that coordinates the training process. The goal of FL is for $n$ clients to collaboratively find a global model, parameterized by $\theta$, such that
\begin{align}\label{eq:opt}
\theta^*=\arg\min_{\theta\in\mathbb R^d} f(\theta) \eqdef \arg\min_{\theta\in\mathbb R^d} \frac{1}{n} \sum_{i=1}^n f_i(\theta),
\end{align}
where $f_i(\theta)\eqdef \mathbb E_{D\sim \mathcal D_i}\big[F_i(\theta;D)\big]$ is an (in general) non-convex loss function for the $i$-th client w.r.t. the local data distribution $\mathcal D_i$. In a typical FL system design, in each training round, the server first broadcasts the model to the clients. Then, each client trains the model based on the local data, after which the updated local models are transmitted back to the server and aggregated~\citep{mcmahan2017communication,stich2019local,chen2020toward}. The number of clients, $n$, can be tens/hundreds in some applications, e.g., \textit{cross-silo} FL~\citep{marfoq2020throughput,huang2021personalized} where clients are companies/organizations. In other scenarios, $n$ can be as large as millions or even billions, e.g.,  \textit{cross-device} FL~\citep{karimireddy2021breaking,khan2021federated}, where clients can be personal/IoT devices.

There are two primary benefits from using FL: (i) the clients train the model simultaneously, which is efficient in terms of computational resources; (ii) each client's data are kept local throughout training and are not required to be transmitted to other parties, which promotes data privacy. However, the efficiency and broad application scenarios also brings challenges for FL method design:
\begin{itemize}
    \item \textbf{Communication cost:} In FL algorithms, clients are allowed to conduct multiple training steps (e.g., local SGD updates) in each round. While this has reduced the communication frequency, the one-time communication cost is still a challenge in FL systems with limited bandwidth (e.g., portable devices at the wireless network edges), which has gained growing research interest in wireless/satellite  communication~\citep{amiri2020federated,niknam2020federated,yang2020federated,yang2020scheduling,khan2021federated,yang2021federated_6G,chen2022satellite}. In particular, as modern machine learning models (e.g, deep neural networks) often contain millions or billions of model parameters, it has been an emergent task to develop efficient algorithms for  transmitting the model/gradient information.

    \item \textbf{Data heterogeneity:} Unlike in the classical distributed training, the local data distribution in FL ($\mathcal D_i$ in (\ref{eq:opt})) can be different (non-iid), reflecting the typical practical scenarios where the local data held by different clients (e.g., app/website users) are highly personalized~\citep{zhao2018federated,kairouz2021advances,li2022federated_icde}. When multiple local training steps are taken, the local models could become ``biased'' towards minimizing the local losses, instead of the global loss. This phenomenon of data heterogeneity may hinder the global model to quickly converge to a good solution~\citep{mohri2018agnostic,li2020convergence,zhao2018federated,li2020federated}.

    \item \textbf{Partial participation (PP):} Another practical issue of FL systems is the partial participation (PP) where the clients do not join training consistently, e.g., due to unstable connection, user change or active selection~\citep{li2020federated}. That is, only a fraction of clients are involved in each FL training round to update their local models and send the local update information. This may also slow down the convergence of the global model, intuitively because less data/information is used per round~\citep{charles2021large,cho2022towards}.
\end{itemize}

\vspace{0.1in}

\noindent\textbf{FL with compression.} To overcome the main challenge of communication bottleneck, several works have analyzed FL algorithms with compressed message passing. Examples include FedPaQ~\citep{reisizadeh2020fedpaq}, FedCOM~\citep{haddadpour2021federated} and FedZip~\citep{malekijoo2021fedzip}. These algorithms are built upon directly compressing model updates communicated from clients to server. In particular, \cite{reisizadeh2020fedpaq,haddadpour2021federated} proposed to use unbiased stochastic compressors such as stochastic quantization~\citep{alistarh2017qsgd} and sparsification~\citep{wangni2018gradient}, showing that with considerable communication saving, applying unbiased compression in FL could approach the performance of un-compressed FL algorithms.

\vspace{0.1in}
\noindent\textbf{Error feedback (EF) for distributed training.} Besides unbiased compression, biased compressors are
also commonly used in communication-efficient distributed training~\citep{lin2018deep,beznosikov2020biased}. One simpler and popular type of compressor is the deterministic compressor, including fixed quantization~\citep{dettmers20168}, TopK sparsification~\citep{aji2017sparse,alistarh2018convergence,stich2018sparsified}, SignSGD~\citep{bernstein2018signsgd,bernstein2019signsgd}, etc. For these compressors, the output is a biased estimator of the true gradient. In classical distributed learning literature, it has been observed that directly updating with the biased gradients may slow down the convergence or even lead to divergence~\citep{seide20141,karimireddy2019error,beznosikov2020biased}, through experiments or by counter examples. A popular remedy is the so-called \textit{error feedback (EF)} strategy~\citep{seide20141,stich2018sparsified}: in each iteration, the local worker sends a compressed gradient to the server and records the local compression error, which is subsequently used to adjust the gradient computed in next iteration, conceptually ``correcting the bias'' due to compression. With error feedback, using biased compression in distributed training can achieve the same convergence rate as the full-precision counterparts~\citep{karimireddy2019error,li2022distributed}.

\vspace{0.2in}
\noindent\textbf{Our contributions.} Despite the rich literature on EF in classical distributed training settings, EF has not been well explored in the context of federated learning. In this work, we provide a thorough analysis of EF in FL. In particular, the three key features of FL: local steps, data heterogeneity and partial participation, pose challenging questions regarding the performance of EF in federated learning: \textit{(i) Can EF still achieve the same convergence rate as full-precision FL algorithms, possibly with highly non-iid local data distribution? (ii) How does partial participation change the situation and the results?} In this paper, we present new algorithm and results to address these questions:
\begin{itemize}
    \item We study Fed-EF, an FL framework with biased compression and error feedback, with two variants (Fed-EF-SGD and Fed-EF-AMS) depending on the global optimizer (SGD and adaptive AMSGrad~\citep{reddi2019convergence}, respectively). Our investigation starts with an analysis of directly applying biased compression in FL, showing that it does not converge to zero asymptotically. Then we prove under data heterogeneity, Fed-EF has asymptotic convergence rate $\mathcal O(\frac{1}{\sqrt{TKn}})$ where $T$ is the number of communication rounds, $K$ is the number of local training steps and $n$ is the number of clients. Our new algorithms and analysis achieve linear speedup with respect to the number of clients, improving the previous convergence result~\citep{basu2019qsparse} on error compensated FL (see detailed comparisons in Section~\ref{sec:main}). Moreover, Fed-EF-AMS is the first compressed adaptive FL algorithm in the literature.

    \item Partial participation (PP) has not been studied in the literature of distributed learning with (standard) error feedback. We initiate a new analysis of Fed-EF in this setting, by  considering both local steps and non-iid data situations. We show that under PP, Fed-EF exhibits a slow-down factor of $\sqrt{n/m}$ compared with the best full-precision rate, where $m$ is the number of active clients per round. We name this as the  \textit{``delayed error compensation''} effect.

    \item Experiments are conducted to illustrate the effectiveness of the proposed methods. We show that Fed-EF matches the performance of full-precision FL with a significant reduction in communication cost, and the proposed method compares favorably with  FL algorithms using unbiased compression without EF. Numerical examples are also provided to justify our theory.
\end{itemize}

\vspace{-0.2in}

\section{Background and Related Work}\label{sec:related}

\noindent\textbf{Gradient compression in distributed optimization.} In distributed SGD training systems, extensive works have applied various compression techniques to the communicated gradients. For example, the so-called  ``unbiased compressors'' are commonly used, which include the stochastic rounding and QSGD~\citep{alistarh2017qsgd,zhang2017zipml,wu2018error,liu2020double,xu2021agile}, unbiased sketching~\citep{ivkin2019communication,haddadpour2020fedsketch}, and the magnitude based random sparsification~\citep{wangni2018gradient}. The works~\citep{seide20141,bernstein2018signsgd,bernstein2019signsgd,karimireddy2019error,jin2020stochastic} analyzed communication compression using only the sign (1-bit) information of the gradients. Unbiased compressors can be combined with variance reduction techniques for convergence acceleration; see e.g.,~\cite{gorbunov2021marina}. On the other hand, the ``biased compressors'' are also popular. Common examples are the TopK compressor~\citep{alistarh2018convergence,stich2018sparsified,shi2019convergence,li2022distributed} (which only transmits gradient coordinates with largest magnitudes),  fixed (or learned) quantization~\citep{dettmers20168,zhang2017zipml,yu2018gradiveq,malekijoo2021fedzip}, and low-rank approximation~\citep{vogels2019powersgd}. See~\cite{beznosikov2020biased} for a summary of more biased compressors. Note that our analysis assumes a fairly general compressor which   apply to a wide range of compression schemes.

\vspace{0.1in}
\noindent\textbf{Error feedback (EF).} It has been shown that directly implementing biased compression in distributed SGD may lead to divergence, through empirical observations or counter examples~\citep{seide20141,karimireddy2019error,beznosikov2020biased}. Error feedback (EF) are proposed to fix this issue~\citep{seide20141,stich2018sparsified,karimireddy2019error}. In particular, with EF, distributed SGD under biased compression can match the convergence rate of the full-precision distributed SGD, e.g., also achieving linear speedup ($\mathcal O(1/\sqrt{Tn})$) w.r.t. the number of workers $n$ in distributed SGD~\citep{alistarh2018convergence,jiang2018linear,shen2018towards,stich2019error,zheng2019communication}. Among the limited related literature on adopting EF to FL, the most relevant method is QSparse-local-SGD~\citep{basu2019qsparse}, which can be viewed as a special instance of the proposed Fed-EF-SGD variant (with fixed global learning rate). The analysis of~\citet{basu2019qsparse} did not consider data heterogeneity and partial client participation, and their convergence rate $\mathcal O(1/\sqrt{TK})$ does not achieve linear speedup w.r.t. the number of clients. See Section~\ref{sec:main} and Section~\ref{sec:theory} for detailed comparisons. Recently, \citet{richtarik2021ef21} proposed ``EF21'' as an alternative to the standard EF. \citet{fatkhullin2021ef21} applied EF21 to FL. Our work differs from~\citet{fatkhullin2021ef21} in that we study the standard EF (which is a different algorithm from EF21) and our theoretical analysis results exhibit a linear speedup in the number of clients.

\vspace{0.1in}
\noindent\textbf{Distributed adaptive gradient methods.} Our proposed Fed-EF algorithm, in addition to SGD, also exploits AMSGrad~\citep{reddi2019convergence}, which is an adaptive gradient method widely used in distributed and federated learning~\citep{chen2020toward,karimi2021fed,reddi2021adaptive,li2022distributed,chen2022convergence} and industrial massive-scale CTR prediction models~\citep{zhao2022communication_arxiv}. The core idea of adaptive gradient algorithms is to assign different implicit learning rates to different coordinates adaptively guided by the training trajectory, leading to faster convergence and less effort needed for parameter tuning. Readers are referred to the extensive literature on adaptive gradient methods, e.g., \citep{duchi2011adaptive,zeiler2012adadelta,kingma2015adam,chen2019convergence, chen2020toward, zhou2020towards,reddi2021adaptive,wang2021optimistic}.

\newpage

\section{Fed-EF: Compressed Federated Learning with Error Feedback}\label{sec:main}

\subsection{Biased Compression Operators in Federated Learning}

In this section, we introduce some existing and new deterministic compressors used in our paper which are simple and computational efficient. Throughout the paper, $[n]$ will denote the integer set $\{1, ... ,n\}$. $\|\cdot \|$ denotes the $l_2$ norm and $\|\cdot\|_1$ is the $l_1$ norm.

\begin{definition}[$q_{\mathcal C}$-deviate compressor] \label{def:quant}
The biased $q_{\mathcal C}$-deviate compressor $\mathcal C:\mathbb R^d\mapsto \mathbb R^d$ is defined such that for $\forall x\in\mathbb R^d$, $\exists$ $0\leq q_{\mathcal C} < 1$ s.t. $\norm{\mathcal C(x)-x}^2 \leq q_{\mathcal C}^2 \norm{x}^2$. In particular, two examples are~\citet{stich2018sparsified,zheng2019communication}:
\begin{itemize}
    \item Let $\mathcal S=\{i \in [d]: |x_i|\geq t\}$ where $t$ is the $(1-k)$-quantile of $|x_i|$, $i\in [d]$. The \textbf{TopK} compressor with compression rate $k$ is defined as $\mathcal C(x)_i=x_i$, if $i\in\mathcal S$; $\mathcal C(x)_i=0$ otherwise.

    \item  Divide $[d]$ into $M$ groups (e.g., neural network layers) with index sets $\mathcal I_i$, $i=1,...,M$, and $d_i\eqdef |\mathcal I_i|$. The \textbf{(Grouped) Sign} compressor is defined as $\mathcal C(x)=\big[\frac{\|x_{\mathcal I_1}\|_1}{d_1}sign(x_{\mathcal I_1}), ..., \frac{\|x_{\mathcal I_M}\|_1}{d_M}sign(x_{\mathcal I_M})\big]$, with $x_{\mathcal I_i}$ the sub-vector of $x$ at indices $\mathcal I_i$.
\end{itemize}
\end{definition}

Larger $q_\mathcal C$ indicates heavier compression, and $q_{\mathcal C}=0$ implies no compression, i.e. $\mathcal C(x)=x$. The following Proposition~\ref{prop:topk,sign} is well-known, and we include the proof for clarity and completeness.

\begin{proposition} \label{prop:topk,sign}
For the \textbf{TopK} compressor which selects top $k$-percent of coordinates, we have $q_{\mathcal C}^2=1-k$. For the \textbf{(Group) Sign} compressor, $q_{\mathcal C}^2=1-\min_{i\in [M]} \frac{1}{d_i}$.
\end{proposition}
\begin{proof}
For \textbf{TopK}, the proof is trivial: since $\mathcal C(x)-x$ only contain $(1-k)d$ coordinates with lowest magnitudes, we know $\|C(x)-x\|^2/\|x^2\|\leq 1-k$.

For \textbf{Sign}, recall that $\mathcal I_i$ is the index set of block (group) $i$. By definition, for the $i$-th block (group) $x_{\mathcal I_i}\in\mathbb R^{d_i}$, we have
\begin{align*}
    \|\mathcal C(x_{\mathcal I_i})-x_{\mathcal I_i}\|^2&=\|x_{\mathcal I_i}-\frac{\|x_{\mathcal I_i}\|_1}{d_i}sign(x_{\mathcal I_i})\|^2 \\
    &=\|x_{\mathcal I_i}\|^2+ \frac{\|x_{\mathcal I_i}\|_1^2}{d_i^2}\cdot d_i-\frac{2\|x_{\mathcal I_i}\|_1^2}{d_i}\\
    &=\|x_{\mathcal I_i}\|^2-\|x_{\mathcal I_i}\|_1^2/d_i.
\end{align*}
Since we have $M$ blocks, concatenating the blocks leads to
\begin{align*}
    \|\mathcal C(x)-x\|^2&=\sum_{i=1}^M\Big(\|x_{\mathcal I_i}\|^2-\|x_{\mathcal I_i}\|_1^2/d_i \Big)\\
    &=\|x\|^2-\sum_{i=1}^M\|x_{\mathcal I_i}\|_1^2/d_i \\
    &=\big(1-\frac{\sum_{i=1}^M\|x_{\mathcal I_i}\|_1^2/d_i}{\|x\|^2}  \big) \|x\|^2 \\
    &\leq \big(1-\min_{i\in [M]} \frac{\|x_{\mathcal I_i}\|_1^2}{d_i\|x_{\mathcal I_i}\|^2}\big)\|x\|^2\leq (1-\min_{i\in [M]} \frac{1}{d_i})\|x\|^2,
\end{align*}
where the last inequality is because $l_1$ norm is lower bounded by $l_2$ norm.
\end{proof}

Additionally, these two compression operators can be combined to derive the so-called ``\textbf{heavy-Sign}'' compressor, where we first apply \textbf{TopK} and then \textbf{Sign}, for even higher compression rate.

\begin{definition}[Heavy-Sign compressor]\label{def:heavy-sign}
Let $\mathcal C_k(\cdot)$ and $\mathcal C_s(\cdot)$ be the \textbf{TopK} and \textbf{Sign} operator as in Definition~\ref{def:quant}. Then the \textbf{Heavy-Sign} operator is defined as $\mathcal C_{hv}(x)\eqdef \mathcal C_s\big(\mathcal C_k(x)\big)$ for $x\in \mathbb R^d$.
\end{definition}

\begin{proposition}
The \textbf{heavy-Sign} compressor satisfies Definition~\ref{def:quant} with $q_{\mathcal C}^2=1-\min_{i\in [M]}\frac{k}{d_i}$.
\end{proposition}

\begin{proof}
Recall Definition~\ref{def:heavy-sign} that $\mathcal C_k$ denotes the \textbf{TopK} compressor and $\mathcal C_s$ is the \textbf{Sign} operator. The \textbf{heavy-Sign} operator $\mathcal C(x)=\mathcal C_s\big(\mathcal C_k(x) \big)$ admits
\begin{align*}
    \|\mathcal C_{hv}(x)-x\|^2&=\|\mathcal C_s\big(\mathcal C_k(x) \big)-\mathcal C_k(x) +\mathcal C_k(x) -x\|^2 \\
    &= \|\mathcal C_s\big(\mathcal C_k(x) \big)-\mathcal C_k(x)\|^2 + \|\mathcal C_k(x) -x\|^2,
\end{align*}
where the second equality holds because \textbf{TopK} zeros out the unpicked coordinates. By Proposition~\ref{prop:topk,sign}, we continue to obtain
\begin{align*}
    \|\mathcal C_{hv}(x)-x\|^2&\leq (1-\min_{i\in [M]} \frac{1}{d_i})\|\mathcal C_k(x)\|^2+ \|\mathcal C_k(x) -x\|^2\\
    &= \|x\|^2-\min_{i\in [M]} \frac{1}{d_i}\|\mathcal C_k(x)\|^2 \leq (1-\min_{i\in [M]}\frac{k}{d_i})\|x\|^2,
\end{align*}
where $M$ is the number of blocks in \textbf{Sign} and we use the fact that $\|\mathcal C_k(x)\|^2+\|\mathcal C_k(x)-x\|^2=\|x\|^2$, and $\|\mathcal C(x)\|\geq k\|x\|^2$ by Proposition~\ref{prop:topk,sign}.
\end{proof}

\begin{figure}[h]
  \begin{center}
  \vspace{-0.1in}
    \includegraphics[width=3in]{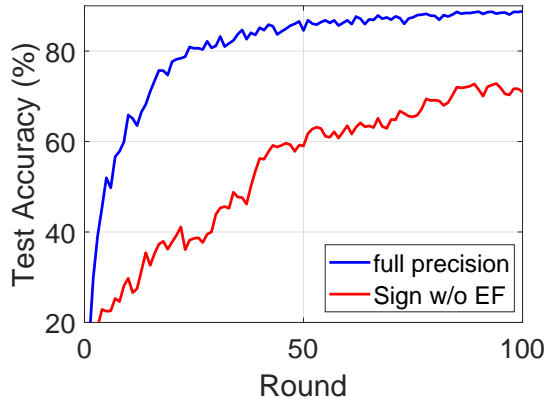}
  \end{center}
  \vspace{-0.15in}
  \caption{Test accuracy of MLP trained by Fed-SGD on MNIST dataset~\citep{lecun1998gradient}: full-precision vs. \textbf{Sign} compression (Algorithm~\ref{alg:no-EF-SGD}), $\eta=1$, $\eta_l=0.01$, $n=200$ non-iid clients.}
  \label{fig:no_EF}
\end{figure}


Given the simplicity of the deterministic compressors, one may ask: can we directly apply biased compressors in communication-efficient FL? As an example, in Figure~\ref{fig:no_EF}, we report the test accuracy of a multi-layer perceptron (MLP) trained on MNIST dataset in non-iid FL environment (see Section~\ref{sec:experiment} for more descriptions), of Fed-SGD~\citep{stich2019local} with full communication (blue) versus compression using \textbf{Sign} (red), i.e., clients directly send compressed local model update to the server for aggregation. In Algorithm~\ref{alg:no-EF-SGD}, we present this simple approach. We observe a catastrophic performance downgrade of using biased compression directly. In Section~\ref{sec:main}, we will demonstrate that adopting biased compression directly leads to an undesirable asymptotically non-vanishing term in the convergence rate, theoretically justifying this empirical performance degradation.

\begin{algorithm}[t]
\caption{Fed-SGD with Biased Compression} \label{alg:no-EF-SGD}
\begin{algorithmic}[1]
\State{\textbf{Input}: learning rates $\eta$, $\eta_l$}
\State{\textbf{Initialize}: central server parameter $\theta_{1} \in \mathbb R^d \subseteq \mathbb R^d$}
\vspace{0.03in}
\State{\textbf{for $t=1, \ldots, T$ do}}
\State{\quad\textbf{parallel for worker $i \in [n]$ do}:}

\State{\quad\quad  Receive model parameter $\theta_{t}$ from central server, set $\theta_{t,i}^{(1)}=\theta_t$}
\State{\quad\quad  \textbf{for $k=1, \ldots, K$ do}}
\State{\quad\quad\quad  Compute stochastic gradient $g_{t,i}^{(k)}$ at $\theta_{t,i}^{(k)}$}
\State{\quad\quad\quad  Local update $\theta_{t,i}^{(k+1)}=\theta_{t,i}^{(k)}-\eta_l g_{t,i}^{(k)}$}
\State{\quad\quad  \textbf{end for}}

\State{\quad\quad Compute the local model update $\del_{t,i}=\theta_{t,i}^{(K+1)}-\theta_{t}$ }
\State{\quad\quad  Send compressed model update $\widetilde\del_{t,i}=\mathcal C(\del_{t,i})$ to central server}

\State{\quad\textbf{end parallel}}

\State{\quad\textbf{Central server do:}}
\State{\quad Global aggregation $\overline{\widetilde\del}_{t}=\frac{1}{n}\sum_{i=1}^n \widetilde\del_{t,i}$  \label{line:g}}
\State{\quad Update the global model $\theta_{t+1}=\theta_{t}-\eta\overline{\widetilde\del}_{t}$}

\State{\textbf{end for}}
\end{algorithmic}

\end{algorithm}

\begin{algorithm}[h]
\caption{Compressed Federated Learning with Error Feedback (Fed-EF)} \label{alg:Fed-EF}
\begin{algorithmic}[1]
\State{\textbf{Input}: learning rates $\eta$, $\eta_l$, \colorbox{blue!20!white}{hyper-parameters $\beta_1$, $\beta_2$, $\epsilon$}  }
\State{\textbf{Initialize}: central server parameter $\theta_{1} \in \mathbb R^d \subseteq \mathbb R^d$; $e_{1,i}=\bm{0}$ the accumulator for each worker; \colorbox{blue!20!white}{$m_0=\bm{0}$, $v_0=\bm{0}$, $\hat v_0=\bm{0}$}}
\vspace{0.03in}
\State{\textbf{for $t=1, \ldots, T$ do}}
\State{\quad\textbf{parallel for worker $i \in [n]$ do}:}

\State{\quad\quad  Receive model parameter $\theta_{t}$ from central server, set $\theta_{t,i}^{(1)}=\theta_t$}
\State{\quad\quad  \textbf{for $k=1, \ldots, K$ do}}
\State{\quad\quad\quad  Compute stochastic gradient $g_{t,i}^{(k)}$ at $\theta_{t,i}^{(k)}$}
\State{\quad\quad\quad  Local update $\theta_{t,i}^{(k+1)}=\theta_{t,i}^{(k)}-\eta_l g_{t,i}^{(k)}$}
\State{\quad\quad  \textbf{end for}}

\State{\quad\quad Compute the local model update $\del_{t,i}=\theta_{t,i}^{(K+1)}-\theta_{t}$ }
\State{\quad\quad  Send compressed adjusted local model update $\widetilde\del_{t,i}=\mathcal C(\del_{t,i}+e_{t,i})$ to central server }
\State{\quad\quad  Update the error $e_{t+1,i}=e_{t,i}+\del_{t,i}-\widetilde\del_{t,i}$}

\State{\quad\textbf{end parallel}}

\State{\quad\textbf{Central server do:}}
\State{\quad Global aggregation $\overline{\widetilde\del}_{t}=\frac{1}{n}\sum_{i=1}^n \widetilde\del_{t,i}$  \label{line:g}}
\State{\quad \colorbox{green!20!white}{Update the global model $\theta_{t+1}=\theta_{t}-\eta\overline{\widetilde\del}_{t}$} \hfill\Comment{\colorbox{green!20!white}{Fed-EF-SGD}} }

\State{\quad \colorbox{blue!20!white}{$m_t=\beta_1 m_{t-1}+(1-\beta_1)\overline{\widetilde\del}_{t}$} \hfill\Comment{\colorbox{blue!20!white}{Fed-EF-AMS}}  }
\State{\quad \colorbox{blue!20!white}{$v_t=\beta_2 v_{t-1}+(1-\beta_2)\overline{\widetilde\del}_{t}^2$,\quad $\hat v_t=\max(v_t,\hat v_{t-1})$} \label{line:v}}
\State{\quad \colorbox{blue!20!white}{Update the global model $\theta_{t+1}=\theta_{t}-\eta\frac{m_t}{\sqrt{\hat v_t+\epsilon}}$} }

\State{\textbf{end for}}
\end{algorithmic}
\end{algorithm}

\subsection{Fed-EF Algorithm}

To resolve this problem, error feedback (EF), which is a popular tool in distributed training, can be adapted to federated learning. In Algorithm~\ref{alg:Fed-EF}, we present the general compressed FL framework named Fed-EF, whose main steps are summarized below. In round $t$: 1) The server broadcast the global model $\theta_t$ to all clients (line 5); 2) The $i$-th client performs $K$ steps of local SGD updates to get local model $\theta_{t,i}^{(K)}$, compute the compressed local model update $\tilde\del_{t,i}$, updates the local error accumulator $e_{t,i}$, and sends the compressed $\tilde\del_{t,i}$ back to the server (line 6-12); 3) The server receives $\tilde\del_{t,i}$, $i\in [n]$ from all clients, takes the average, and perform a global model update using the averaged compressed local model updates (line 15-19).

Depending on the global model optimizer, we propose two variants: Fed-EF-SGD (green) which applies SGD global updates, and Fed-EF-AMS (blue), whose global optimizer is AMSGrad~\citep{reddi2019convergence}. In Fed-EF-AMS, by the nature of adaptive gradient methods, we incorporate momentum ($m_t$) with different implicit dimension-wise learning rates $\eta/\hat v_t$. Additionally, for conciseness, the presented algorithm employs one-way compression (clients-to-server). In Appendix~\ref{app sec:two-way}, we also provide a two-way compressed Fed-EF framework and demonstrate that adding the server-to-clients compression would not affect the convergence rates.

\vspace{0.1in}
\noindent\textbf{Comparison with prior work.} Compared with EF approaches in the classical distributed training, e.g., \cite{stich2018sparsified,karimireddy2019error,zheng2019communication,liu2020double,ghosh2021communication,li2022distributed}, our algorithm allows local steps (more communication-efficiency) and uses two-side learning rates. When $\eta\equiv 1$, the Fed-EF-SGD method reduces to QSparse-local-SGD~\citep{basu2019qsparse}. In Section~\ref{sec:theory}, we will demonstrate how the two-side learning rate schedule improves the convergence analysis of the one-side learning rate approach~\citep{basu2019qsparse}. On the other hand, several recent works considered compressed FL using unbiased stochastic compressors (all of which used SGD as the global optimizer). FedPaQ~\citep{reisizadeh2020fedpaq} applied stochastic quantization without error feedback to local SGD, which is improved by \cite{haddadpour2021federated} using a gradient tracking trick that, however, requires communicating an extra vector from server to clients, which is less efficient than Fed-EF. \cite{malekijoo2021fedzip} provided an empirical study on directly compressing the local updates using various compressors in Fed-SGD, while we use EF to compensate for the bias. \cite{mitra2021linear} proposed FedLin, which only uses compression for synchronizing a local memory term but still requires transmitting full-precision updates. Finally, to our knowledge, Fed-EF-AMS is the first compressed adaptive FL method in literature.

\section{Theoretical Results}\label{sec:theory}

\begin{assumption}[Smoothness] \label{ass:smooth}
For $\forall i \in [n]$, $f_i$ is  L-smooth: $\norm{\nabla f_i (x) - \nabla f_i (y)} \leq L \norm{x-y}$.
\end{assumption}

\begin{assumption}[Bounded variance] \label{ass:var}
For $\forall t \in [T]$, $\forall i \in [n]$, $\forall k\in [K]$: (i) the stochastic gradient is unbiased: $\EE\big[g_{t,i}^{(k)}\big] = \nabla f_i(\theta_{t,i}^{(k)})$; (ii) the \textbf{local variance} is bounded: $\EE\big[\|g_{t,i}^{(k)} - \nabla f_i(\theta_{t,i}^{(k)})\|^2\big] < \sigma^2$; (iii) the \textbf{global variance} is bounded: $\frac{1}{n}\sum_{i=1}^n\|\nabla f_i(\theta_t)-\nabla f(\theta_t)\|^2\leq \sigma_g^2$.
\end{assumption}

Both assumptions are standard in the convergence analysis of stochastic gradient methods. The global variance bound $\sigma_g^2$ in Assumption~\ref{ass:var} characterizes the difference among local objective functions, which, is mainly caused by different local data distribution $\mathcal X_i$ in \eqref{eq:opt}, i.e., data heterogeneity. Following prior works on compressed FL, we also make the following additional assumption.

\begin{assumption}[Compression discrepancy] \label{ass:compress_diff}
There exists some $q_{\mathcal A}<1$ such that $\mathbb E\big[\| \frac{1}{n}\sum_{i=1}^n \mathcal C\big(\del_{t,i}+e_{t,i}\big)-\frac{1}{n}\sum_{i=1}^n (\del_{t,i}+e_{t,i}) \|^2\big]\leq q_{\mathcal A}^2 \mathbb E\big[\| \frac{1}{n}\sum_{i=1}^n (\del_{t,i}+e_{t,i}) \|^2\big]$ in every round $t\in [T]$.
\end{assumption}

In Assumption~\ref{ass:compress_diff}, if we replace ``the average of compression'', $\frac{1}{n}\sum_{i=1}^n \mathcal C\big(\del_{t,i}+e_{t,i}\big)$, by ``the compression of average'', $\mathcal C\big(\frac{1}{n}\sum_{i=1}^n (\del_{t,i}+e_{t,i})\big)$, the statement immediately holds by Definition~\ref{def:quant} with $q_{\mathcal A}=q_{\mathcal C}$. Thus, Assumption~\ref{ass:compress_diff} basically says that the above two terms stay close during training. This is a common assumption in related work on compressed distributed learning, for example, a similar assumption is used in~\cite{alistarh2018convergence} analyzing sparsified SGD. In \cite{haddadpour2021federated}, for unbiased compression without EF, a similar condition is also assumed with an absolute bound. In Section~\ref{sec:discuss_assumption}, we provide more discussion and empirical justification to validate this analytical assumption in practice.

\subsection{Convergence of Directly Using Biased Compression in FL}

In Figure~\ref{fig:no_EF}, we have seen a simple example that naively transmitting the condensed local model updates by biased compressors without EF may perform poorly empirically. We now provide the theoretical convergence results on this naive strategy (Algorithm~\ref{alg:no-EF-SGD}). More specifically, in the standard Fed-SGD algorithm~\citep{stich2019local}, in each round $t$, after conducting $K$ local training steps to get the model update $\del_{t,i}$, the $i$-th client computes $\widetilde\del_{t,i}\eqdef \mathcal C(\del_{t,i})$ and sends it to the server. The server takes the average of the compressed gradients and updates with $\overline{\widetilde\del}_t=\frac{1}{n}\sum_{i=1}^n \widetilde\del_{t,i}$ using SGD. We have the following convergence result.

\begin{Theorem}[Fed-SGD with biased compression] \label{theo:no-EF rate}
Let $\theta^*=\arg\min f(\theta)$, and denote $q=\max\{q_{\mathcal C}, q_{\mathcal A}\}$. Consider Algorithm~\ref{alg:no-EF-SGD} where Fed-SGD is applied with biased communication compression. Under Assumptions~\ref{ass:smooth} to~\ref{ass:compress_diff}, when $\eta_l\leq \frac{1}{8KL\max\{1,8(1+q^2)\eta\}}$, we have
\begin{align*}
    \frac{1}{T}\sum_{t=1}^T\mathbb E\big[\|\nabla f(\theta_t)\|^2\big]
    &\lesssim \frac{f(\theta_1)-\mathbb E[f(\theta_{t+1})]}{\eta\eta_l TK}+\frac{4\eta\eta_l(1+q^2) L}{n}\sigma^2 + \frac{64q^2}{Kn}\sigma^2 \\
    &+ 20\eta\eta_l^3(1+q^2)K^2L^3(\sigma^2+6K\sigma_g^2)+(320q^2+3)\eta_l^2K L^2(\sigma^2+6K\sigma_g^2).
\end{align*}
If we choose $\eta_l=\Theta(\frac{1}{K\sqrt T})$ and $\eta=\Theta(\sqrt{Kn})$, we have
\begin{align}
    \frac{1}{T}\sum_{t=1}^T\mathbb E\big[\|\nabla f(\theta_t)\|^2\big]=\mathcal O\Big( \frac{1+q^2}{\sqrt{TKn}}+\frac{1+q^2}{TK}(\sigma^2+K\sigma_g^2)+\frac{q^2\sigma^2}{Kn} \Big).  \label{eqn:asym-rate-no-EF}
\end{align}
\end{Theorem}
\begin{remark}
When $q=0$ (no compression), the constant term in the convergence rate is erased and Theorem~\ref{theo:no-EF rate} recovers the $\mathcal O(\frac{1}{\sqrt{TKn}})$ asymptotic rate of full-precision Fed-SGD~\citep{yang2021achieving}.
\end{remark}
\begin{remark}
If we use unbiased compressors (e.g., stochastic quantization) instead, we can remove the constant bias term in the proof and recover the $\mathcal O(\frac{1}{\sqrt{TKn}})$ rate in the recent work~\citep{haddadpour2021federated} for federated optimization with unbiased compression.
\end{remark}

In Theorem~\ref{theo:no-EF rate}, the left-hand side is the expected squared gradient norm at a uniformly chosen global model from $t=1, ..., T$, which is a standard measure of convergence in non-convex optimization (i.e., ``norm convergence''). In general, we see that larger $q$ (i.e., higher compression) would slow down the convergence. In the asymptotic rate (\ref{eqn:asym-rate-no-EF}), the first term depends on the initialization and local variance. The second term containing $\sigma_g^2$ represents the influence of data heterogeneity. The non-vanishing (as $T\rightarrow \infty$) third term is a consequence of the bias introduced by the biased compressors which decreases with smaller $q$ (i.e., less compression). This constant term (when $q\neq 0$) implies that Fed-SGD does not converge to a stationary point when biased compression is directly applied. When $q=0$ (no compression), we recover the full-precision rate as expected.

\subsection{Convergence of Fed-EF: Linear Speedup Under Data Heterogeneity}

We now analyze our Fed-EF algorithm which compensates the compression bias with error feedback.

\begin{Theorem}[Fed-EF-SGD]  \label{theo:rate SGD}
Let $\theta^*=\arg\min f(\theta)$, and denote $q=\max\{q_{\mathcal C}, q_{\mathcal A}\}$, $C_1\eqdef 2+\frac{4q^2}{(1-q^2)^2}$. Under Assumptions~\ref{ass:smooth} to~\ref{ass:compress_diff}, when $\eta_l\leq  \frac{1}{2KL\cdot\max\{4,\eta(C_1+1)\}}$, the squared gradient norm of Fed-EF-SGD iterates in Algorithm~\ref{alg:Fed-EF} can be bounded by
\begin{align*}
    \frac{1}{T}\sum_{t=1}^T\mathbb E\big[\|\nabla f(\theta_t)\|^2\big]&\lesssim
    \frac{f(\theta_1)-f(\theta^*)}{\eta\eta_lT K}+\frac{2\eta\eta_lC_1 L}{n}\sigma^2  +10\eta\eta_l^3C_1K^2L^3(\sigma^2+6K\sigma_g^2).
\end{align*}
\end{Theorem}

Similar to Theorem~\ref{theo:no-EF rate}, we may decompose the convergence rate into several parts: the first term is dependent on the initialization, the second term $\sigma^2$ comes from the local stochastic variance, and the last term represents the influence of data heterogeneity. We also see that larger $q$ (i.e., higher compression) slows down the convergence as expected. We will present the convergence rate under specific learning rates in Corollary~\ref{coro:rates}.

In our analysis of Fed-EF-AMS, we will make the following additional assumption of bounded stochastic gradients, which is common in the convergence analysis of adaptive methods, e.g.,~\cite{reddi2019convergence,zhou2018convergence,chen2019convergence,li2022distributed}. Note that this assumption is only used for Fed-EF-AMS, but not for Fed-EF-SGD.

\begin{assumption}[Bounded gradients] \label{ass:boundgrad}
It holds that $\|g_{t,i}^{(k)}\| \leq G$, $\forall t >0$, $\forall i \in [n]$, $\forall k\in [K]$.
\end{assumption}

We provide the convergence analysis of adaptive FL with error feedback (Fed-EF-AMS) as below.

\begin{Theorem}[Fed-EF-AMS]  \label{theo:rate AMS}
With same notations as in Theorem~\ref{theo:rate SGD}, let $C_1\eqdef \frac{\beta_1}{1-\beta_1}+\frac{2q}{1-q^2}$. Under Assumptions~\ref{ass:smooth} to~\ref{ass:boundgrad}, if the learning rates satisfy $\eta_l\leq \frac{\sqrt\epsilon}{8KL}\min \big\{ \frac{1}{\sqrt\epsilon}, \frac{2(1-q^2)L}{(1+q^2)^{1.5}G}, \frac{1}{\max\{16,32C_1^2\}\eta},   \frac{1}{3\eta^{1/3}} \big\}$, the Fed-EF-AMS iterates in Algorithm~\ref{alg:Fed-EF} satisfy
\begin{align*}
    \frac{1}{T}\sum_{t=1}^T\mathbb E\big[\| \nabla f(\theta_t) \|^2\big]&\lesssim  \frac{f(\theta_1)- f(\theta^*)}{\eta\eta_l TK} +\Big[ \frac{5\eta_l^2 K L^2}{2\sqrt\epsilon}+\frac{\eta\eta_l^3(30+20C_1^2) K^2L^3}{\epsilon} \Big] (\sigma^2+6K\sigma_g^2)   \\
    &\hspace{0.5in} + \frac{\eta\eta_l L (6+4C_1^2)}{n\epsilon}\sigma^2+ \frac{(C_1+1)  G^2d}{T\sqrt\epsilon}+\frac{3\eta \eta_l C_1^2 L K G^2d}{T\epsilon}.
\end{align*}
\end{Theorem}

With some properly chosen learning rates, we have the following simplified results.

\begin{Corollary}[Fed-EF, specific learning rates] \label{coro:rates}
Suppose the conditions in Theorem~\ref{theo:rate SGD} and Theorem~\ref{theo:rate AMS} are satisfied respectively. Choosing $\eta_l=\Theta(\frac{1}{K\sqrt T})$ and $\eta=\Theta(\sqrt{Kn})$, Fed-EF-SGD satisfies
\begin{align*}
    \frac{1}{T}\sum_{t=1}^T\mathbb E\big[\| \nabla f(\theta_t) \|^2\big]=\mathcal O\Big(\frac{f(\theta_1)-f(\theta^*)}{\sqrt{TKn}} + \frac{1}{\sqrt{TKn}}\sigma^2 + \frac{\sqrt n}{T^{3/2}\sqrt K}(\sigma^2+K\sigma_g^2) \Big),
\end{align*}
and for Fed-EF-AMS, it holds that
\begin{align*}
    \frac{1}{T}\sum_{t=1}^T\mathbb E\big[\| \nabla f(\theta_t) \|^2\big]=\mathcal O\Big(\frac{f(\theta_1)-f(\theta^*)}{\sqrt{TKn}}+ \frac{1}{\sqrt{TKn}}\sigma^2+ (\frac{1}{TK}+\frac{\sqrt n}{T^{3/2}\sqrt K})(\sigma^2+K\sigma_g^2) \Big).
\end{align*}
\end{Corollary}

\vspace{0.1in}
\noindent\textbf{Discussion.} From Corollary~\ref{coro:rates}, we see that when $T\geq K$, Fed-EF-AMS and Fed-EF-SGD have the same rate of convergence asymptotically. Therefore, our following discussion applies to the general Fed-EF scheme with both variants. In Corollary~\ref{coro:rates}, when $T\geq Kn\frac{\sigma_g^2}{\sigma^2}$, the global variance term $\sigma_g^2$ vanishes and the convergence rate becomes $\mathcal O(1/\sqrt{TKn})$\footnote{This also implies $K=\mathcal O(\frac{T\sigma^2}{n\sigma_g^2})$  is needed to achieve $\mathcal O(1/\sqrt{TKn})$ convergence rate which decreases in $K$, i.e., local steps helps. When $\sigma_g^2$ is large (i.e., high data heterogeneity), $K$ needs to be small to achieve this rate; when $\sigma_g^2$ is small (more homogeneous client data), we can tolerate larger $K$, i.e., local steps help even $K$ is already large. Intuitively, this is because in heterogeneous setting, the local losses might be very different from the global loss. Applying too many local steps may not always help for the global convergence.}. Thus, the proposed Fed-EF enjoys linear speedup w.r.t. the number of clients $n$, i.e., it reaches a $\delta$-stationary point (i.e., $\frac{1}{T}\sum_{t=1}^T\mathbb E\big[\| \nabla f(\theta_t) \|^2\big]\leq \delta$) as long as $TK=\Theta(1/n\delta^2)$, which matches the recent results of the full-precision counterparts~\citep{yang2021achieving,reddi2021adaptive} (Note that \cite{reddi2021adaptive} only analyzed the special case $\beta_1=0$, while our analysis is more general). The condition $T\geq Kn$ to reach linear speedup considerably improves $\mathcal O(K^3n^3)$ of the federated momentum SGD analysis in~\cite{yu2019linear}. In terms of communication complexity, by setting $K=\Theta(1/n\delta)$, Fed-EF only requires $T=\Theta(1/\delta)$ rounds of communication to converge. This matches one of the state-of-the-art FL communication complexity results of SCAFFOLD~\citep{karimireddy2020scaffold}.

\vspace{0.1in}
\noindent\textbf{Comparison with prior results on compressed FL.} As a special case of Fed-EF-SGD ($\eta\equiv 1$) and the most relevant previous work, the analysis of QSparse-local-SGD~\citep{basu2019qsparse} did not consider data heterogeneity, and their convergence rate $\mathcal O(1/\sqrt{TK})$ did not achieve linear speedup either. Our new analysis improves this result, showing that EF can also match the best rate of using full communication in federated learning. For FL with direct unbiased compression (without EF), the convergence rate of FedPaQ~\citep{reisizadeh2020fedpaq} is also $\mathcal O(1/\sqrt{TK})$. \cite{haddadpour2021federated} refined the analysis and algorithm of FedPaQ, which matches our $\mathcal O(1/\delta)$ communication complexity. To sum up, both Fed-EF-SGD and Fed-EF-AMS are able to achieve the convergence rates of the corresponding full-precision FL counterparts, as well as the state-of-the-art rates of federated optimization with unbiased compression.

\subsection{Analysis of Fed-EF Under Partial Client Participation}  \label{sec:partial}

Whilst being a popular strategy in classical distributed training, error feedback has not been analyzed under partial participation (PP), which is an important feature of FL. Next, we provide new analysis and results of EF under this setting, considering both local steps and data heterogeneity in federated learning. In each round $t$, assume only $m$ randomly chosen clients (without replacement) indexed by $\mathcal M_t\subseteq [n]$ are active and participate in training (i.e., changing $i\in [n]$ to $i\in \mathcal M_t$ at line 4 of Algorithm~\ref{alg:Fed-EF}). For the remaining $(n-m)$ inactive clients, we simply set $e_{t,i}=e_{t-1,i}$, $\forall i\in [n]\setminus \mathcal M_t$. The convergence rate is given as below.

\begin{Theorem}[Fed-EF, partial participation] \label{theo:partial-simple}
In each round, suppose $m$ randomly chosen clients in $\mathcal M_t$ participate in the training. Under Assumptions~\ref{ass:smooth} to~\ref{ass:compress_diff}, suppose the learning rates satisfy
$\eta_l\leq \min\Big\{ \frac{1}{6},\frac{m}{96C'\eta},\frac{m^2}{53760(n-m) C_1\eta}, \frac{1}{4\eta}, \frac{1}{32C_1\eta}\Big\}\frac{1}{KL}$. Fed-EF-SGD admits
\begin{align*}
    &\frac{1}{T}\sum_{t=1}^T\mathbb E\big[\| \nabla f(\theta_t) \|^2\big]\lesssim \frac{f(\theta_1)-f(\theta^*)}{\eta\eta_l TK}+ \Big[ \frac{\eta\eta_l L}{m}+\frac{8\eta\eta_l C_1Ln}{m^2} \Big]\sigma^2 + \frac{3\eta\eta_lC'KL}{m}\sigma_g^2\\
    &\hspace{1.2in} + \Big[ \frac{5\eta_l^2KL^2}{2}+\frac{15\eta\eta_l^3C'K^2L^3}{m}+\frac{560\eta\eta_lC_1(n-m)L}{m^2} \Big](\sigma^2+6K\sigma_g^2),
\end{align*}
where $C_1=\frac{q^2}{(1-q^2)^3}$ and $C'=\frac{n-m}{n-1}$. Choosing $\eta=\Theta(\sqrt{Km})$, $\eta_l=\Theta(\frac{\sqrt m}{K\sqrt{Tn}})$, we have
\begin{align*}
    \frac{1}{T}\sum_{t=1}^T\mathbb E\big[\| \nabla f(\theta_t) \|^2\big]&=\mathcal O\Big(\frac{\sqrt n}{\sqrt m}\big( \frac{f(\theta_1)-f(\theta^*)}{\sqrt{TKm}} + \frac{1}{\sqrt{TKm}}\sigma^2 +\frac{\sqrt{K}}{\sqrt{Tm}}\sigma_g^2 \big)\Big).
\end{align*}
\end{Theorem}

\begin{remark}
We present Fed-EF-SGD for simplicity. With more complicated analysis, similar result applies to Fed-EF-AMS yielding the same asymptotic convergence rate as Fed-EF-SGD.
\end{remark}

\begin{remark}
When $m=n$ (full participation), Theorem~\ref{theo:partial-simple} reduces to the $\mathcal O(1/\sqrt{TKn})$ rate in Corollary~\ref{coro:rates}. When $q=0$ (no compression), we can recover the $\mathcal O(\sqrt{K/Tm})$ rate of full-precision Fed-SGD under PP~\citep{yang2021achieving}.
\end{remark}

The convergence rate in Theorem~\ref{theo:partial-simple} involves $m$ in the denominator, instead of $n$ as in Corollary~\ref{coro:rates}, which is a result of larger gradient estimation variance due to client sampling. Importantly, compared with the $\mathcal O(\sqrt{K/Tm})$ rate of~\cite{yang2021achieving} for full-precision local SGD under partial participation, Theorem~\ref{theo:partial-simple} extracts an additional slow-down factor of $\sqrt{n/m}$.

\vspace{0.1in}
\noindent\textbf{Effect of delayed error compensation.} We argue that this is a consequence of the mechanism of error feedback. Intuitively, with full participation where each client is active in every round, EF itself can, to a large extent, be regarded as subtly ``delaying'' the ``untransmitted'' gradient information ($\mathcal C(\del_t)-\del_t$) to the next iteration. However, under partial participation, in each round $t$, the error accumulator of a chosen client actually contains the latest information from round $t-s$, where $s$ can be viewed as the ``lag'' which follows a geometric distribution with $\mathbb E[s]=n/m$. In some sense, this shares similar spirit to the problem of asynchronous distributed optimization with delayed gradients (e.g.,~\cite{agarwal2011distributed,lian2015asynchronous}). The delayed error information in Fed-EF under PP is likely to pull the model away from heading towards a stationary point (i.e., slower down the norm convergence), especially for highly non-convex loss functions. In Section~\ref{sec:experiment}, we will propose a simple strategy to empirically justify (and to an extent mitigate) the negative impact of the stale error compensation on the norm convergence.

\section{Numerical Study}\label{sec:experiment}

We provide numerical results to show the empirical efficacy of Fed-EF in communication-efficient FL problems and justify our theoretical analysis. Our main objective is to show: 1) Fed-EF is able to provide matching performance as full-precision FL, with significantly less communication; 2) the stale error compensation effect would indeed slower down the norm convergence of Fed-EF under partial participation. We include representative main results here and place the implementation details and more figures and tables in Appendix~\ref{app sec:experiment}.

\subsection{Experiment Setup}

\noindent\textbf{Datasets.} We present experiments on three popular FL datasets. The MNIST dataset~\citep{lecun1998gradient} contains 60000 training examples and 10000 test samples of $28\times 28$ gray-scale hand-written digits from 0 to 9. The FMNIST dataset~\citep{xiao2017fashion} has the same input size and train/test split as MNIST, but the samples are fashion products (e.g., clothes and bags). The CIFAR-10~\citep{cifar} dataset includes 50000 natural images of size $32\times 32$ each with 3 RGB channels for training and 10000 images for testing. There are 10 classes, e.g., airplanes, cars, cats, etc. We follow a standard strategy for CIFAR-10 dataset to pre-process the training images by a random crop, a random horizontal flip and a normalization of the pixel values to have zero mean and unit variance. For test images, we only apply the normalization step.

\vspace{0.1in}
\noindent\textbf{Federated setting.} In our experiments, we test $n=200$ clients. The clients' local data are set to be highly non-iid (heterogeneous), where we restrict the local data samples of each client to come from at most two classes: we first split the data samples into $2n=400$ shards each containing samples from only one class; then each client is assigned with two shards uniformly at random. We run $T=100$ rounds, where one FL training round is finished after all the clients have performed one epoch of local training. The local mini-batch size is $32$, which means that the clients conduct 10 local iterations per round. Regarding partial participation, we uniformly randomly sample $m$ clients in each round. We present the results at multiple sampling proportion $p=m/n$ (e.g., $p=0.1$ means choosing $20$ active clients per round). To measure the communication cost, we report the accumulated number of bits transmitted from the client to server (averaged over all clients), assuming that full-precision gradients are $32$-bit encoded.

\vspace{0.1in}
\noindent\textbf{Methods and compressors.} For both Fed-EF variants, we implement \textbf{Sign} compressor, and \textbf{TopK} compressor\footnote{For \textbf{TopK}, in our implementation we also apply it in a ``layer-wise'' manner similar to \textbf{Sign}. Let $k$ denote the proportion of coordinates selected. For each layer with $d_i$ parameters, we pick $\max(1,\lfloor kd_i\rfloor)$ gradient dimensions. The maximum operator avoids the case where a layer is never updated.} with compression rate $k\in\{0.001,0.01,0.05\}$. We also employ a more compressive strategy \textbf{heavy-Sign} (Definition~\ref{def:heavy-sign}) where \textbf{Sign} is applied after \textbf{TopK} (i.e., a further $32$x compression over \textbf{TopK} under same sparsity). We test \textbf{heavy-Sign} with $k\in\{0.01,0.05,0.1\}$. We compare our method with the analogue federated learning approaches using full-precision updates and using unbiased stochastic quantization \textbf{``Stoc'' without error feedback}~\citep{alistarh2017qsgd} without error feedback. For this compressor, we test parameter $b\in\{1,2,4\}$. For SGD, this algorithm is equivalent to FedCOM/FedPaQ~\citep{reisizadeh2020fedpaq,haddadpour2021federated}. The detailed introduction of the competing methods (with both SGD and AMSGrad variants) can be found in Algorithm~\ref{alg:competing} in Appendix~\ref{app sec:competing}. In our experiments, the reported results are averaged over multiple independent runs.

\subsection{Fed-EF Matches Full-Precision FL with Substantially Less Communication}

Firstly, we demonstrate the feasible performance of Fed-EF in practical FL tasks. For both datasets, we train a ReLU activated CNN with two convolutional layers followed by one max-pooling, one dropout and two fully-connected layers before the softmax output. We test each compression strategy with different compression ratios, and report the test accuracy in Figure~\ref{fig:MNIST-acc-compressor-0.5} and Figure~\ref{fig:FMNIST-acc-compressor-0.5} for MNIST and FMNIST, respectively, when the participation rate $p=0.5$. The set of results when $p=0.1$ is placed in Appendix~\ref{app sec:experiment}. We see that in general, the performance gets worse when we increase the compression rate, as expected from the theory.

\begin{figure}[t!]
    \begin{center}
        \mbox{\hspace{-0.1in}
        \includegraphics[width=2.25in]{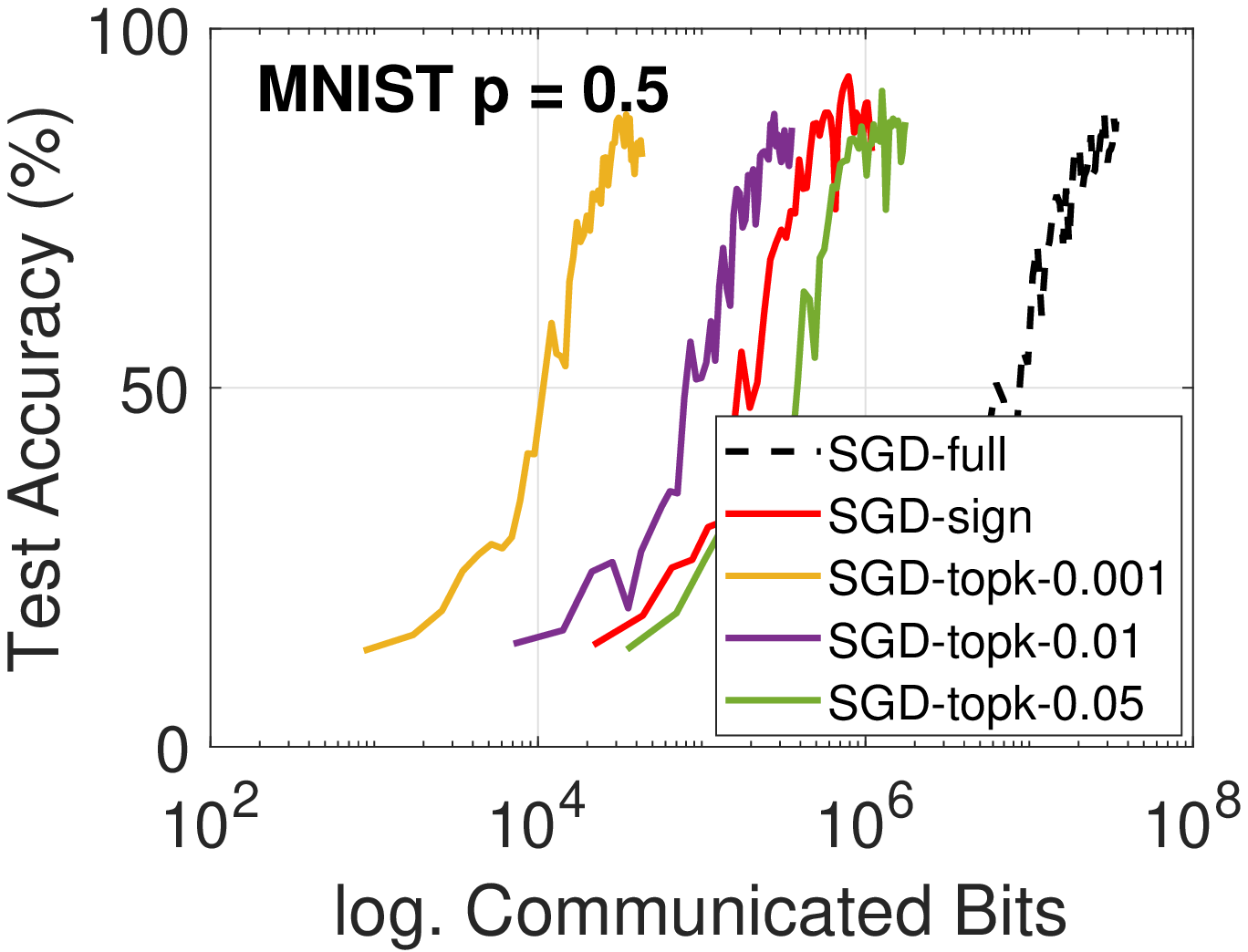}\hspace{-0.1in}
        \includegraphics[width=2.25in]{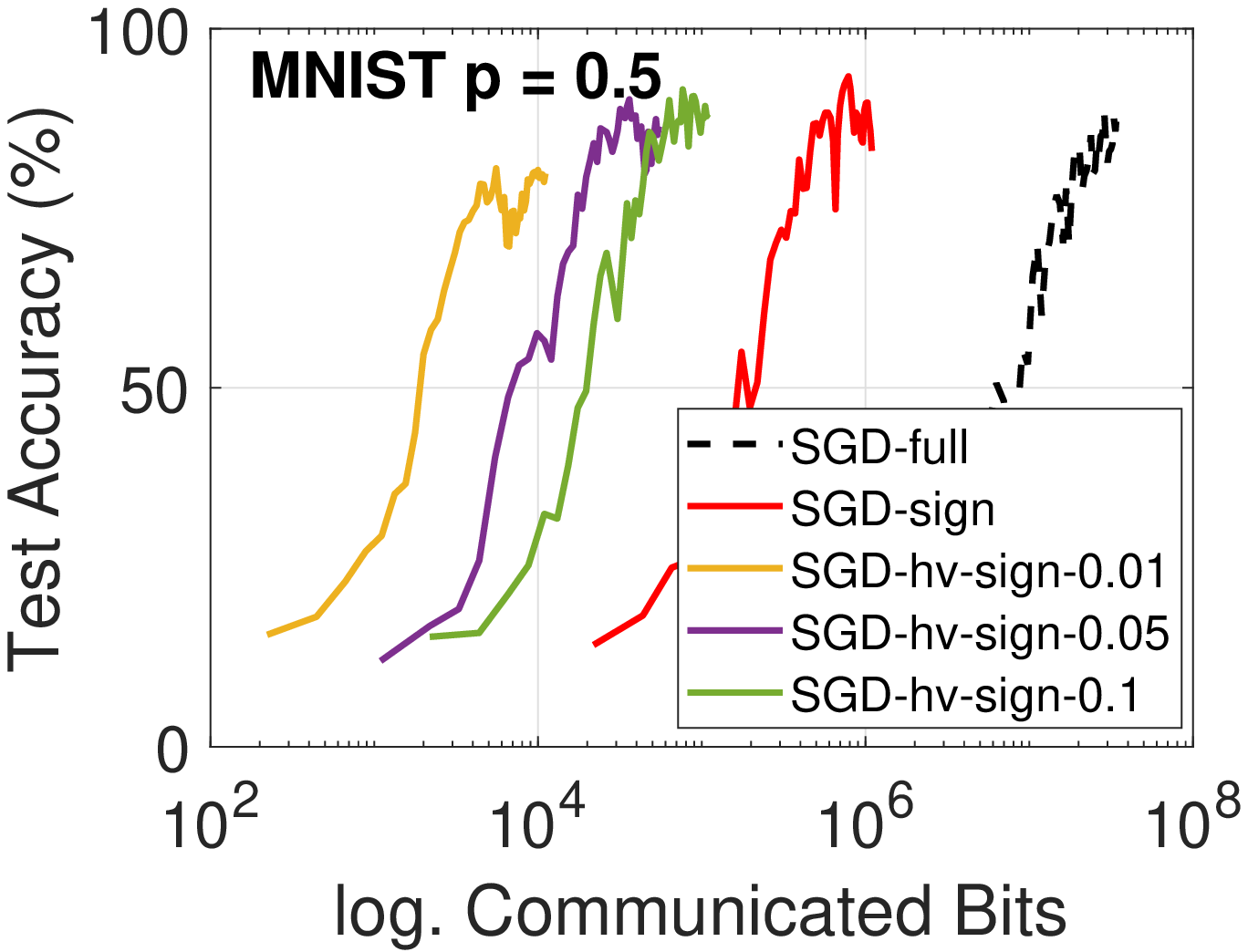}\hspace{-0.1in}
        \includegraphics[width=2.25in]{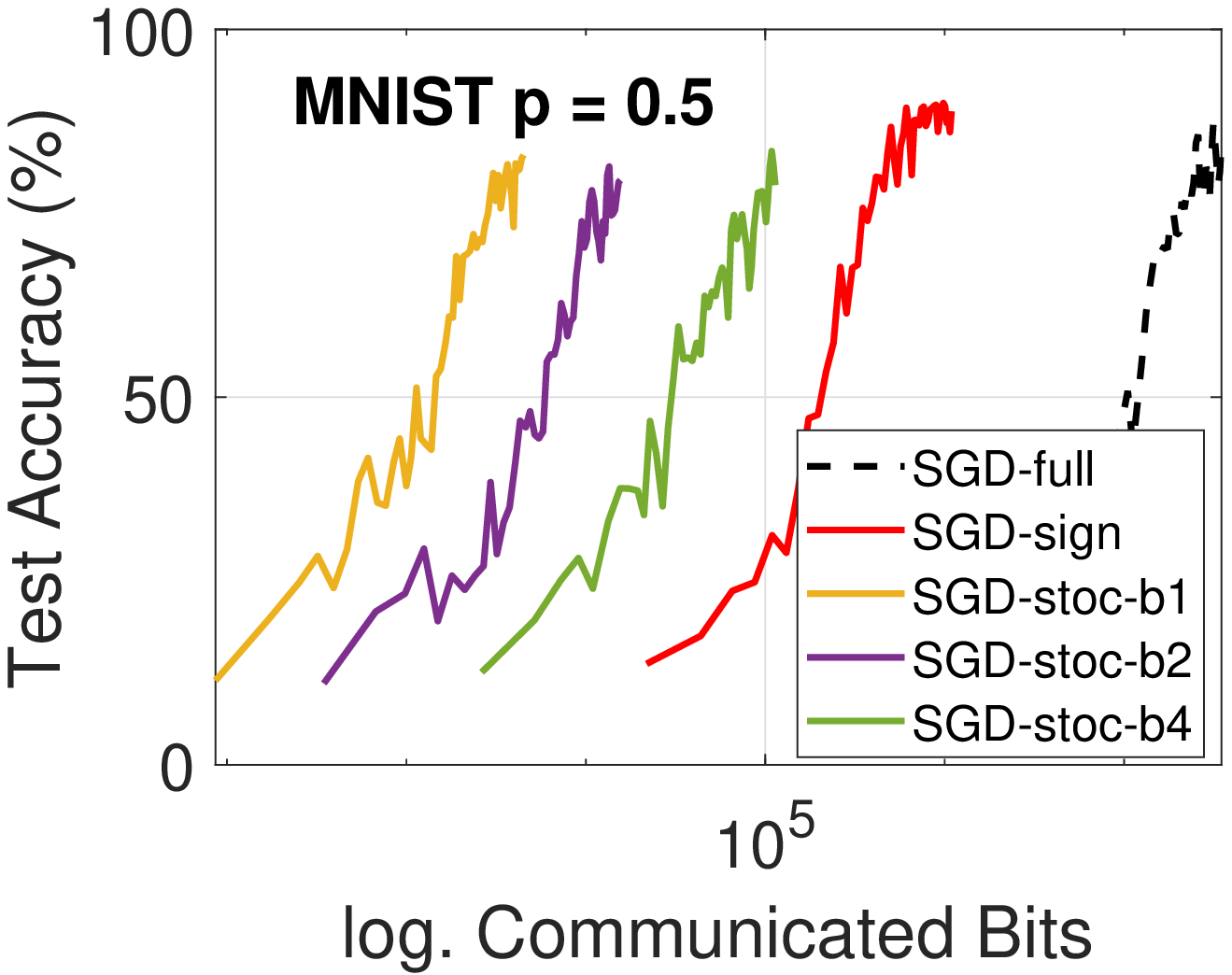}
        }
        \mbox{\hspace{-0.1in}
        \includegraphics[width=2.25in]{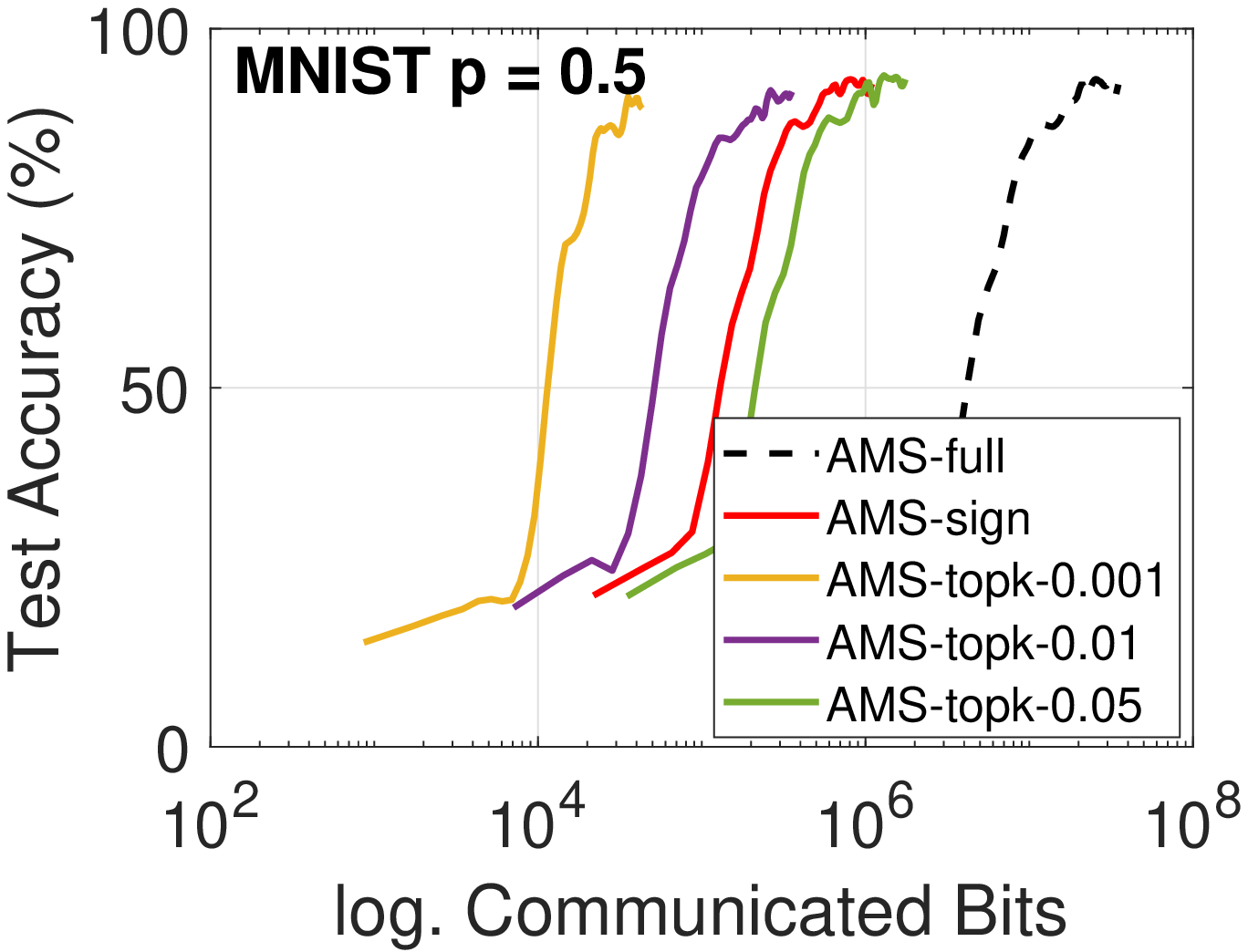}\hspace{-0.1in}
        \includegraphics[width=2.25in]{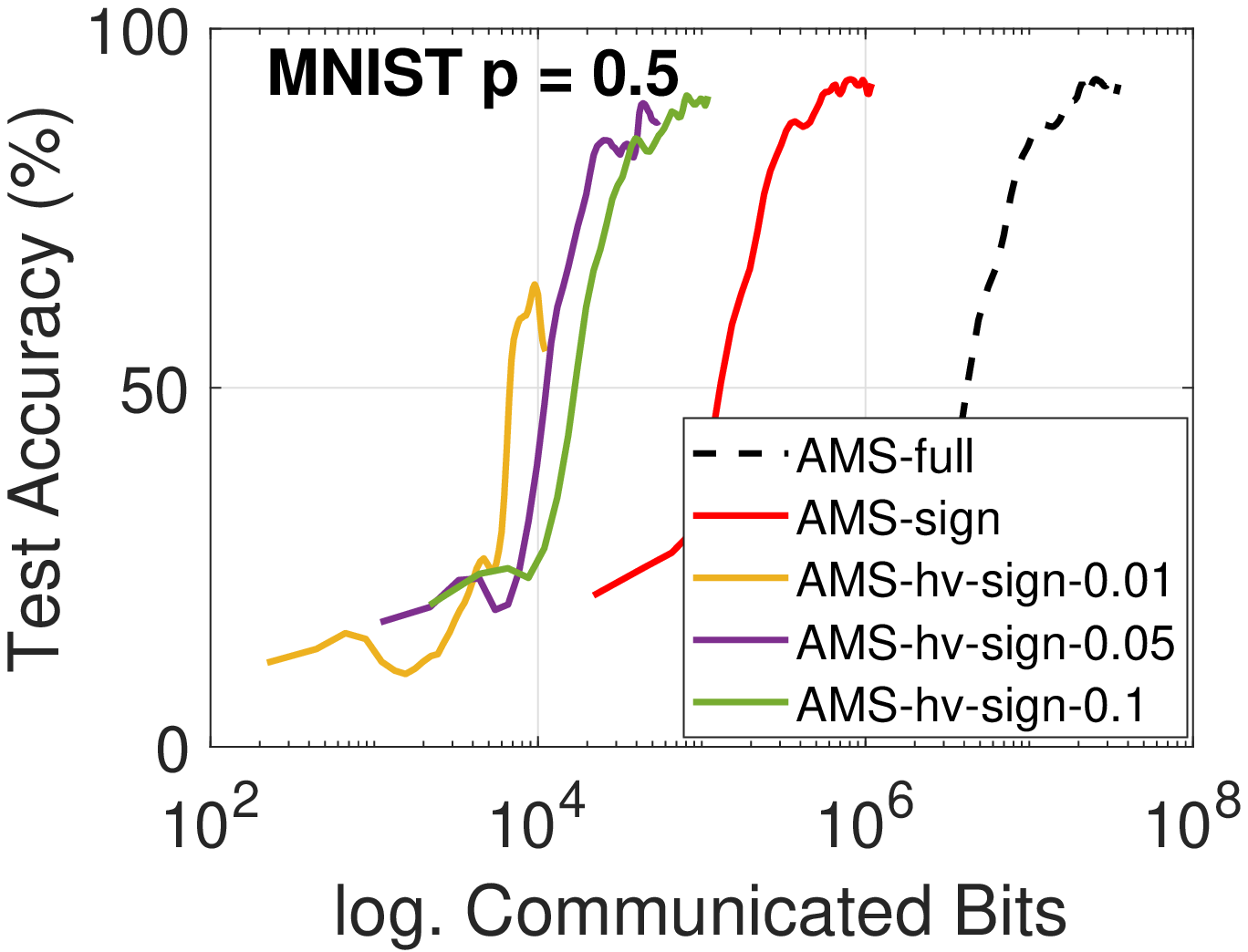}\hspace{-0.1in}
        \includegraphics[width=2.25in]{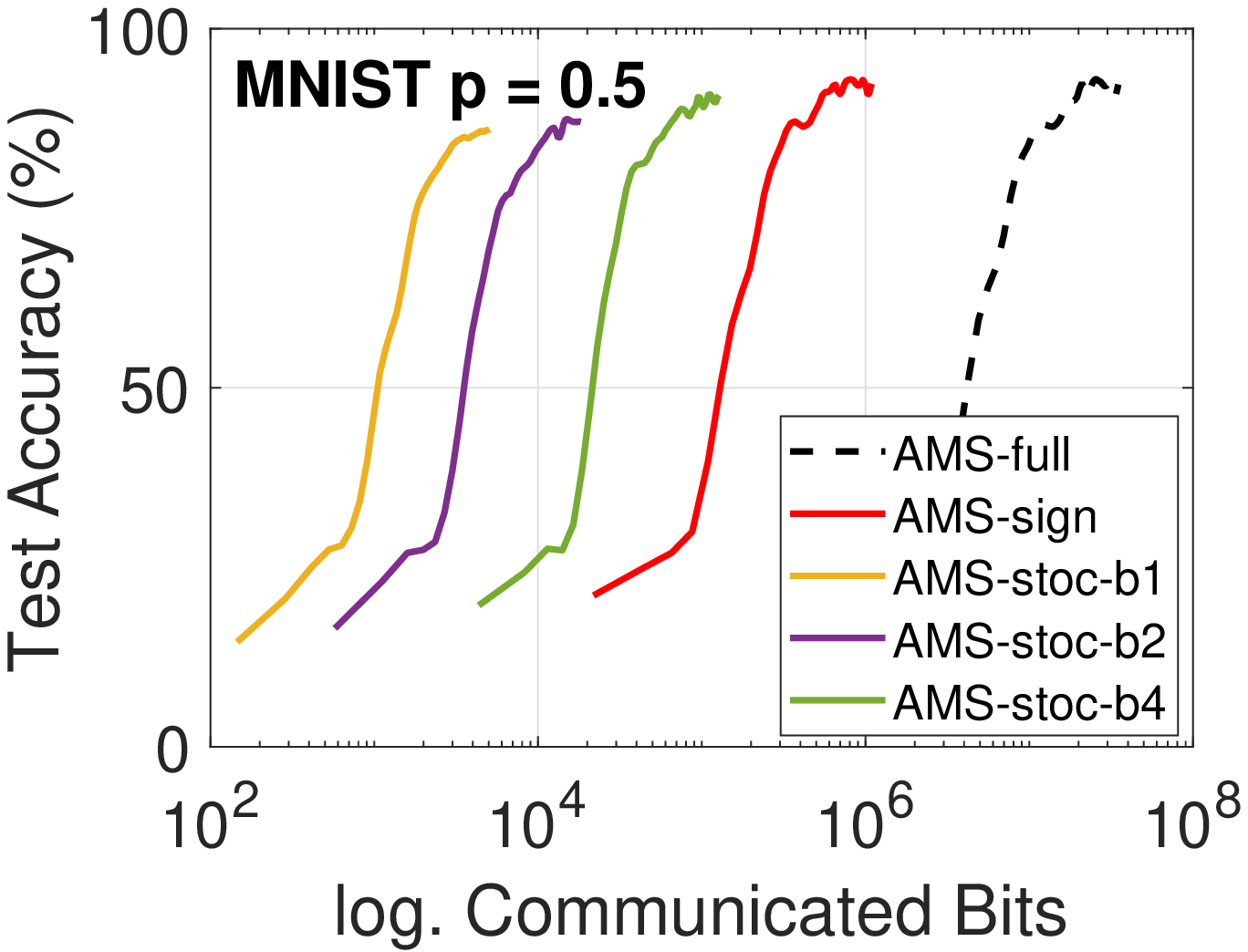}
        }
    \end{center}

\vspace{-0.2in}

	\caption{Test accuracy of Fed-EF on MNIST dataset trained by CNN. ``sign'', ``topk'' and ``hv-sign'' are applied with Fed-EF, while ``Stoc'' is the stochastic quantization without EF. Participation rate $p=0.5$, non-iid data. 1st row: Fed-EF-SGD. 2nd row: Fed-EF-AMS.}
	\label{fig:MNIST-acc-compressor-0.5}\vspace{0.2in}
\end{figure}

\begin{figure}[h!]
    \begin{center}
        \mbox{\hspace{-0.1in}
        \includegraphics[width=2.25in]{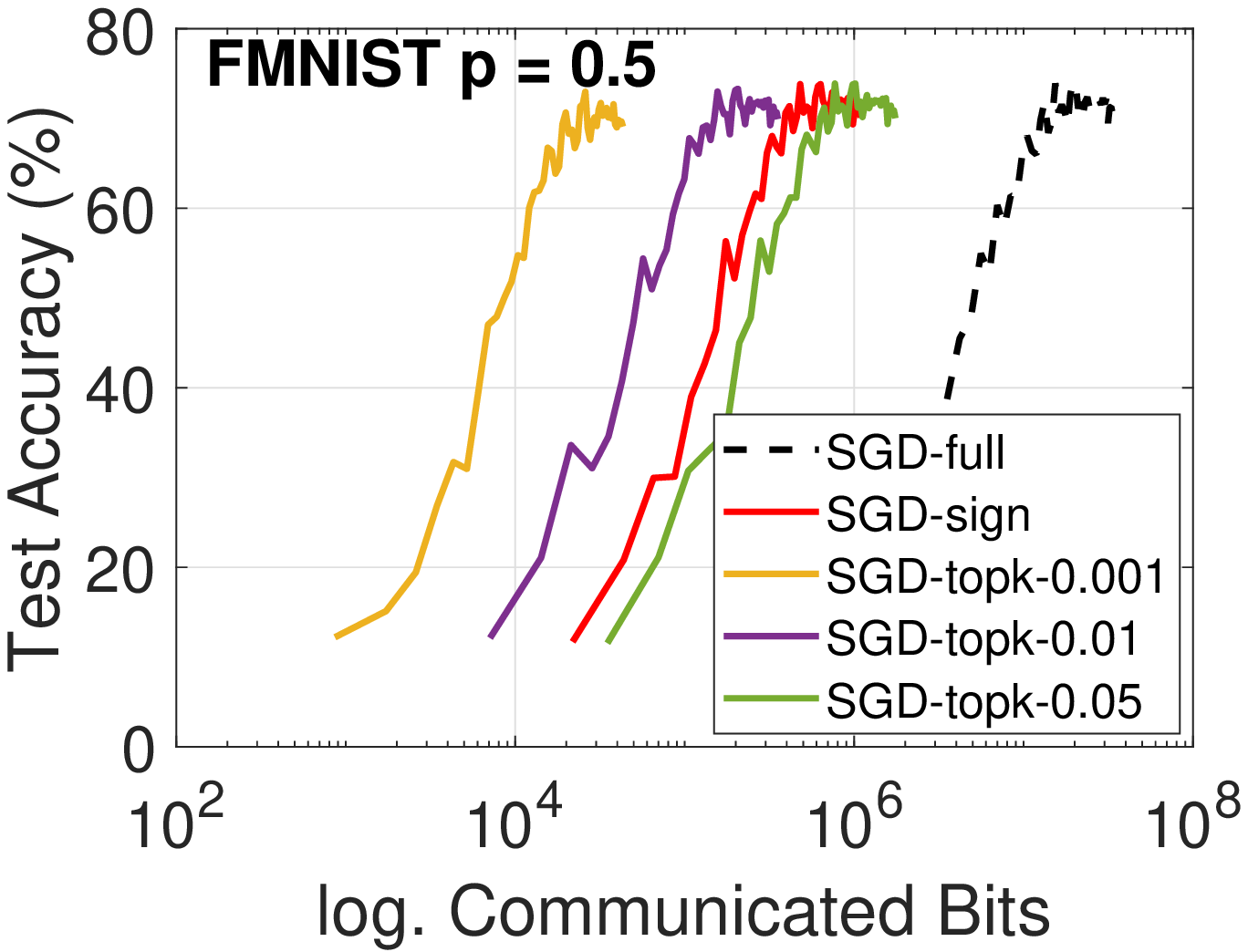}\hspace{-0.1in}
        \includegraphics[width=2.25in]{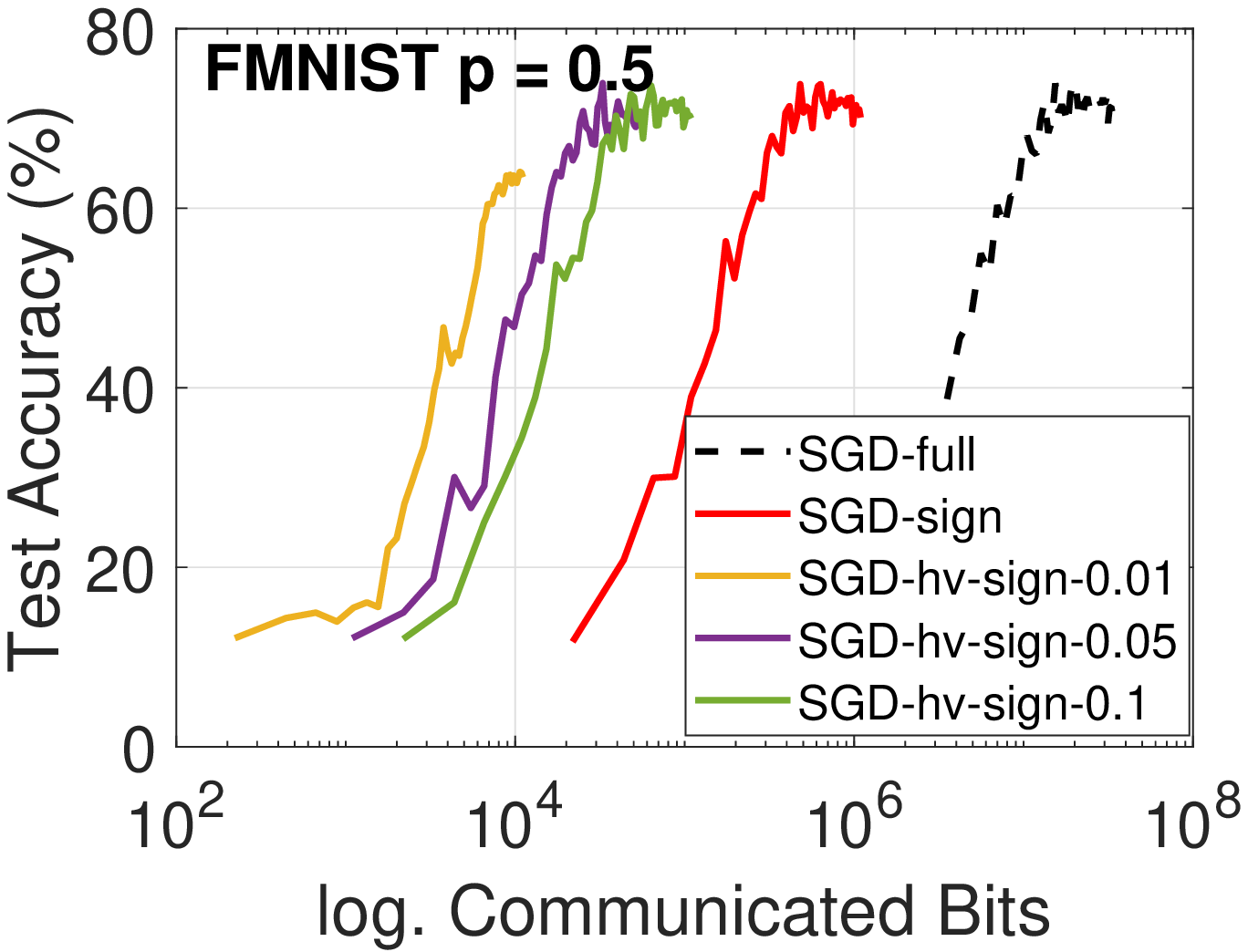}\hspace{-0.1in}
        \includegraphics[width=2.25in]{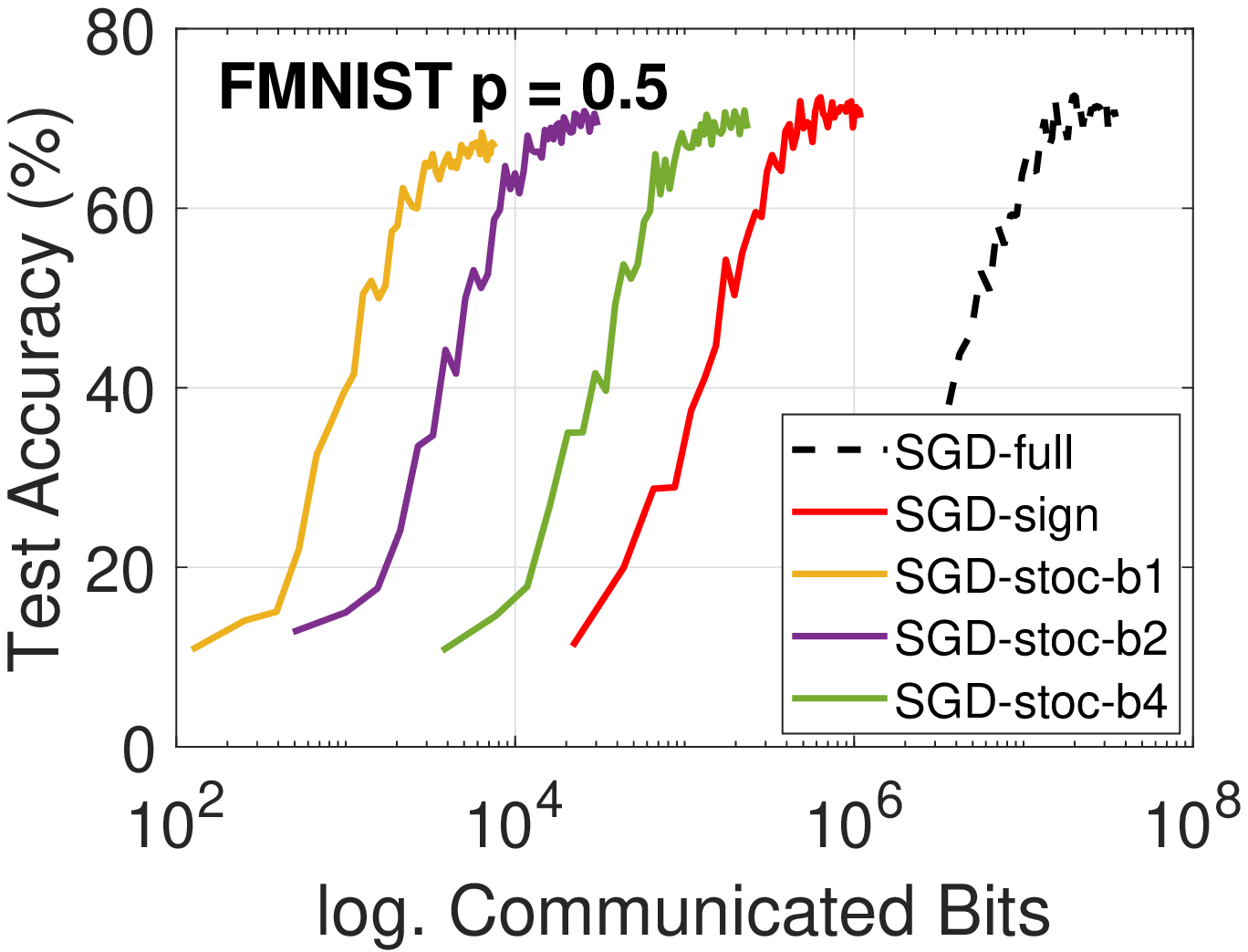}
        }
        \mbox{\hspace{-0.1in}
        \includegraphics[width=2.25in]{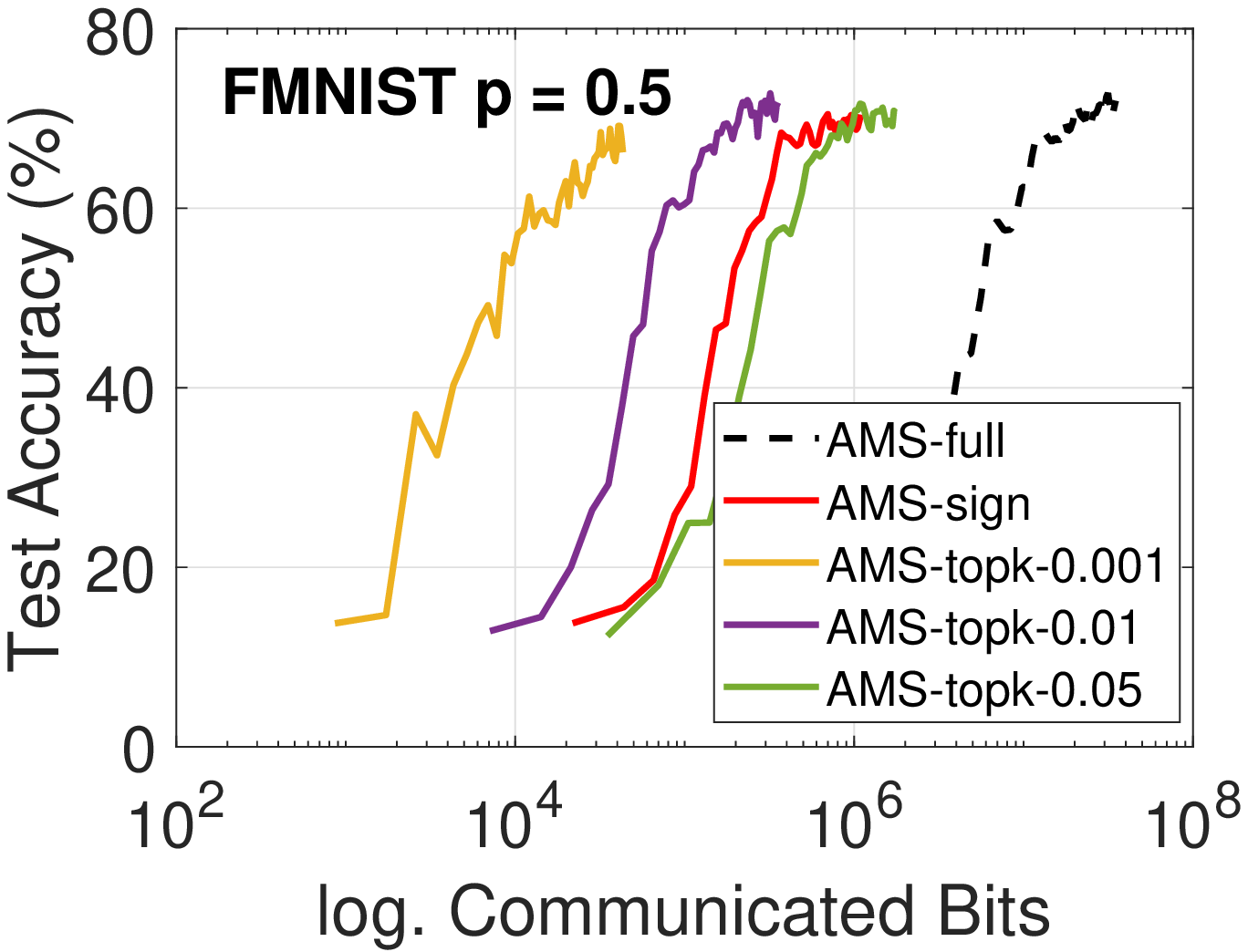}\hspace{-0.1in}
        \includegraphics[width=2.25in]{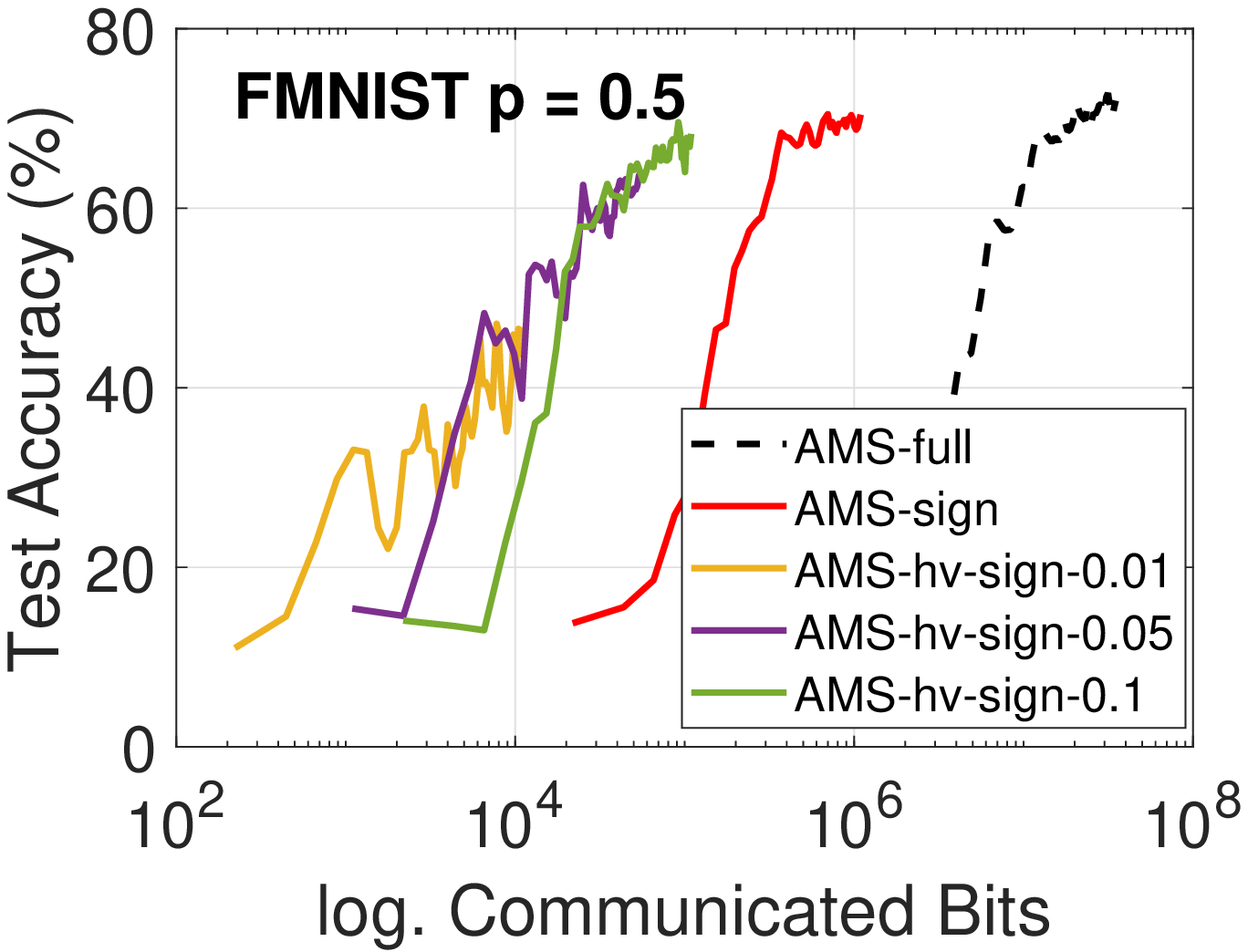}\hspace{-0.1in}
        \includegraphics[width=2.25in]{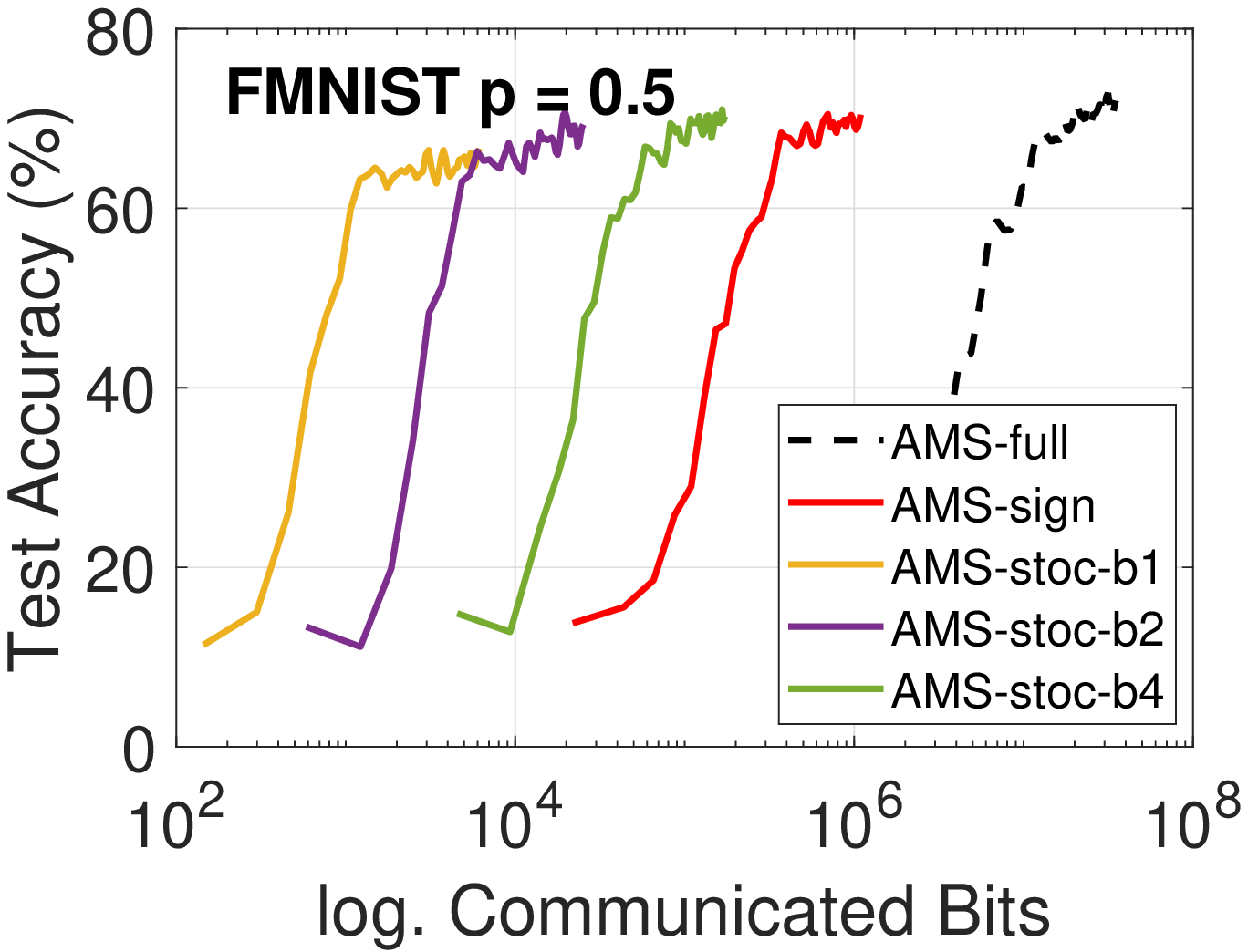}
        }
    \end{center}

\vspace{-0.2in}

	\caption{Test accuracy of Fed-EF on FMNIST dataset trained by CNN. ``sign'', ``topk'' and ``hv-sign'' are applied with Fed-EF, while ``Stoc'' is the stochastic quantization without EF. Participation rate $p=0.5$, non-iid data. 1st row: Fed-EF-SGD. 2nd row: Fed-EF-AMS.}
	\label{fig:FMNIST-acc-compressor-0.5}
\end{figure}


\begin{figure}[t!]
    \begin{center}
        \mbox{\hspace{-0.15in}
        \includegraphics[width=1.75in]{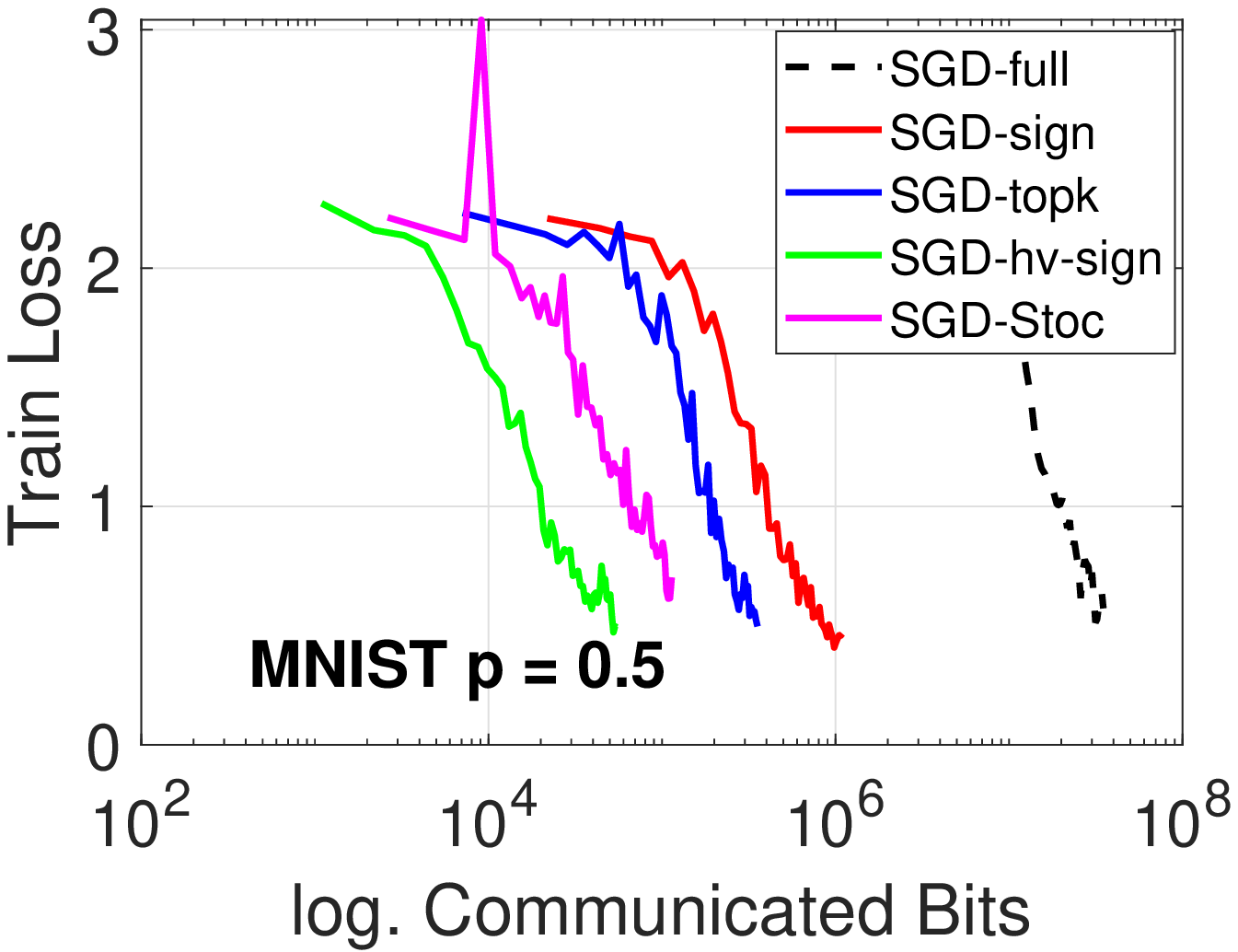}\hspace{-0.12in}
        \includegraphics[width=1.75in]{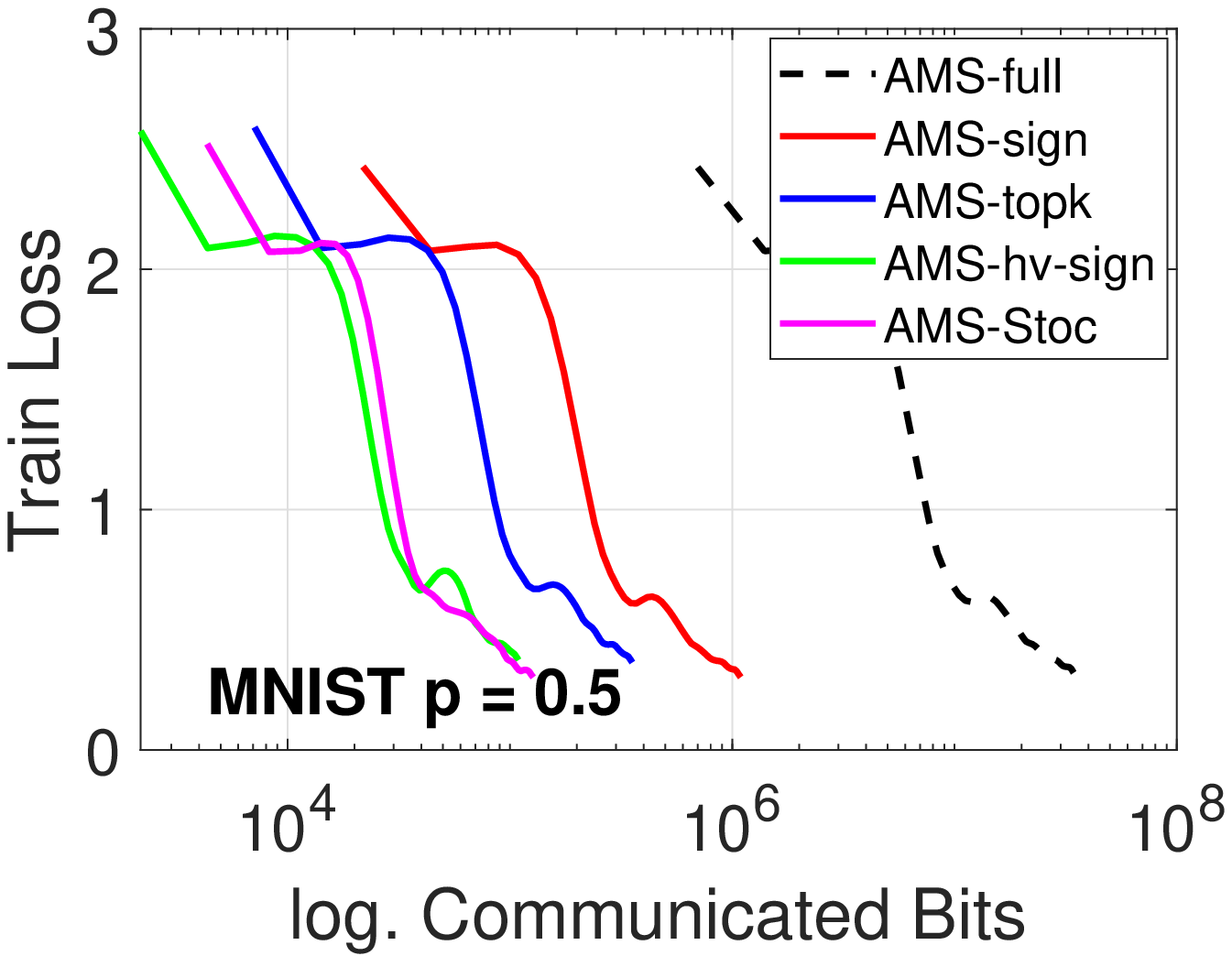}\hspace{-0.12in}
        \includegraphics[width=1.75in]{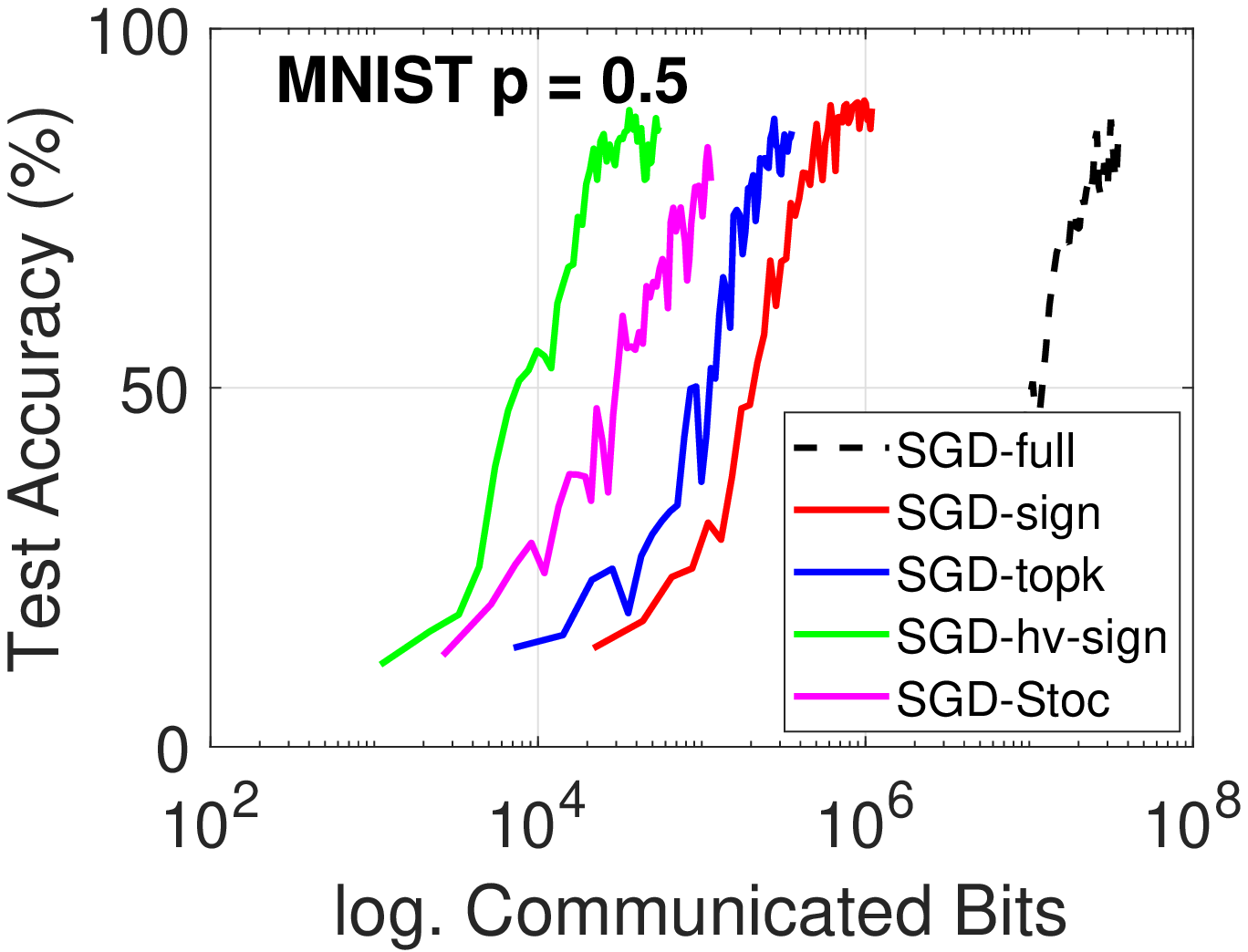}\hspace{-0.12in}
        \includegraphics[width=1.75in]{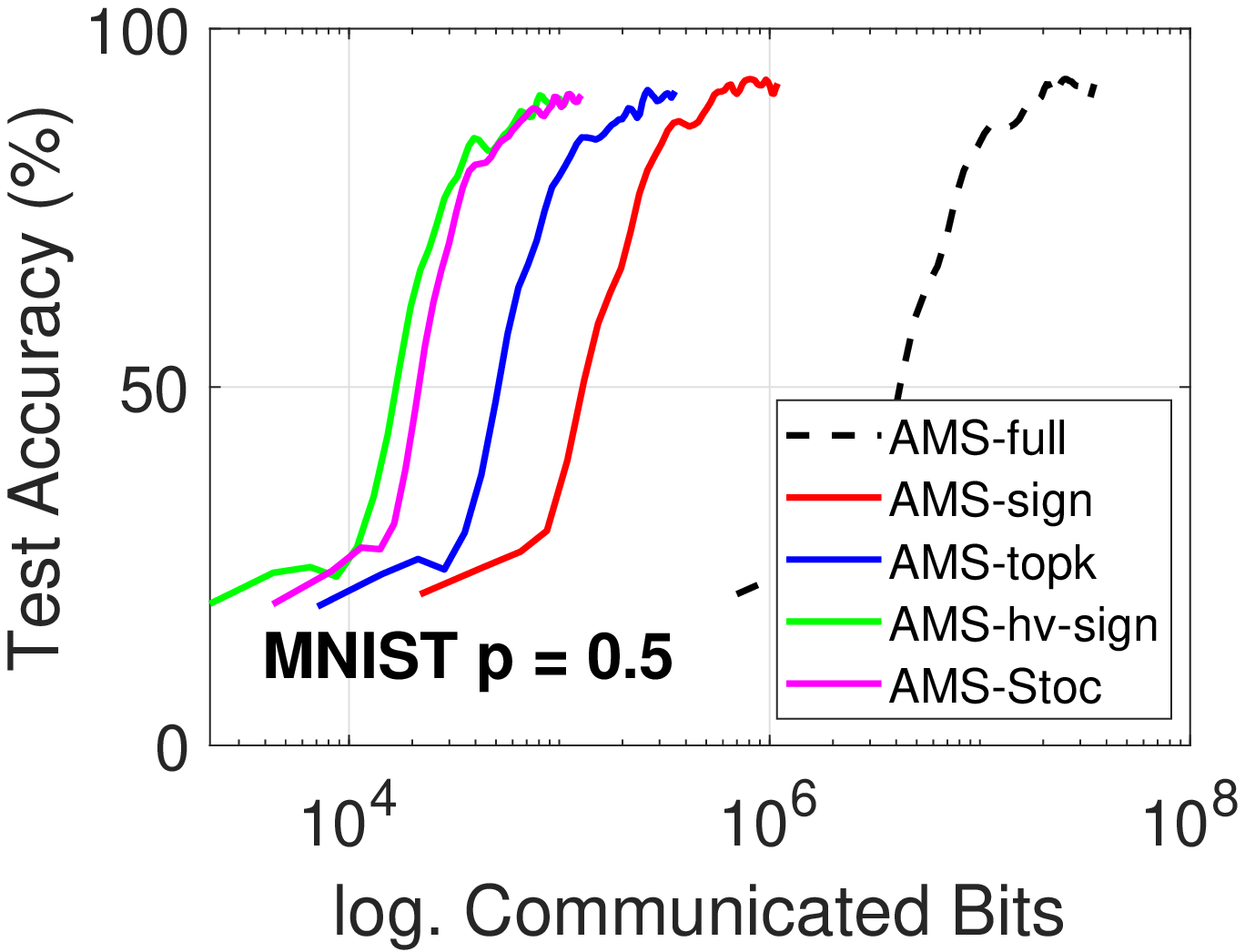}
        }

        \mbox{\hspace{-0.15in}
        \includegraphics[width=1.75in]{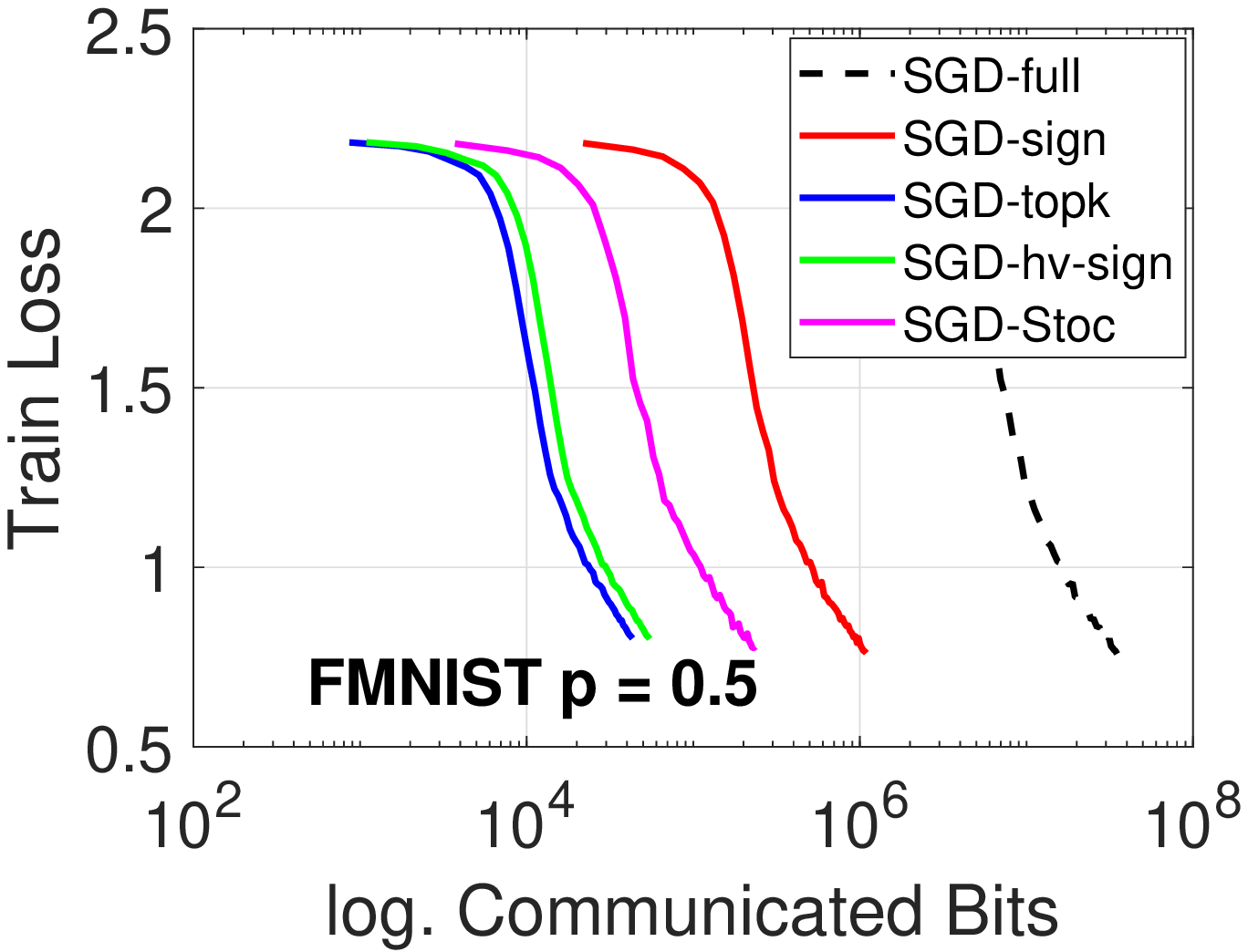}\hspace{-0.12in}
        \includegraphics[width=1.75in]{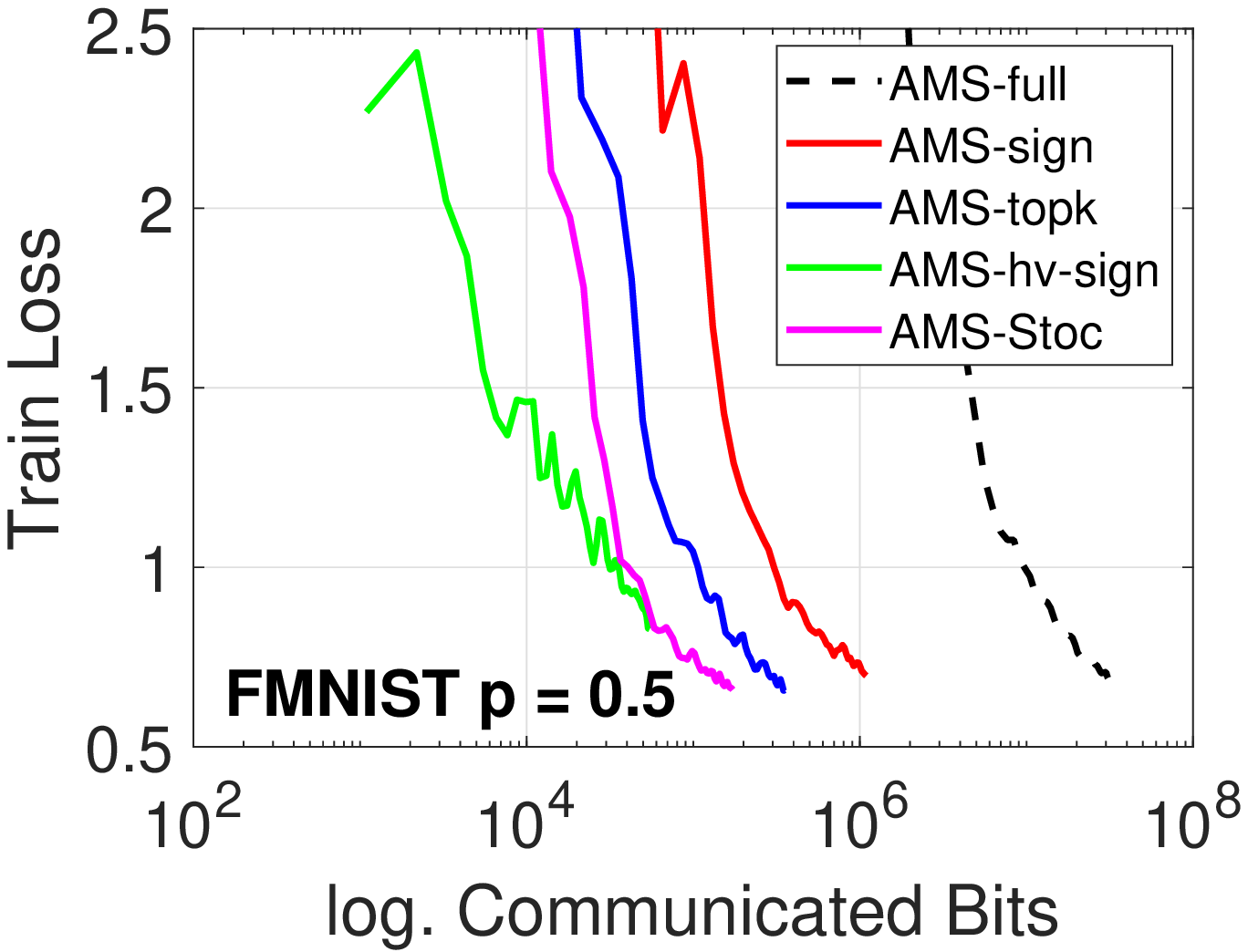}\hspace{-0.12in}
        \includegraphics[width=1.75in]{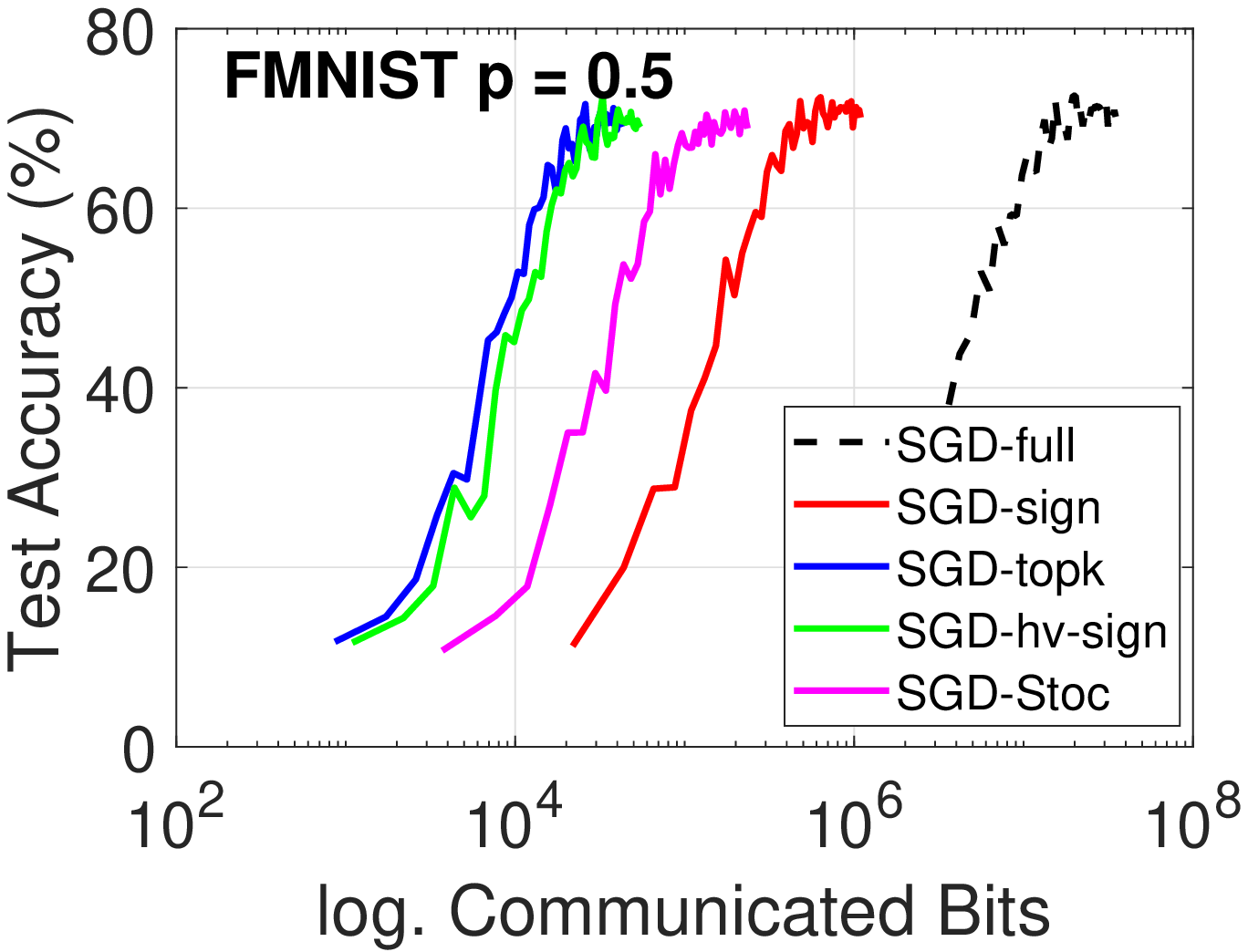}\hspace{-0.12in}
        \includegraphics[width=1.75in]{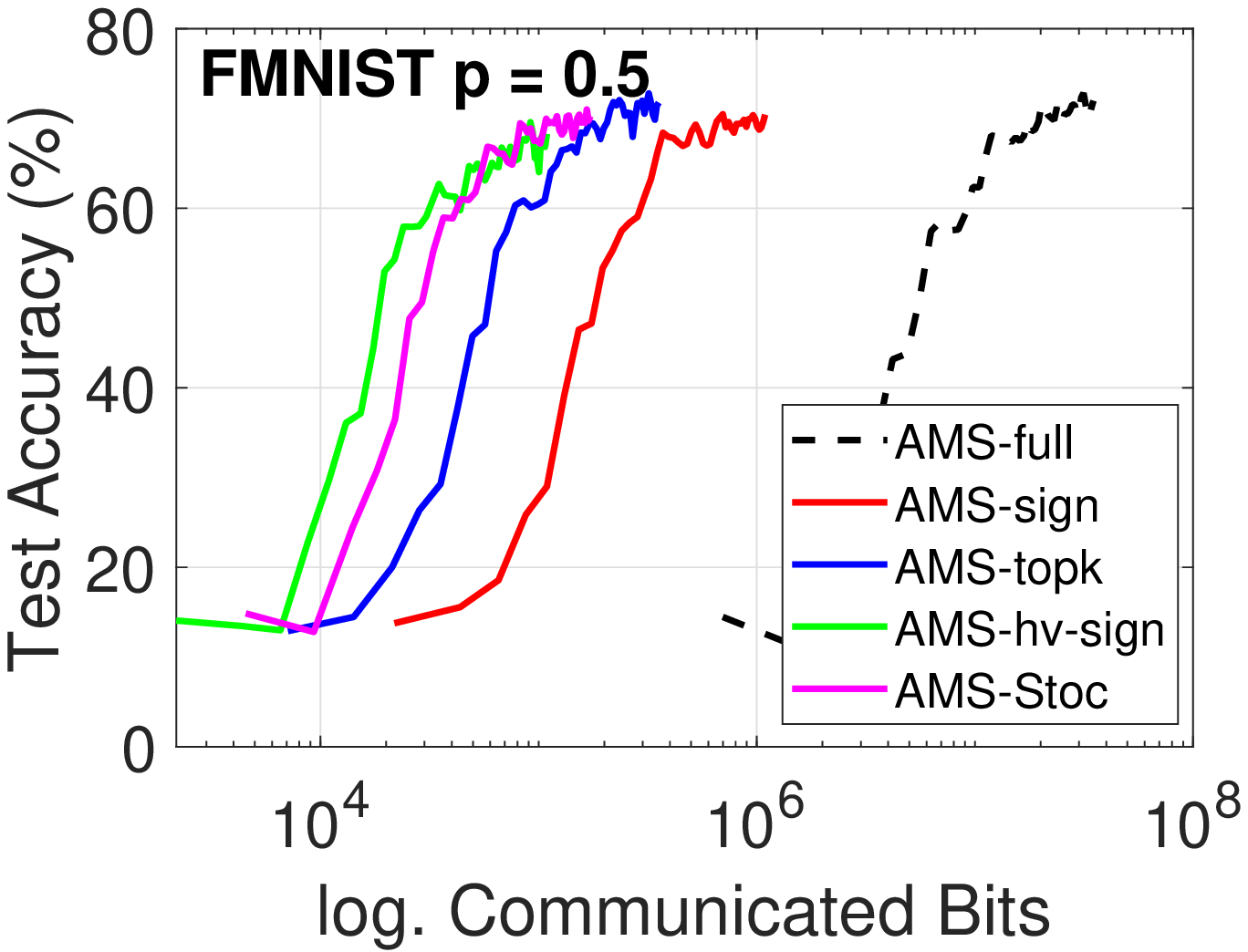}
        }
    \end{center}
    \vspace{-0.2in}

	\caption{Training loss and test accuracy v.s. communicated bits on MNIST and FMNIST datasets, participation rate $p=0.5$. ``sign'', ``topk'' and ``hv-sign'' are applied with Fed-EF, while ``Stoc'' is the stochastic quantization without EF.}
	\label{fig:MNIST-FMNIST-0.5}
\end{figure}

\begin{figure}[h!]
    \begin{center}
        \mbox{\hspace{-0.15in}
        \includegraphics[width=1.75in]{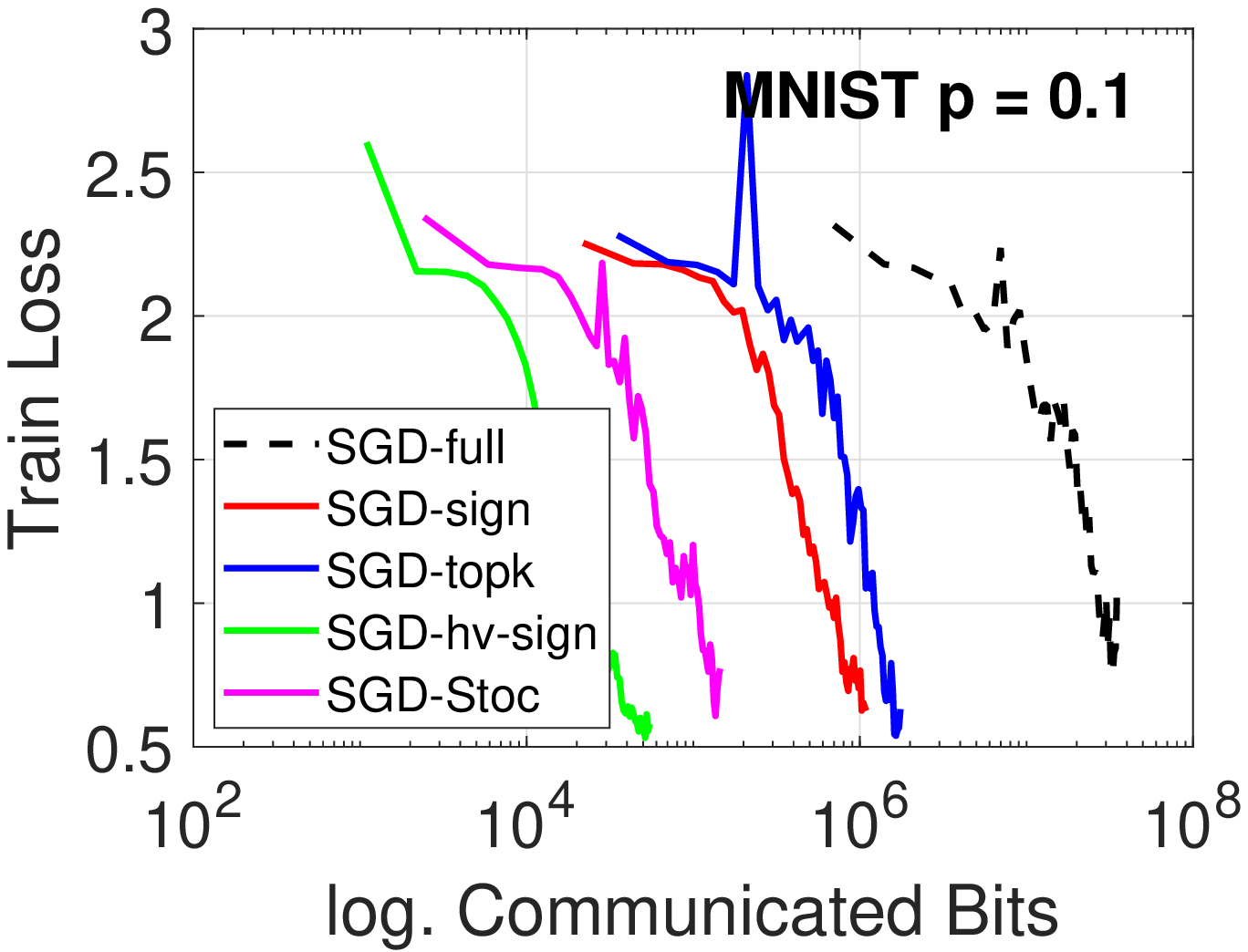}\hspace{-0.12in}
        \includegraphics[width=1.75in]{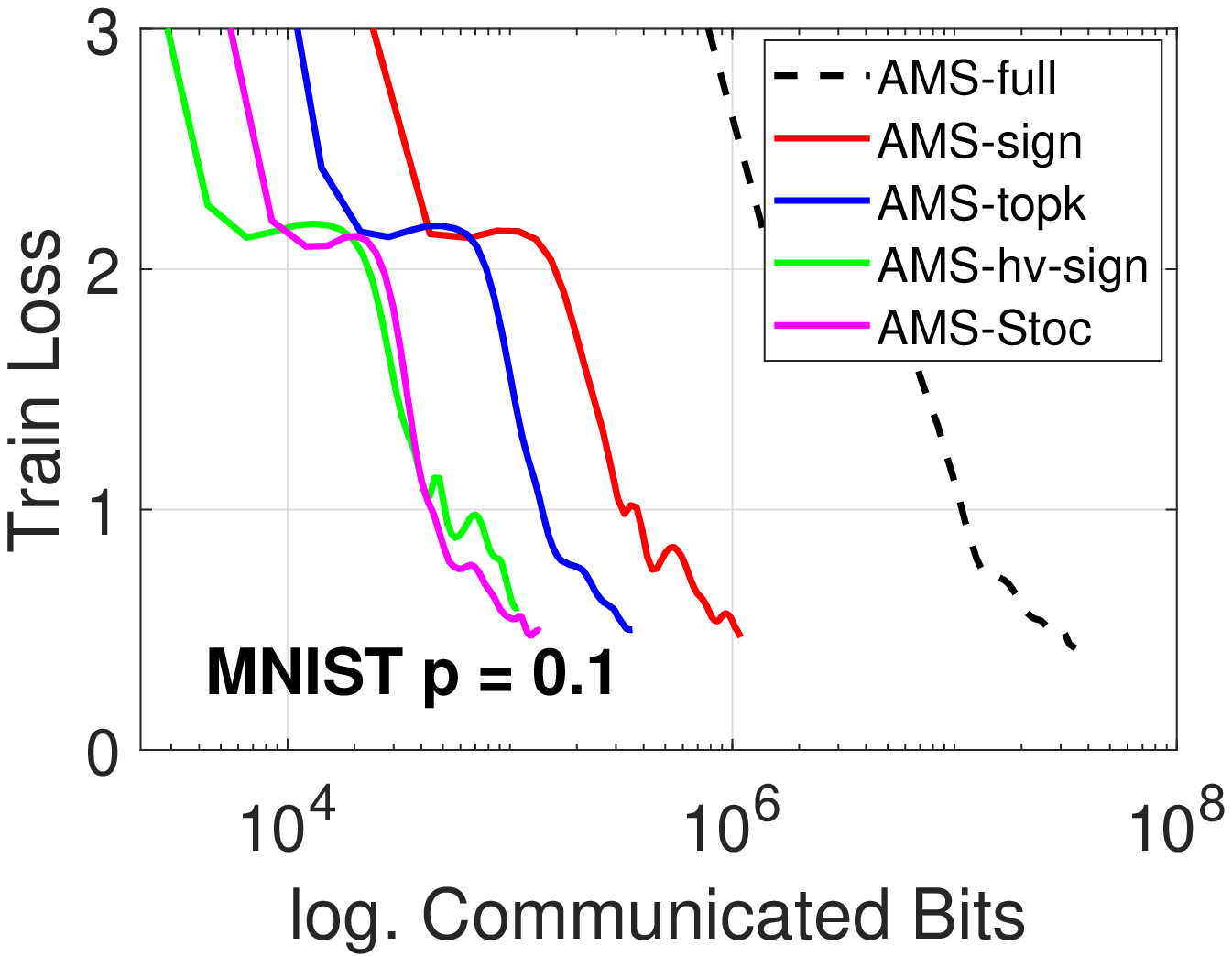}\hspace{-0.12in}
        \includegraphics[width=1.75in]{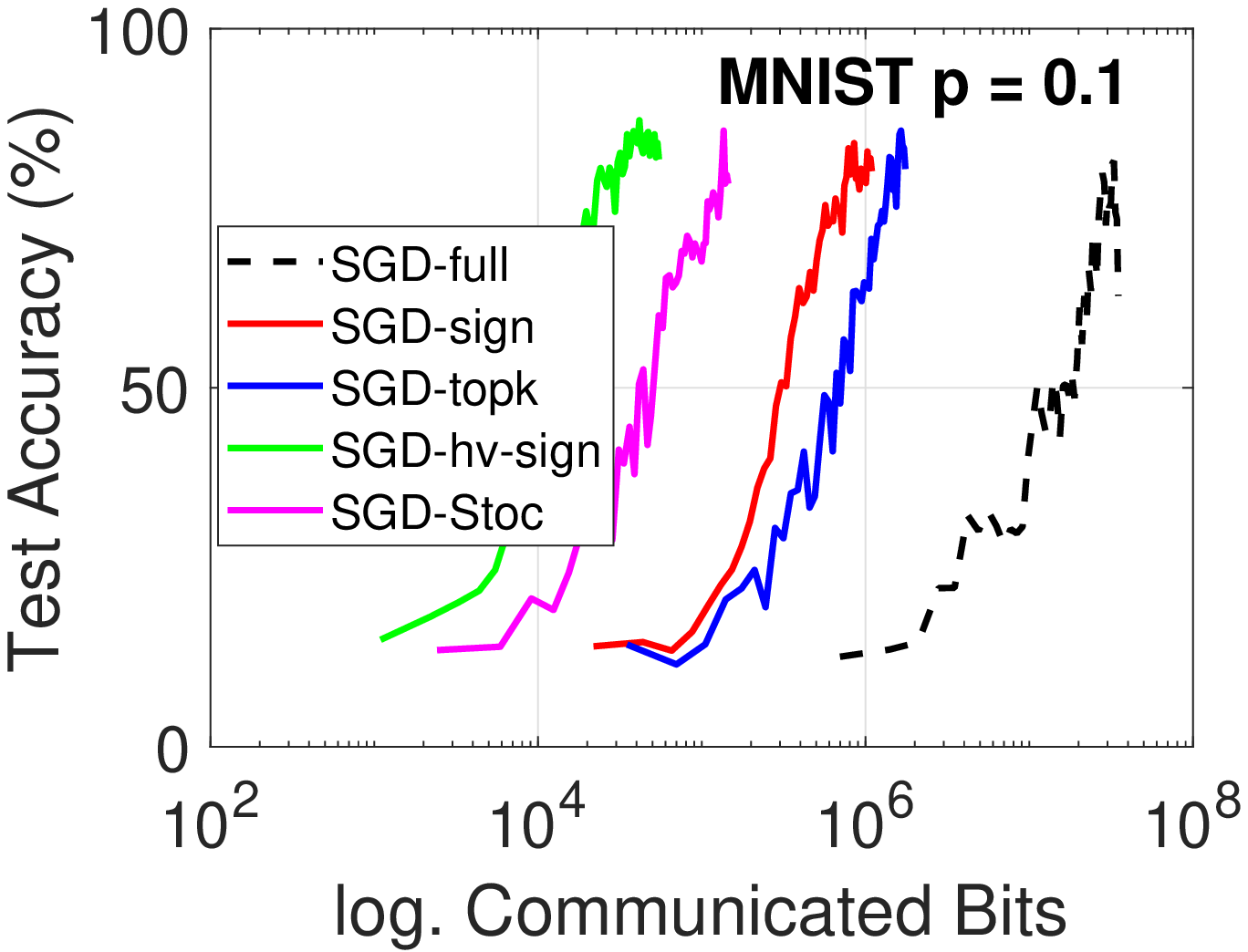}\hspace{-0.12in}
        \includegraphics[width=1.75in]{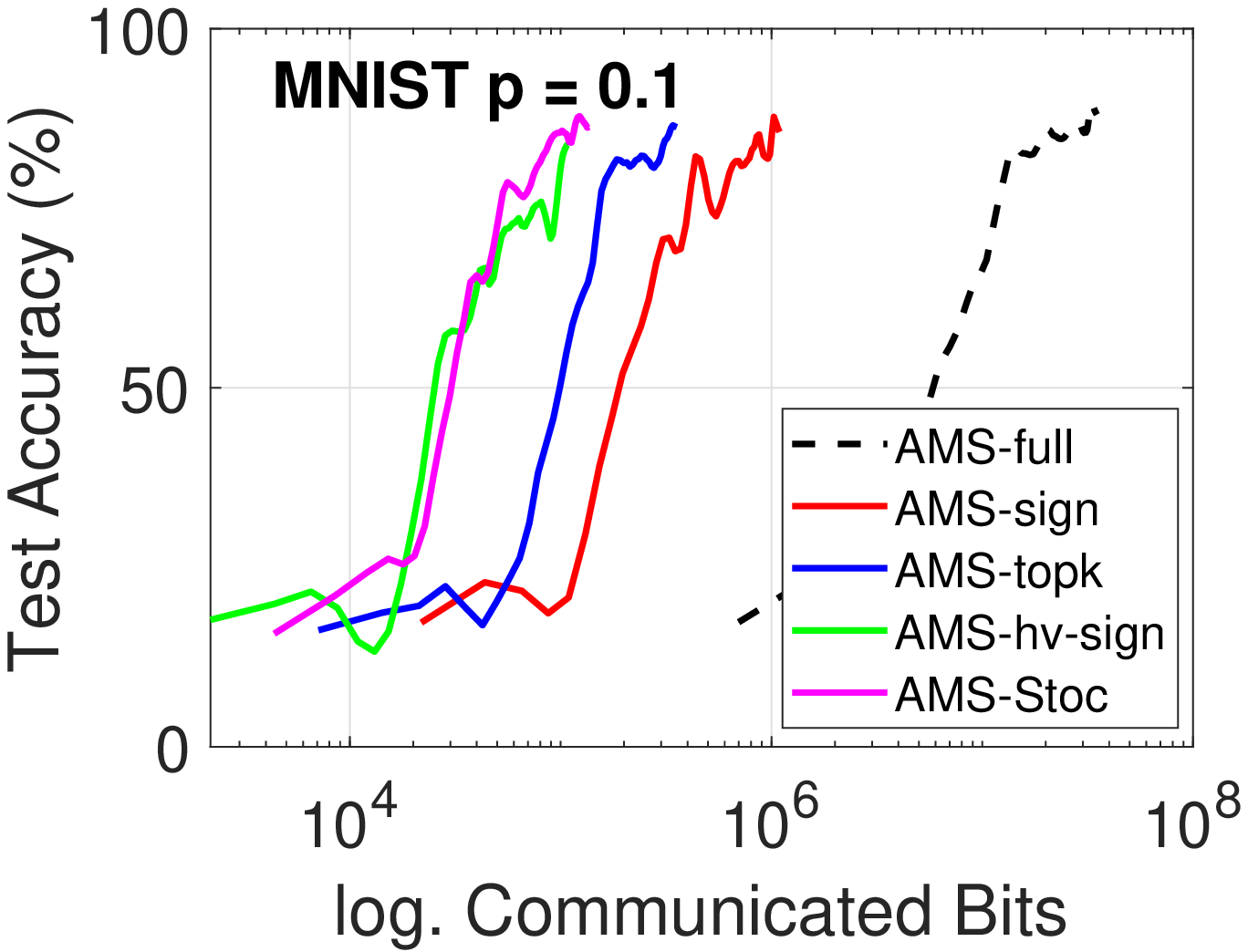}
        }

        \mbox{\hspace{-0.15in}
        \includegraphics[width=1.75in]{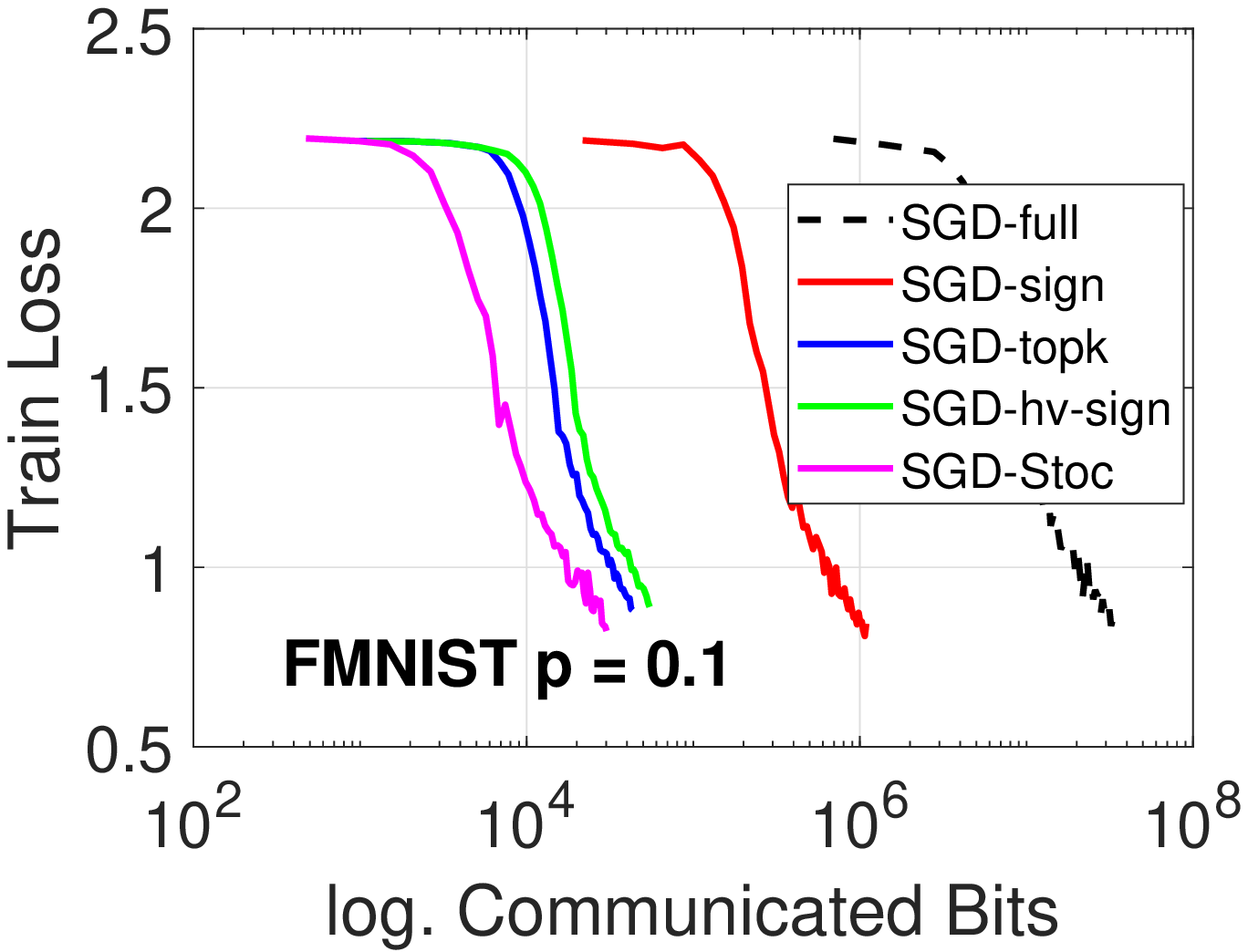}\hspace{-0.12in}
        \includegraphics[width=1.75in]{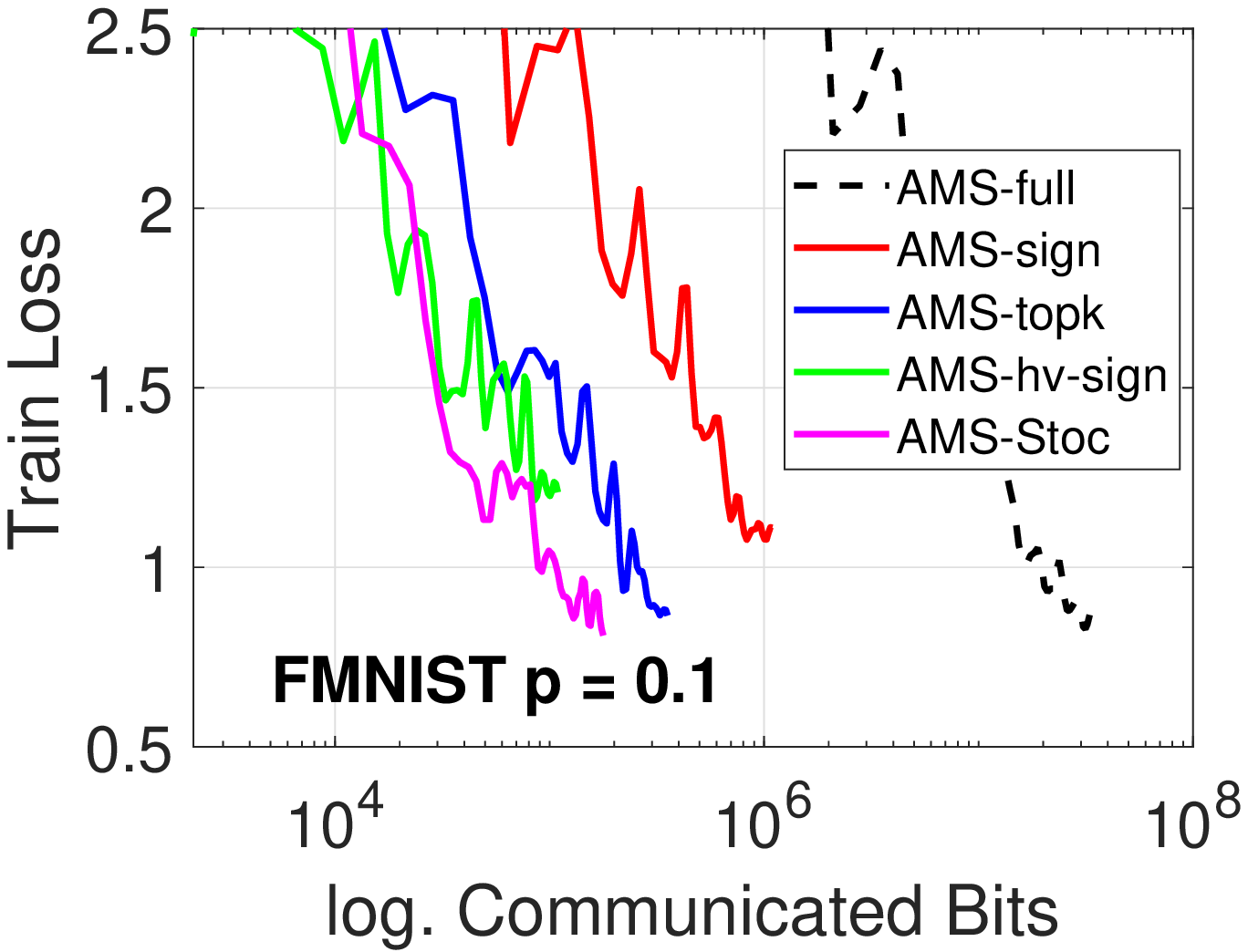}\hspace{-0.12in}
        \includegraphics[width=1.75in]{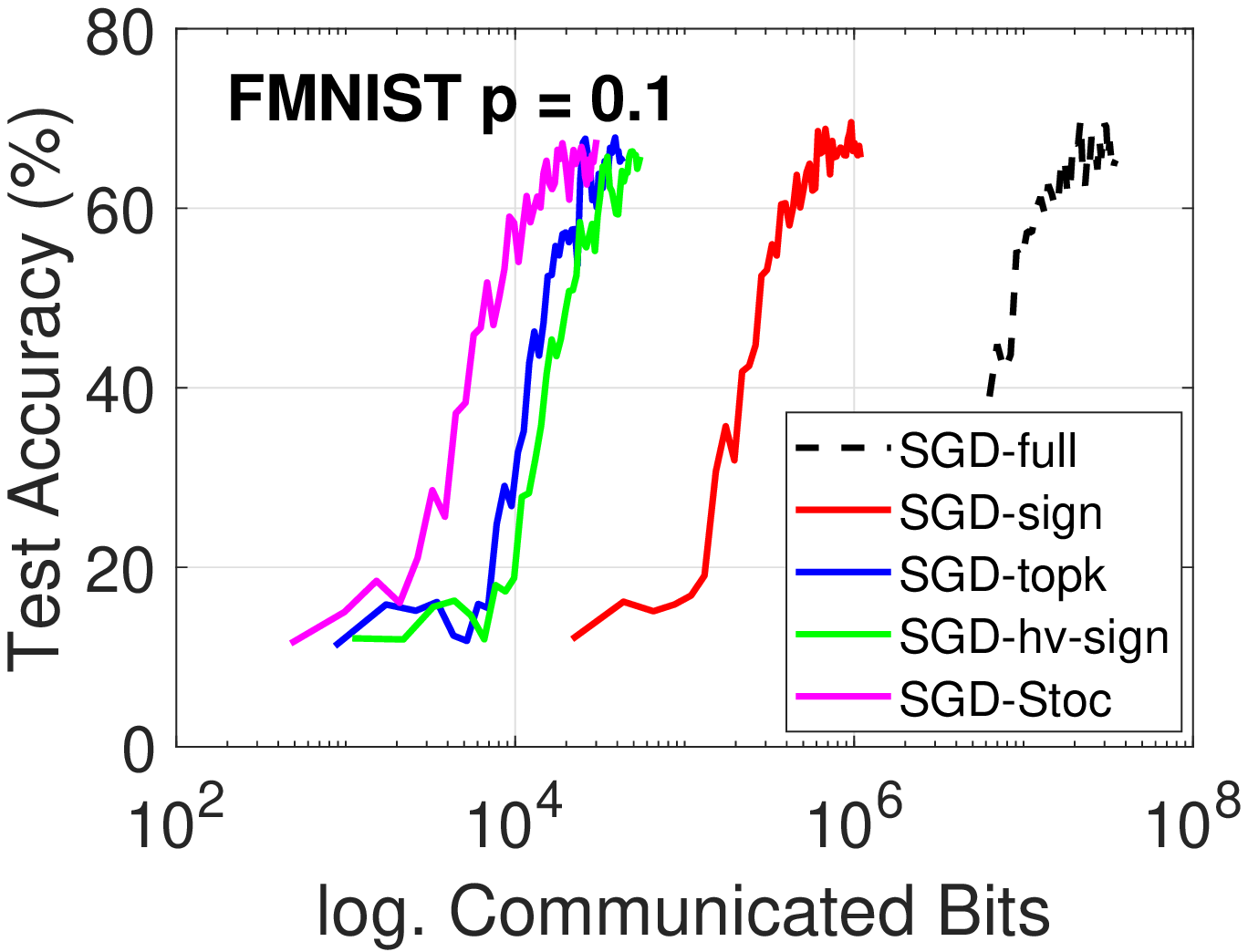}\hspace{-0.12in}
        \includegraphics[width=1.75in]{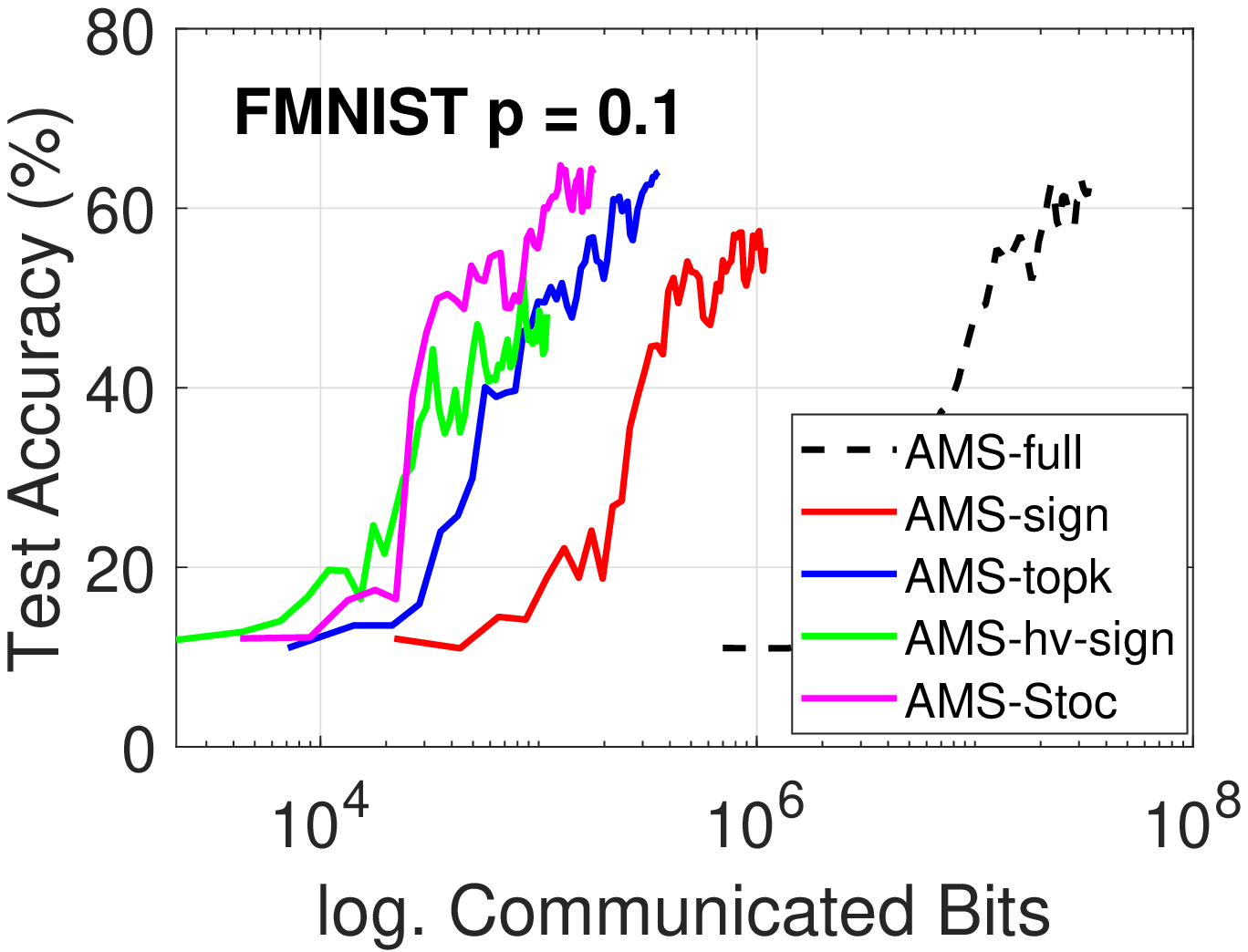}
        }
    \end{center}

    \vspace{-0.2in}

	\caption{Training loss and test accuracy v.s. communicated bits on MNIST and FMNIST datasets, participation rate $p=0.1$. ``sign'', ``topk'' and ``hv-sign'' are applied with Fed-EF, while ``Stoc'' is the stochastic quantization without EF.}
	\label{fig:MNIST-FMNIST-0.1}
\end{figure}

To make a more straightforward comparison, in Figure~\ref{fig:MNIST-FMNIST-0.5} ($p=0.5$) and Figure~\ref{fig:MNIST-FMNIST-0.1} ($p=0.1$), we present the training loss and test accuracy curves (respective compression ratios) chosen by the following rule: for each method, we present the curve with highest compression level that achieves the best full-precision test accuracy (within $0.1\%$); if the method does not match the full-precision performance, we present the curve with the highest test accuracy. The exact test accuracy and standard deviations can be found in Appendix~\ref{app sec:experiment}. We observe the following:
\begin{itemize}
    \item The proposed Fed-EF (including both variants) is able to achieve the same performance as the full-precision methods with substantial communication reduction, e.g., \textbf{heavy-Sign} and \textbf{TopK} reduce the communication by more than 100x without losing accuracy. \textbf{Sign} also provides 30x compression with matching accuracy as full-precision training.

    \item In Figure~\ref{fig:MNIST-FMNIST-0.5}, on MNIST, the test accuracy of \textbf{Stoc} (stochastic quantization without EF) is slightly lower than Fed-EF-SGD with \textbf{heavy-Sign}, yet requiring more communication.

    \item With more aggressive $p=0.1$ (Figure~\ref{fig:MNIST-FMNIST-0.1}),  Fed-EF with a proper compressor still performs on a par with full-precision algorithms. While \textbf{Sign} performs well on MNIST, we notice that fixed sign-based compressors (\textbf{Sign} and \textbf{heavy-Sign}) are outperformed by \textbf{TopK} for Fed-EF-AMS on FMNIST. We conjecture that this is because with small participation rate, sign-based compressors tend to assign a same implicit learning rate across coordinates (controlled by the second moment $\hat v$), making adaptive method less effective. In contrast,magnitude-preserving compressors (e.g., \textbf{TopK} and \textbf{Stoc}) may better exploit the adaptivity of AMSGrad.
\end{itemize}

\subsection{Results on CIFAR-10 and ResNet}

We present additional experiments to illustrate that Fed-EF is able to match the full-precision training on larger models, on the task of CIFAR-10 image classification. For this experiment, we train a ResNet-18~\citep{he2016deep} network for 200 rounds. The clients local data are distributed in the same way as described above which is highly non-iid.

\begin{figure}[h]
    \begin{center}
        \mbox{\hspace{-0.1in}
        \includegraphics[width=2.25in]{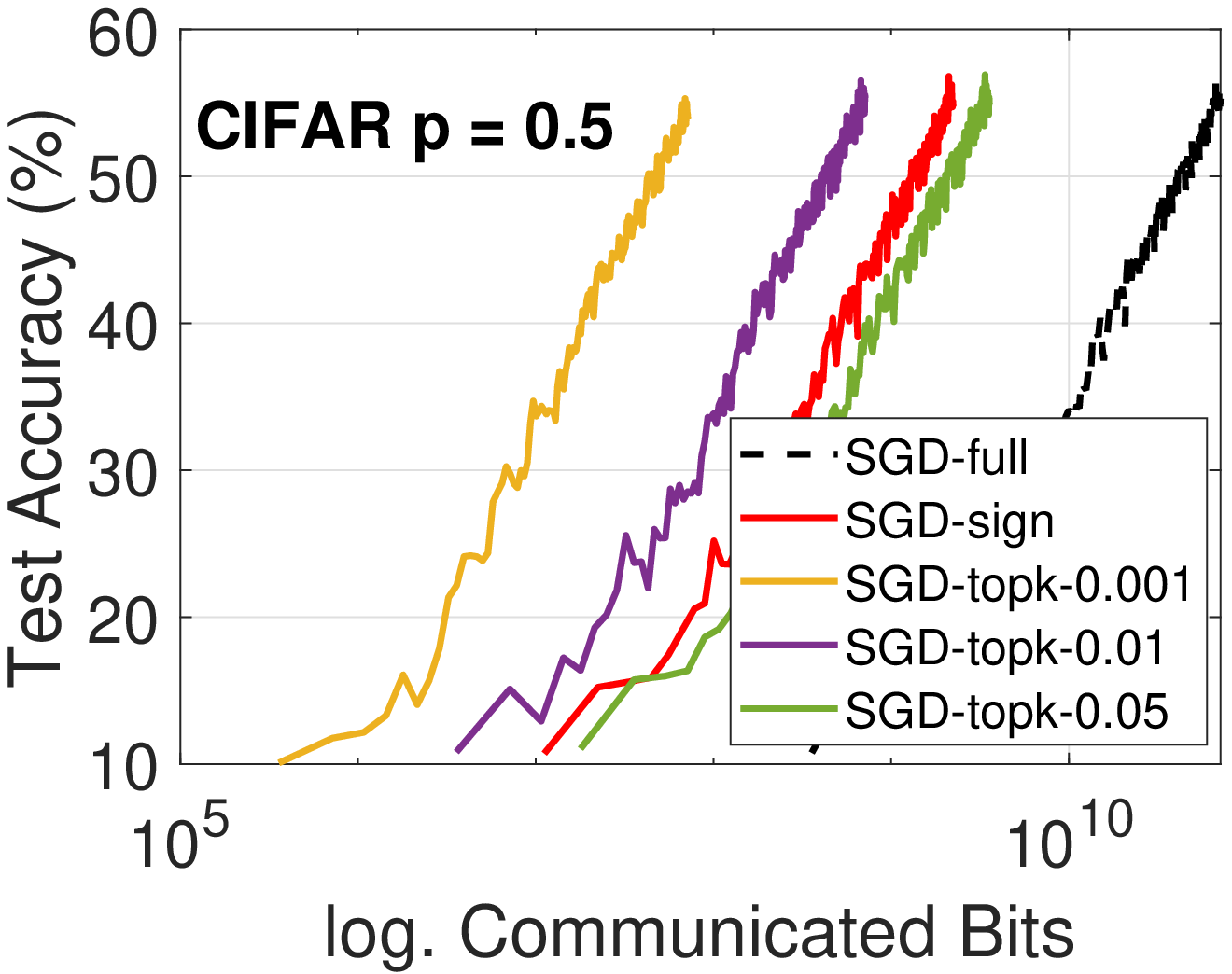}\hspace{-0.1in}
        \includegraphics[width=2.25in]{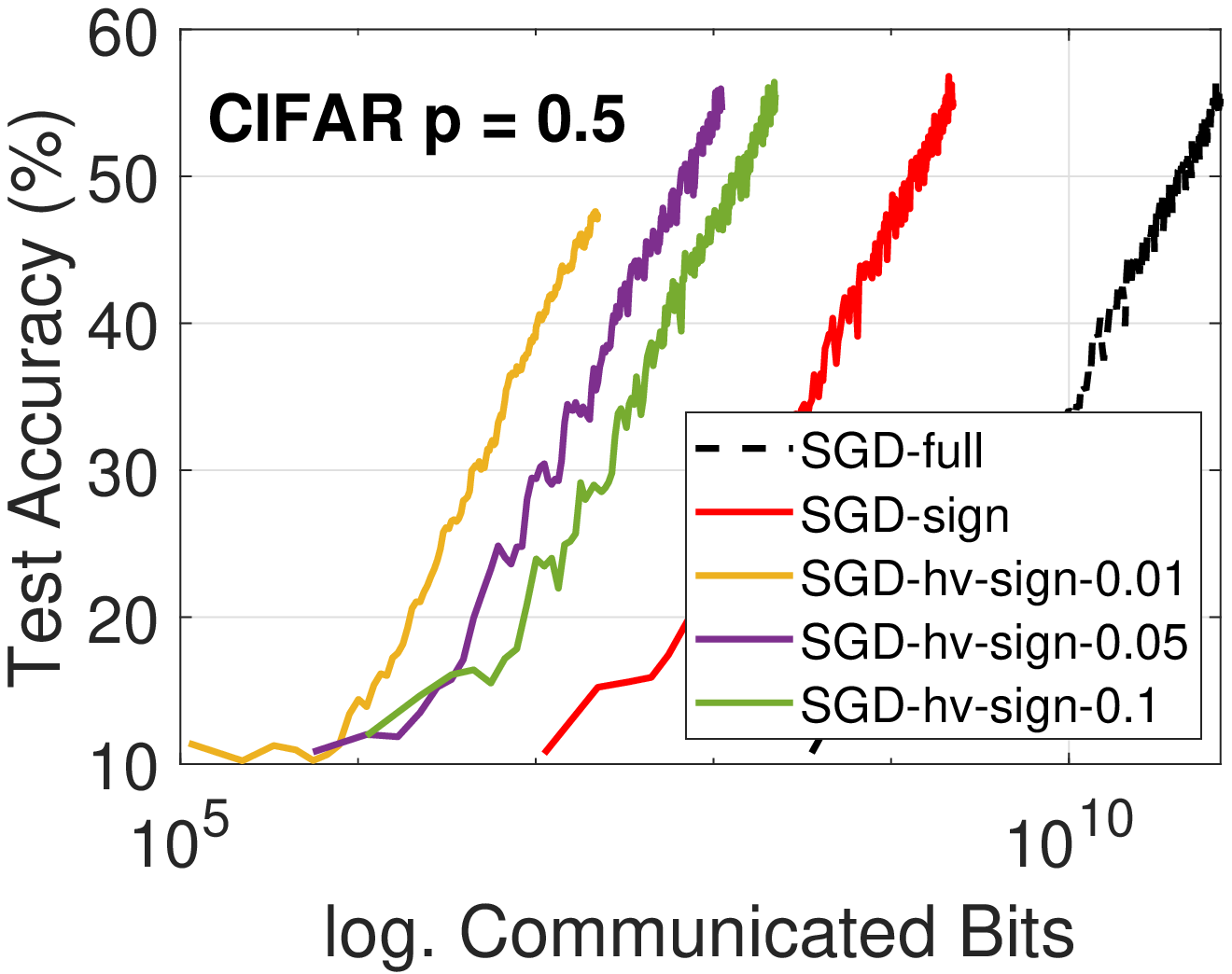}\hspace{-0.1in}
        \includegraphics[width=2.25in]{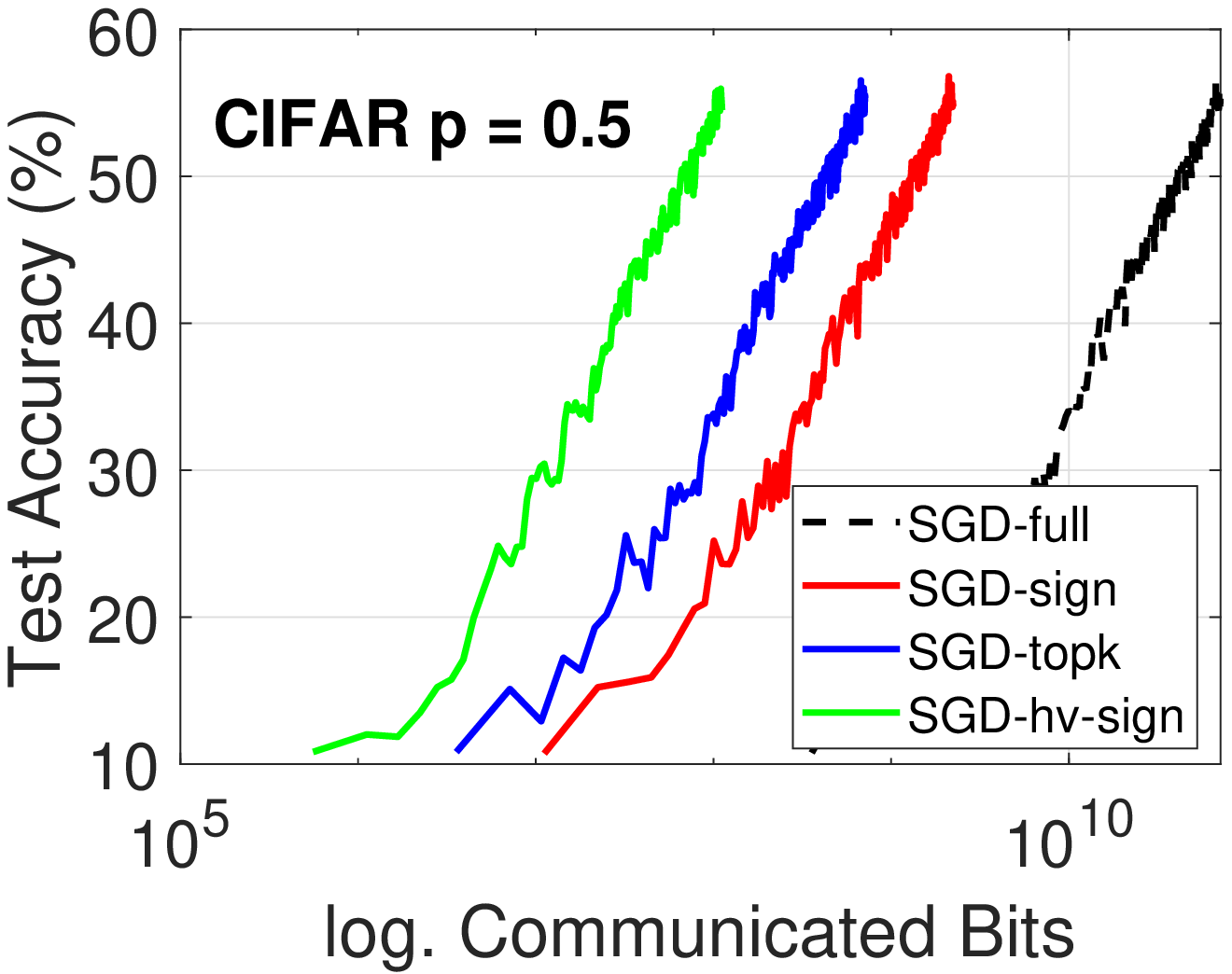}
        }
        \mbox{\hspace{-0.1in}
        \includegraphics[width=2.25in]{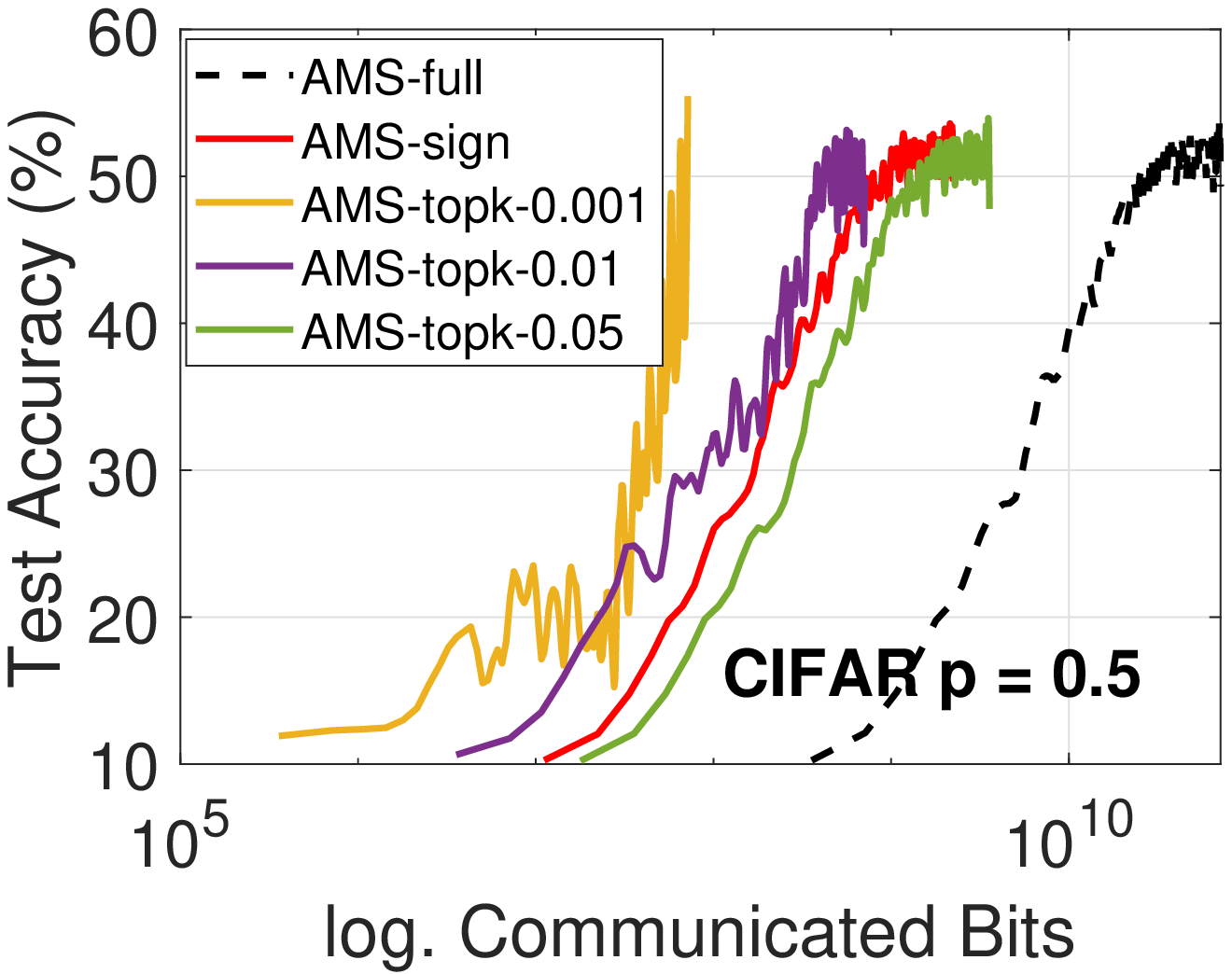}\hspace{-0.1in}
        \includegraphics[width=2.25in]{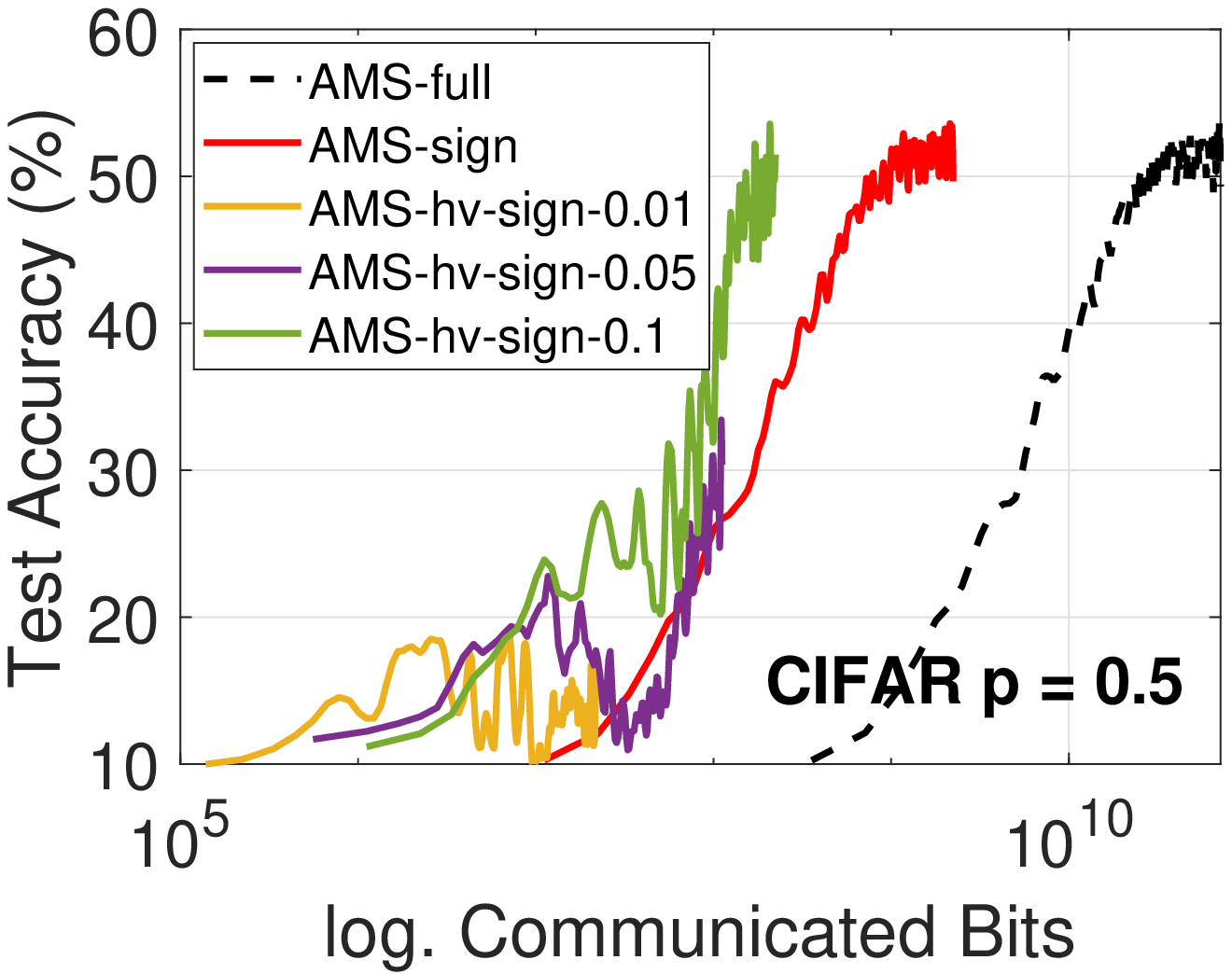}\hspace{-0.1in}
        \includegraphics[width=2.25in]{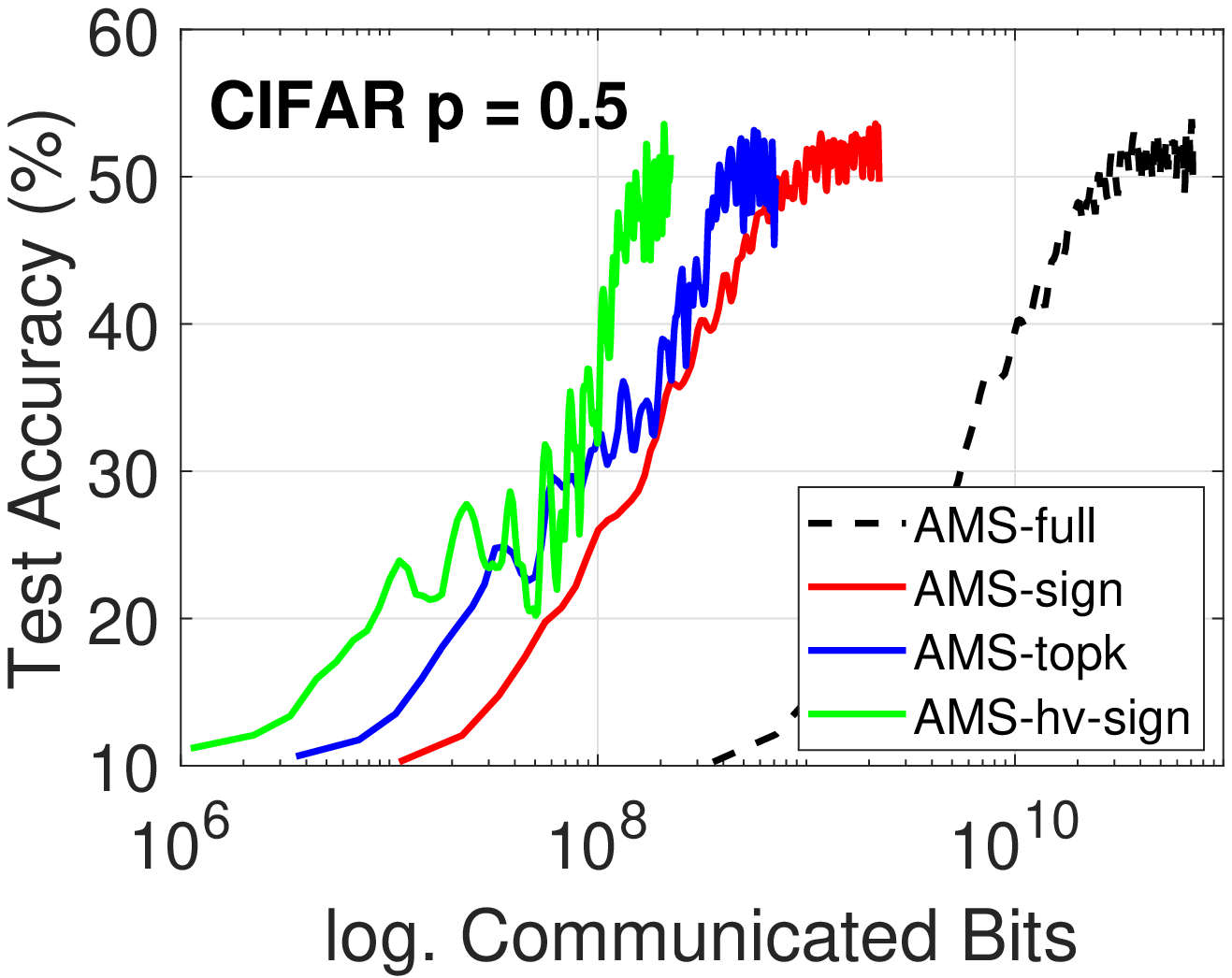}
        }
    \end{center}
    \vspace{-0.1in}
	\caption{CIFAR-10 dataset trained by ResNet-18. Test accuracy of Fed-EF with \textbf{TopK}, \textbf{Sign} and \textbf{heavy-Sign} compressors. Participation rate $p=0.5$, non-iid data. 1st row: Fed-EF-SGD. 2nd row: Fed-EF-AMS. The last column presents the corresponding curves that achieve the full-precision accuracy using lowest communication.}
	\label{fig:CIFAR-acc-compressor-0.5}
\end{figure}

In Figure~\ref{fig:CIFAR-acc-compressor-0.5}, we plot the test accuracy of Fed-EF with different compressors. Again, we see that Fed-EF (both variants) is able to attain the same accuracy level as the corresponding full-precision federated learning algorithms. For Fed-EF-SGD, the compression rate is around 32x for \textbf{Sign}, 100x for \textbf{TopK} and $\sim$300x for \textbf{heavy-Sign}. For Fed-EF-AMS, the compression ratio can also be around hundreds. Note that for Fed-EF-AMS, the training curve of \textbf{TopK-0.001} is not stable. Though it reaches a high accuracy, we still plot \textbf{TopK-0.01} in the third column for comparison.

In Figure~\ref{fig:CIFAR-acc-compressor-0.1} we report the results for partial participation with $p=0.1$. Similarly, for SGD, all three compressors are able to match the full-precision accuracy, with significantly reduced number of communicated bits. For Fed-EF-AMS, similar to the observations on FMNIST, we see that \textbf{TopK} outperforms \textbf{Sign} and \textbf{heavy-Sign}, and achieve the performance of full-precision method with 100x compression ratio. \textbf{Sign} also performs reasonably well.

In conclusion, our results on CIFAR-10 and ResNet again confirm that compared with standard full-precision FL algorithms, the proposed Fed-EF scheme can provide significant communication reduction without empirical performance drop in test accuracy, under both data heterogeneity and partial client participation.

\begin{figure}[t]
    \begin{center}
        \mbox{\hspace{-0.1in}
        \includegraphics[width=2.25in]{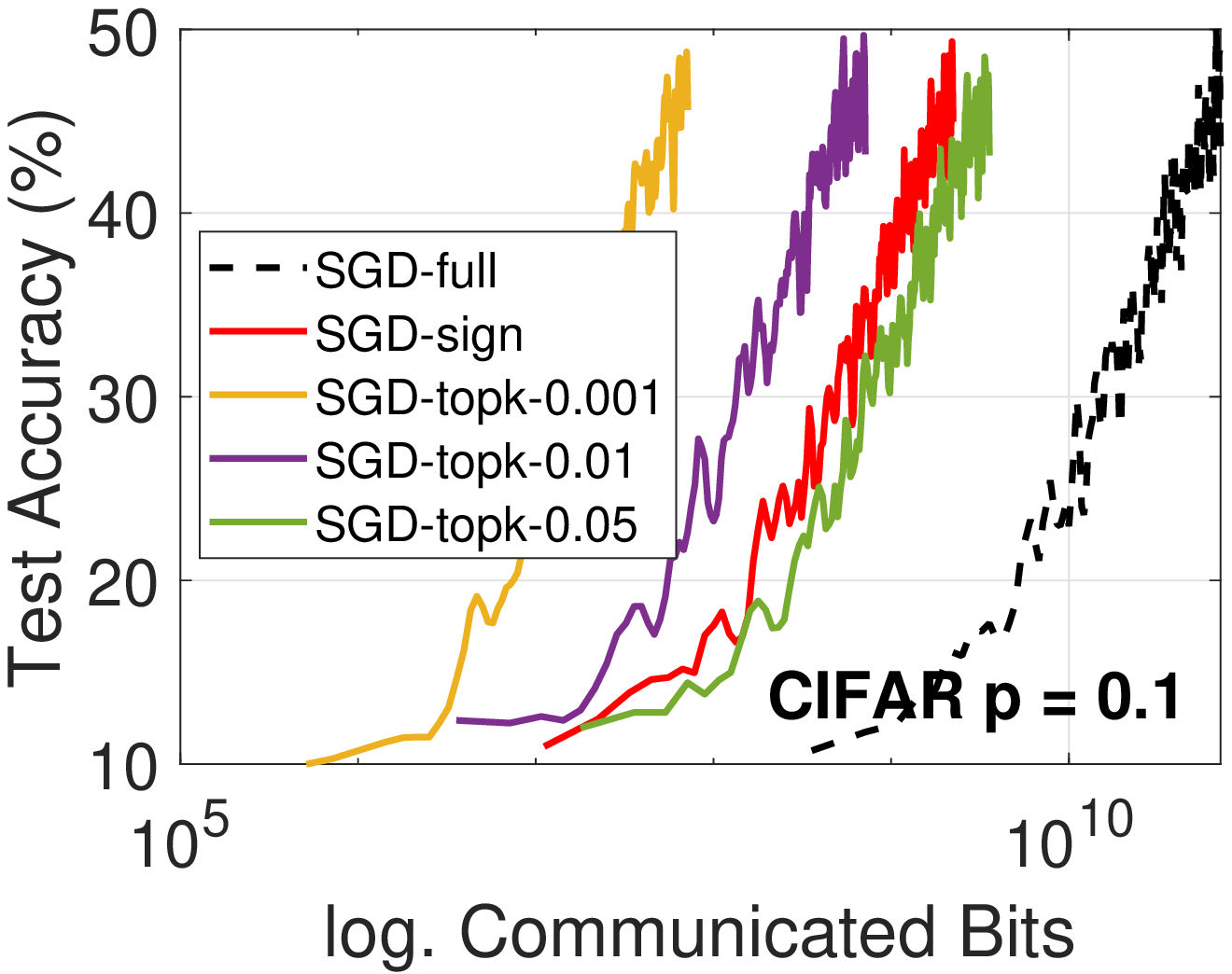}\hspace{-0.1in}
        \includegraphics[width=2.25in]{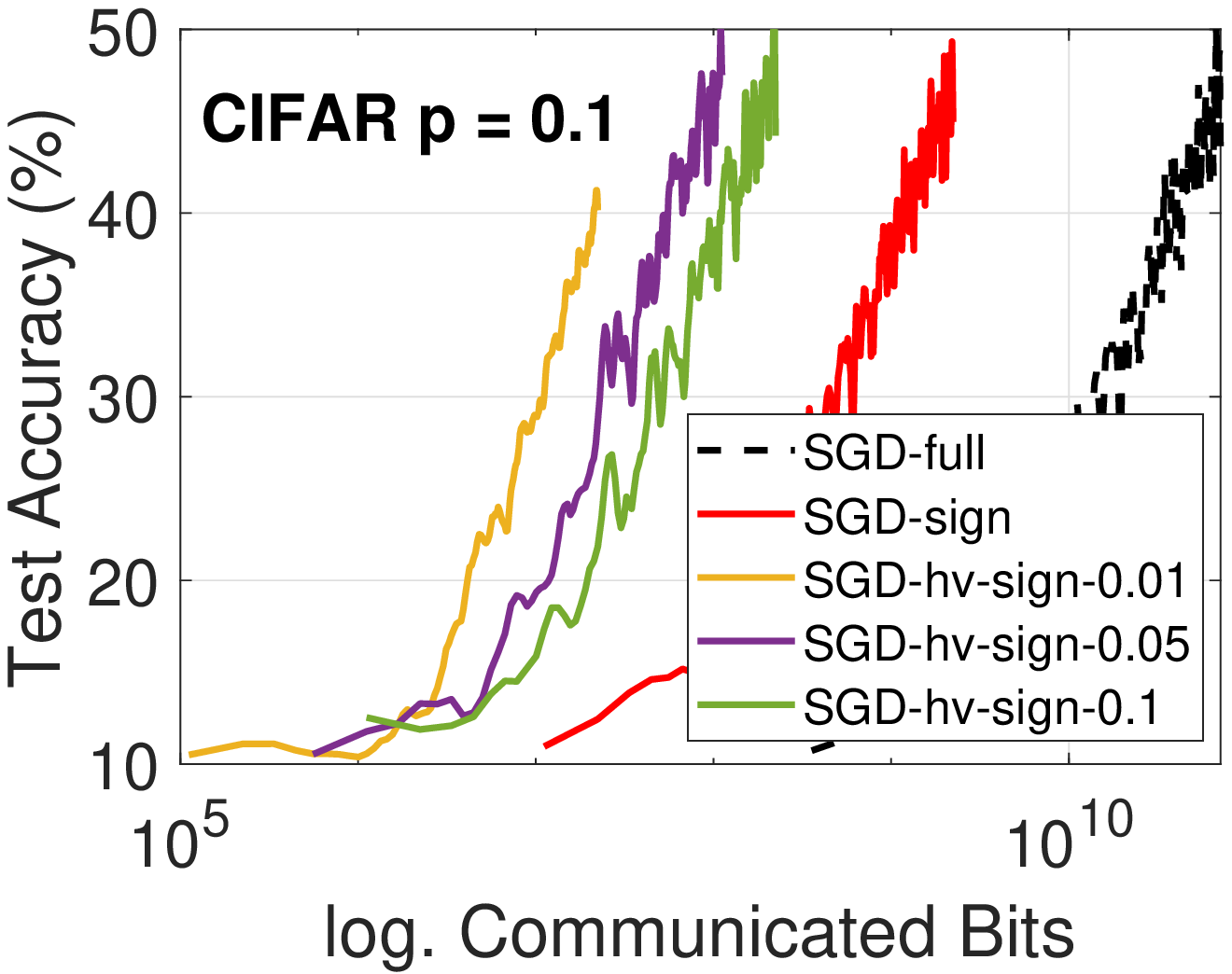}\hspace{-0.1in}
        \includegraphics[width=2.25in]{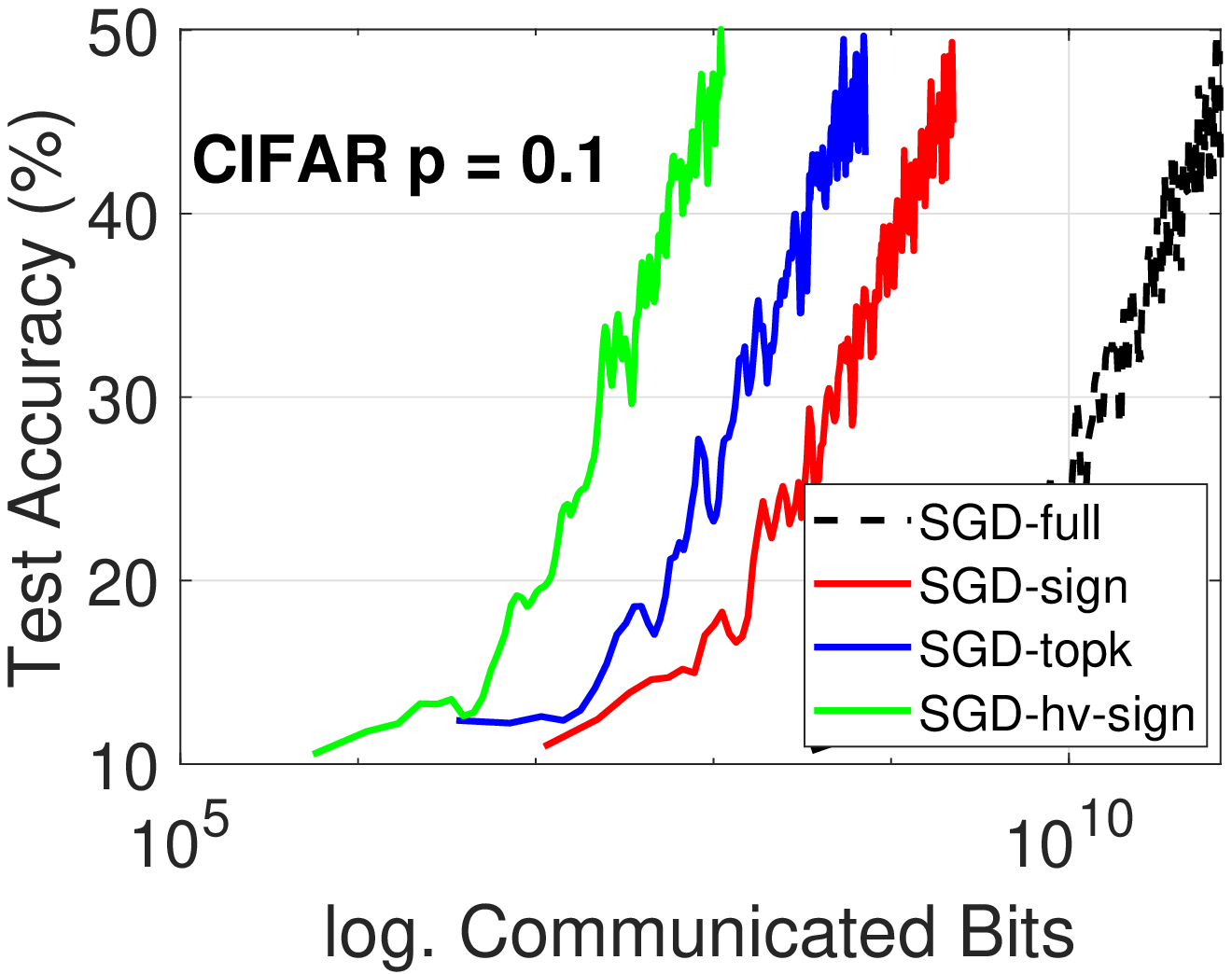}
        }
        \mbox{\hspace{-0.1in}
        \includegraphics[width=2.25in]{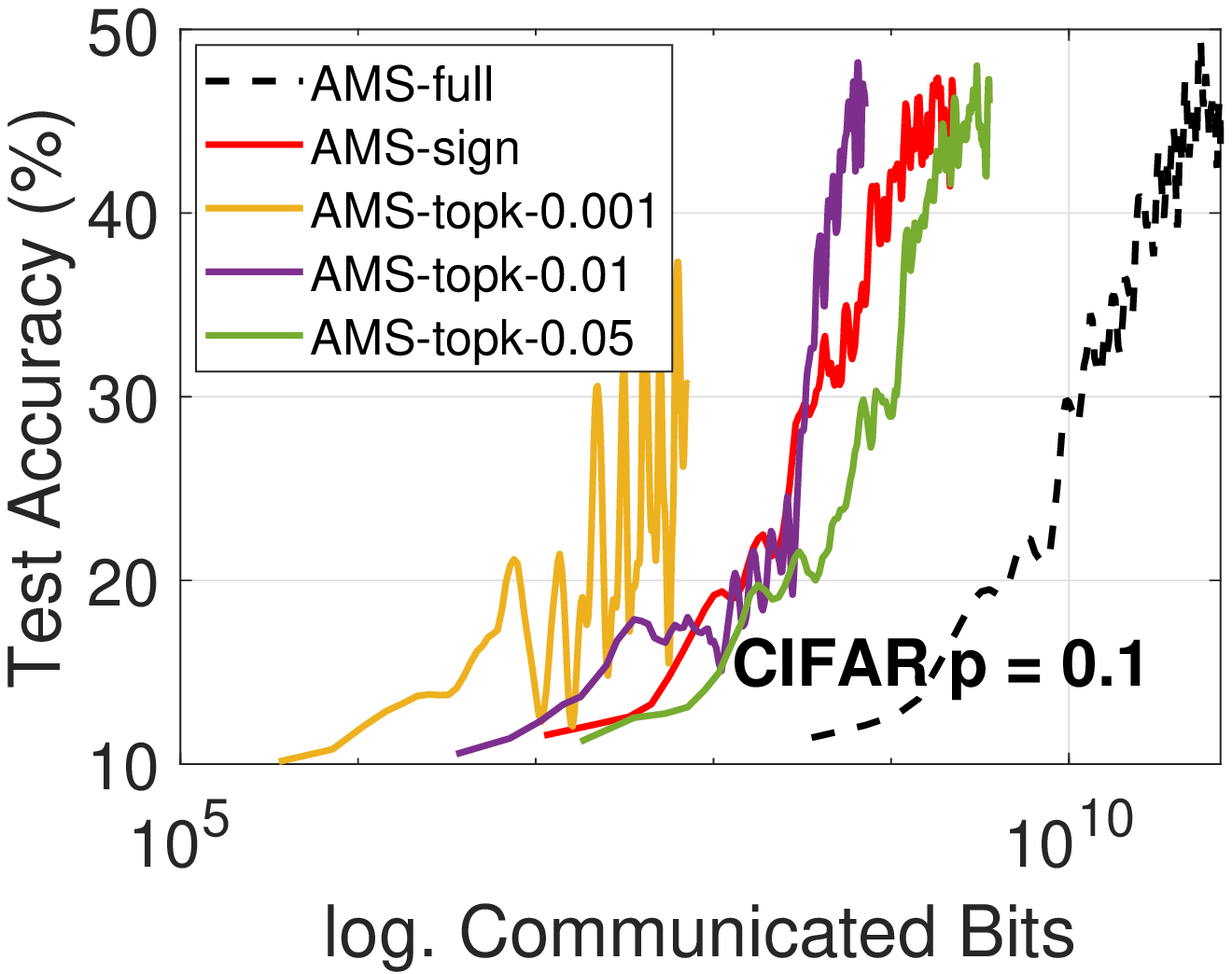}\hspace{-0.1in}
        \includegraphics[width=2.25in]{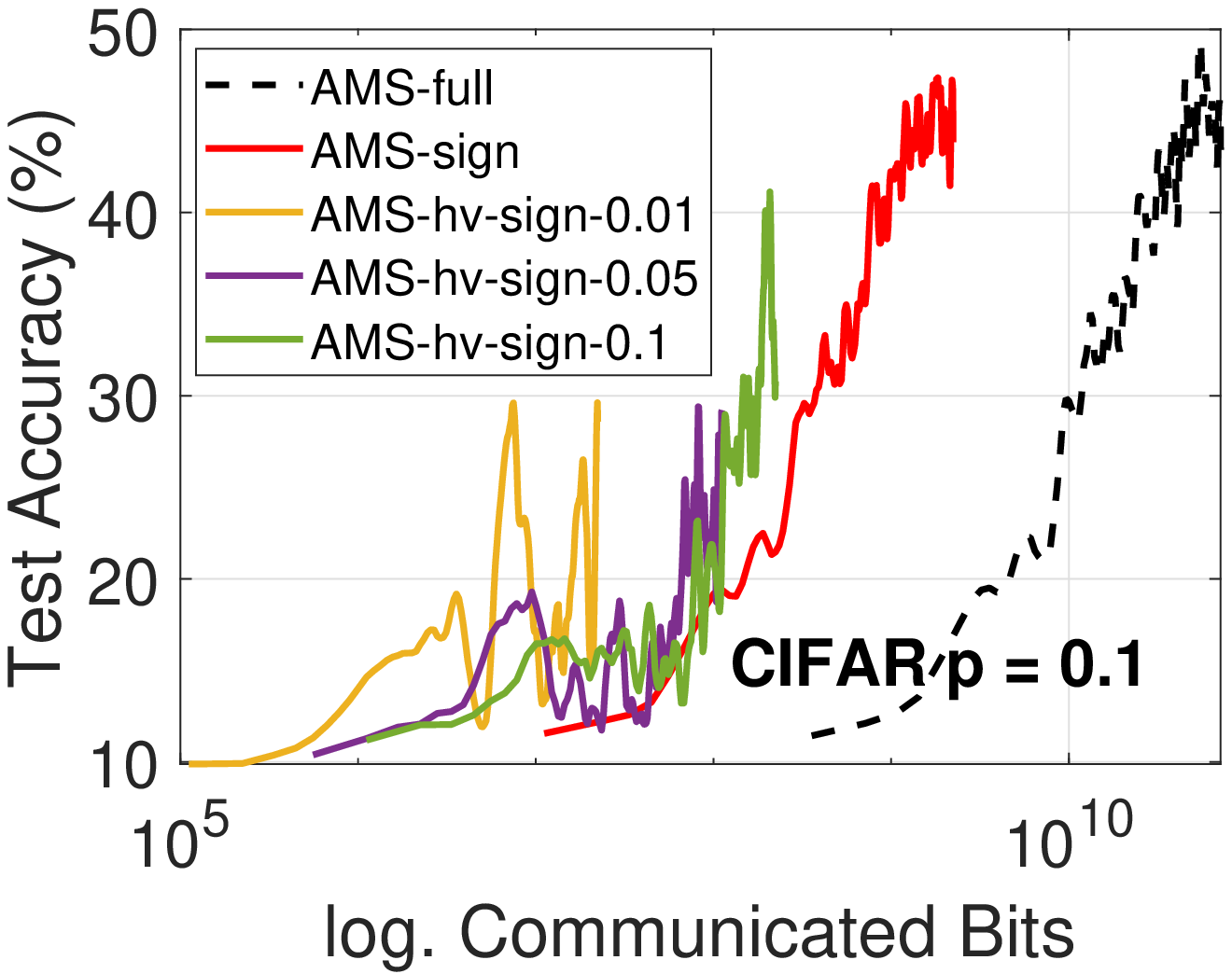}\hspace{-0.1in}
        \includegraphics[width=2.25in]{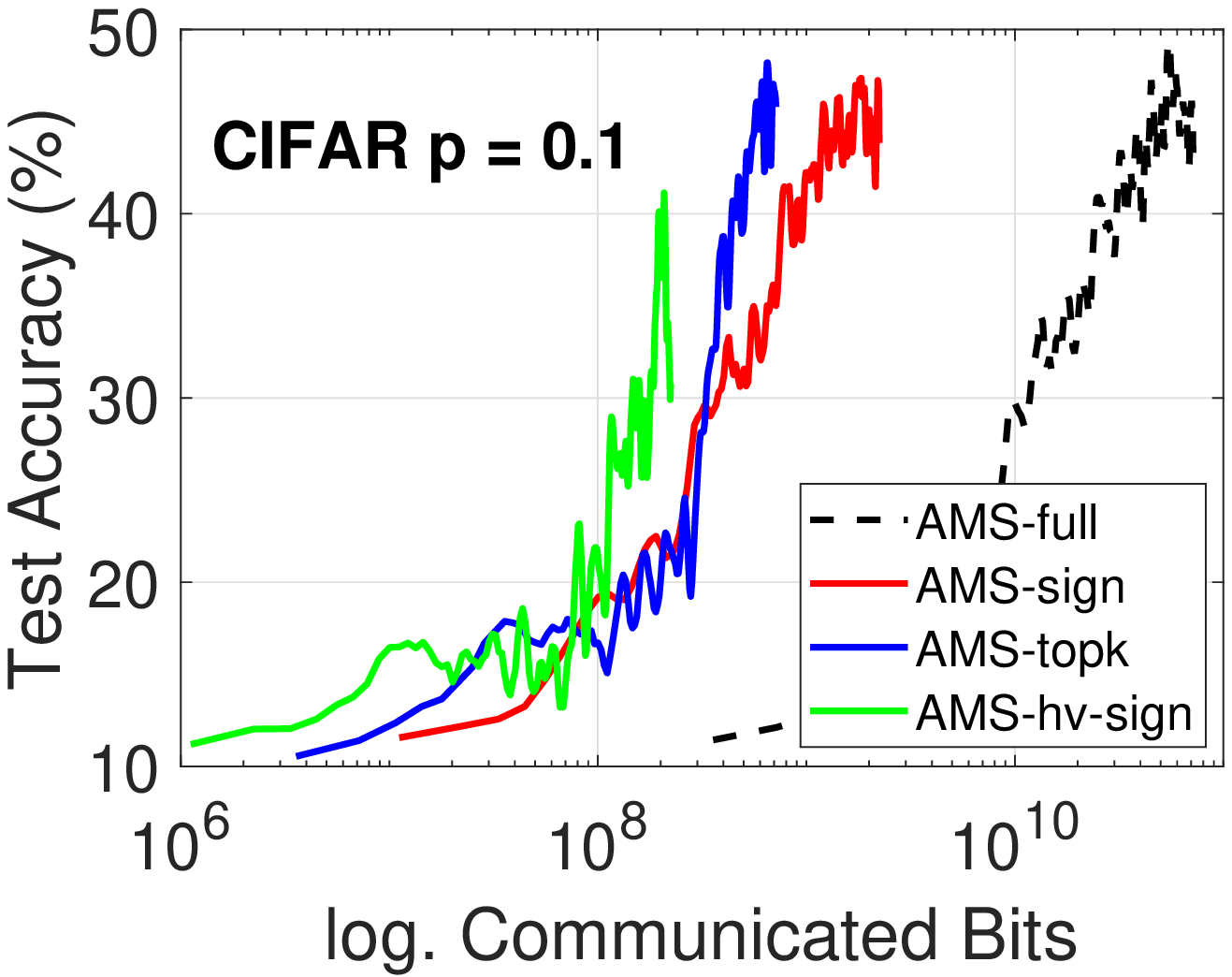}
        }
    \end{center}
    \vspace{-0.1in}
	\caption{CIFAR-10 dataset trained by ResNet-18. Test accuracy of Fed-EF with \textbf{TopK}, \textbf{Sign} and \textbf{heavy-Sign} compressors. Participation rate $p=0.1$, non-iid data. 1st row: Fed-EF-SGD. 2nd row: Fed-EF-AMS. The last column presents the corresponding curves that achieve the full-precision accuracy using lowest communication.}
	\label{fig:CIFAR-acc-compressor-0.1}
\end{figure}

\subsection{Analysis of Norm Convergence and Delayed Error Compensation}  \label{sec:restart}

We empirically evaluate the norm convergence to verify the theoretical speedup properties of Fed-EF and the effect of delayed error compensation in partial participation (PP). Recall that from Theorem~\ref{coro:rates}, in full participation case, reaching a $\delta$-stationary point requires running $\Theta(1/n\delta^2)$ rounds of Fed-EF (i.e., linear speedup). From Theorem~\ref{theo:partial-simple}, when $n$ is fixed, our result implies that the speedup should be super-linear against $m$, the number of participating clients, due to the additional stale error effect. In other words, altering $m$ under PP is expected to have more impact on the convergence than altering $n$ with full participation.

\begin{figure}[h]
    \begin{center}
        \mbox
        {\hspace{-0.15in}
        \includegraphics[width=1.75in]{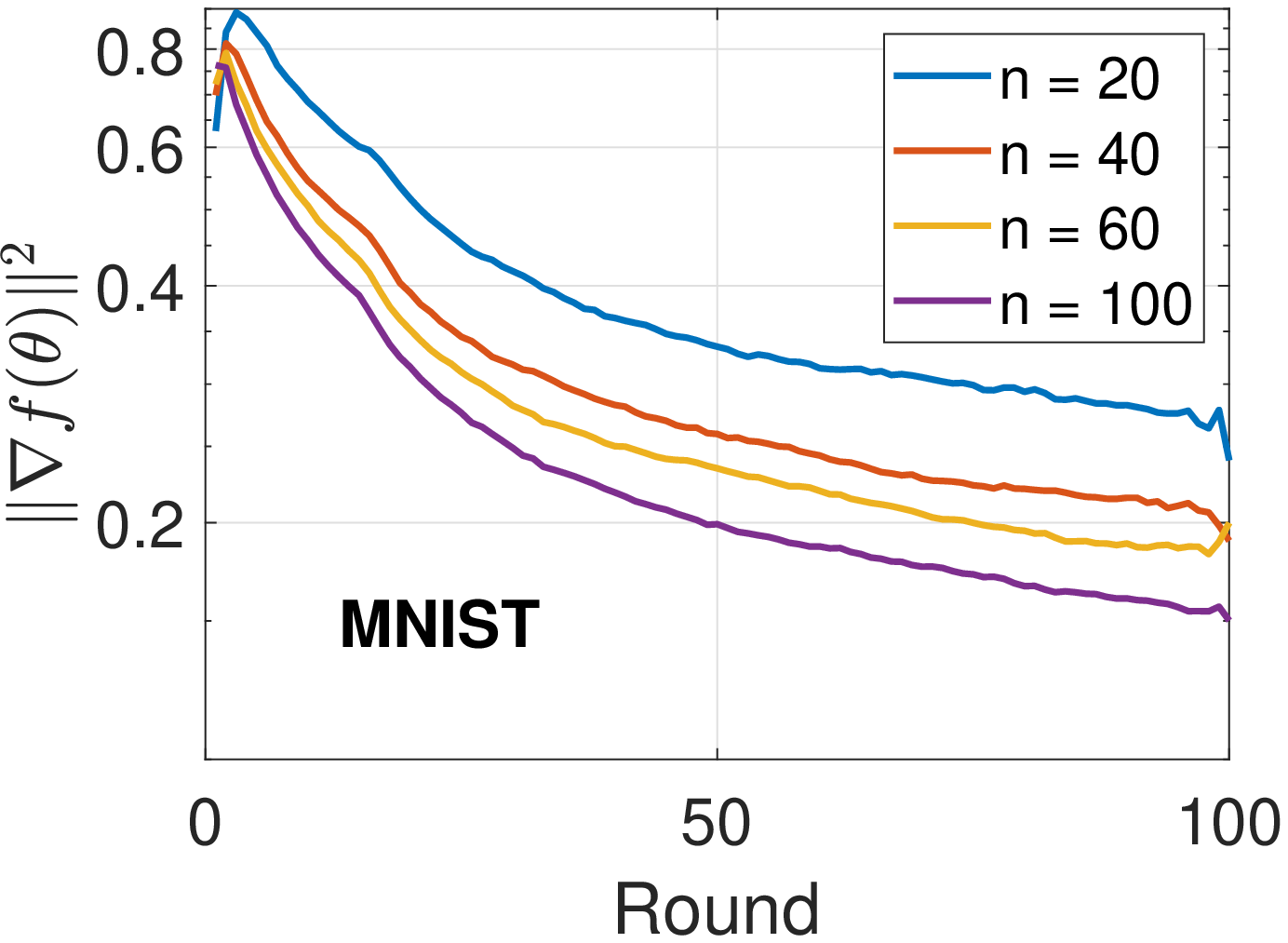}\hspace{-0.12in}
        \includegraphics[width=1.75in]{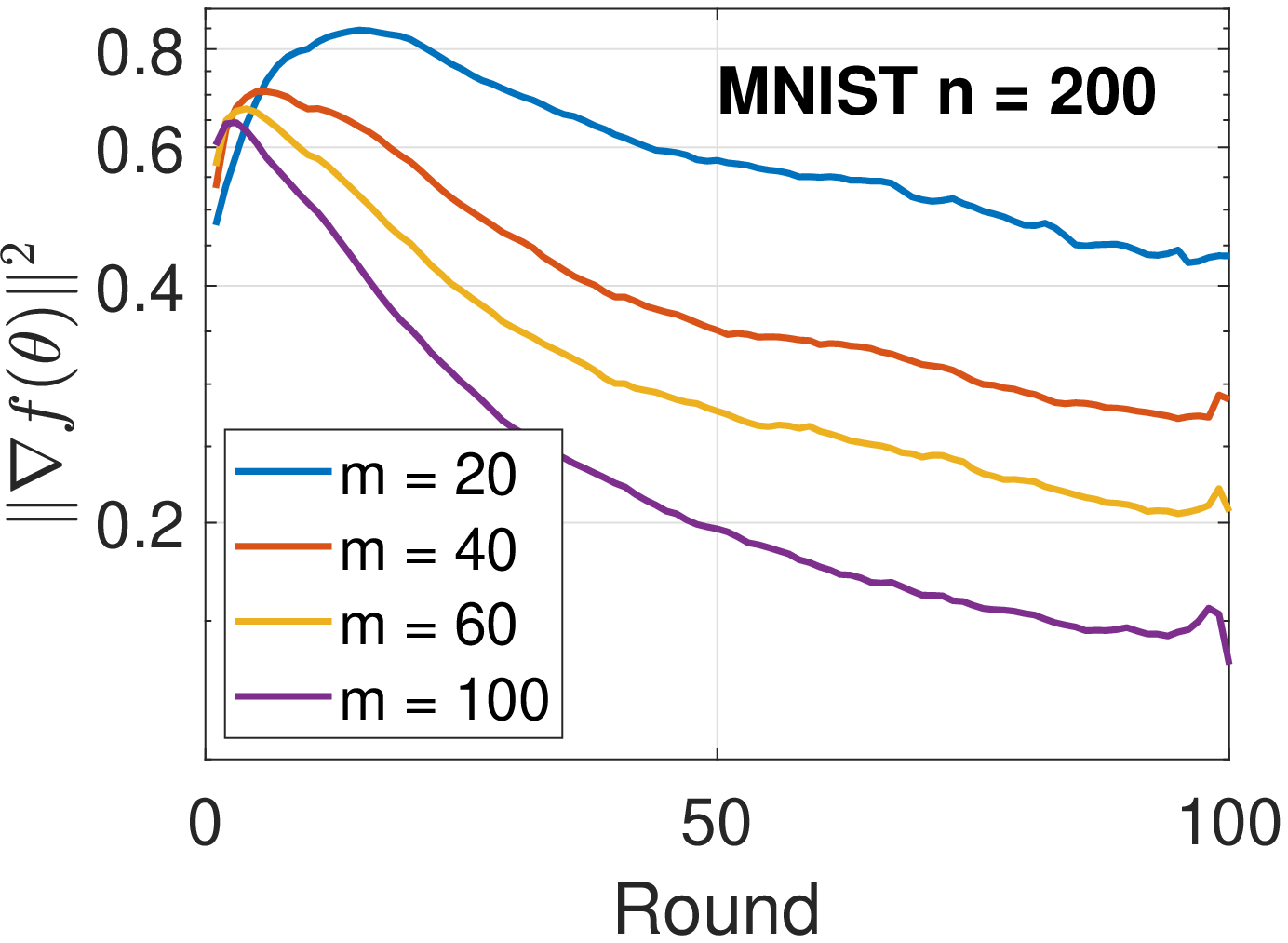}\hspace{-0.12in}
        \includegraphics[width=1.75in]{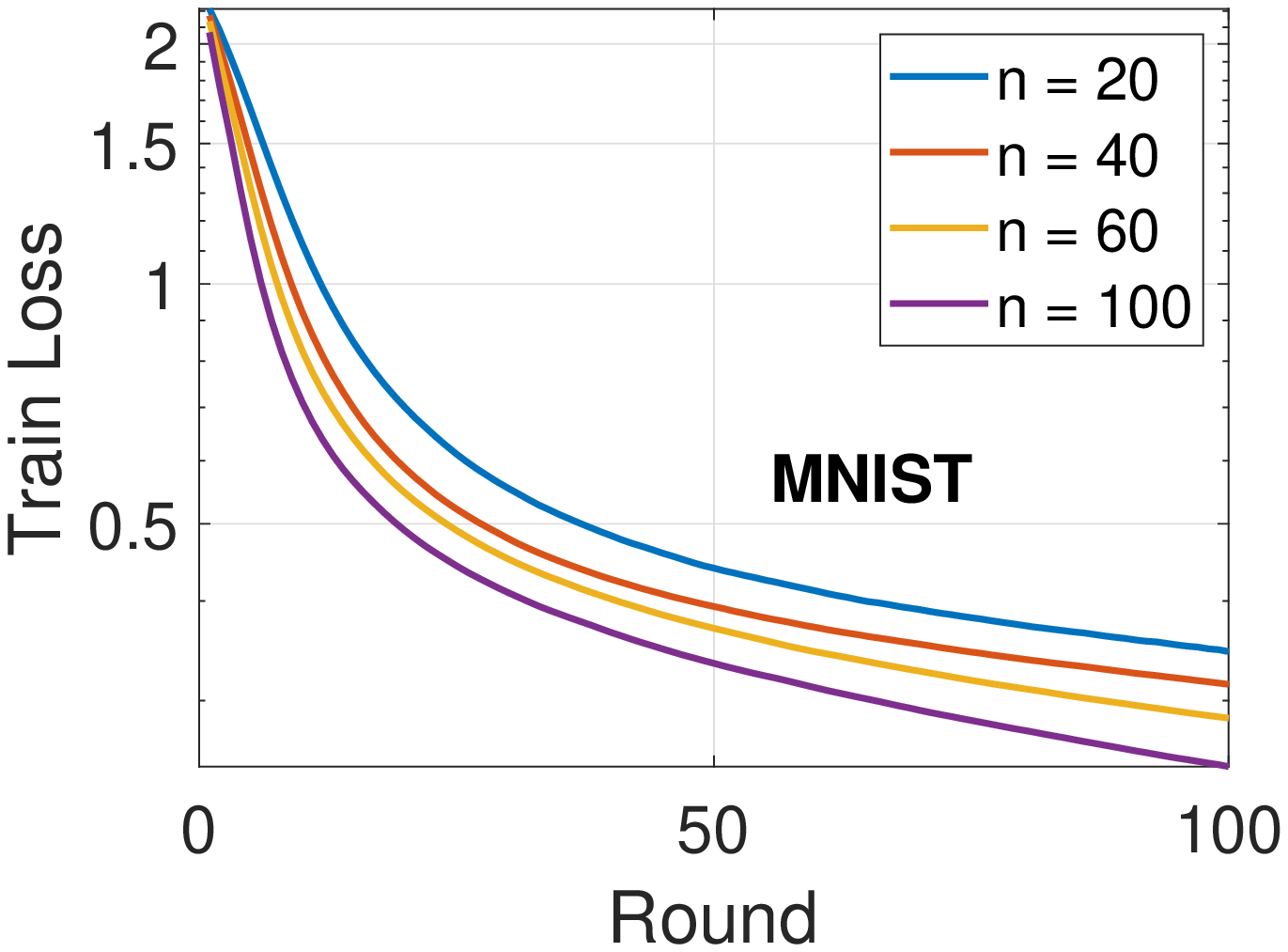}\hspace{-0.12in}
        \includegraphics[width=1.75in]{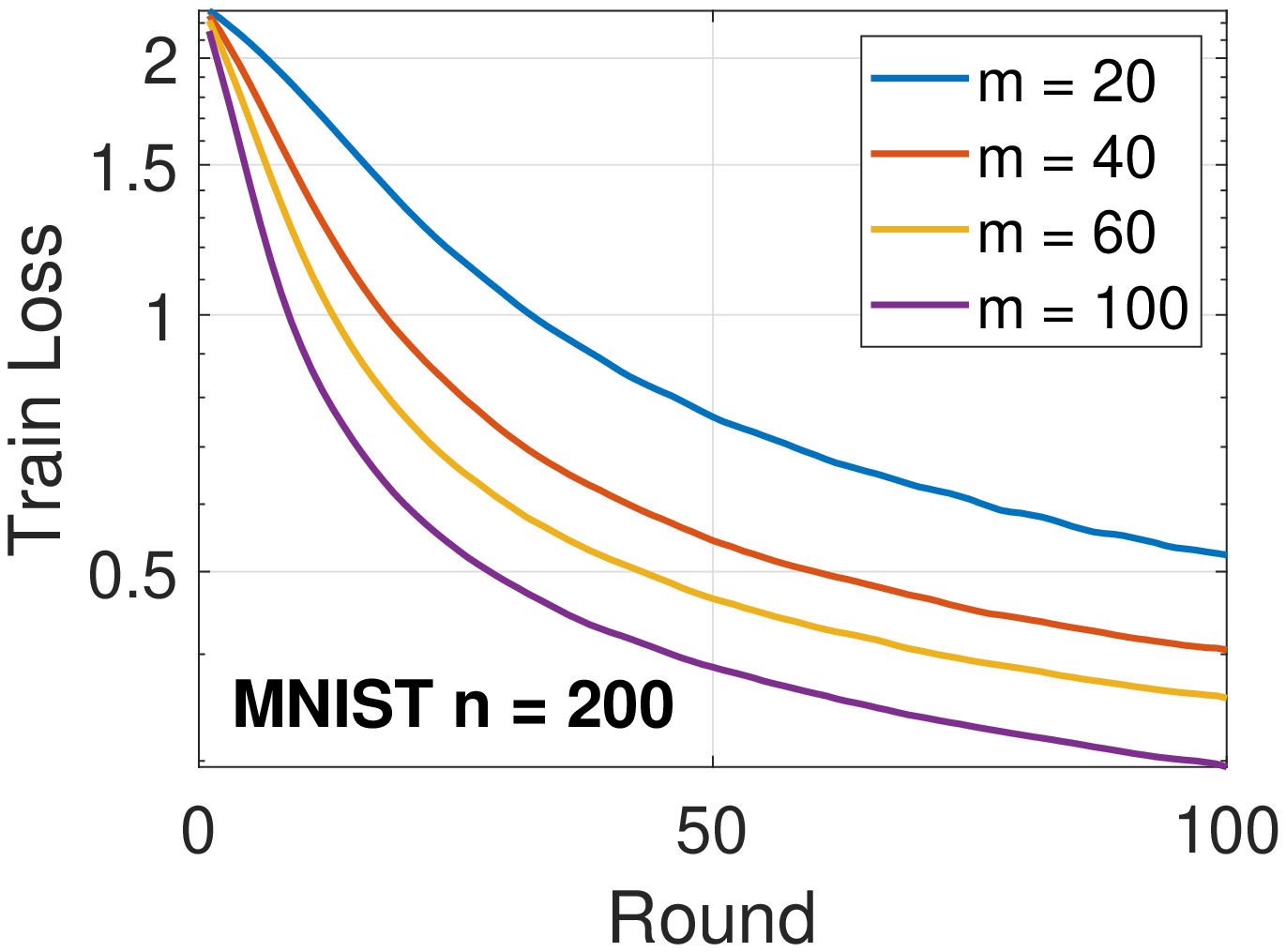}
        }
    \end{center}

    \vspace{-0.2in}

	\caption{MLP on MNIST with \textbf{TopK-0.01} compressed Fed-EF: Squared gradient norm (left two) and train loss (right two) against the number of training rounds, averaged over 20 independent runs.}
	\label{fig:speedup}
\end{figure}

We train an MLP (which is also used for Figure~\ref{fig:no_EF}) with one hidden layer of 200 neurons. In Figure~\ref{fig:speedup}, we report the squared gradient norm and the training loss on MNIST under the same non-iid FL setting as above (the results on FMNIST are similar). In the full participation case, we implement Fed-EF-SGD with $n=20,40,60,100$ clients; for the partial participation case, we fix $n=200$ and alter $m=20,40,60,100$. According to our theory, we set $\eta=0.1\sqrt{n}$ (or $0.1\sqrt{m}$) and  $\eta_l=0.1$. We see that: 1) In general, the convergence of Fed-EF is faster with increasing $n$ or $m$, which confirms the speedup property; 2) The gaps among curves in the PP setting (the 2nd and 4th plot) is larger than those in the full-participation case (the 1st and 3rd plot), which suggests that the acceleration brought by increasing $m$ under PP is more significant than that of increasing $n$ in full-participation by a same proportion, which is consistent with our theoretical implications.

\vspace{0.1in}
\noindent\textbf{Error restarting.} To further embody the intuitive impact of delayed error compensation under PP, we test a simple strategy called \textit{``error restarting''} as follows: for each client $i$ in round $t$, if the error accumulator was last updated more than $S$ (a threshold) rounds ago (i.e., before round $t-S$), we simply restart its error accumulator by setting $e_{t,i}=0$, which effectively eliminates the error information that is ``too old''.

\begin{figure}[h]
    \begin{center}
        \mbox
        {\hspace{-0.15in}
        \includegraphics[width=1.75in]{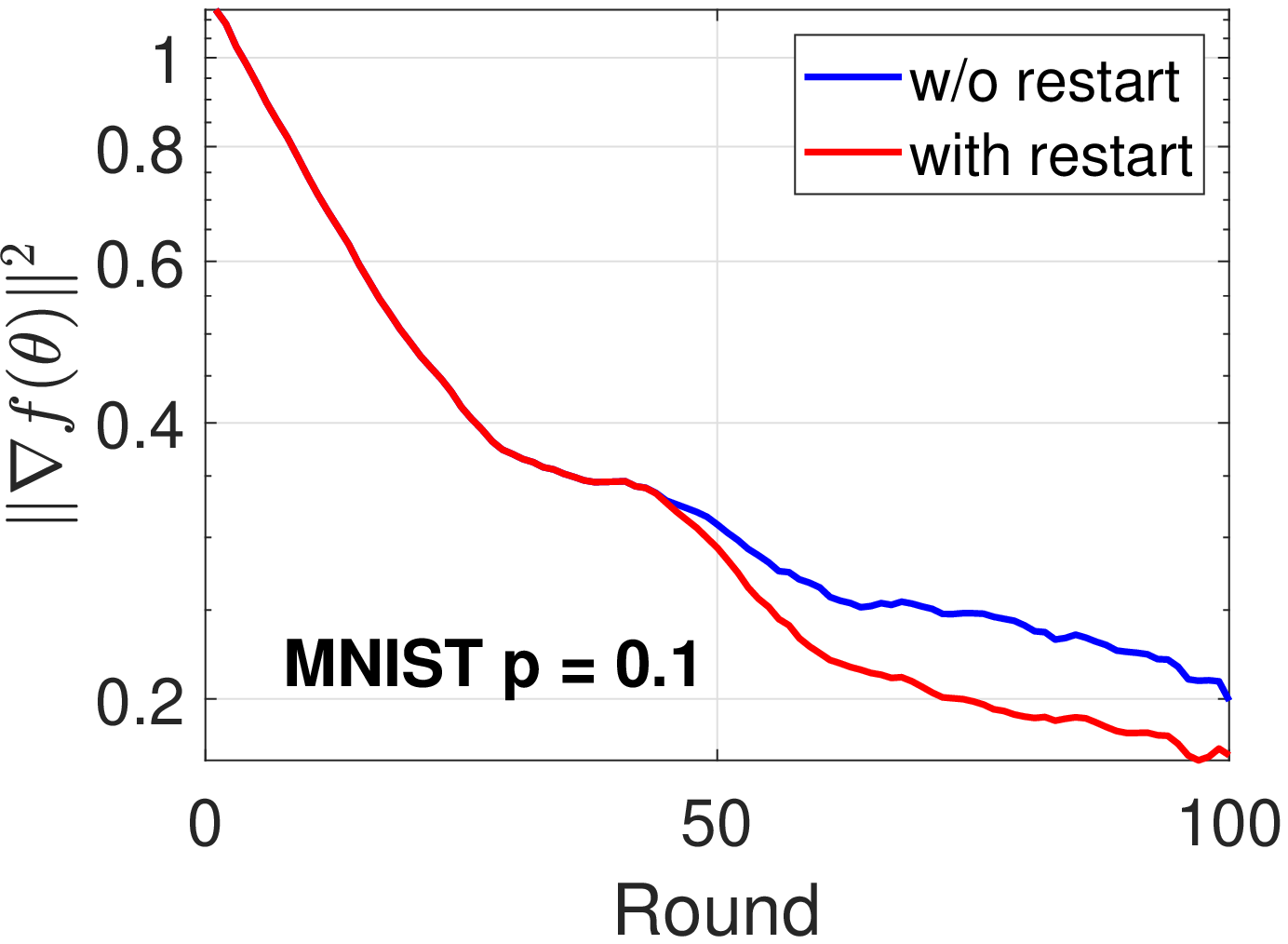}\hspace{-0.12in}
        \includegraphics[width=1.75in]{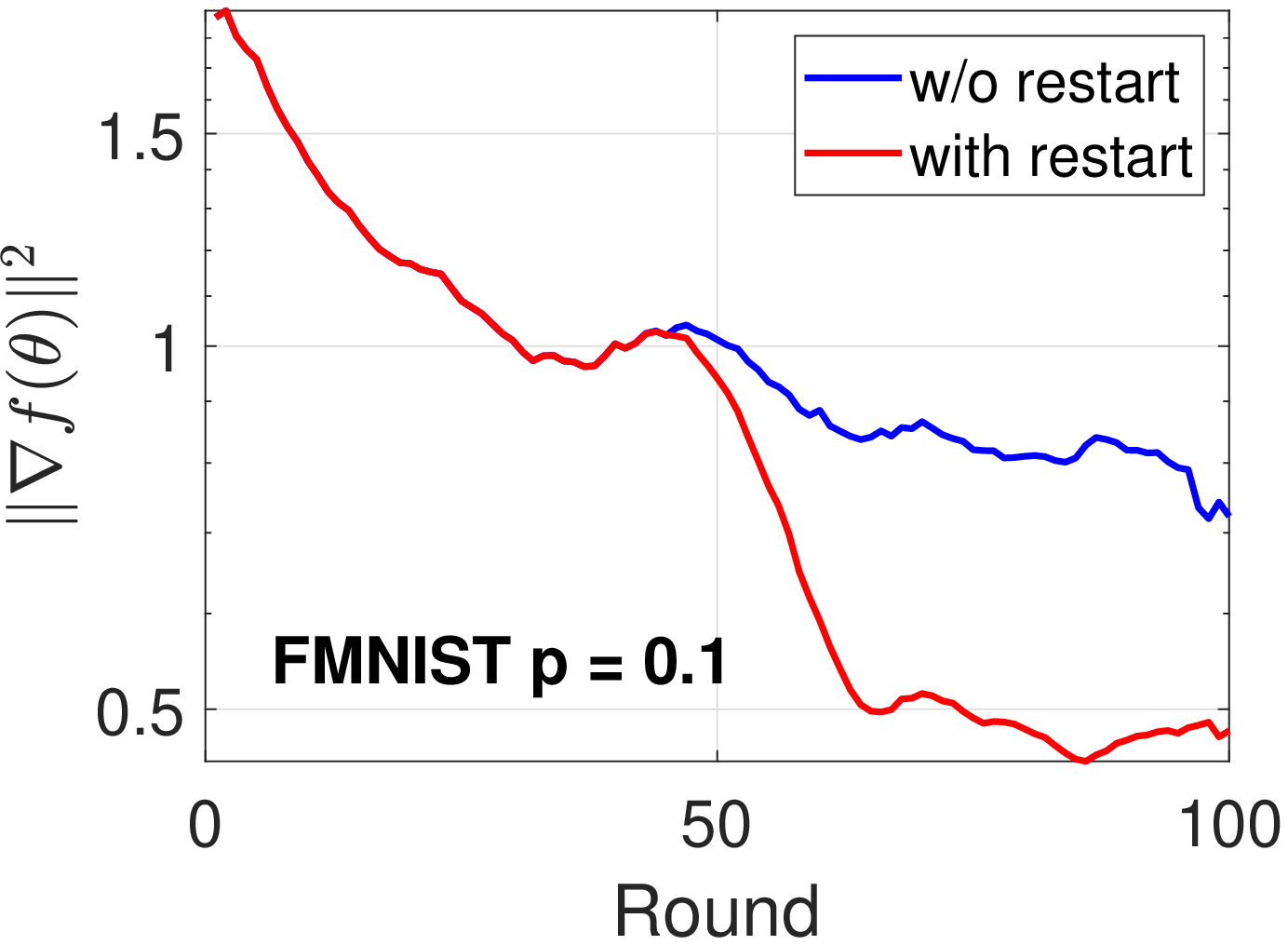}\hspace{-0.12in}
        \includegraphics[width=1.75in]{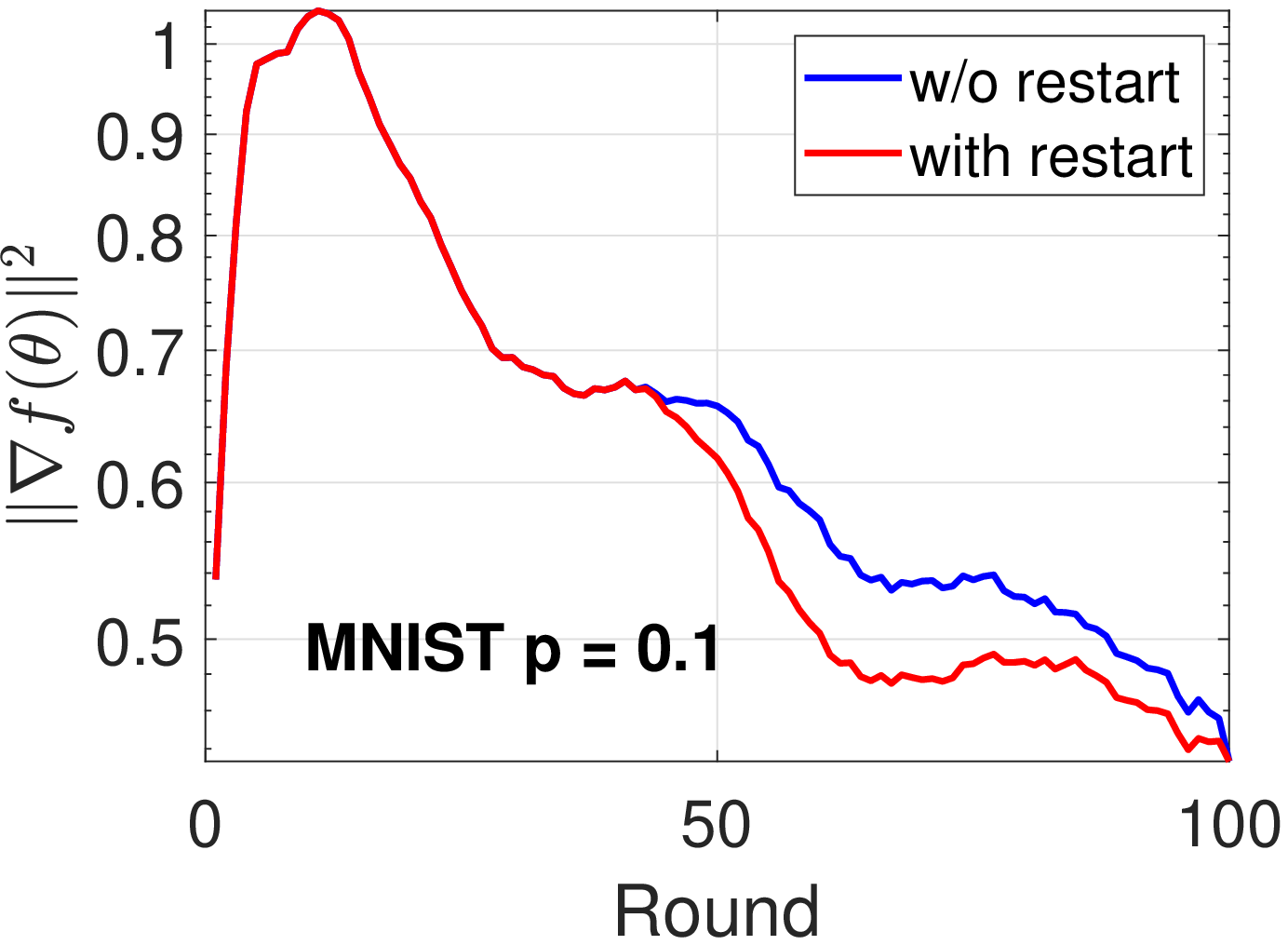}\hspace{-0.12in}
        \includegraphics[width=1.75in]{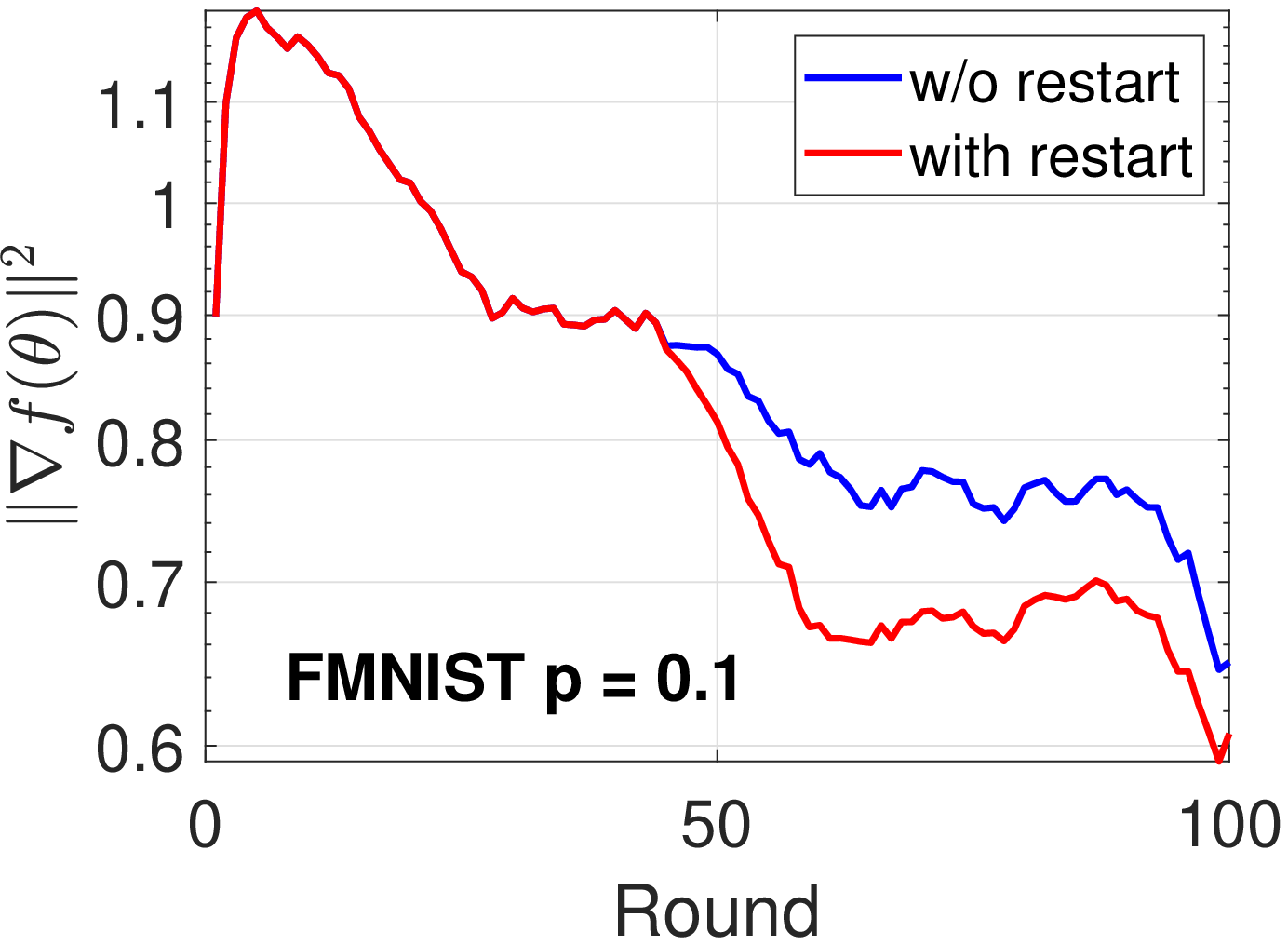}
        }
    \end{center}

    \vspace{-0.2in}

	\caption{Squared gradient norm of Fed-EF (\textbf{TopK-0.01}) under PP with error restarting, $S=10$. Left two panels: logistic regression. Right two panels: MLP. $n=200$, $\eta=1$, $\eta_l=0.1$.}
	\label{fig:stale-error}
\end{figure}

In Figure~\ref{fig:stale-error}, we first run Fed-EF for 50 rounds, and then apply error restarting with threshold $S=10$. As we see, after prohibiting heavily delayed error information, the gradient norm is smaller than that of continuing running standard Fed-EF, i.e., the model finds a stationary point faster. These results illustrate the influence of stale error compensation in Fed-EF, and that properly handling this staleness might be a promising direction for theoretical improvement in the future.

\vspace{0.1in}
\begin{remark}
Note that the error restarting trick is a simple heuristic to illustrate the impact of delayed error compensation in Fed-EF. If we consider theoretical convergence rates, this strategy may not bring improvement since zeroing out the error accumulators, although reducing the stale errors, also leads to biased compression when it is triggered. Therefore, as shown by Theorem~\ref{theo:no-EF rate}, it would result in an additional non-vanishing bias term in the convergence rate, worse than Theorem~\ref{theo:partial-simple} asymptotically. In fact, this strategy plays a trade-off between the staleness in error feedback and biased compression, which could be beneficial empirically if properly balanced.
\end{remark}

\section{Two-Way Compression in Fed-EF}  \label{app sec:two-way}

\begin{algorithm}[t!]
\caption{Fed-EF Scheme with Two-Way Compression} \label{alg:Fed-EF-two-way}
\begin{algorithmic}[1]
\State{\textbf{Input}: learning rates $\eta$, $\eta_l$, \colorbox{blue!20!white}{hyper-parameters $\beta_1$, $\beta_2$, $\epsilon$} }
\State{\textbf{Initialize}: central server parameter $\theta_{1} \in \mathbb R^d \subseteq \mathbb R^d$; $e_{1,i}=\bm{0}$ the accumulator for each worker; \colorbox{blue!20!white}{$m_0=\bm{0}$, $v_0=\bm{0}$, $\hat v_0=\bm{0}$}; \colorbox{yellow!20!white}{$\phi_1=\tilde H_t=\bm 0$; $\theta_{0,i}^{(1)}=\theta_1$ for all $i\in [n]$} }
\vspace{0.03in}
\State{\textbf{for $t=1, \ldots, T$ do}}
\State{\quad\textbf{parallel for worker $i \in [n]$ do}:}
\State{\quad\quad \colorbox{yellow!20!white}{Receive $\tilde H_t$ from the server and set $\theta_{t,i}^{(1)}=\theta_{t-1,i}^{(1)}+\tilde H_t$} \hfill\Comment{\colorbox{yellow!20!white}{Download compression}}  }
\State{\quad\quad  \textbf{for $k=1, \ldots, K$ do}}
\State{\quad\quad\quad  Compute stochastic gradient $g_{t,i}^{(k)}$ at $\theta_{t,i}^{(k)}$}
\State{\quad\quad\quad  Local update $\theta_{t,i}^{(k+1)}=\theta_{t,i}^{(k)}-\eta_l g_{t,i}^{(k)}$}
\State{\quad\quad  \textbf{end for}}

\State{\quad\quad Compute the local model update $\del_{t,i}=\theta_{t,i}^{(K+1)}-\theta_{t}$ }
\State{\quad\quad  Send compressed adjusted local update $\widetilde\del_{t,i}=\mathcal C(\del_{t,i}+e_{t,i})$ to central server}
\State{\quad\quad  Update the error $e_{t+1,i}=e_{t,i}+\del_{t,i}-\widetilde\del_{t,i}$}

\State{\quad \textbf{end parallel}}

\State{\quad\textbf{Central server do:}}
\State{\quad Global aggregation $\overline{\widetilde\del}_{t}=\frac{1}{n}\sum_{i=1}^n \widetilde\del_{t,i}$  \label{line:g}}
\State{\quad \colorbox{green!20!white}{ Compress $\tilde H_t=\mathcal C(\overline{\widetilde\del}_{t}+\phi_t)$} \hfill\Comment{\colorbox{green!20!white}{Fed-EF-SGD}} }

\State{\quad \colorbox{blue!20!white}{$m_t=\beta_1 m_{t-1}+(1-\beta_1)\overline{\widetilde\del}_{t}$} \hfill\Comment{\colorbox{blue!20!white}{Fed-EF-AMS}}  }
\State{\quad \colorbox{blue!20!white}{$v_t=\beta_2 v_{t-1}+(1-\beta_2)\overline{\widetilde\del}_{t}^2$,\quad $\hat v_t=\max(v_t,\hat v_{t-1})$} \label{line:v}}
\State{\quad \colorbox{blue!20!white}{ Compress $\tilde H_t=\mathcal C(\frac{m_t}{\sqrt{\hat v_t+\epsilon}}+\phi_t)$ } }

\State{\quad \colorbox{yellow!20!white}{Update server error accumulator $\phi_{t+1}=\phi_t+(\theta_{t+1}-\theta_t)-\tilde H_t$ } }
\State{\quad Update global model $\theta_{t+1}=\theta_t-\eta\tilde H_t$ and broadcast $\tilde H_t$ to clients}
\State{\textbf{end for}}
\end{algorithmic}
\end{algorithm}

As discussed in Section~\ref{sec:main}, our Fed-EF scheme can also extend to two-way compression, for both uploading (clients-to-server) and downloading (server-to-clients) channels. This can lead to even more communication reduction in practice. The steps can be found in Algorithm~\ref{alg:Fed-EF-two-way}. The general approach is: 1) the clients transmit $\tilde\del_{t,i}$ to the server which are compressed; 2) the server again compresses the aggregated update $\overline{\widetilde\del}_t$ and broadcast the compressed $\tilde H_t$ to the clients, also using an error feedback at the central node. Note that this approach requires the clients to additionally store the model at the beginning of each round.

Next, we briefly demonstrate that Algorithm~\ref{alg:Fed-EF-two-way} has the same convergence rate as Algorithm~\ref{alg:Fed-EF}. For simplicity, we will focus on two-way compressed Fed-EF-SGD here, while same arguments hold for Fed-EF-AMS. Assume the same conditions as in Theorem~\ref{theo:rate SGD}. To study the convergence of Algorithm~\ref{alg:Fed-EF-two-way}, we consider a series of virtual iterates as
\begin{align*}
    \tilde\theta_{t+1}=\theta_{t+1}-\eta\phi_{t+1}&=\theta_t - \eta(\tilde H_t+\phi_{t+1})\\
    &=\theta_t-\frac{1}{n}\sum_{i=1}^n \widetilde\del_{t,i}-\phi_t\\
    &=\tilde\theta_t-\overline{\widetilde\del}_t,
\end{align*}
where we use the fact of EF that $\phi_{t+1}+\tilde H_t=\phi_t+\overline{\widetilde\del}_t$. Then we can construct a similar sequence $x_t$ as in (\ref{eq:virtual-SGD}) associated with $\tilde\theta_t$ by
\begin{align*}
    x_{t+1}&=\tilde\theta_{t+1}-\eta \bar e_{t+1}=x_t-\eta \bar\del_t.
\end{align*}

\newpage

We can then apply same analysis to derive the convergence bound as in Section~\ref{app sec:SGD}. The only difference is in (\ref{eq:sgd1}), where the second term becomes
\begin{align}
    \frac{\eta^2 L^2}{2}\mathbb E\big[\|\bar e_{t}+\phi_t\|^2\big]&\leq \eta^2 L^2\mathbb E\big[\|\bar e_t\|^2\big]+\eta^2 L^2\mathbb E\big[\|\phi_t\|^2\big].  \label{eq:two-way}
\end{align}
The first term can be bounded in the same way as in (\ref{eq:sgd1}). Regarding the second term, we can use a similar trick as Lemma~\ref{lemma:bound e_t} that under Assumption~\ref{ass:compress_diff},
\begin{align*}
    \|\phi_{t+1}\|^2&=\|\phi_t+\overline{\widetilde\del}_t-\mathcal C(\phi_t+\overline{\widetilde\del}_t)\|^2 \\
    &\leq q_{\mathcal C}^2 \|\phi_t+\overline{\widetilde\del}_t\|^2 \\
    &\leq \frac{1+q_{\mathcal C}^2}{2}\|\phi_t\|^2+\frac{2q_{\mathcal C}^2}{1-q_{\mathcal C}^2}\|\overline{\widetilde\del}_t\|^2.
\end{align*}
Then, by recursion and the geometric sum, $\|\phi_{t+1}\|^2$ can be bounded by the second term in above up to a constant. We can write
\begin{align*}
    \mathbb E\big[\|\overline{\widetilde\del}_t\|^2\big]&\leq \mathbb E\big[\|\bar\del_t+\bar e_t-\bar e_{t+1}\|^2\big]\\
    &\leq 3 (\mathbb E\big[\|\bar\del_t \|^2\big] + \mathbb E\big[\|\bar e_t \|^2\big] + \mathbb E\big[\|\bar e_{t+1} \|^2\big]).
\end{align*}
As a result, it holds that $\mathbb E\big[\|\phi_t\|^2\big]= \Theta(\mathbb E\big[\|\bar e_t\|\big]^2)$ since $\mathbb E\big[\|\bar e_t\|^2\big]= \Theta(\mathbb E\big[\|\bar\del_t\|\big]^2)$ by Lemma~\ref{lemma:bound delta} and Lemma~\ref{lemma:bound e_t}. Therefore, (\ref{eq:two-way}) has same order as (\ref{eq:sgd1}). Since other parts of the proof are the same, we conclude that two-way compression does not change the asymptotic convergence rate of Fed-EF.

\section{Conclusion}\label{sec:conclusion}

In this paper, we study Fed-EF, a federated learning (FL) framework with compressed communication and error feedback (EF). We consider two variants, Fed-EF-SGD and Fed-EF-AMS, where the global optimizer are the SGD and the adaptive gradient method, respectively. Theoretically, we demonstrate the non-convergence issue of directly using biased compression, and we provide convergence analysis in non-convex optimization showing that Fed-EF achieves the same convergence rate as the full-precision FL counterparts, which improves upon previous results. The Fed-EF-AMS variant is the first compressed adaptive FL method in the literature. Moreover, we develop a new analysis of error feedback in distributed training systems under the partial client participation setting. Our analysis provides an additional slow-down factor related to the participation rate due to the delayed error compensation of the EF mechanism. Experiments validate that, compared with full-precision training, Fed-EF achieves a significant communication reduction without performance drop. Additionally, we also present numerical results to justify the theory and provide intuition regarding the impact of the delayed error compensation on the norm convergence of Fed-EF.

\vspace{0.1in}

\noindent In summary, our work provides a thorough theoretical investigation of the standard error feedback technique in federated learning, and formally analyzes its convergence under practical FL settings with partial participation. Our paper expands several interesting future directions, e.g., to improve Fed-EF by other techniques, especially under partial participation, and to study more closely the properties of different compressors. In addition, more mechanisms in FL (e.g., variance reduction, fairness) can also be incorporated into our Fed-EF scheme for efficient training in practice.

\newpage

\bibliographystyle{plainnat}
\bibliography{refs_scholar}


\clearpage
\newpage

{\hspace{2in}\textbf{\huge Appendix}}

\appendix





\section{Additional Algorithms}  \label{app sec:competing}

\begin{algorithm}[H]
\caption{Federated Learning with Compressed Communication (Competing method \textbf{Stoc})} \label{alg:competing}
\begin{algorithmic}[1]
\State{\textbf{Input}: learning rates $\eta$, $\eta_l$, \colorbox{blue!20!white}{hyper-parameters $\beta_1$, $\beta_2$, $\epsilon$}  }
\State{\textbf{Initialize}: central server parameter $\theta_{1} \in \mathbb R^d \subseteq \mathbb R^d$; $e_{1,i}=\bm{0}$ the accumulator for each worker; \colorbox{blue!20!white}{$m_0=\bm{0}$, $v_0=\bm{0}$, $\hat v_0=\bm{0}$}}
\vspace{0.03in}
\State{\textbf{for $t=1, \ldots, T$ do}}
\State{\quad\textbf{parallel for worker $i \in [n]$ do}:}

\State{\quad\quad  Receive model parameter $\theta_{t}$ from central server, set $\theta_{t,i}^{(1)}=\theta_t$}
\State{\quad\quad  \textbf{for $k=1, \ldots, K$ do}}
\State{\quad\quad\quad  Compute stochastic gradient $g_{t,i}^{(k)}$ at $\theta_{t,i}^{(k)}$}
\State{\quad\quad\quad  Local update $\theta_{t,i}^{(k+1)}=\theta_{t,i}^{(k)}-\eta_l g_{t,i}^{(k)}$}
\State{\quad\quad  \textbf{end for}}

\State{\quad\quad Compute the local model update $\del_{t,i}=\theta_{t,i}^{(K+1)}-\theta_{t}$ }
\State{\quad\quad  Send quantized model update $\widetilde\del_{t,i}=\mathcal Q(\del_{t,i})$ to central server using (\ref{def:stoc quantizer})}

\State{\quad\textbf{end parallel}}

\State{\quad\textbf{Central server do:}}
\State{\quad Global aggregation $\overline{\widetilde\del}_{t}=\frac{1}{n}\sum_{i=1}^n \widetilde\del_{t,i}$  \label{line:g}}
\State{\quad \colorbox{green!20!white}{Update the global model $\theta_{t+1}=\theta_{t}-\eta\overline{\widetilde\del}_{t}$} \hfill\Comment{\colorbox{green!20!white}{\textbf{Stoc} with SGD}} }

\State{\quad \colorbox{blue!20!white}{$m_t=\beta_1 m_{t-1}+(1-\beta_1)\overline{\widetilde\del}_{t}$} \hfill\Comment{\colorbox{blue!20!white}{\textbf{Stoc} with AMSGrad}}  }
\State{\quad \colorbox{blue!20!white}{$v_t=\beta_2 v_{t-1}+(1-\beta_2)\overline{\widetilde\del}_{t}^2$,\quad $\hat v_t=\max(v_t,\hat v_{t-1})$} \label{line:v}}
\State{\quad \colorbox{blue!20!white}{Update the global model $\theta_{t+1}=\theta_{t}-\eta\frac{m_t}{\sqrt{\hat v_t+\epsilon}}$} }

\State{\textbf{end for}}
\end{algorithmic}

\end{algorithm}

For completeness, in Algorithm~\ref{alg:competing}, we give the details of the competing approach, compressed FL without error feedback. That is, we directly compress the transmitted gradient vectors from clients to server. Similar to Fed-EF, we may also design two variants depending on the global optimizer.

\vspace{0.1in}
\noindent\textbf{When the compressor is biased}. In Algorithm~\ref{alg:competing}, if the compressor $\mathcal Q$ is biased (e.g., \textbf{Sign} and \textbf{TopK}), then the SGD variant of Algorithm~\ref{alg:competing} is essentially Algorithm~\ref{alg:no-EF-SGD} considered in Theorem~\ref{theo:no-EF rate}, the Fed-SGD method directly using biased compression.

\vspace{0.1in}
\noindent\textbf{When the compressor is unbiased}, Algorithm~\ref{alg:competing} becomes the \textbf{Stoc} baseline in our experiments, which directly compresses the transmitted vector from clients to server by unbiased stochastic quantization $\mathcal Q(\cdot)$ proposed by~\cite{alistarh2017qsgd}. For a vector $x\in \mathbb R^d$, the operator $\mathcal Q(\cdot)$ is defined as
\begin{align}
    \mathcal Q_b(x)=\|x\|\cdot sign(x)\cdot \xi(x,b),  \label{def:stoc quantizer}
\end{align}
where $b\geq 1$ is number of bits per non-zero entry of the compressed vector $\mathcal Q(x)$. Suppose $0\leq l< 2^{b-1}$ is the integer such that $|x_i|/\|x\|$ is contained in the interval $[l/2^{b-1},(l+1)/2^{b-1}]$. The random variable $\xi(x,b)$ is defined by
\begin{align*}
    \xi(x,b) = \begin{cases}
      l/s,  & \text{with probability}\ 1-g(\frac{|x_i|}{\|x\|},b), \\
      (l+1)/s, & \text{otherwise},
    \end{cases}
\end{align*}
with $g(a,b)=a\cdot 2^{b-1}-l$ for $a\in[0,1]$. Simply, 0 is always quantized to 0. The \textbf{Stoc} quantizer is unbiased, i.e., $\mathbb E[\mathcal Q(x)|x]=x$. In addition, it also introduces sparsity to the compressed vector in a probabilistic way, with $\mathbb E[\|\mathcal Q(x)\|_0]\leq 2^{b}+2^{b-1}\sqrt d$.

Additionally, we mention that \textbf{Stoc} also has two corresponding variants, one using SGD and one using AMSGrad as the global optimizer. For the SGD variant, \textbf{Stoc} is equivalent to the FedCOM method in~\cite{haddadpour2021federated}, which is also the FedPaQ algorithm~\citep{reisizadeh2020fedpaq} with tunable global learning rate.

\vspace{0.1in}
\noindent\textbf{When there is no compression}, for the full-precision algorithms, we simply set $\widetilde\del_{t,i}=\del_{t,i}$ in line 11 of Algorithm~\ref{alg:competing}. For SGD, it is the one studied in \cite{yang2021achieving} which is the standard local SGD~\citep{mcmahan2017communication} with global learning rate. For the AMSGrad variant, it becomes FedAdam~\citep{reddi2021adaptive}. Note that \cite{reddi2021adaptive} used Adam, while we use AMSGrad (with the max operation) for better stability. Empirically, the performance of these two options are very similar.

\section{Parameter Tuning and More Experiment Results} \label{app sec:experiment}

\subsection{Parameter Tuning}

In our experiments, we fine-tune the global and local learning rates for the baseline methods and Fed-EF with each hyper-parameter of the compressors (i.e., the compression rate). For the AMSGrad optimizer, we set $\beta_1=0.9$, $\beta_2=0.999$ and $\epsilon=10^{-8}$ as the recommended default~\citep{reddi2019convergence}. For each algorithm, we tune $\eta$ over $\{ 10^{-4},10^{-3},10^{-2},10^{-1},1,10 \}$ and $\eta_l$ over $\{10^{-4},10^{-3},10^{-2},10^{-1},1\}$. We found that the compressed methods usually have same optimal learning rates as the full-precision training. The best learning rate combinations achieving highest test accuracy are given in Table~\ref{tab:best lr}.

\begin{table}[h]
\vspace{0.2in}
\centering
\begin{tabular}{c|cc|cc}
\Xhline{2\arrayrulewidth} \hline
  \multirow{2}{*}{}     & \multicolumn{2}{c}{Fed-EF-SGD} & \multicolumn{2}{c}{Fed-EF-AMS} \\ 
       & $\eta$        & $\eta_l$       & $\eta$        & $\eta_l$       \\ \hline
MNIST  &       $10$        &    $10^{-3}$            &      $10^{-3}$         &       $10^{-2}$         \\
FMNIST &      $1$         &     $10^{-1}$           &   $10^{-2}$            &    $10^{-1}$    \\
CIFAR-10 & 1 & $10^{-1}$           &   $10^{-3}$            &    $10^{-2}$    \\ \hline
\Xhline{2\arrayrulewidth}
\end{tabular}
\vspace{0.1in}
\caption{Optimal global ($\eta$) and local ($\eta_l$) learning rate combinations to attain highest test accuracy.}
\label{tab:best lr}
\end{table}

\subsection{More Experiment Results}

We provide the complete set of experimental results on each method under various compression rates. In Table~\ref{tab:acc-SGD-p0.5} - Table~\ref{tab:acc-AMS-p0.1}, for completeness we report the average test accuracy at the end of training and the standard deviations, corresponding to the curves (compression parameters) in Figure~\ref{fig:MNIST-FMNIST-0.5} and Figure~\ref{fig:MNIST-FMNIST-0.1}. Figure~\ref{fig:MNIST-loss-compressor-0.5} and Figure~\ref{fig:FMNIST-loss-compressor-0.5} present the training loss under participation rate $p=0.5$, and Figure~\ref{fig:MNIST-loss-compressor-0.1} to Figure~\ref{fig:FMNIST-acc-compressor-0.1} report the loss and accuracy results for $p=0.1$. All the results suggest that Fed-EF is able to perform on a par with full-precision FL with much less communication cost.

\begin{table}[h]
\centering
\begin{tabular}{c|ccc|cc}
\Xhline{2\arrayrulewidth} \hline
\multirow{2}{*}{} & \multicolumn{3}{c}{Fed-EF-SGD}             & \multicolumn{2}{c}{No EF}      \\ \hline
                  & Sign         & TopK         & Hv-Sign      & Stoc         & Full-precision           \\ \hline
MNIST             & 90.87 ($\pm$0.84) & 91.04 ($\pm$1.05) & 91.18 ($\pm$1.10) & 90.16 ($\pm$0.96) & 90.85 ($\pm$0.89) \\
FMNIST            & 71.13 ($\pm$0.68) & 71.16 ($\pm$0.77) & 71.07 ($\pm$0.83) & 71.26 ($\pm$0.87) & 71.20 ($\pm$0.71) \\
\hline \Xhline{2\arrayrulewidth}
\end{tabular}
\caption{Test accuracy (\%) with client participation rate $p=0.5$, of Fed-EF-SGD with \textbf{Sign}, \textbf{TopK} and \textbf{heavy-Sign} compressor and \textbf{Stoc} (stochastic quantization) without EF. The compression parameters (i.e., $k$ and $b$) of the compressors are consistent with Figure~\ref{fig:MNIST-FMNIST-0.5}.}
\label{tab:acc-SGD-p0.5}
\end{table}

\begin{table}[h]
\centering
\begin{tabular}{c|ccc|cc}
\Xhline{2\arrayrulewidth} \hline
\multirow{2}{*}{} & \multicolumn{3}{c}{Fed-EF-AMS}             & \multicolumn{2}{c}{No EF}      \\ \hline
                  & Sign         & TopK         & Hv-Sign      & Stoc         & Full-precision             \\ \hline
MNIST             & 92.32 ($\pm$0.98) & 92.74 ($\pm$0.84) & 91.77 ($\pm$1.22) & 92.36 ($\pm$0.93) & 92.23 ($\pm$0.73) \\
FMNIST            & 71.35 ($\pm$0.61) & 71.90 ($\pm$0.78) & 70.73($\pm$1.03)  & 71.94 ($\pm$0.95) & 71.97 ($\pm$0.86) \\
\hline \Xhline{2\arrayrulewidth}
\end{tabular}
\caption{Test accuracy (\%) with client participation rate $p=0.5$, of Fed-EF-AMS with \textbf{Sign}, \textbf{TopK} and \textbf{heavy-Sign} compressor and \textbf{Stoc} (stochastic quantization) without EF. The compression parameters (i.e., $k$ and $b$) of the compressors are consistent with Figure~\ref{fig:MNIST-FMNIST-0.5}.}
\label{tab:acc-AMS-p0.5}
\end{table}

\begin{table}[h]
\centering
\begin{tabular}{c|ccc|cc}
\Xhline{2\arrayrulewidth} \hline
\multirow{2}{*}{} & \multicolumn{3}{c}{Fed-EF-SGD}             & \multicolumn{2}{c}{No EF}      \\ \hline
                  & Sign         & TopK         & Hv-Sign      & Stoc         & Full-precision             \\ \hline
MNIST             & 90.15 ($\pm$1.06) & 90.61 ($\pm$0.93) & 90.42 ($\pm$1.09) & 90.27 ($\pm$1.18) & 90.22 ($\pm$0.82) \\
FMNIST            & 67.69 ($\pm$0.73) & 67.47 ($\pm$0.80) & 67.72 ($\pm$0.55) & 67.71 ($\pm$0.78) & 67.50 ($\pm$0.85) \\
\hline \Xhline{2\arrayrulewidth}
\end{tabular}
\caption{Test accuracy (\%) with client participation rate $p=0.1$, of Fed-EF-SGD with \textbf{Sign}, \textbf{TopK} and \textbf{heavy-Sign} compressor and \textbf{Stoc} (stochastic quantization) without EF. The compression parameters (i.e., $k$ and $b$) of the compressors are consistent with Figure~\ref{fig:MNIST-FMNIST-0.1}.}
\label{tab:acc-SGD-p0.1}
\end{table}

\begin{table}[h!]
\centering
\begin{tabular}{c|ccc|cc}
\Xhline{2\arrayrulewidth} \hline
\multirow{2}{*}{} & \multicolumn{3}{c}{Fed-EF-AMS}             & \multicolumn{2}{c}{No EF}      \\ \hline
                  & Sign         & TopK         & Hv-Sign      & Stoc         & Full-precision             \\ \hline
MNIST             & 88.67 ($\pm$1.11) & 88.97 ($\pm$1.16) & 77.49 ($\pm$1.53) & 88.76 ($\pm$1.22) & 89.05 ($\pm$1.04) \\
FMNIST            & 57.60 ($\pm$2.34) & 64.09 ($\pm$0.91) & 50.77($\pm$2.87)  & 64.35 ($\pm$1.06) & 64.18 ($\pm$0.90) \\
\hline \Xhline{2\arrayrulewidth}
\end{tabular}
\caption{Test accuracy (\%) with client participation rate $p=0.1$, of Fed-EF-AMS with \textbf{Sign}, \textbf{TopK} and \textbf{heavy-Sign} compressor and \textbf{Stoc} (stochastic quantization) without EF. The compression parameters (i.e., $k$ and $b$) of the compressors are consistent with Figure~\ref{fig:MNIST-FMNIST-0.1}.}
\label{tab:acc-AMS-p0.1}
\end{table}

\newpage\clearpage

\begin{figure}[t!]
    \begin{center}
        \mbox{\hspace{-0.2in}
        \includegraphics[width=2.25in]{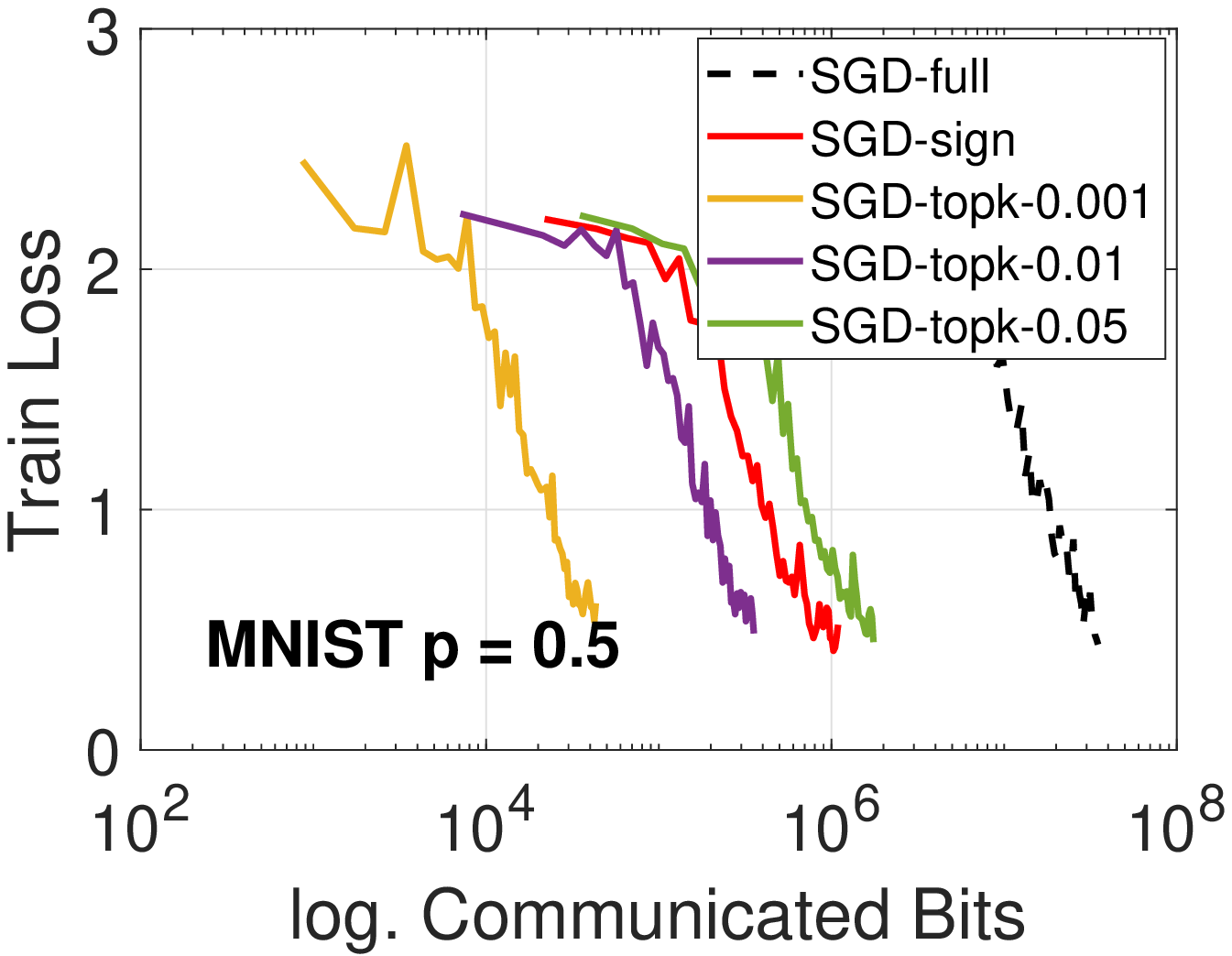}\hspace{-0.1in}
        \includegraphics[width=2.25in]{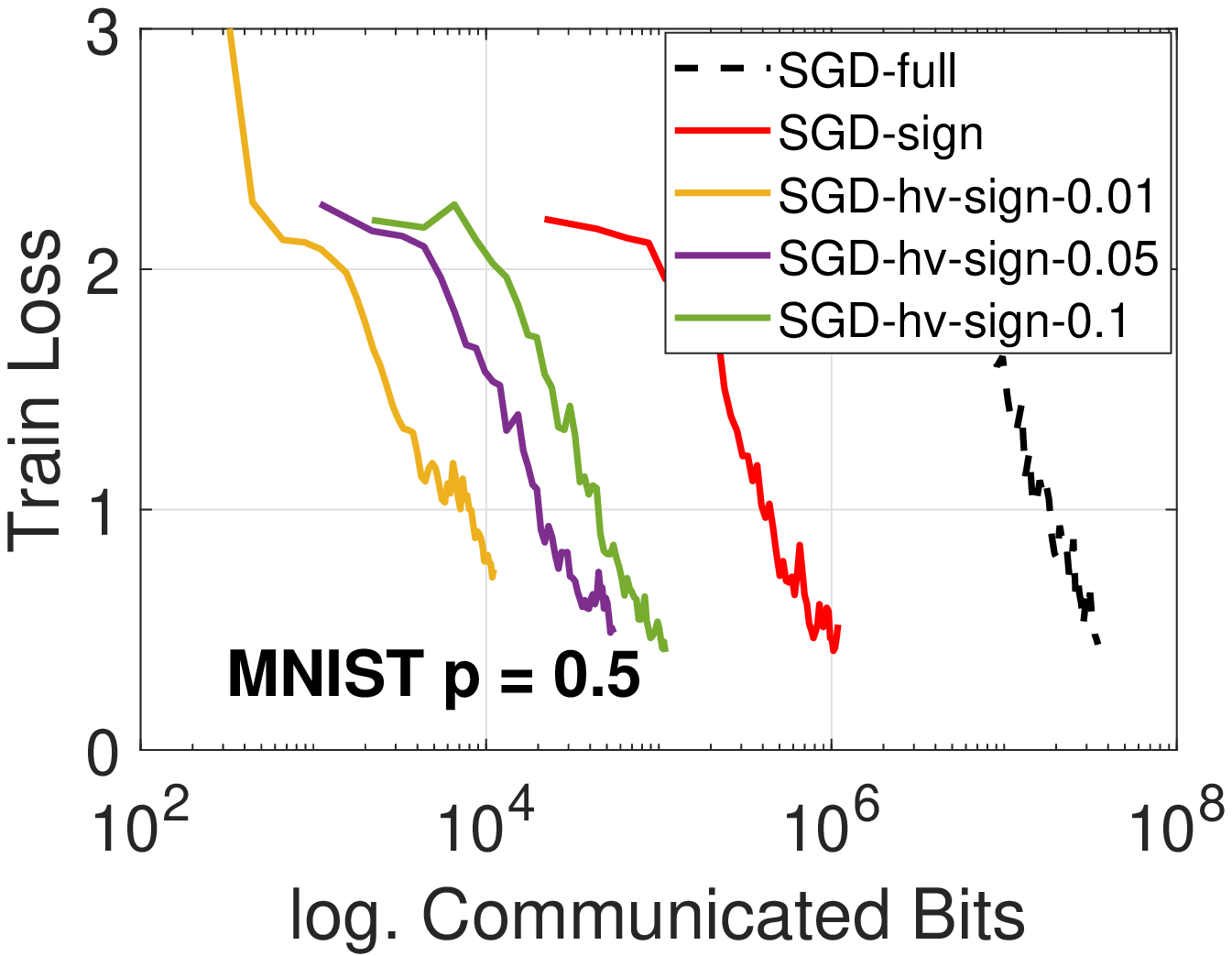}\hspace{-0.1in}
        \includegraphics[width=2.25in]{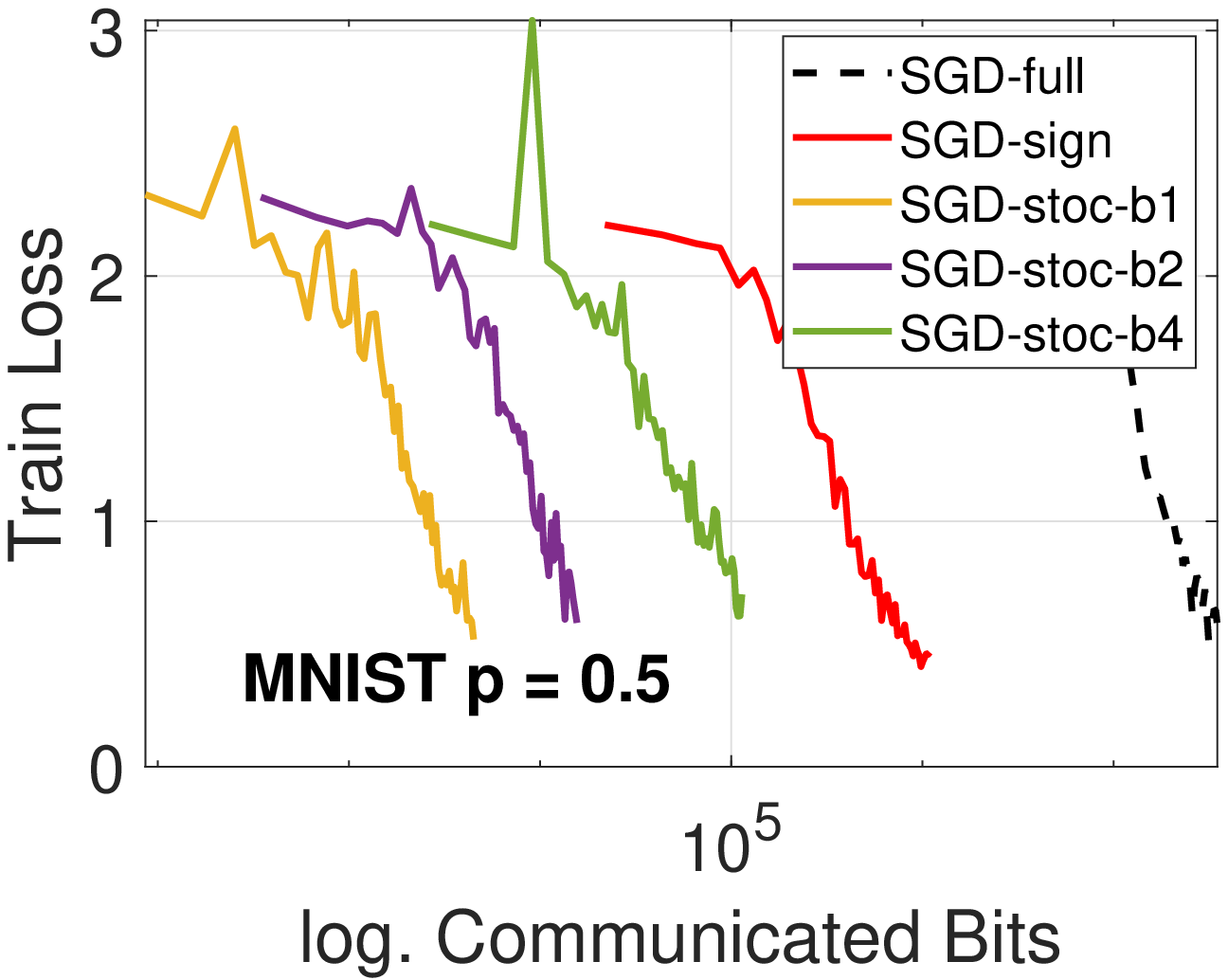}
        }
        \mbox{\hspace{-0.2in}
        \includegraphics[width=2.25in]{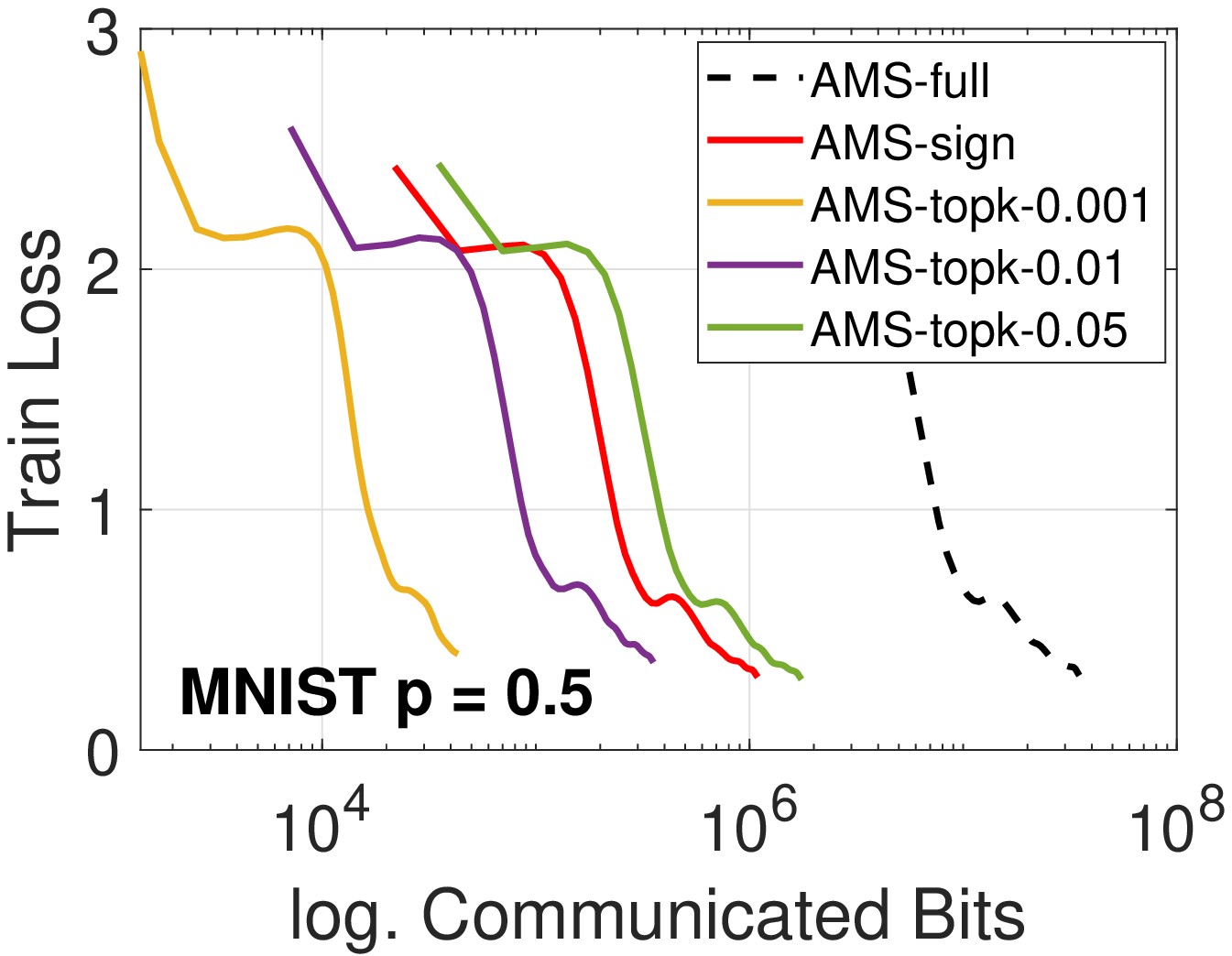}\hspace{-0.1in}
        \includegraphics[width=2.25in]{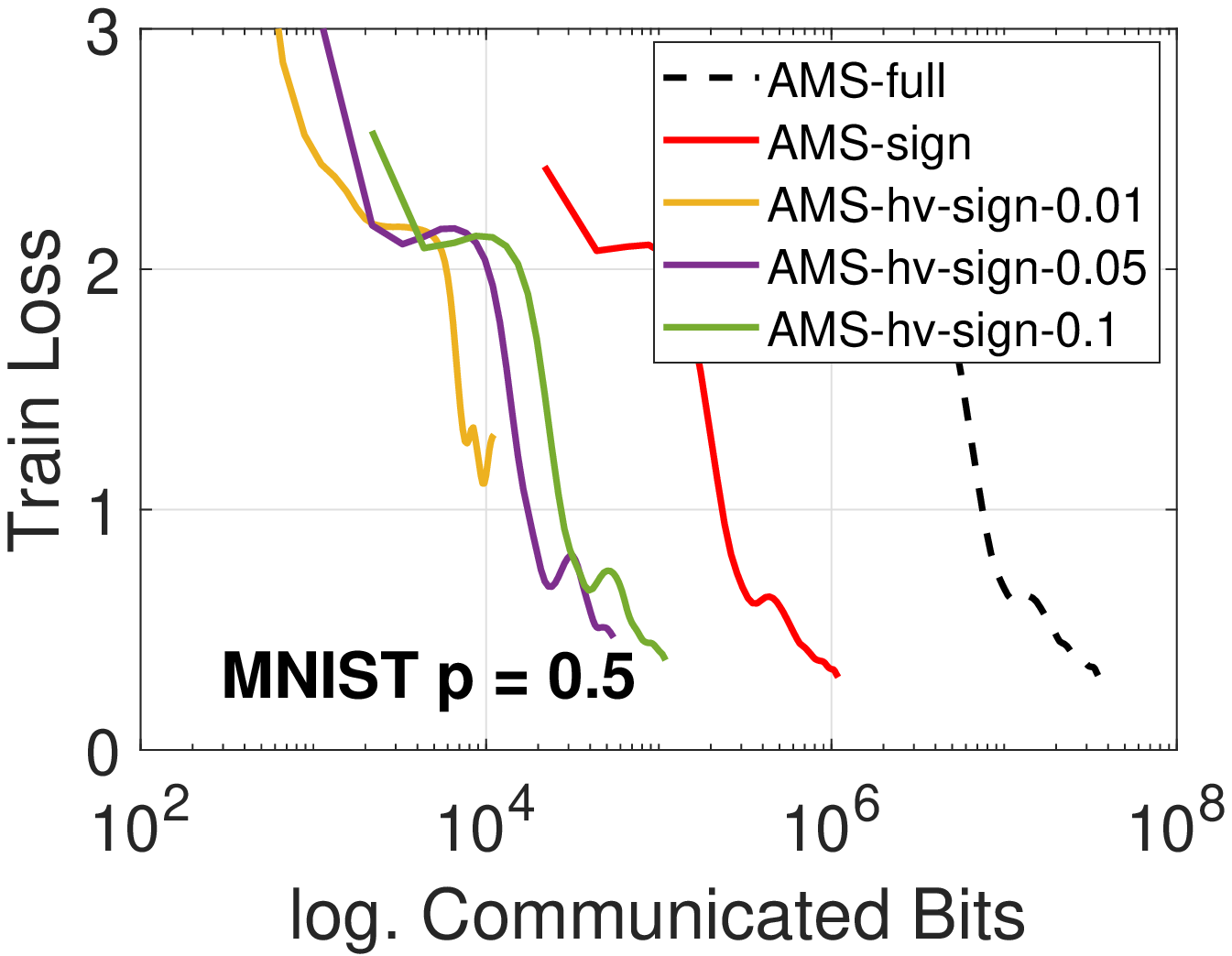}\hspace{-0.1in}
        \includegraphics[width=2.25in]{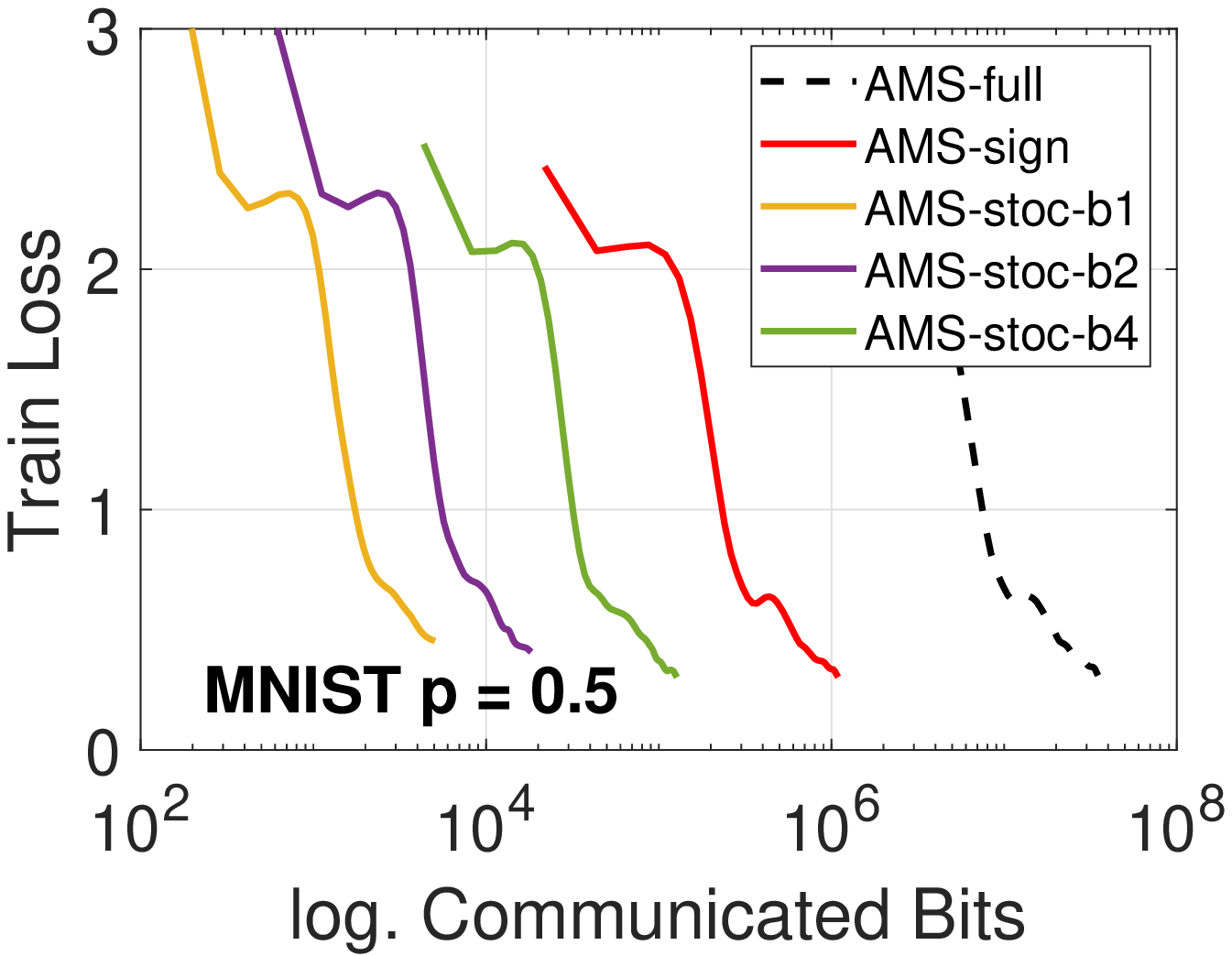}
        }
    \end{center}
    \vspace{-0.1in}

	\caption{Training loss of Fed-EF on MNIST dataset trained by CNN. ``sign'', ``topk'' and ``hv-sign'' are applied with Fed-EF, while ``Stoc'' is the stochastic quantization without EF. Participation rate $p=0.5$, non-iid data. 1st row: Fed-EF-SGD. 2nd row: Fed-EF-AMS.}
	\label{fig:MNIST-loss-compressor-0.5}
\end{figure}

\begin{figure}[h!]
    \begin{center}
        \mbox{\hspace{-0.2in}
        \includegraphics[width=2.25in]{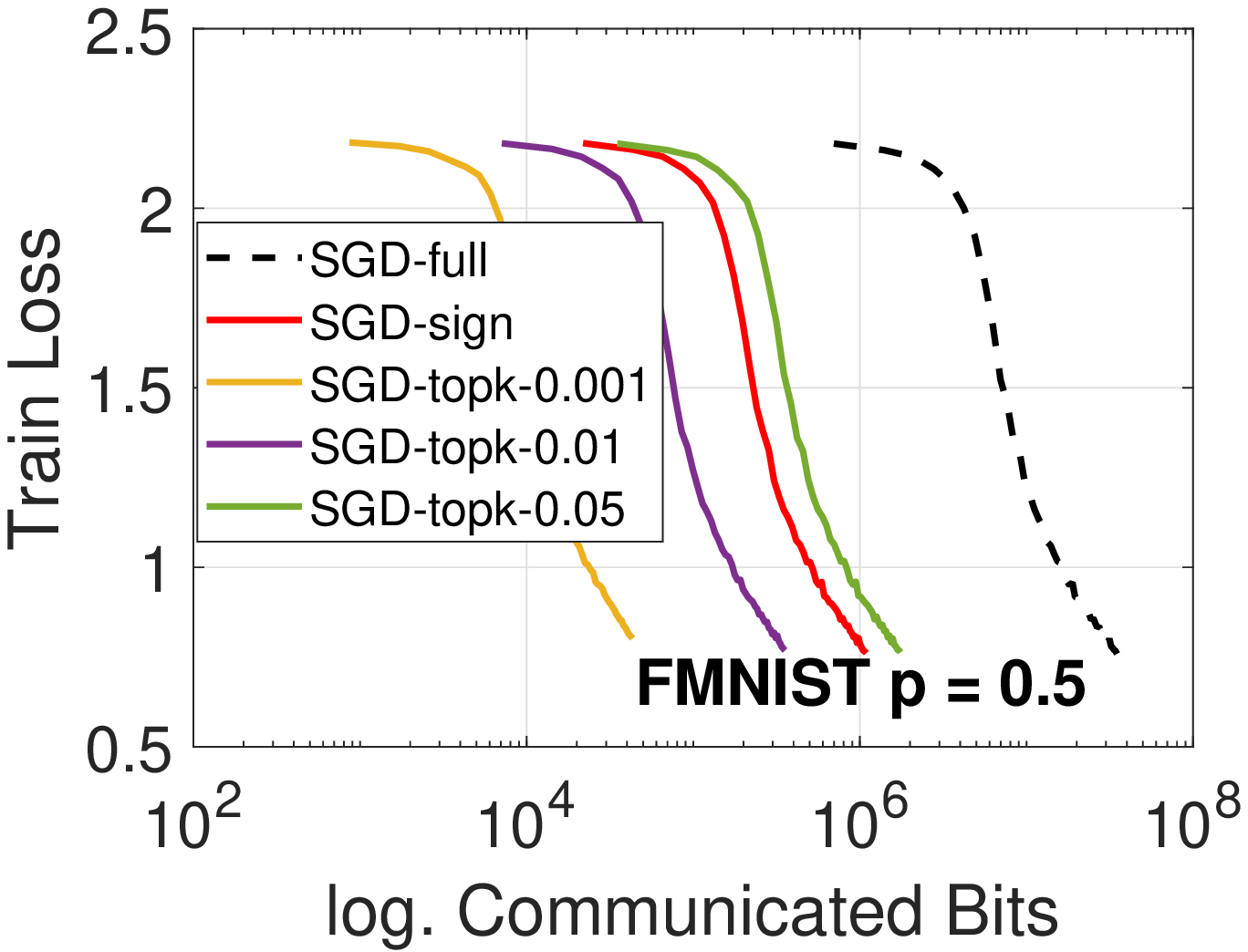}\hspace{-0.1in}
        \includegraphics[width=2.25in]{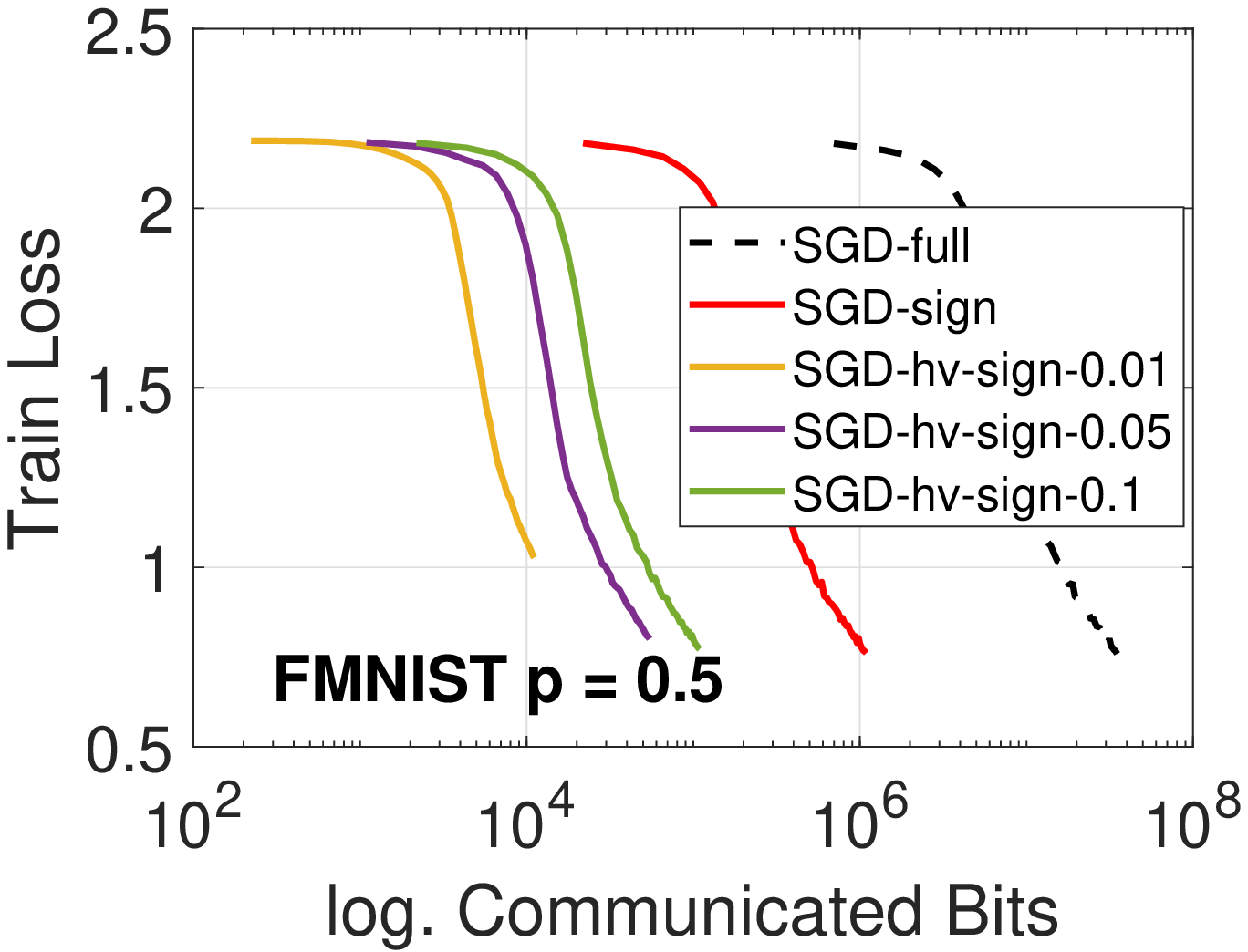}\hspace{-0.1in}
        \includegraphics[width=2.25in]{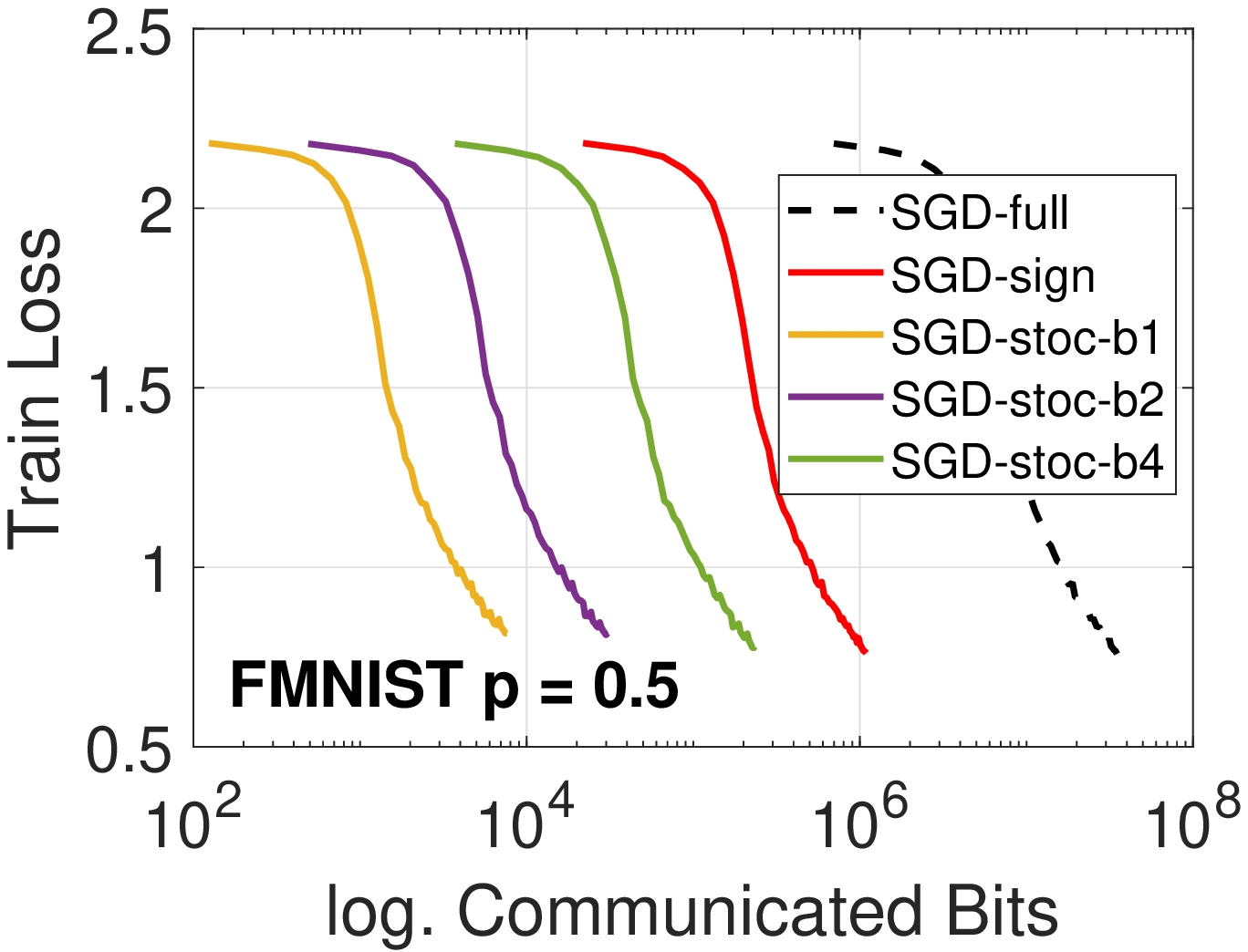}
        }
        \mbox{\hspace{-0.2in}
        \includegraphics[width=2.25in]{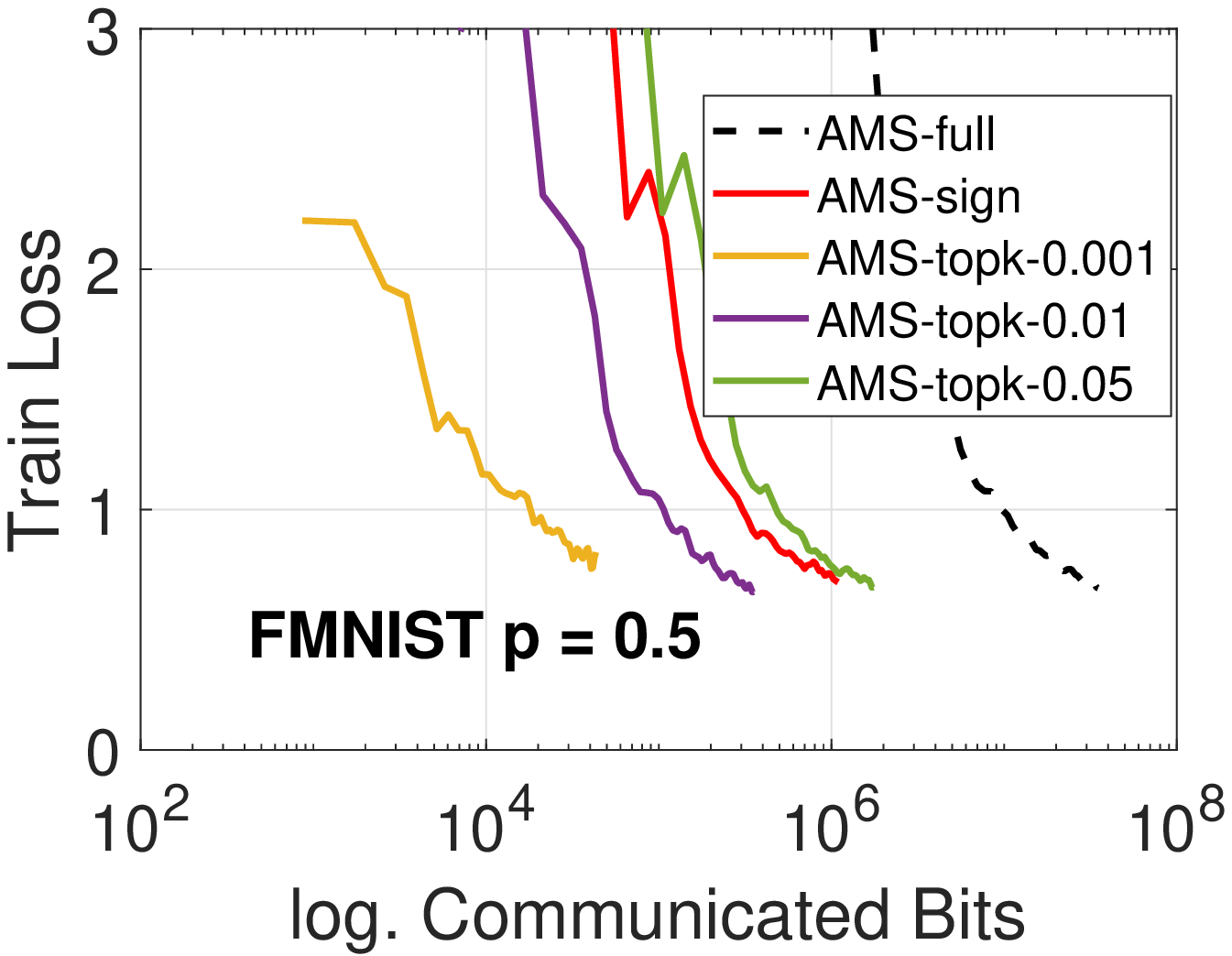}\hspace{-0.1in}
        \includegraphics[width=2.25in]{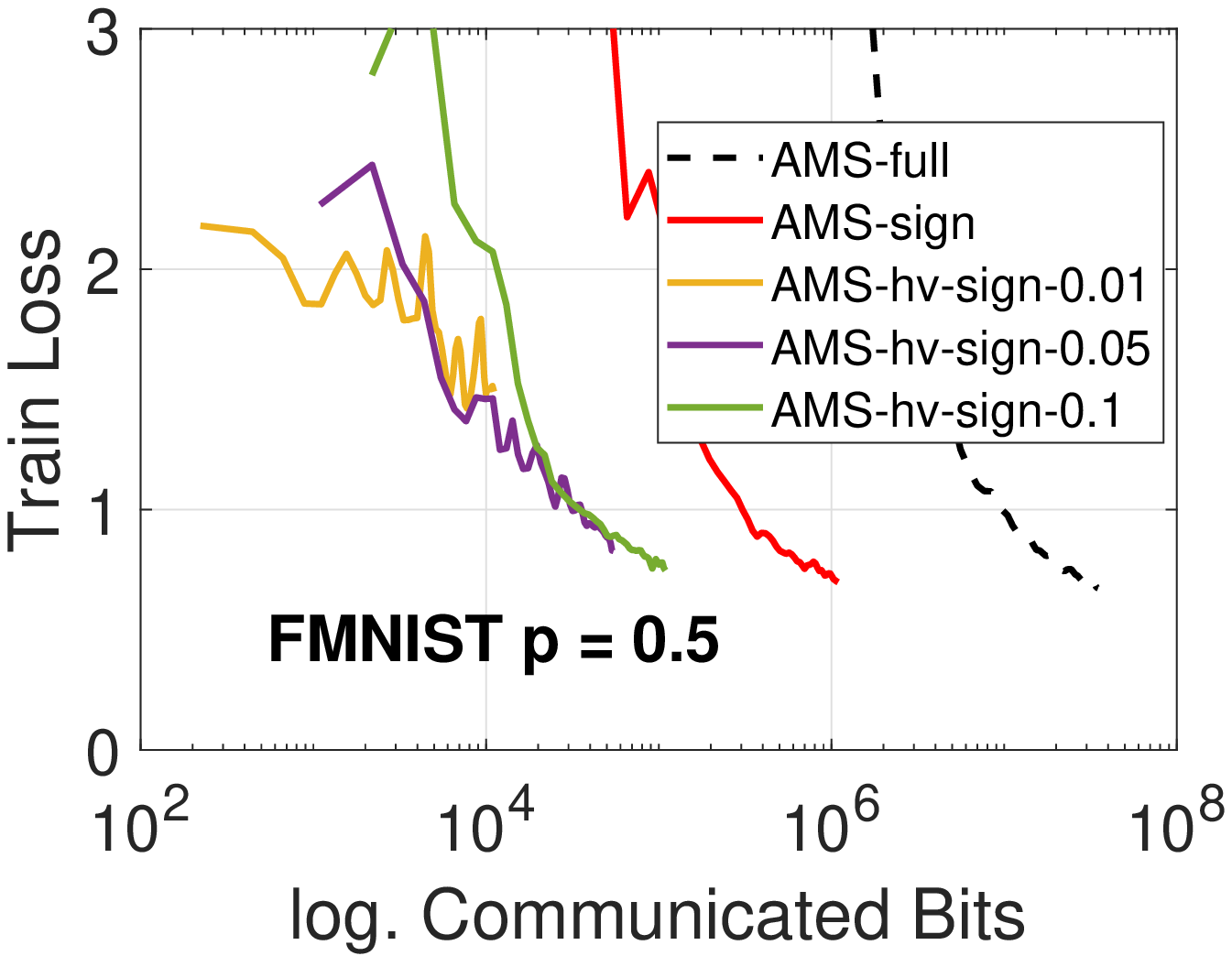}\hspace{-0.1in}
        \includegraphics[width=2.25in]{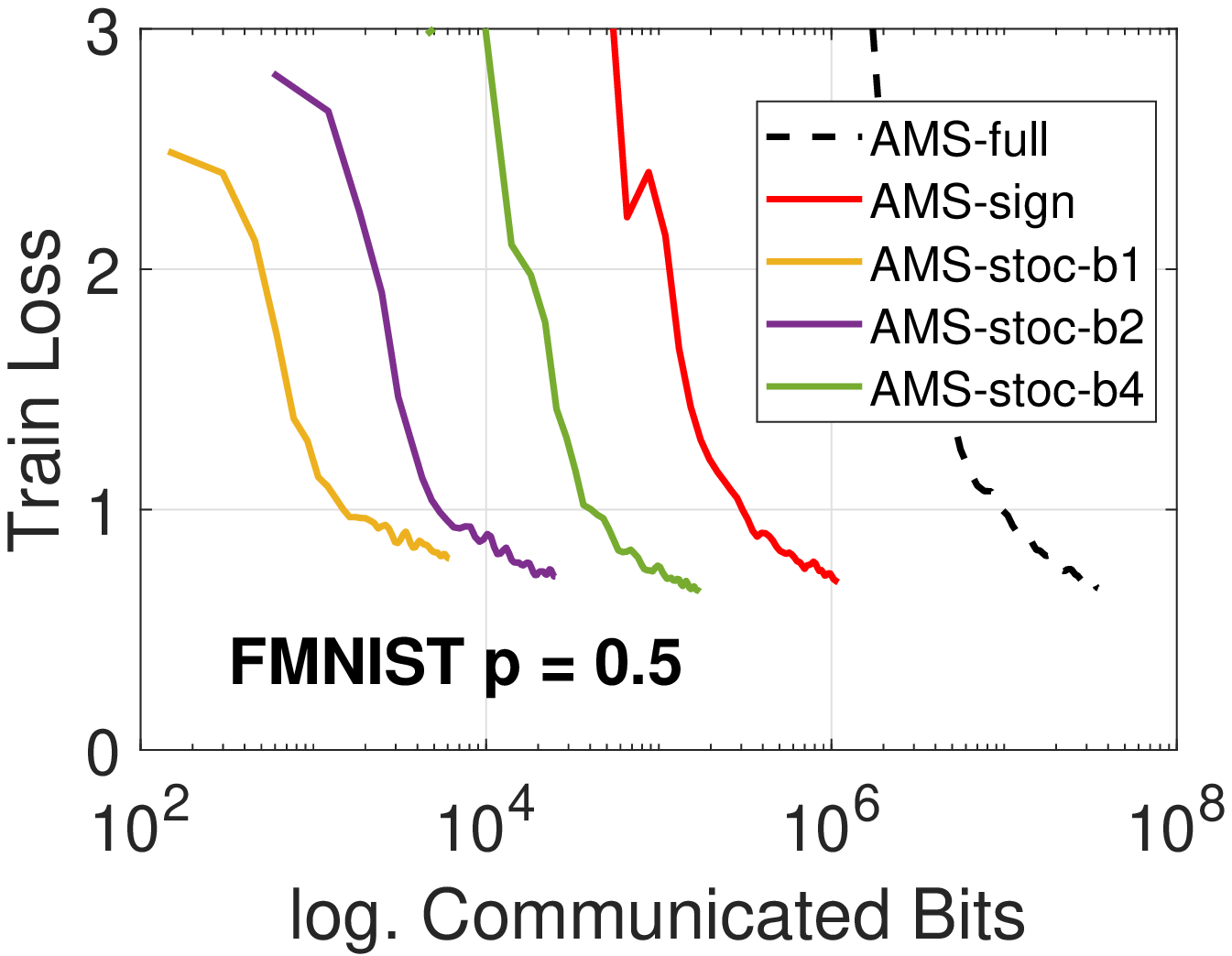}
        }
    \end{center}
    \vspace{-0.1in}
	\caption{Training loss of Fed-EF on FMNIST dataset trained by CNN. ``sign'', ``topk'' and ``hv-sign'' are applied with Fed-EF, while ``Stoc'' is the stochastic quantization without EF. Participation rate $p=0.5$, non-iid data. 1st row: Fed-EF-SGD. 2nd row: Fed-EF-AMS.}
	\label{fig:FMNIST-loss-compressor-0.5}
\end{figure}

\begin{figure}[h]
    \begin{center}
        \mbox{\hspace{-0.2in}
        \includegraphics[width=2.25in]{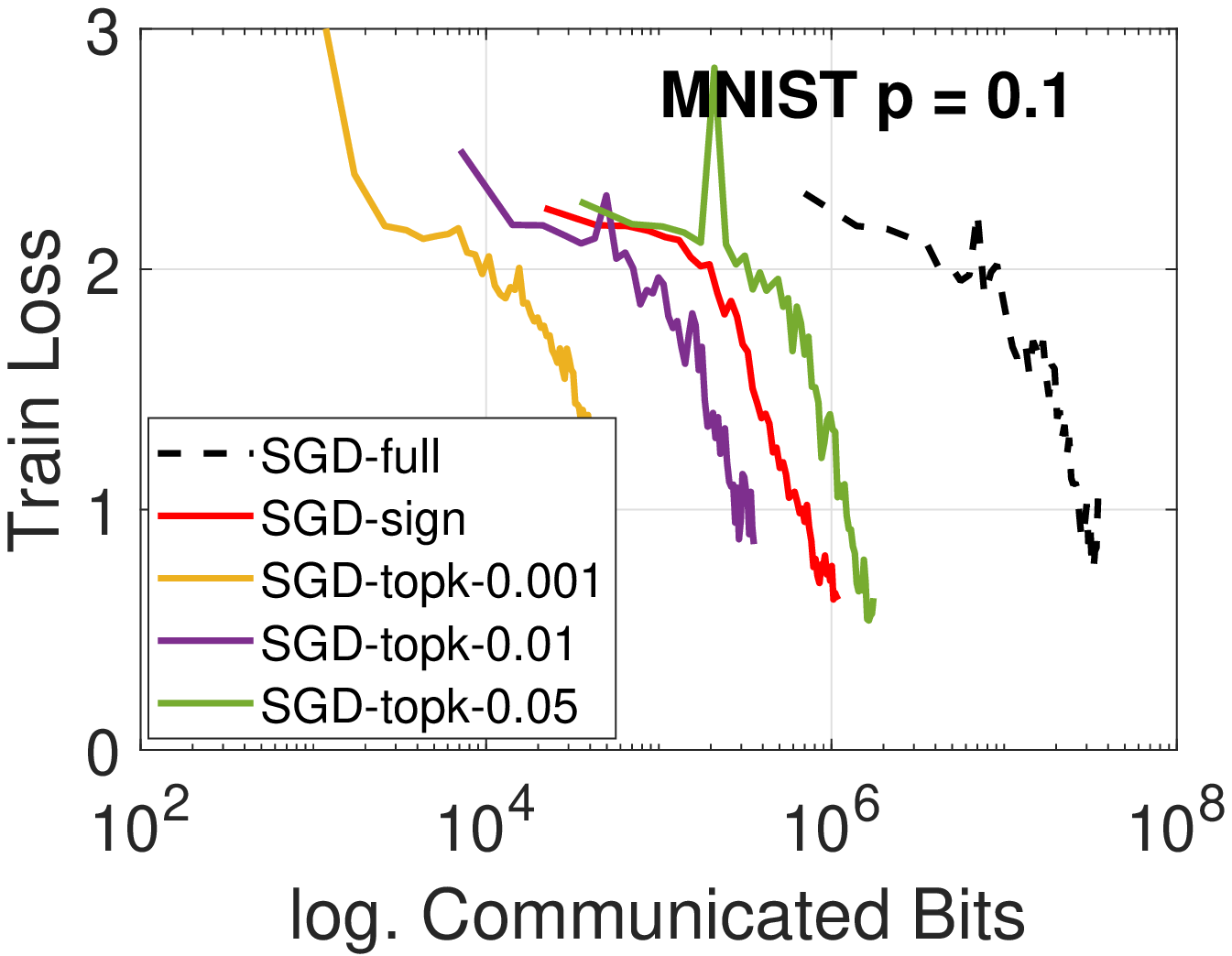}\hspace{-0.1in}
        \includegraphics[width=2.25in]{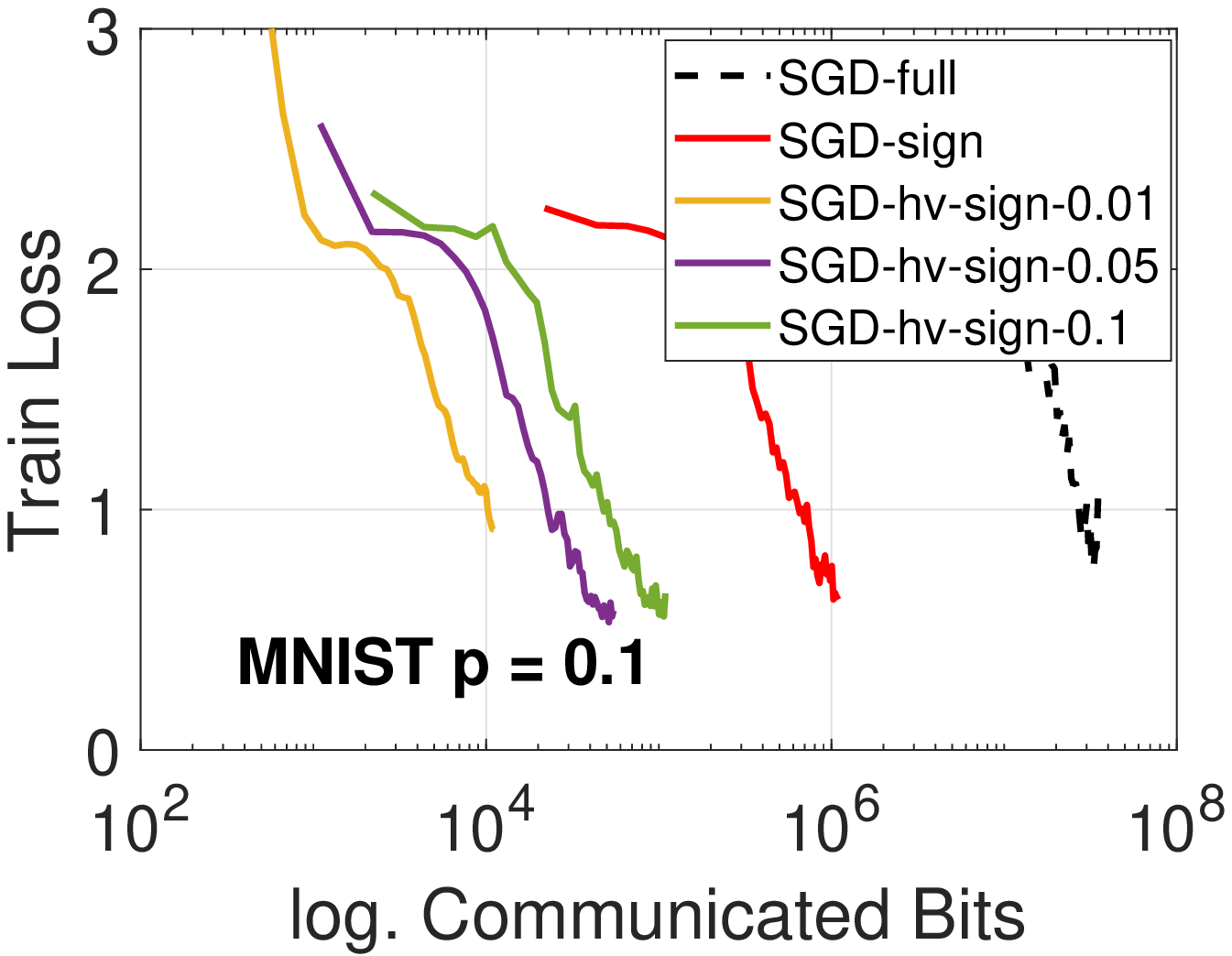}\hspace{-0.1in}
        \includegraphics[width=2.25in]{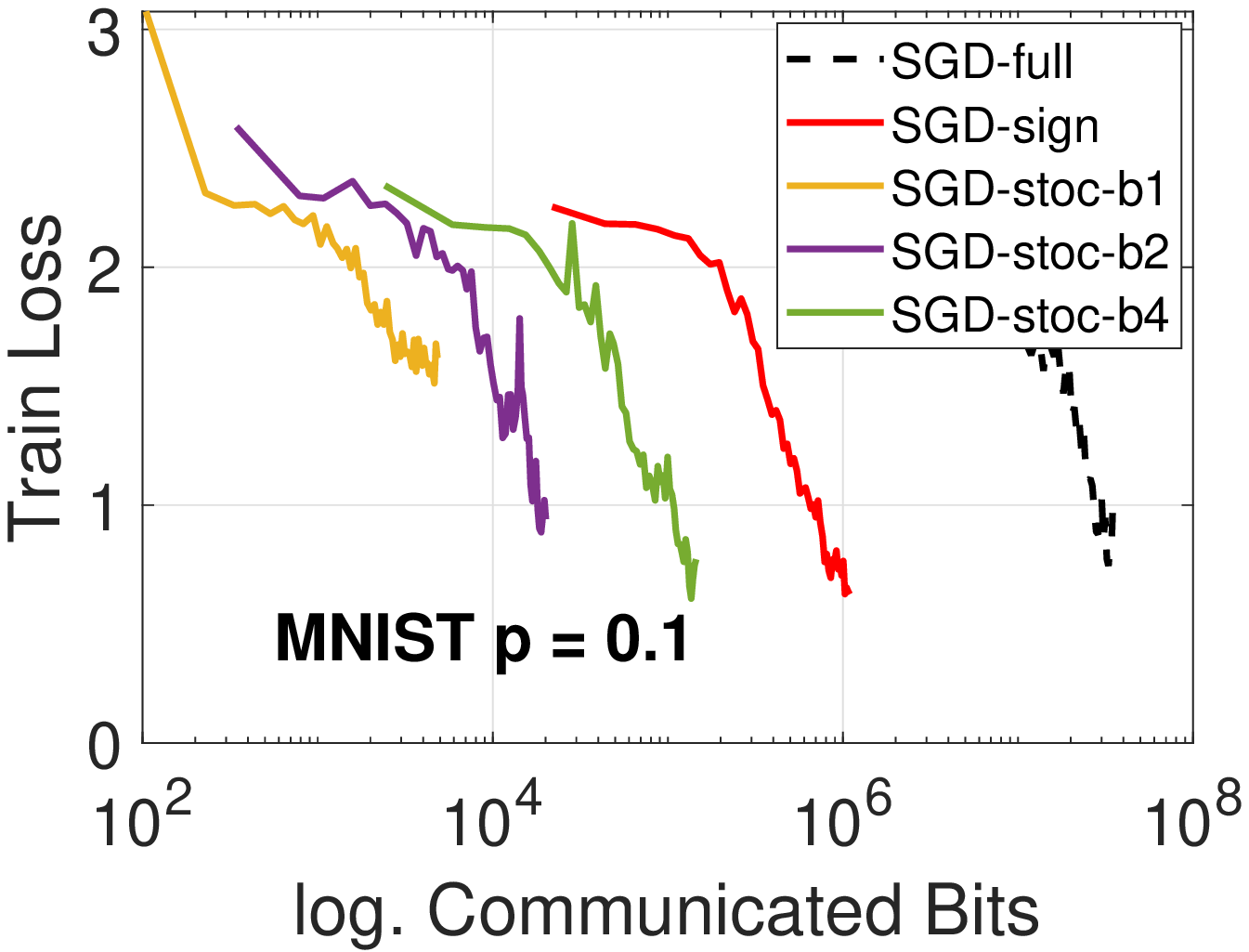}
        }
        \mbox{\hspace{-0.2in}
        \includegraphics[width=2.25in]{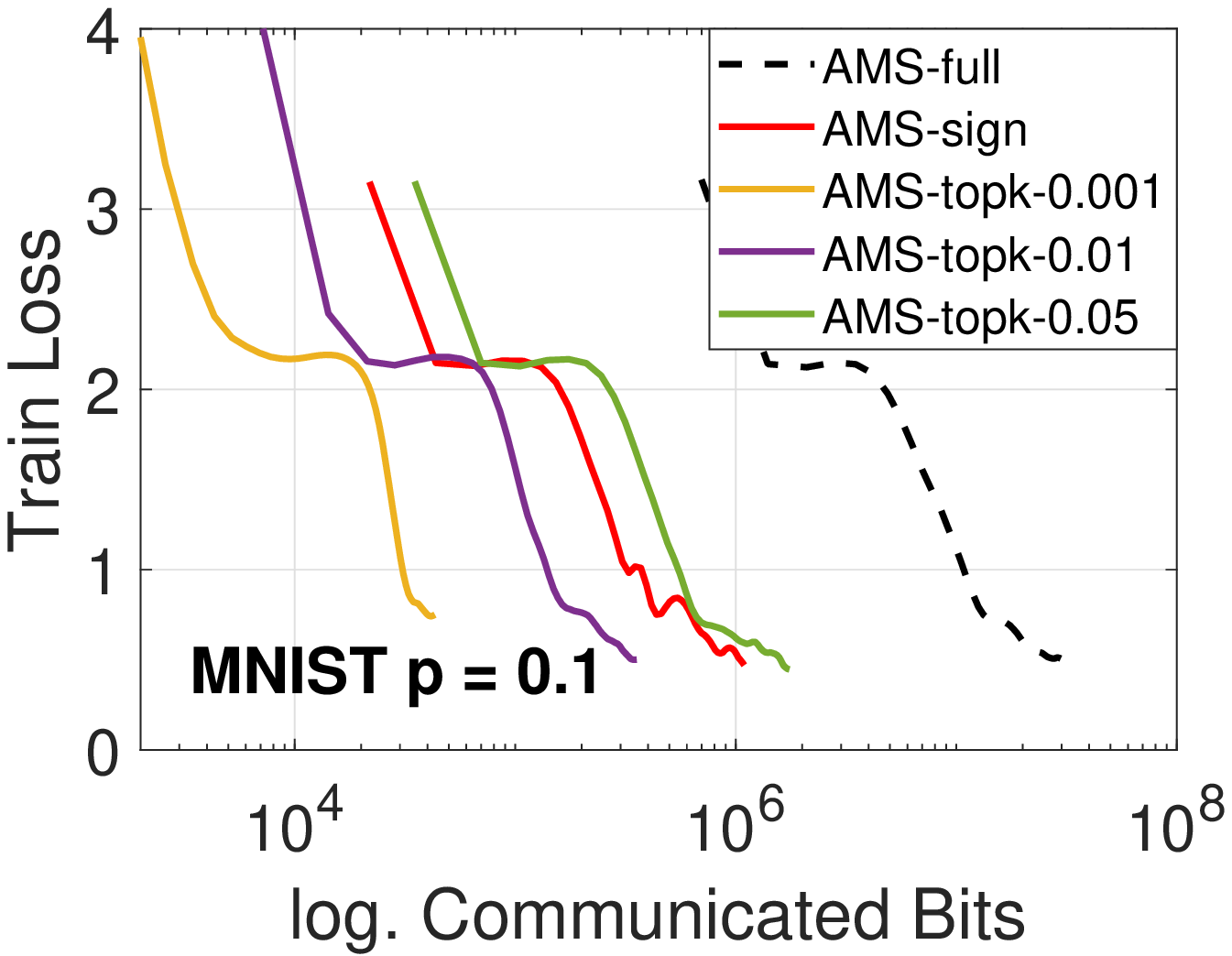}\hspace{-0.1in}
        \includegraphics[width=2.25in]{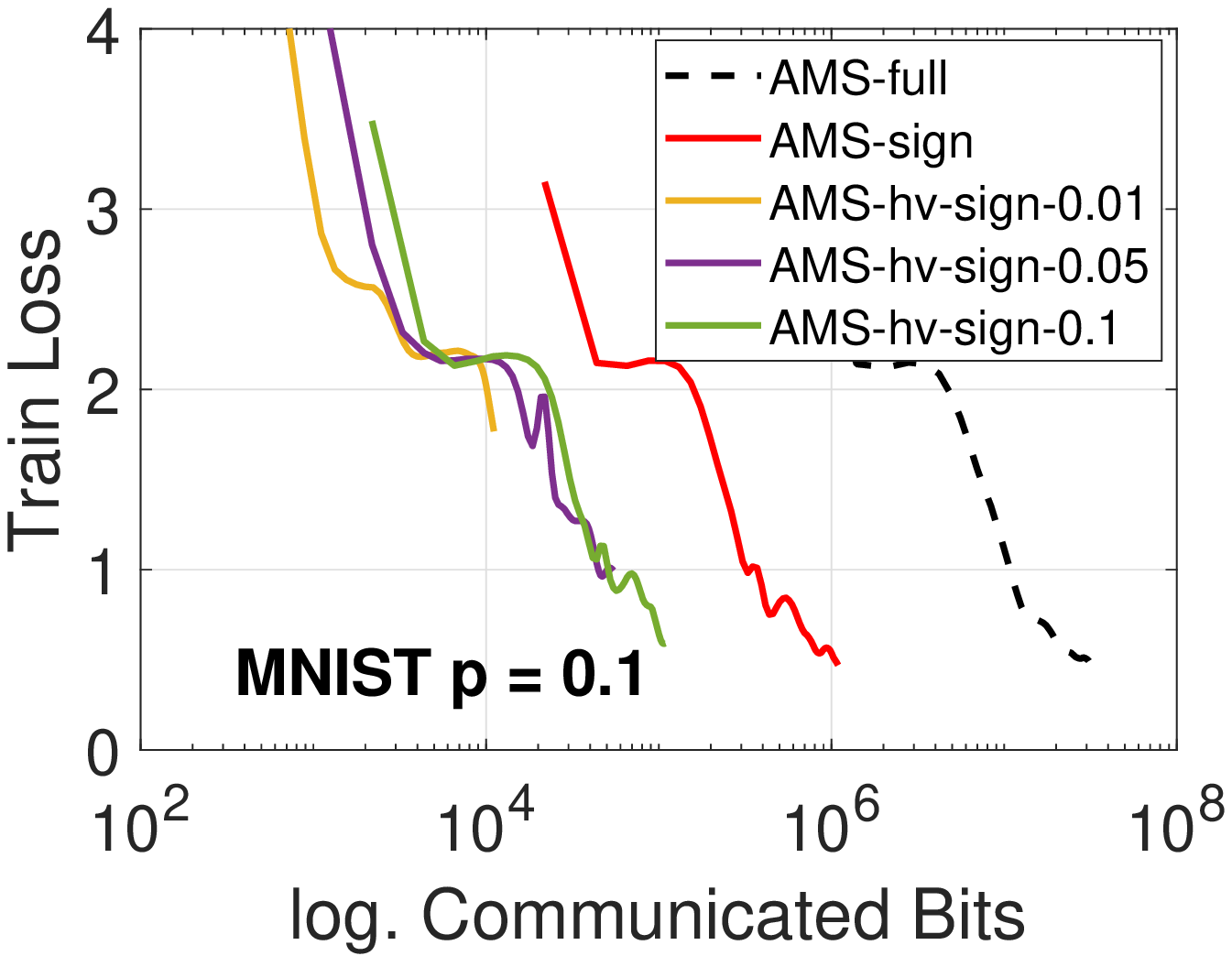}\hspace{-0.1in}
        \includegraphics[width=2.25in]{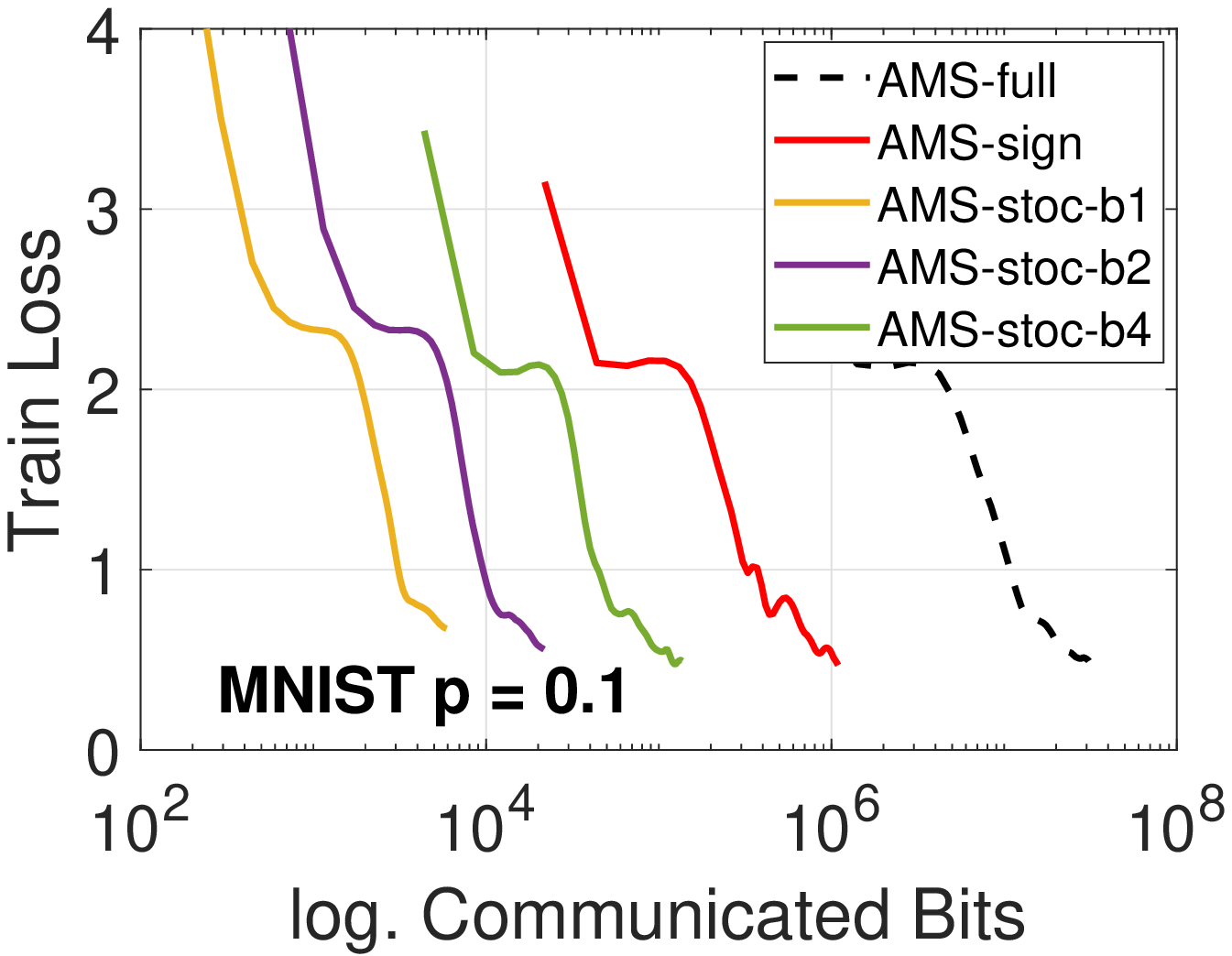}
        }
    \end{center}
    \vspace{-0.1in}
	\caption{Training loss of Fed-EF on MNIST dataset trained by CNN. ``sign'', ``topk'' and ``hv-sign'' are applied with Fed-EF, while ``Stoc'' is the stochastic quantization without EF. Participation rate $p=0.1$, non-iid data. 1st row: Fed-EF-SGD. 2nd row: Fed-EF-AMS.}
	\label{fig:MNIST-loss-compressor-0.1}
	
\end{figure}

\begin{figure}[h]
    \begin{center}
        \mbox{\hspace{-0.2in}
        \includegraphics[width=2.25in]{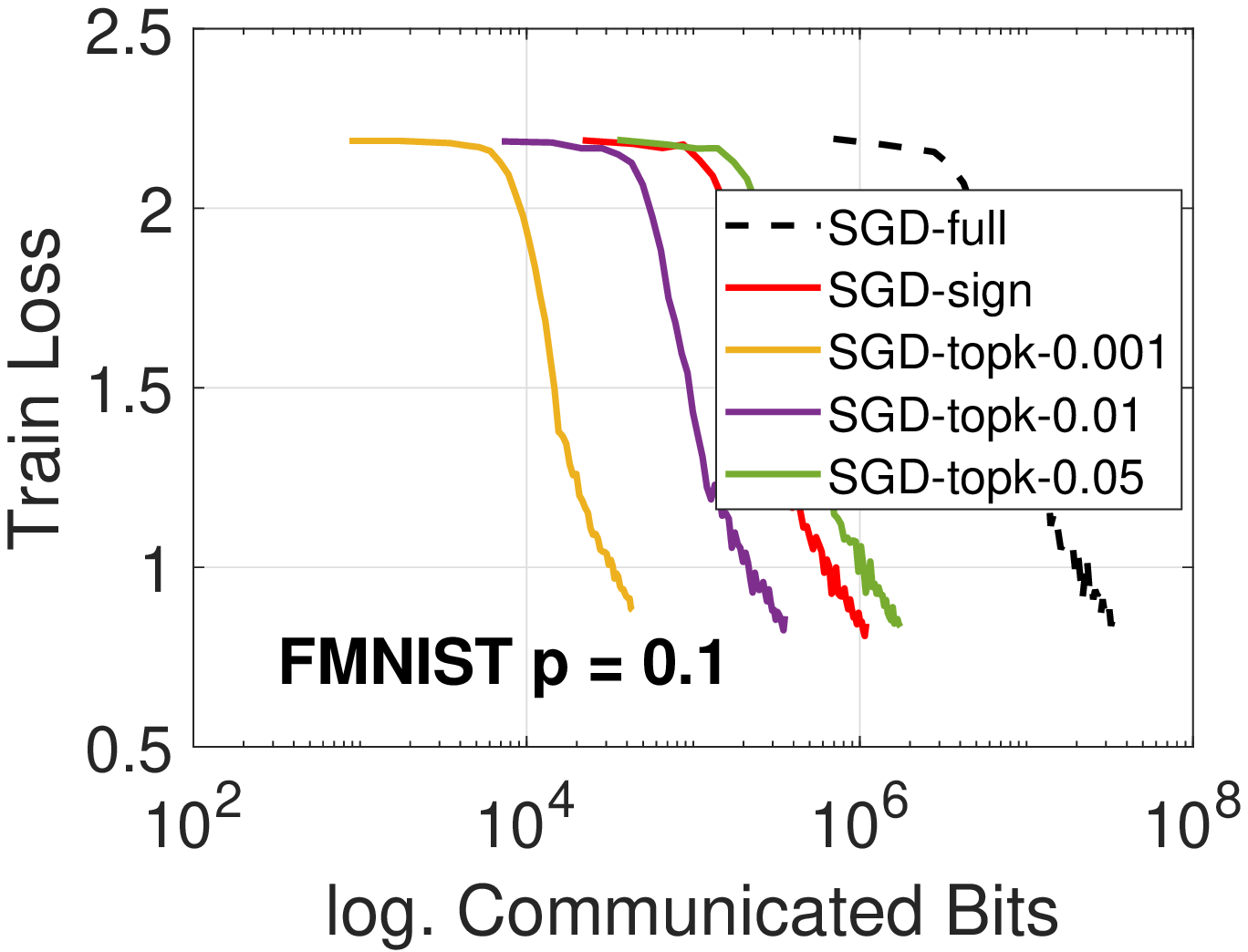}\hspace{-0.1in}
        \includegraphics[width=2.25in]{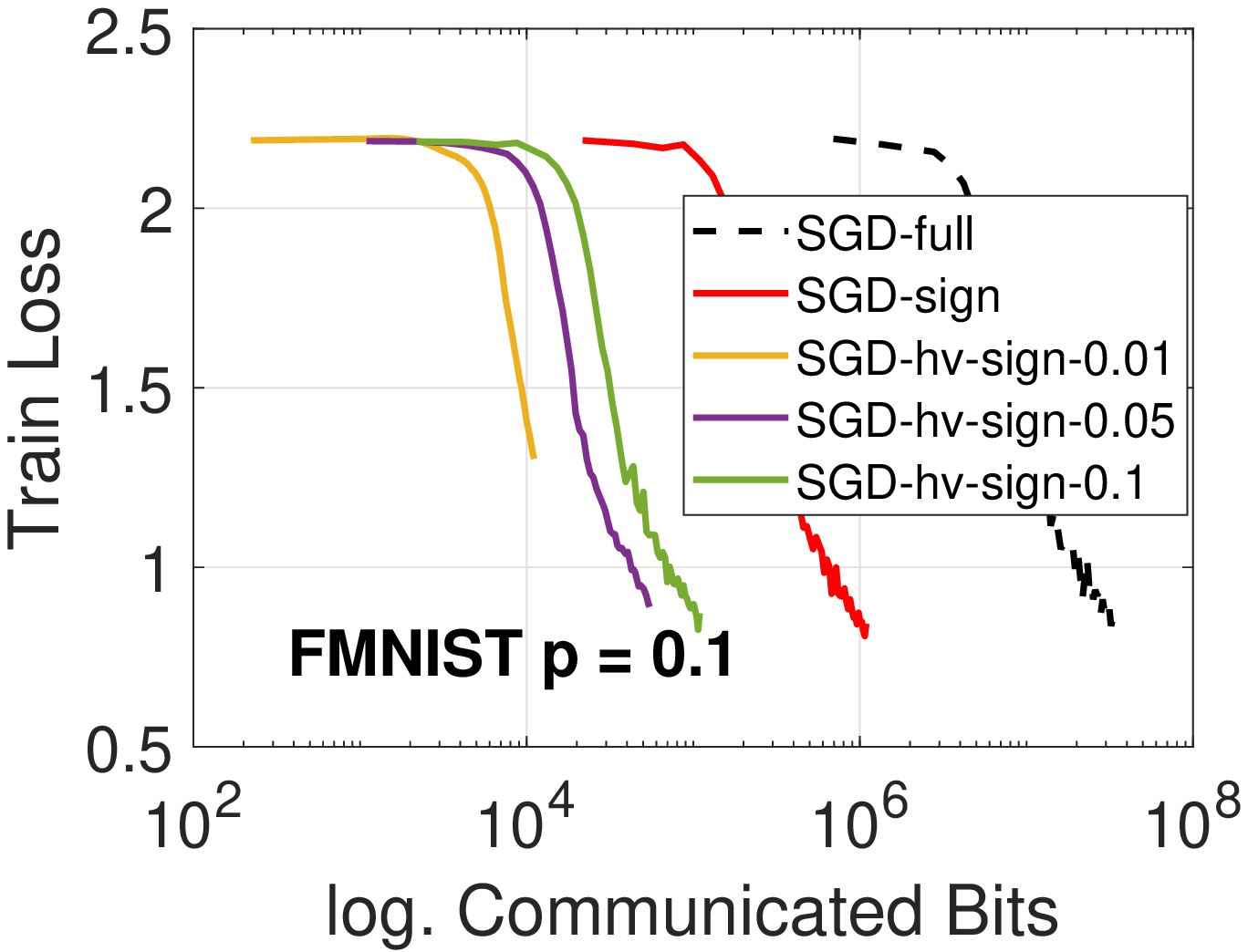}\hspace{-0.1in}
        \includegraphics[width=2.25in]{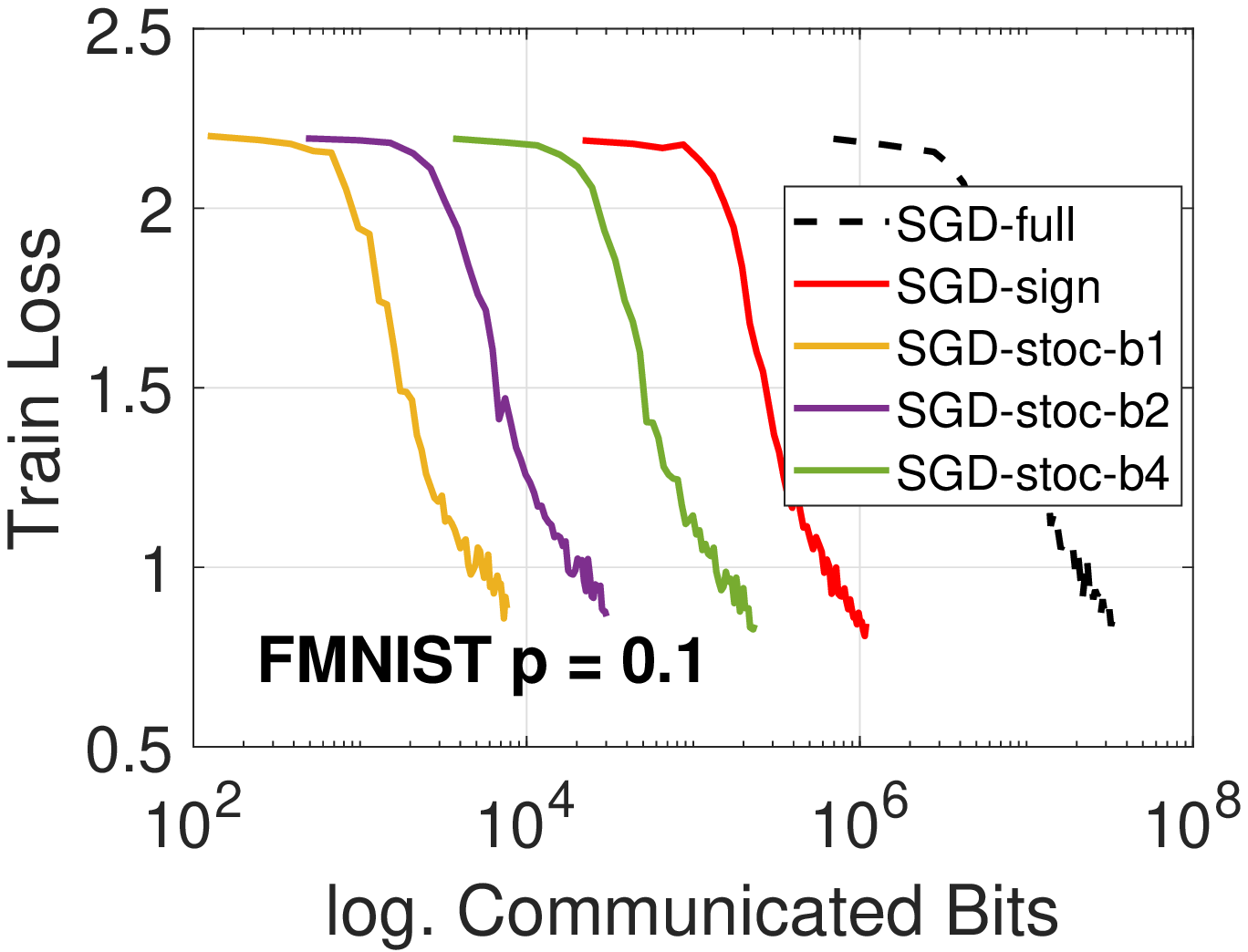}
        }
        \mbox{\hspace{-0.2in}
        \includegraphics[width=2.25in]{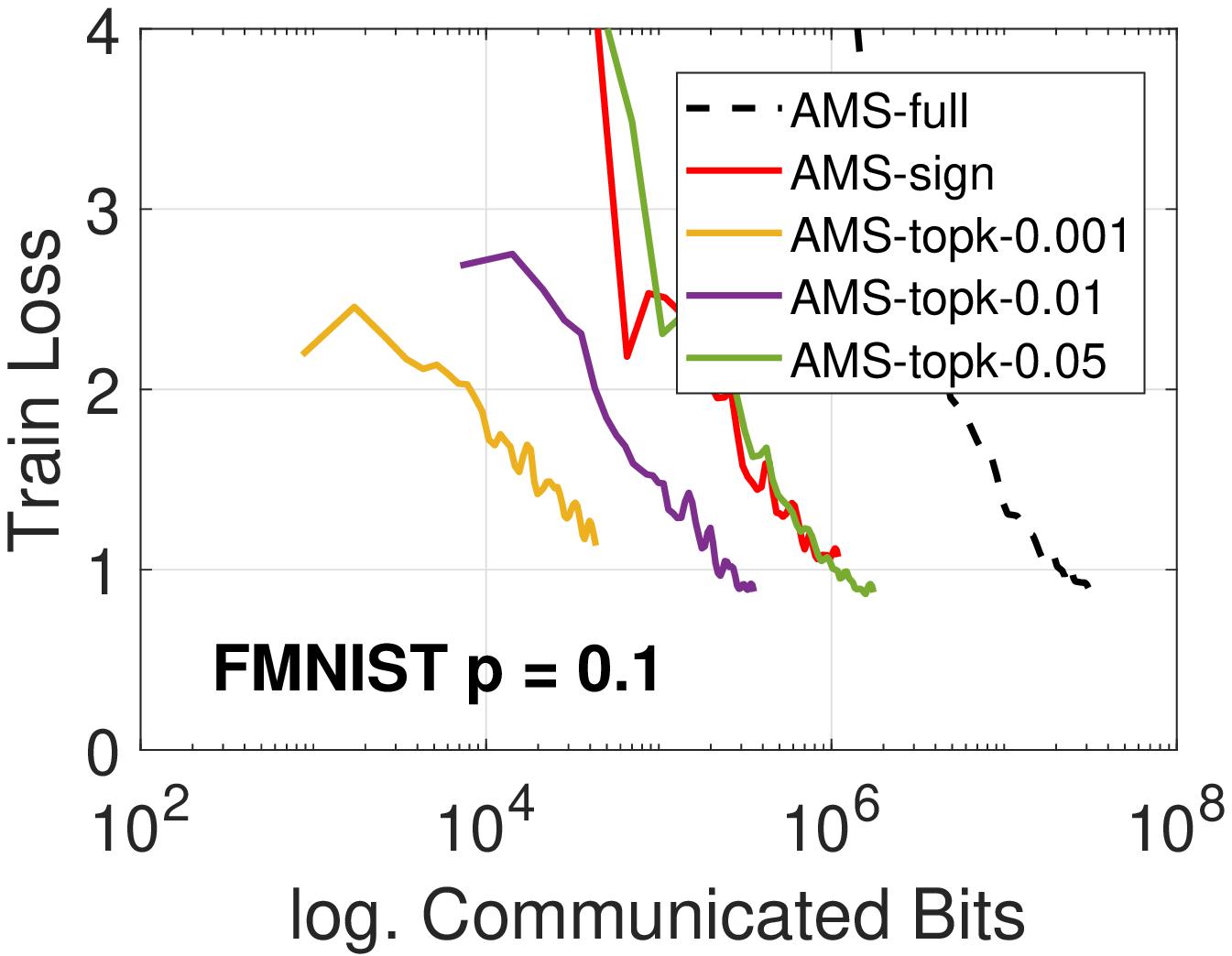}\hspace{-0.1in}
        \includegraphics[width=2.25in]{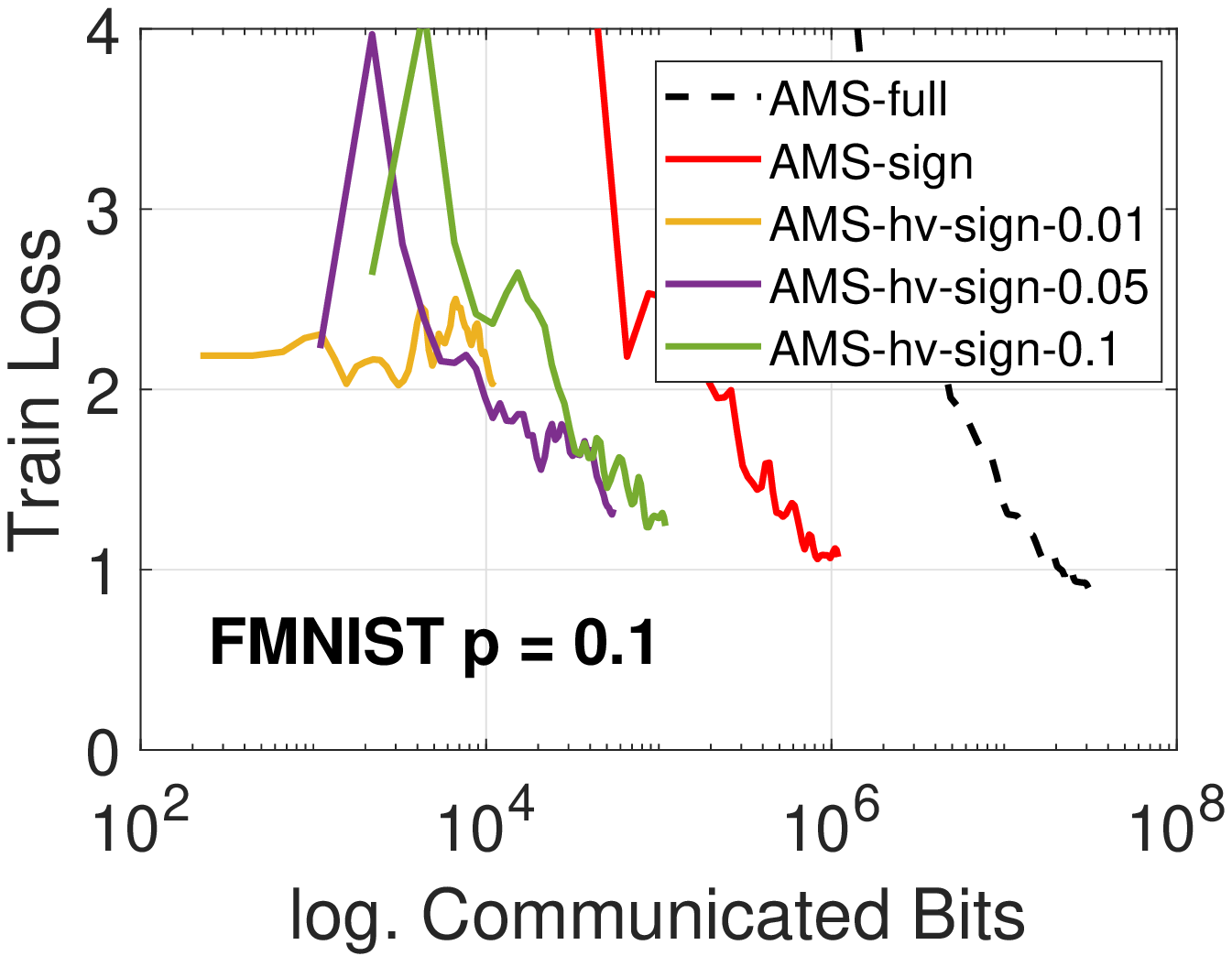}\hspace{-0.1in}
        \includegraphics[width=2.25in]{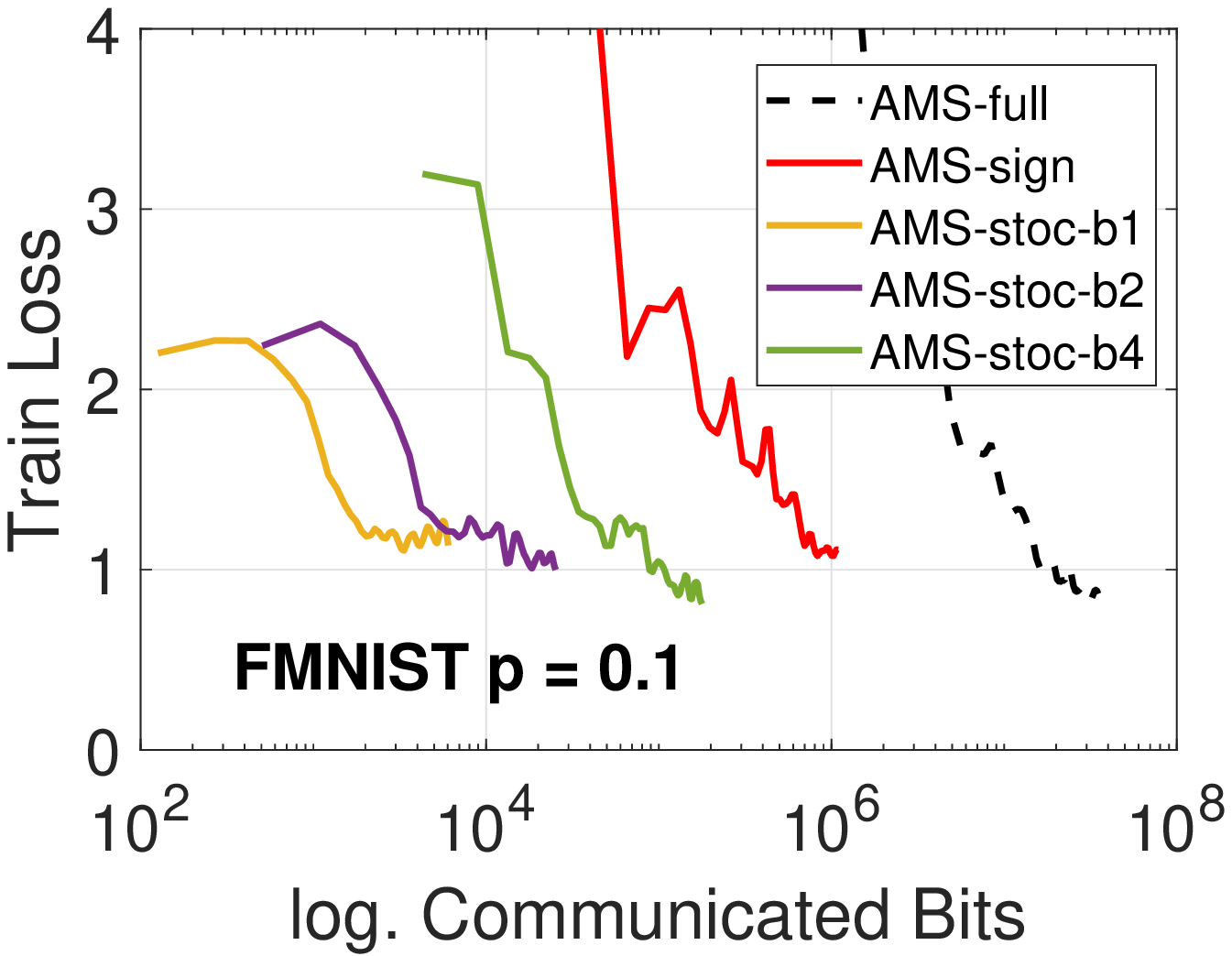}
        }
    \end{center}
    \vspace{-0.1in}
	\caption{Training loss of Fed-EF on FMNIST dataset trained by CNN. ``sign'', ``topk'' and ``hv-sign'' are applied with Fed-EF, while ``Stoc'' is the stochastic quantization without EF. Participation rate $p=0.1$, non-iid data. 1st row: Fed-EF-SGD. 2nd row: Fed-EF-AMS.}
	\label{fig:FMNIST-loss-compressor-0.1}
\end{figure}

\begin{figure}[t]
    \begin{center}
        \mbox{\hspace{-0.2in}
        \includegraphics[width=2.25in]{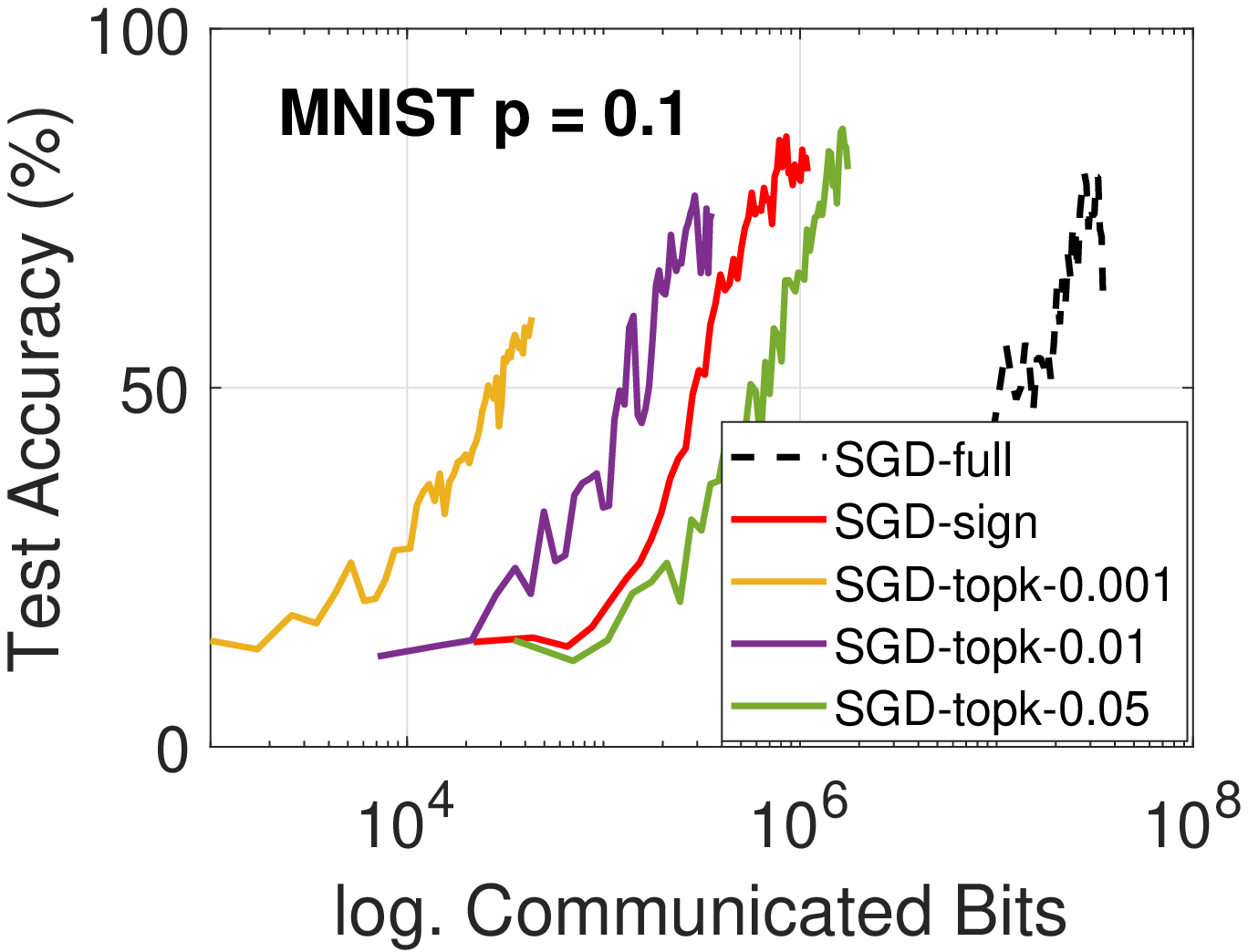}\hspace{-0.1in}
        \includegraphics[width=2.25in]{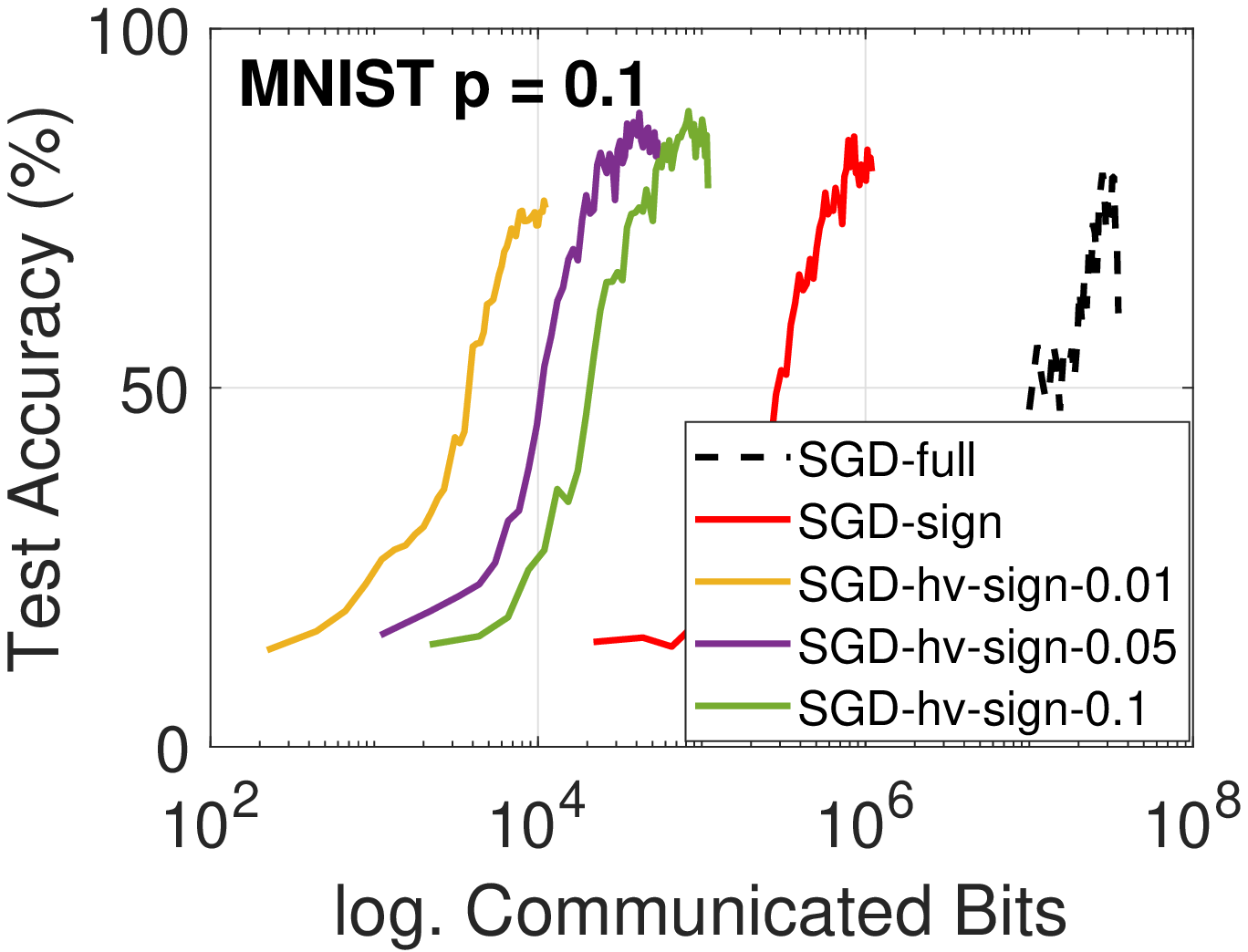}\hspace{-0.1in}
        \includegraphics[width=2.25in]{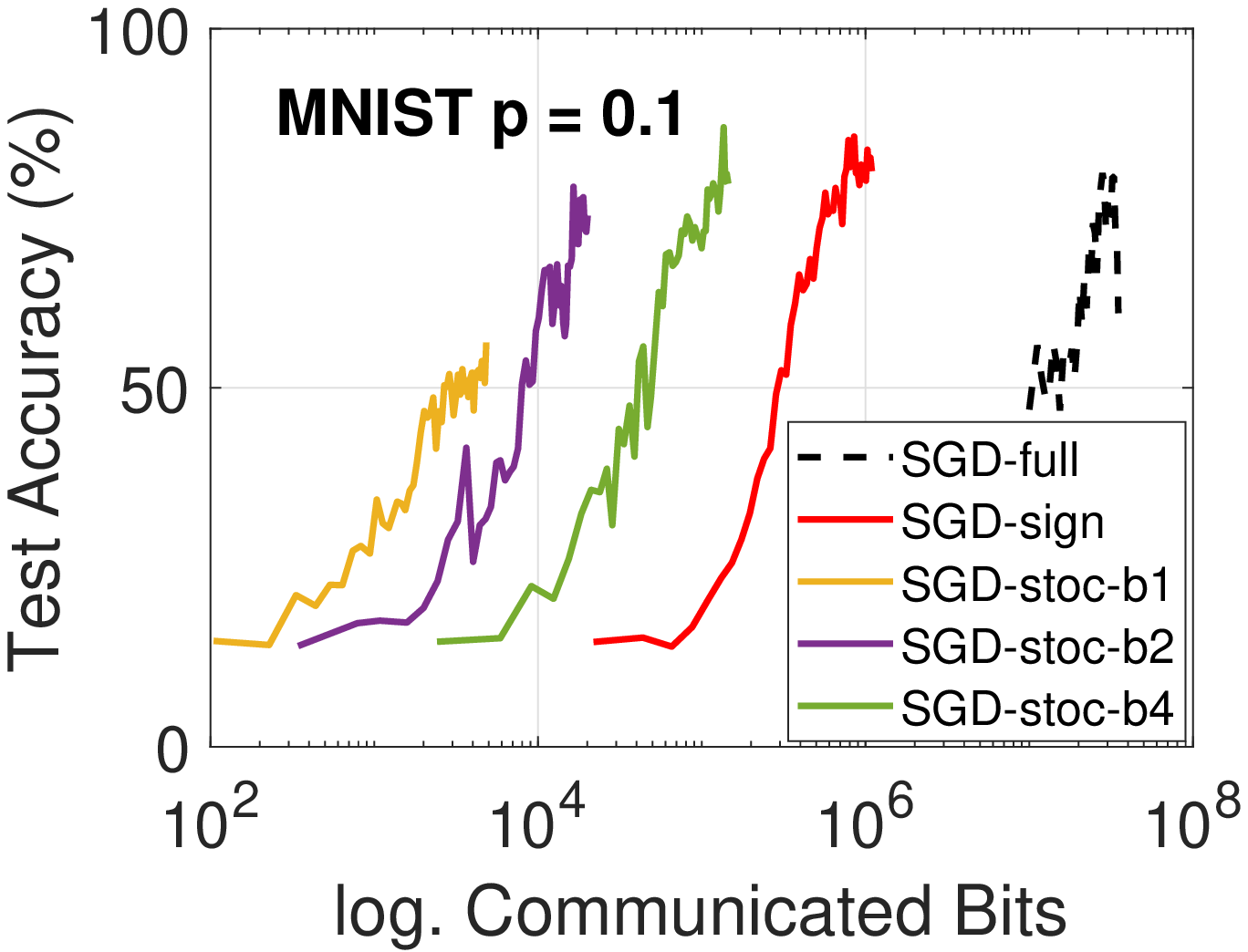}
        }
        \mbox{\hspace{-0.2in}
        \includegraphics[width=2.25in]{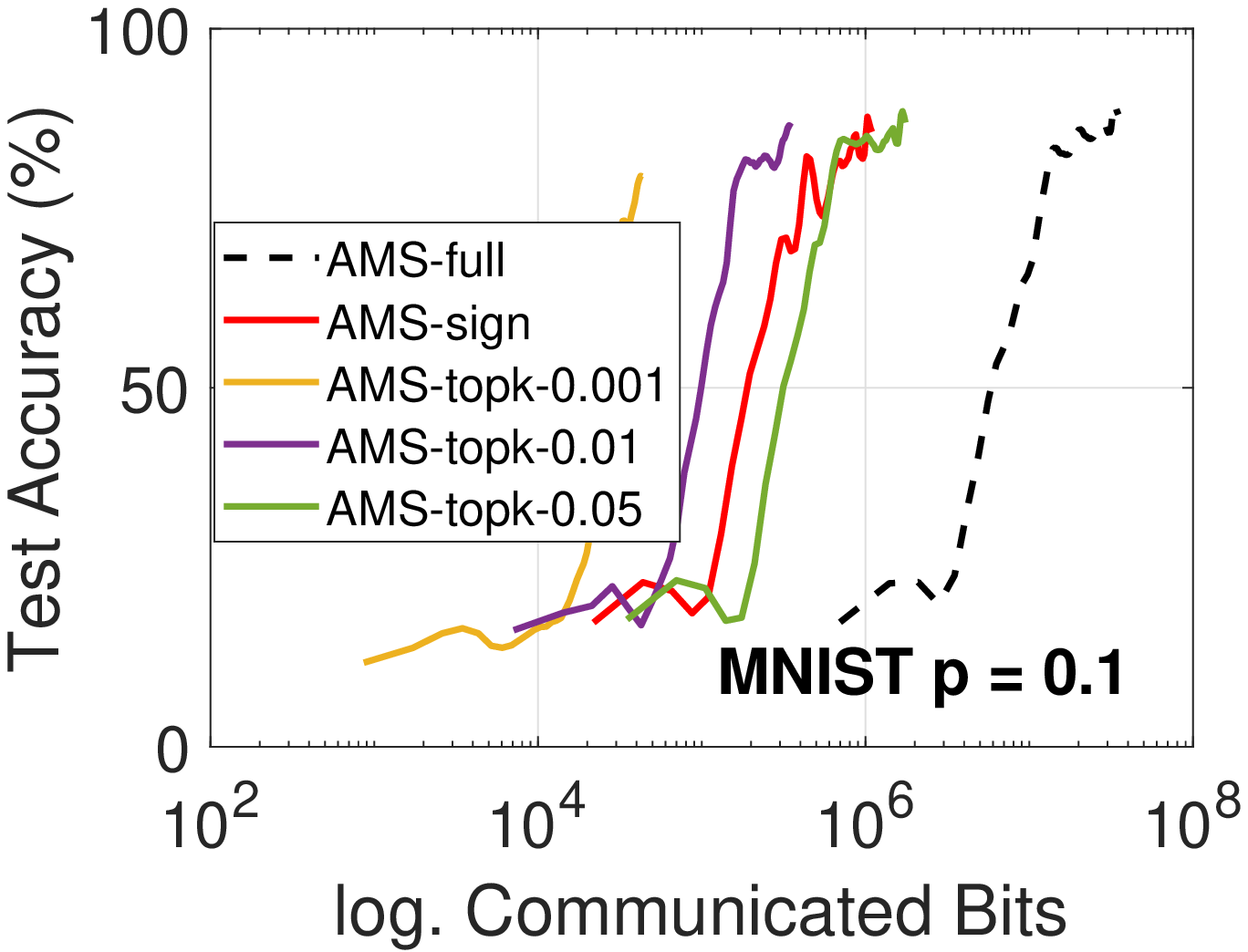}\hspace{-0.1in}
        \includegraphics[width=2.25in]{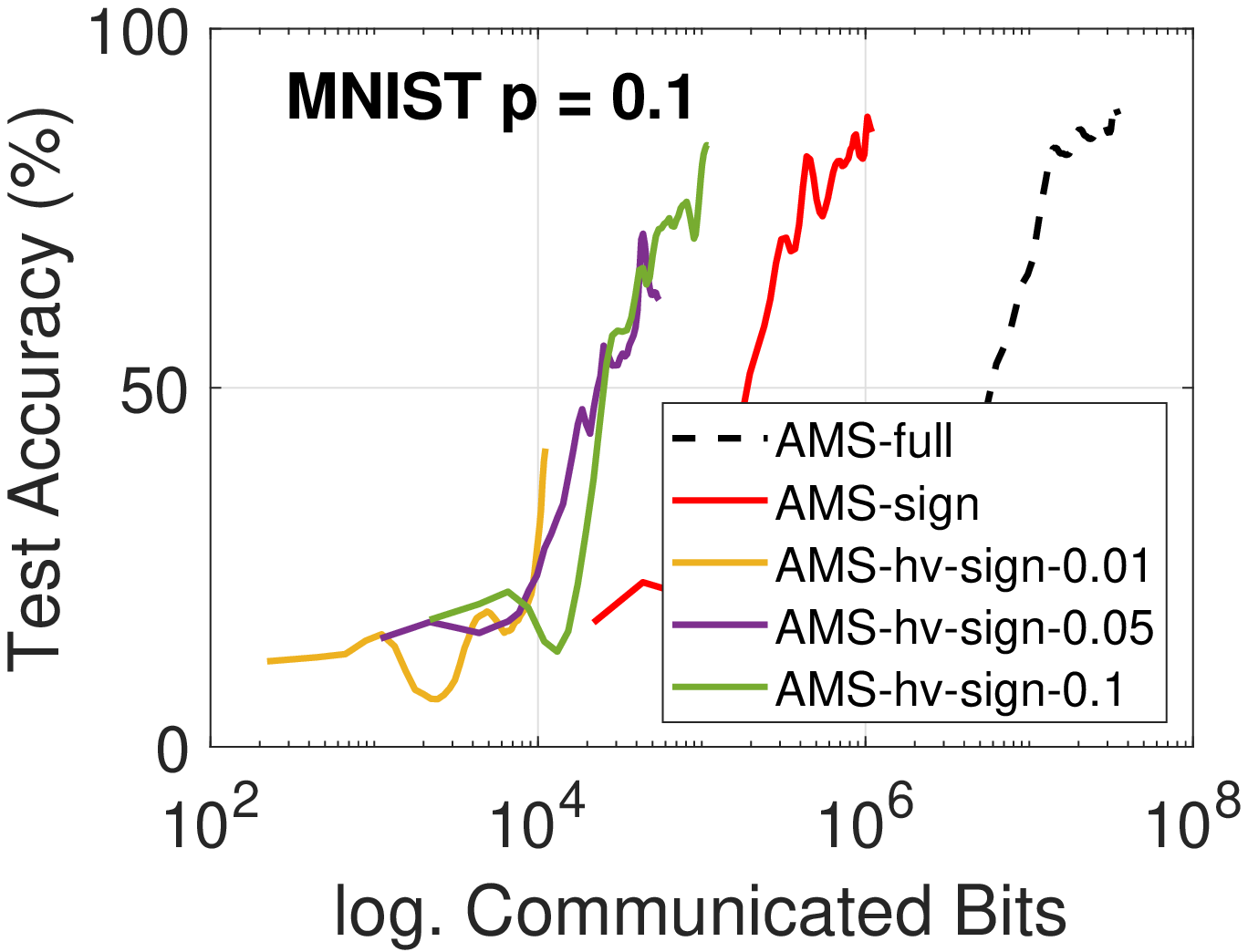}\hspace{-0.1in}
        \includegraphics[width=2.25in]{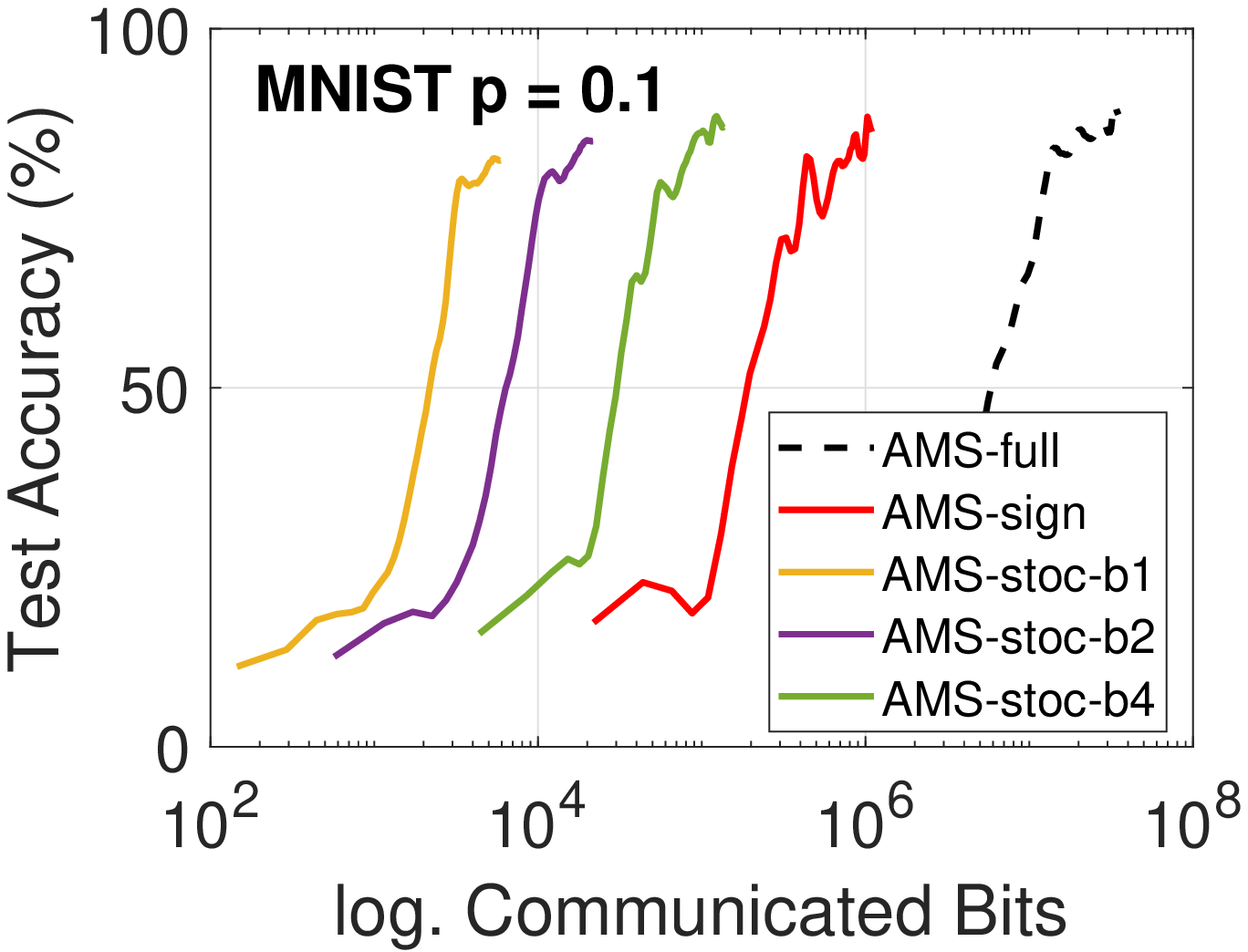}
        }
    \end{center}
    \vspace{-0.1in}
	\caption{Test accuracy of Fed-EF on MNIST dataset trained by CNN. ``sign'', ``topk'' and ``hv-sign'' are applied with Fed-EF, while ``Stoc'' is the stochastic quantization without EF. Participation rate $p=0.1$, non-iid data. 1st row: Fed-EF-SGD. 2nd row: Fed-EF-AMS.}
	\label{fig:MNIST-acc-compressor-0.1}\vspace{0.1in}
\end{figure}


\begin{figure}[h]
    \begin{center}
        \mbox{\hspace{-0.2in}
        \includegraphics[width=2.25in]{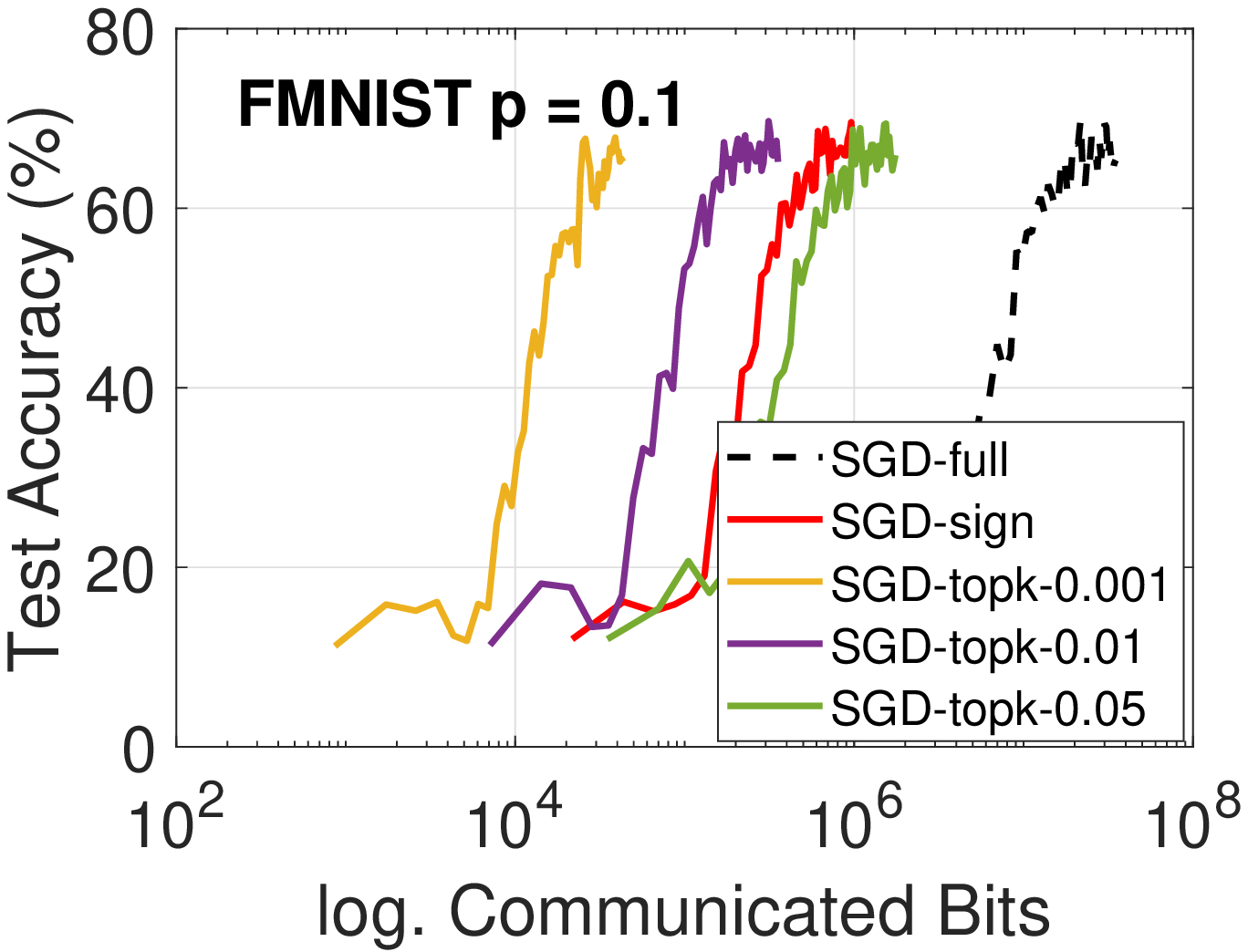}\hspace{-0.1in}
        \includegraphics[width=2.25in]{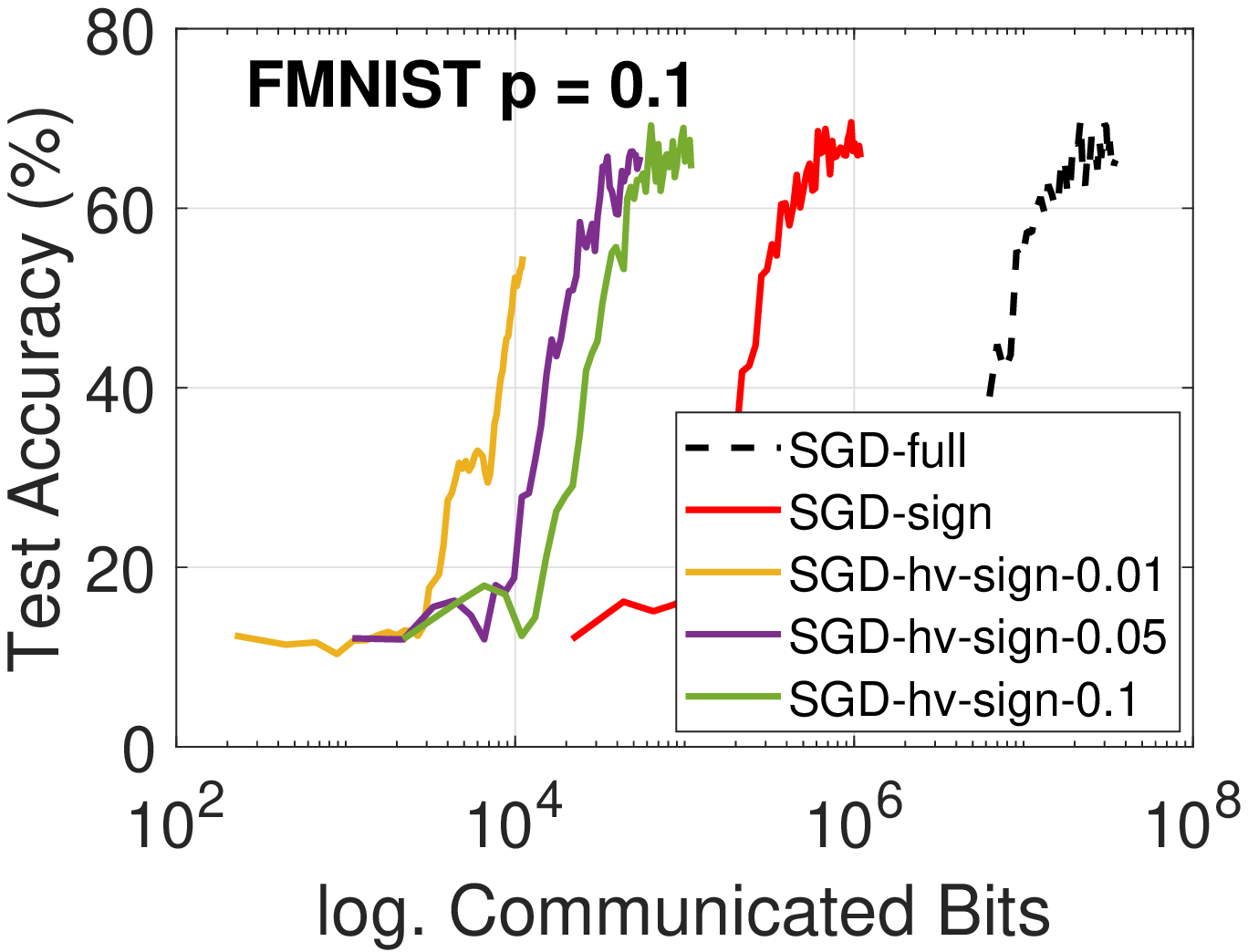}\hspace{-0.1in}
        \includegraphics[width=2.25in]{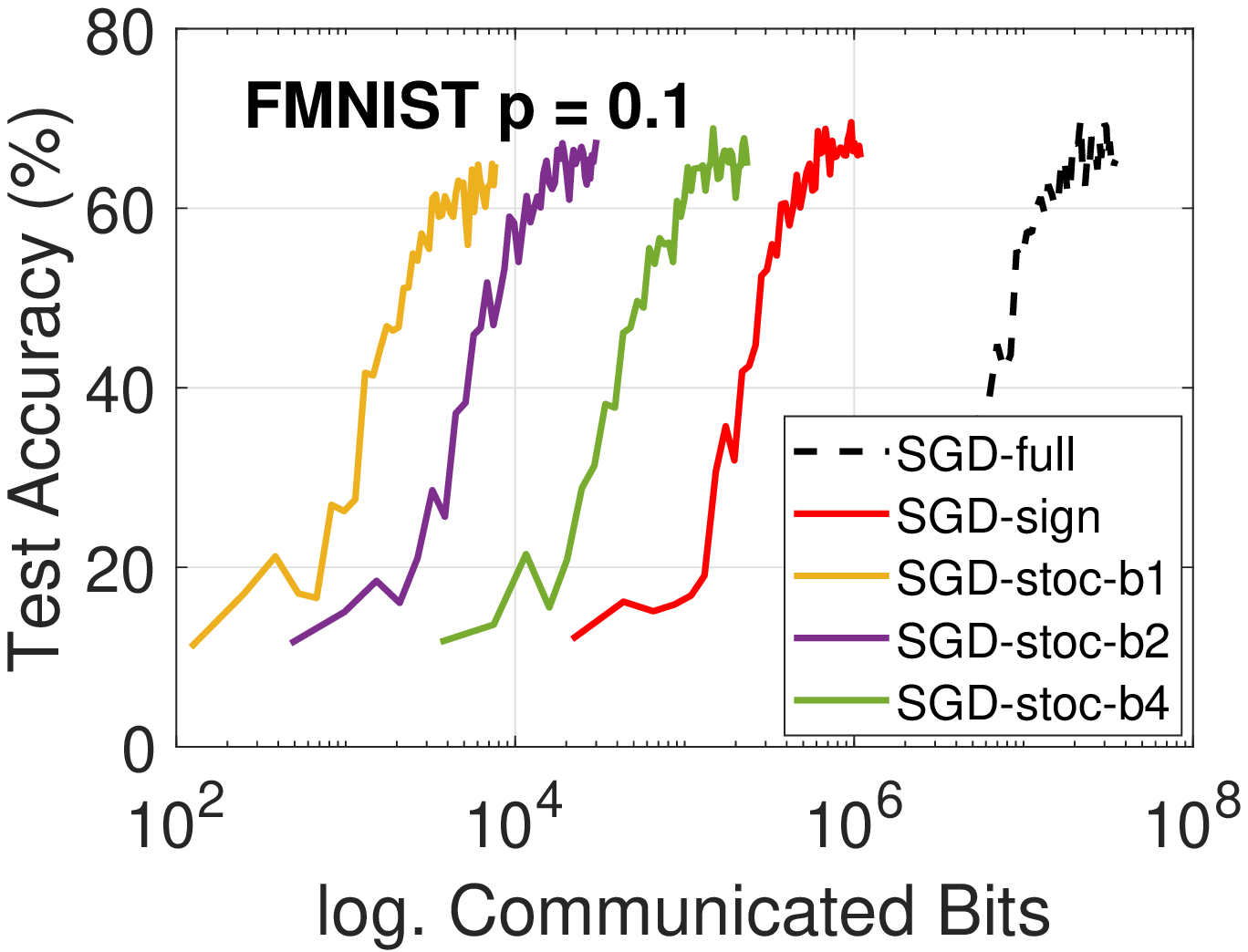}
        }
        \mbox{\hspace{-0.2in}
        \includegraphics[width=2.25in]{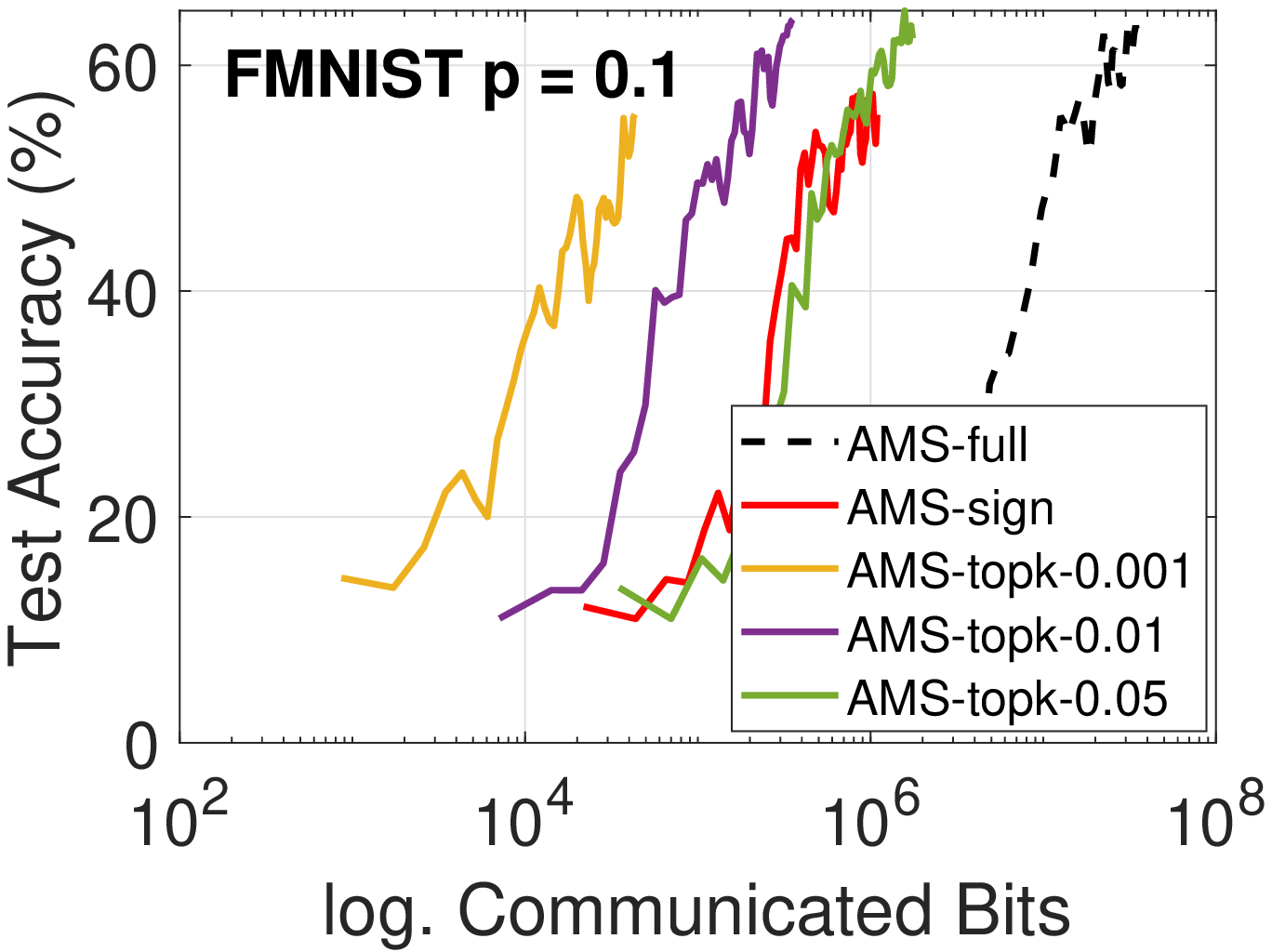}\hspace{-0.1in}
        \includegraphics[width=2.25in]{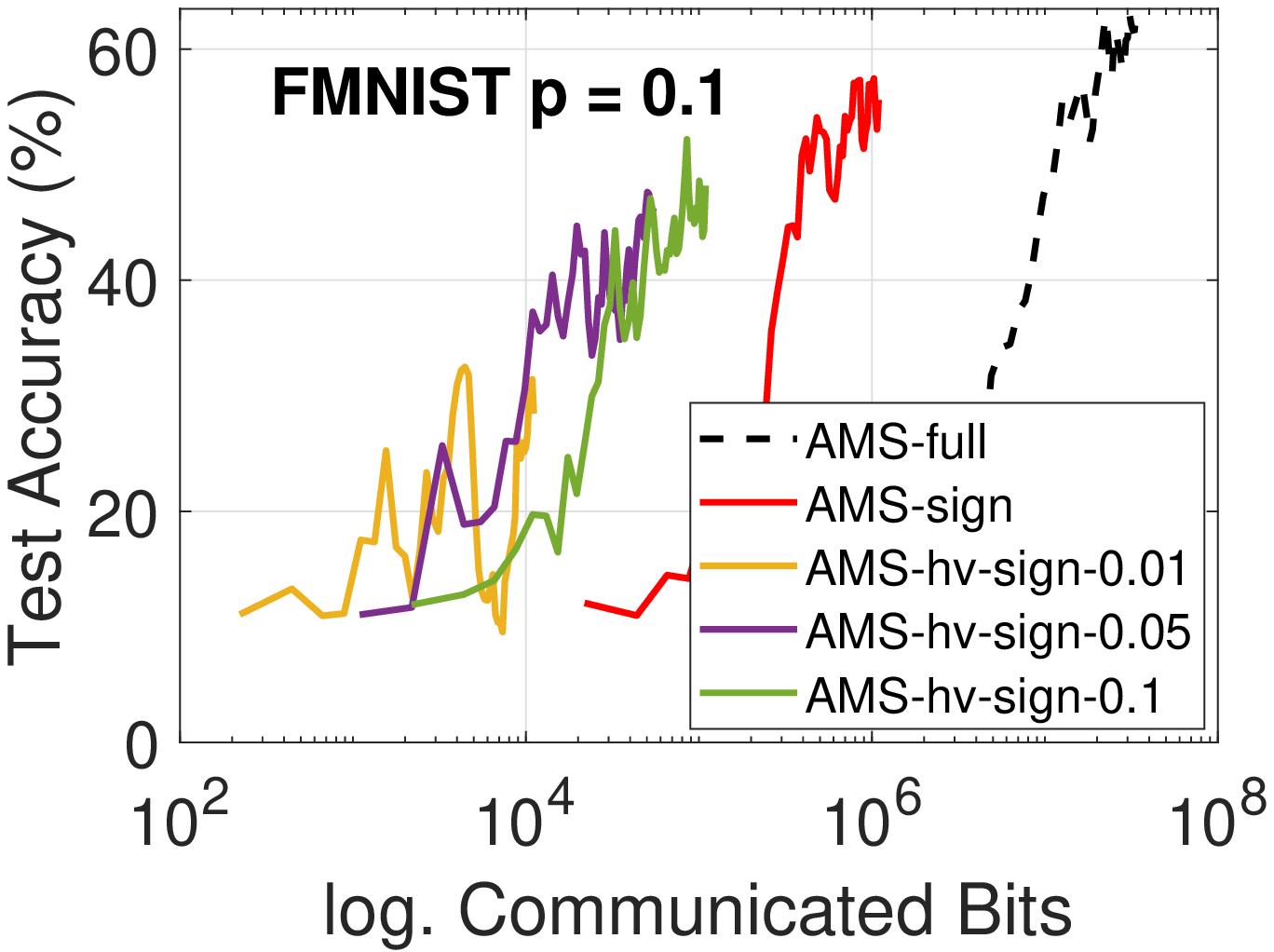}\hspace{-0.1in}
        \includegraphics[width=2.25in]{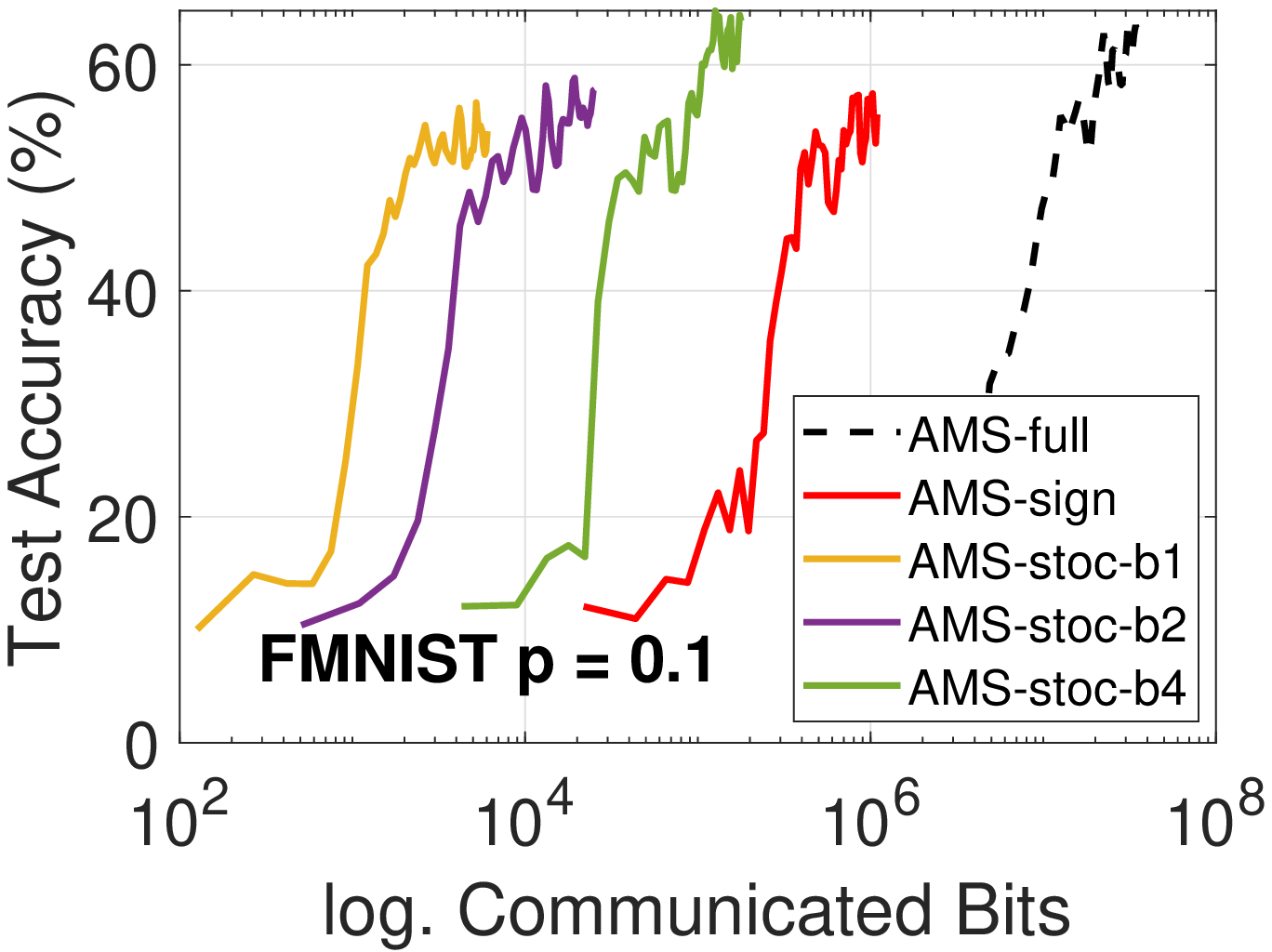}
        }
    \end{center}
    \vspace{-0.1in}
	\caption{Test accuracy of Fed-EF on FMNIST dataset trained by CNN. ``sign'', ``topk'' and ``hv-sign'' are applied with Fed-EF, while ``Stoc'' is the stochastic quantization without EF. Participation rate $p=0.1$, non-iid data. 1st row: Fed-EF-SGD. 2nd row: Fed-EF-AMS.}
	\label{fig:FMNIST-acc-compressor-0.1}
\end{figure}

\newpage
\clearpage

\section{Compression Discrepancy}  \label{sec:discuss_assumption}

In our theoretical analysis for Fed-EF, Assumption~\ref{ass:compress_diff} is needed, which states that $\mathbb E[\| \frac{1}{n}\sum_{i=1}^n \mathcal C\big(\del_{t,i}+e_{t,i}\big)-\frac{1}{n}\sum_{i=1}^n (\del_{t,i}+e_{t,i}) \|^2]\leq q_{\mathcal A}^2 \mathbb E[\| \frac{1}{n}\sum_{i=1}^n (\del_{t,i}+e_{t,i}) \|^2]$ for some $q_{\mathcal A}<1$ during training. In the following, we justify this assumption to demonstrate how it holds in practice. It is worth mentioning that a similar condition is was assumed in \cite{haddadpour2021federated} for the analysis of FL with unbiased compression. To study sparsified SGD, \cite{alistarh2018convergence} also used a similar and stronger (uniform bound instead of in expectation) analytical assumption. As a result, our analysis and theoretical results are also valid under their assumption.

\subsection{Simulated Data}

We first conduct a simulation to investigate how the two compressors, \textbf{TopK} and \textbf{Sign}, affect $q_{\mathcal A}$. In our presented results, for conciseness we use $n=5$ clients and model dimensionality $d=1100$. Similar conclusions hold for much larger $n$ and $d$. We simulate two types of gradients following normal distribution and Laplace distribution (more heavy-tailed), respectively. Examples of the simulated gradients are visualized in Figure~\ref{fig:normal grad} and Figure~\ref{fig:laplace grad}. To mimic non-iid data, we assume that each client has some strong signals (large gradients) in some coordinates, and we scale those gradients by a scaling factor $s=2,10,100$. Conceptually, larger $s$ represents higher data heterogeneity.

\begin{figure}[h]
    \begin{center}
        \mbox{\hspace{-0.18in}
        \includegraphics[width=2in]{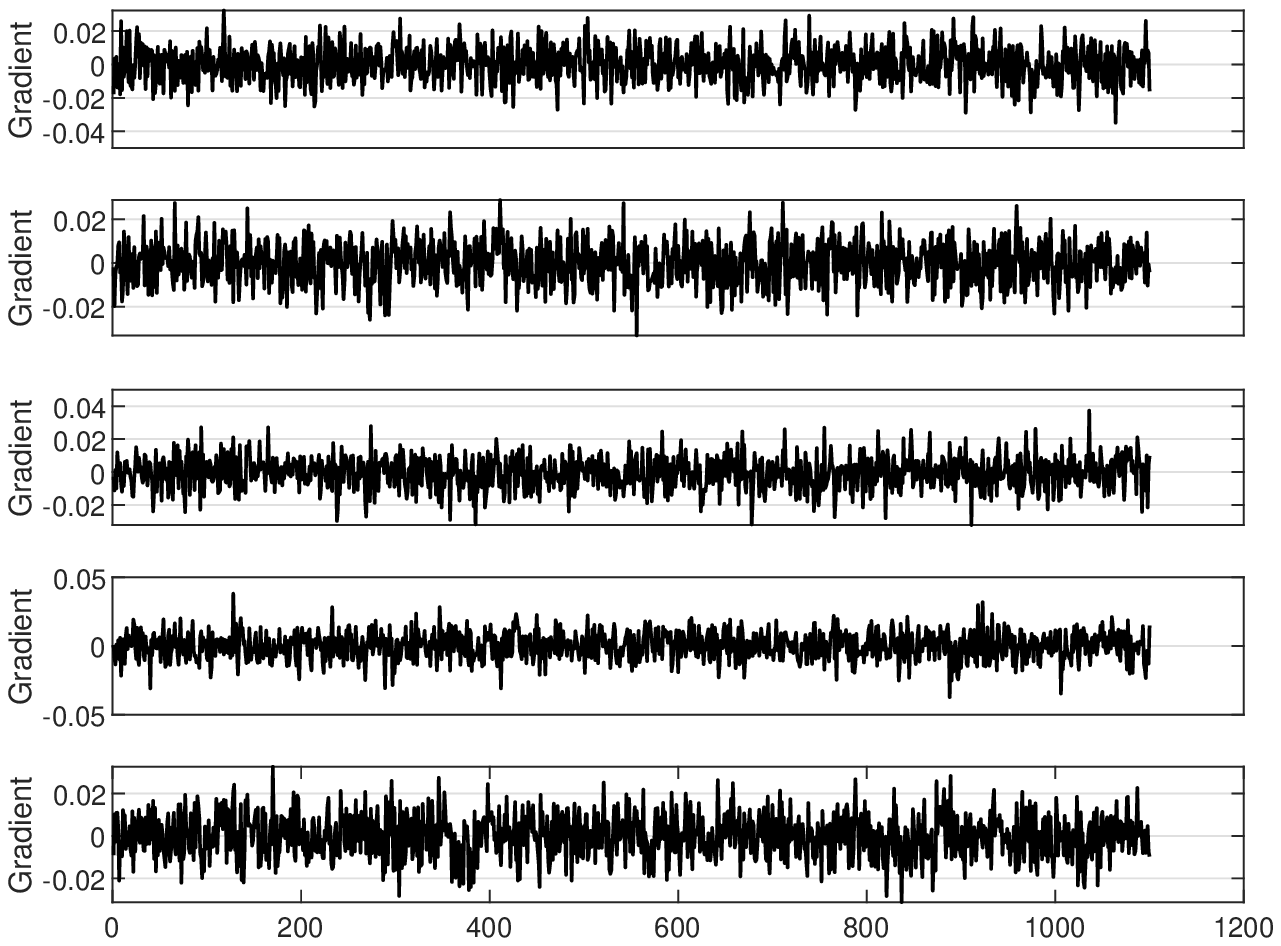}\hspace{-0.1in}
        \includegraphics[width=2in]{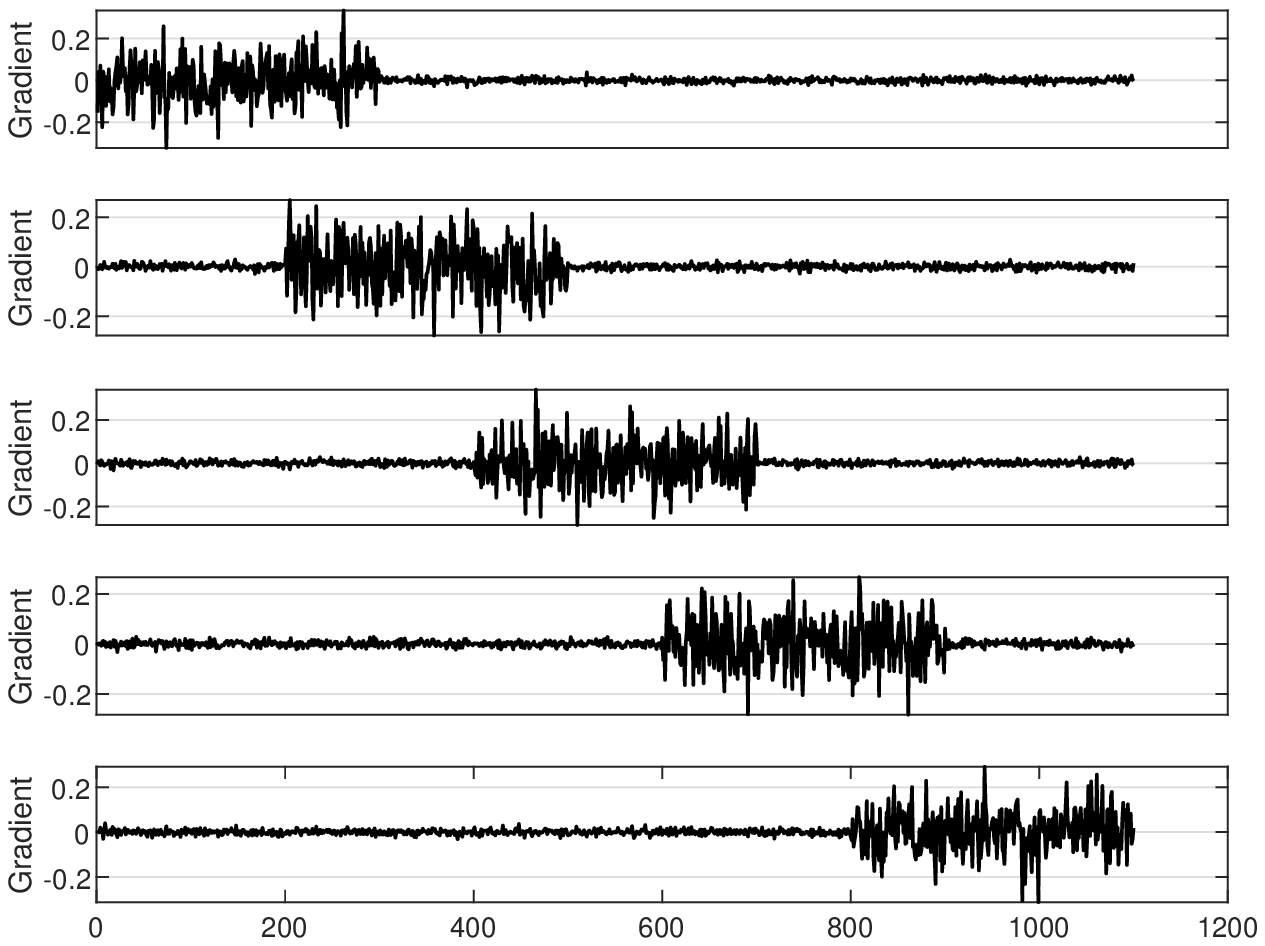}\hspace{-0.1in}
        \includegraphics[width=2in]{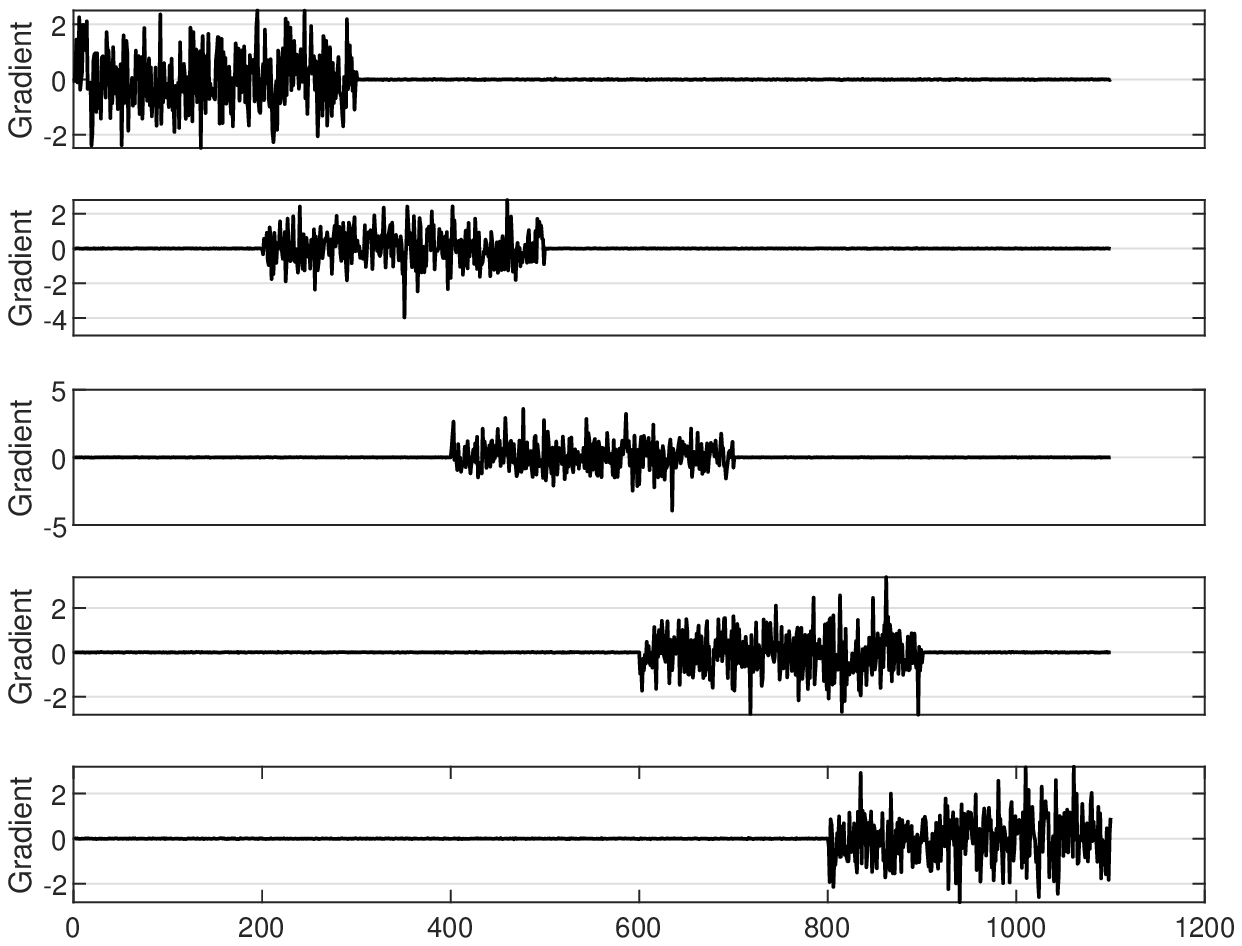}
        }
    \end{center}

	\caption{The simulated gradients of 5 heterogeneous clients, from $N(0,\gamma^2)$ with $\gamma=0.01$. The gradient on each distinct client is scaled by $s=2,10,100$ (left, mid, right), respectively. Larger $s$ implies higher data heterogeneity.}
	\label{fig:normal grad}
\end{figure}

\begin{figure}[h]
    \begin{center}
        \mbox{\hspace{-0.18in}
        \includegraphics[width=2in]{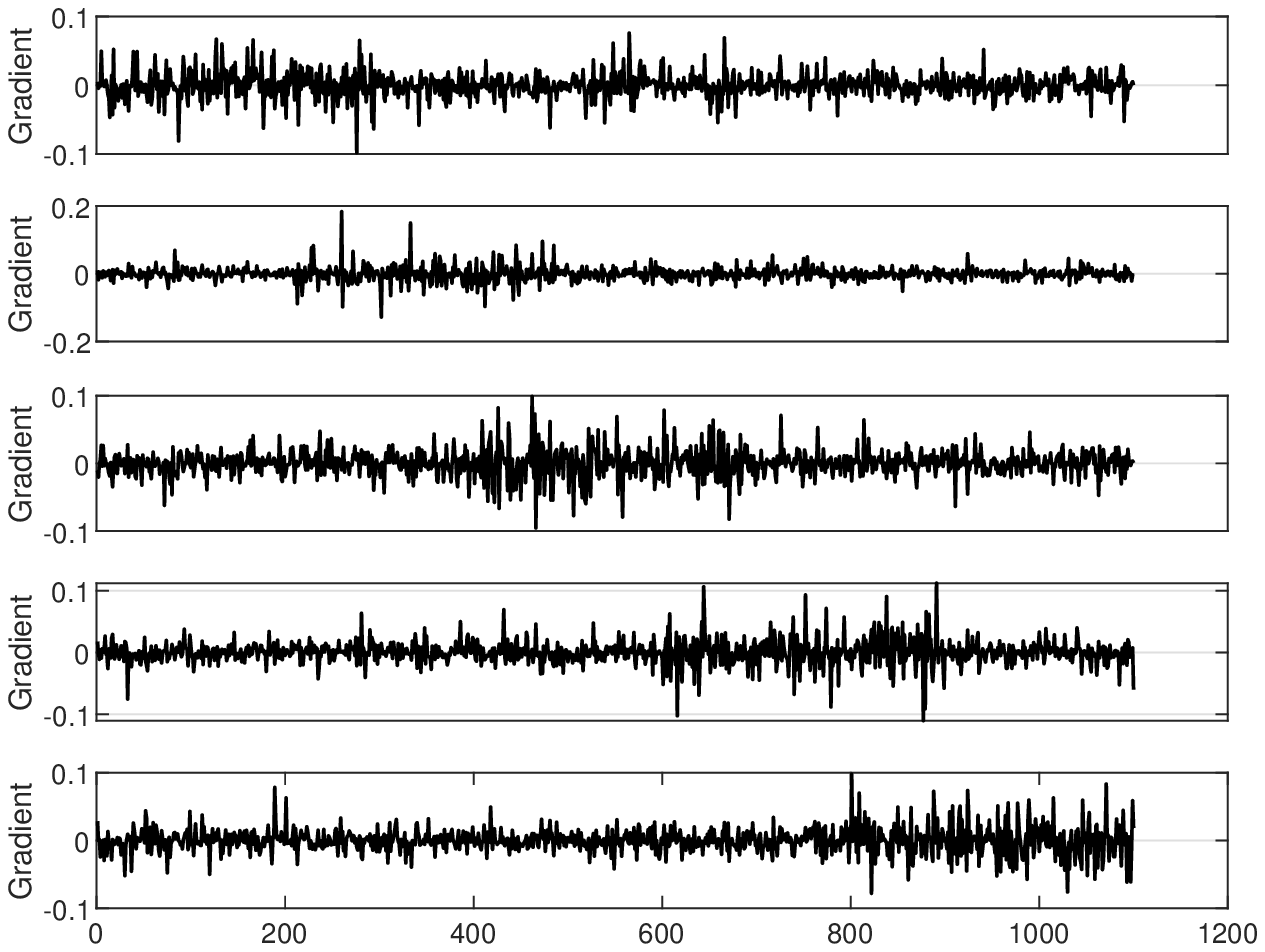}\hspace{-0.1in}
        \includegraphics[width=2in]{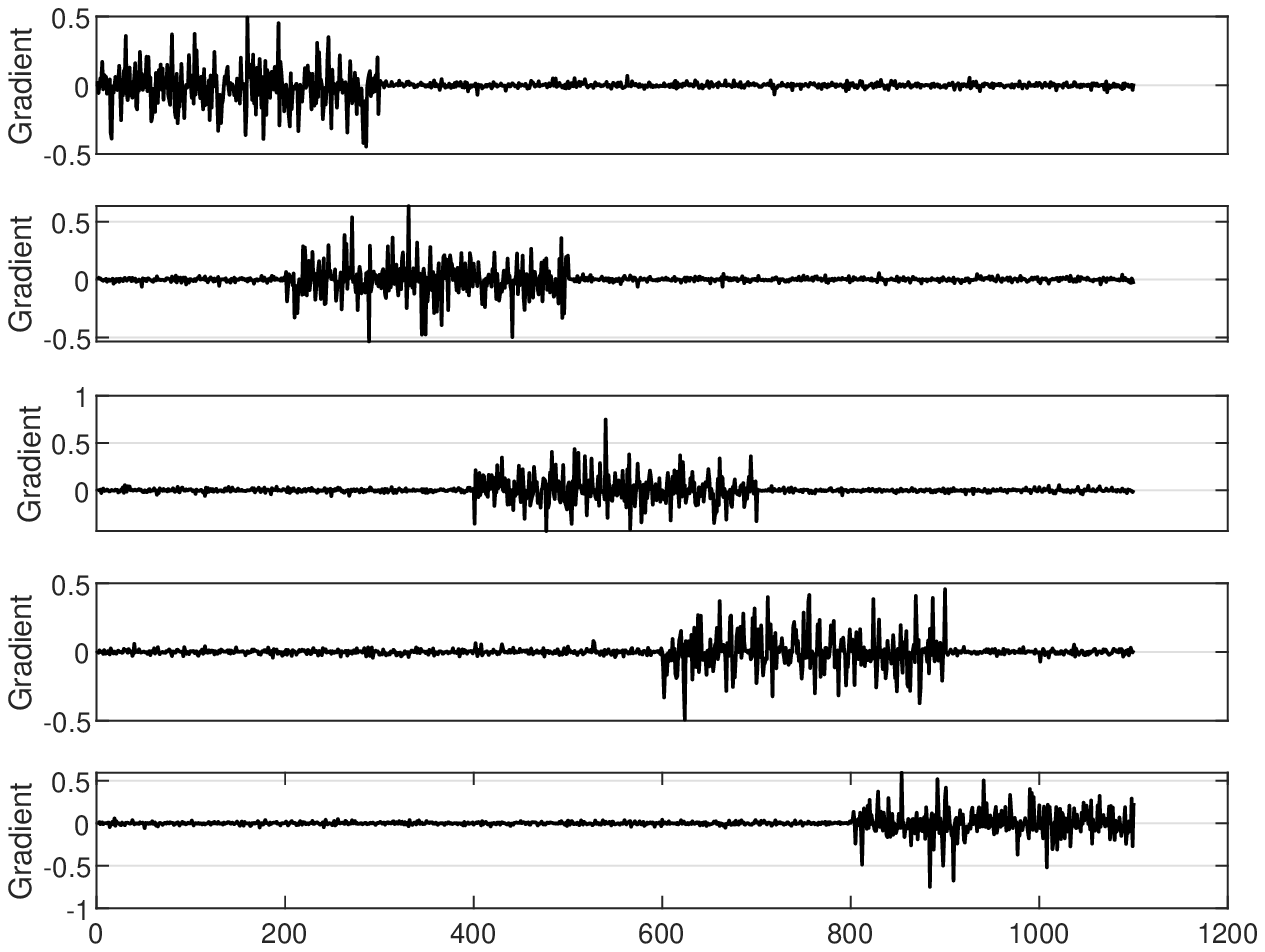}\hspace{-0.1in}
        \includegraphics[width=2in]{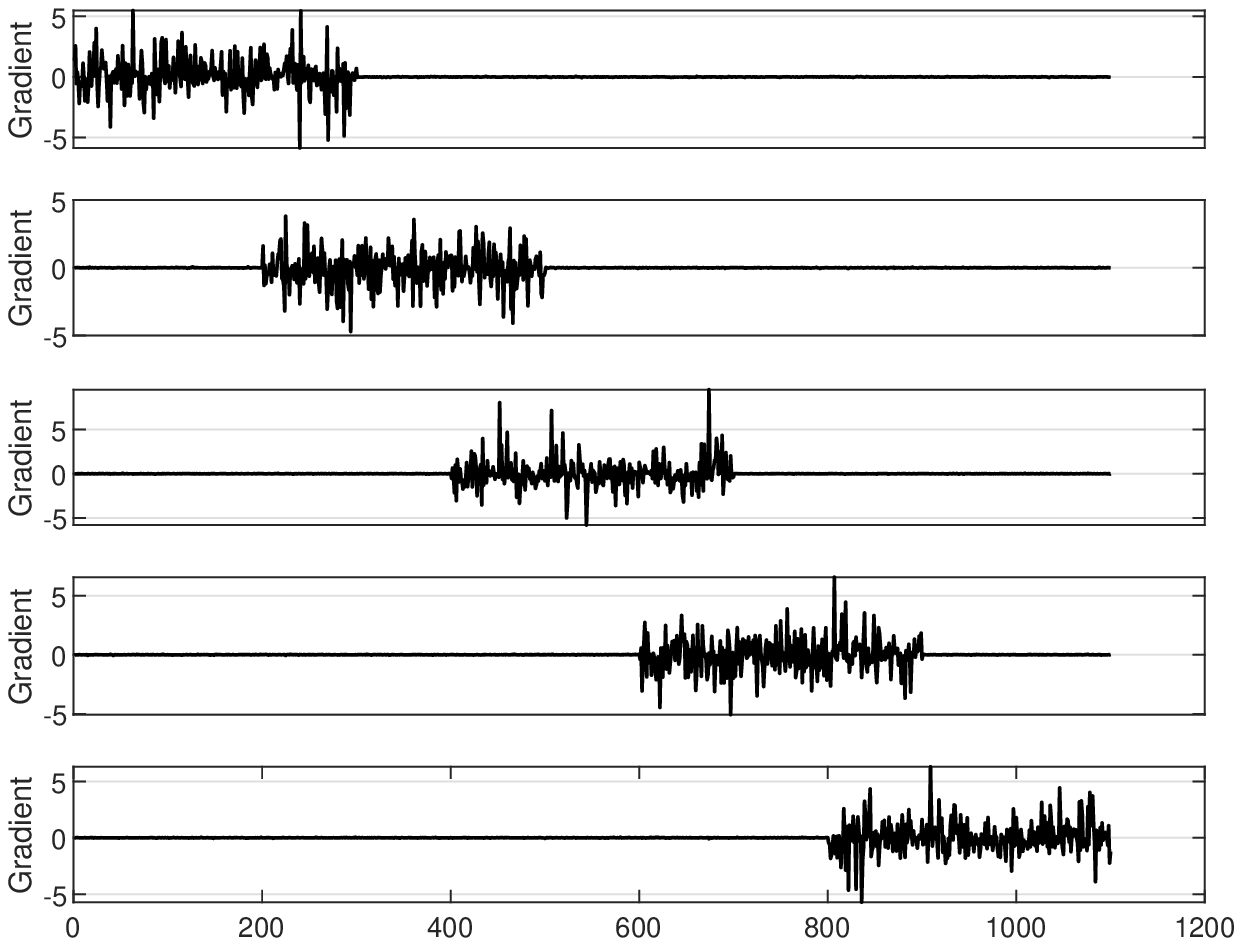}
        }
    \end{center}

	\caption{The simulated gradients of 5 heterogeneous clients, from $Lap(0,\lambda)$ with $\lambda=0.01$. The $x$-axis is the dimension. The gradient on each distinct client is scaled by $s=2,10,100$ (left, mid, right), respectively. Wee the gradients have heavier tail than normal distribution. Larger $s$ implies higher data heterogeneity.}
	\label{fig:laplace grad}
\end{figure}

\newpage

We apply the \textbf{TopK-0.1} and \textbf{Sign} compressor in Definition~\ref{def:quant} to the simulated gradients, and compute the averaged $q_{\mathcal A}^2$ in Figure~\ref{fig:simulate q} over $10^5$ independent runs. The dashed curves are respectively the ``ideal'' compression coefficients $q_\mathcal C$ such that $\mathbb E[\| \mathcal C\big(\frac{1}{n}\sum_{i=1}^n \del_{t,i}+e_{t,i}\big)-\frac{1}{n}\sum_{i=1}^n (\del_{t,i}+e_{t,i}) \|^2]\leq q_{\mathcal C}^2 \mathbb E[\| \frac{1}{n}\sum_{i=1}^n (\del_{t,i}+e_{t,i}) \|^2]$ from Definition~\ref{def:quant}. We see that in all cases, $q_{\mathcal A}$ is indeed less than 1. This still holds even when the data heterogeneity increases to as large as 100.

\begin{figure}[t]
  \begin{center}
   \mbox{
    \includegraphics[width=2.5in]{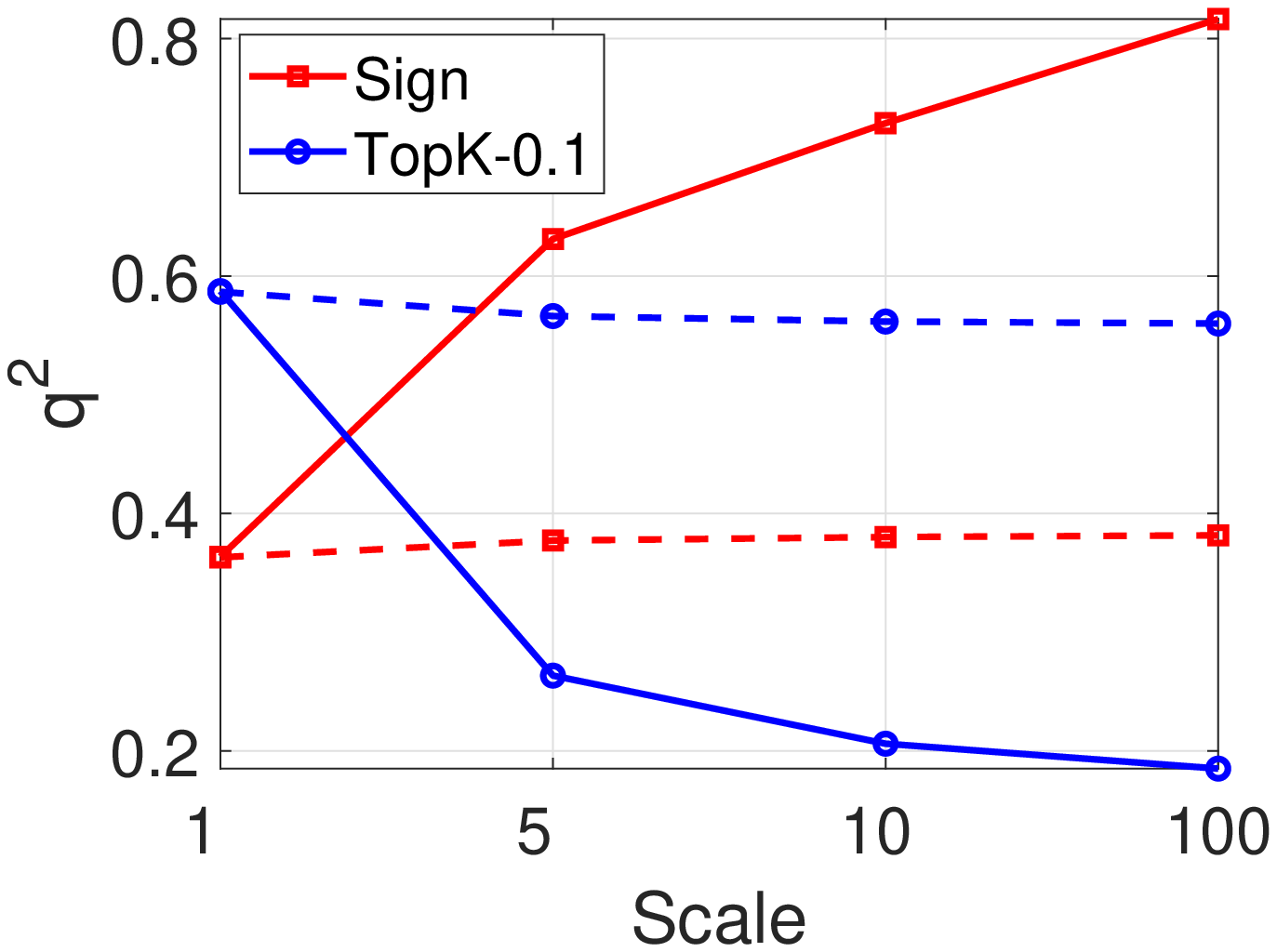}\hspace{0.2in}
    \includegraphics[width=2.5in]{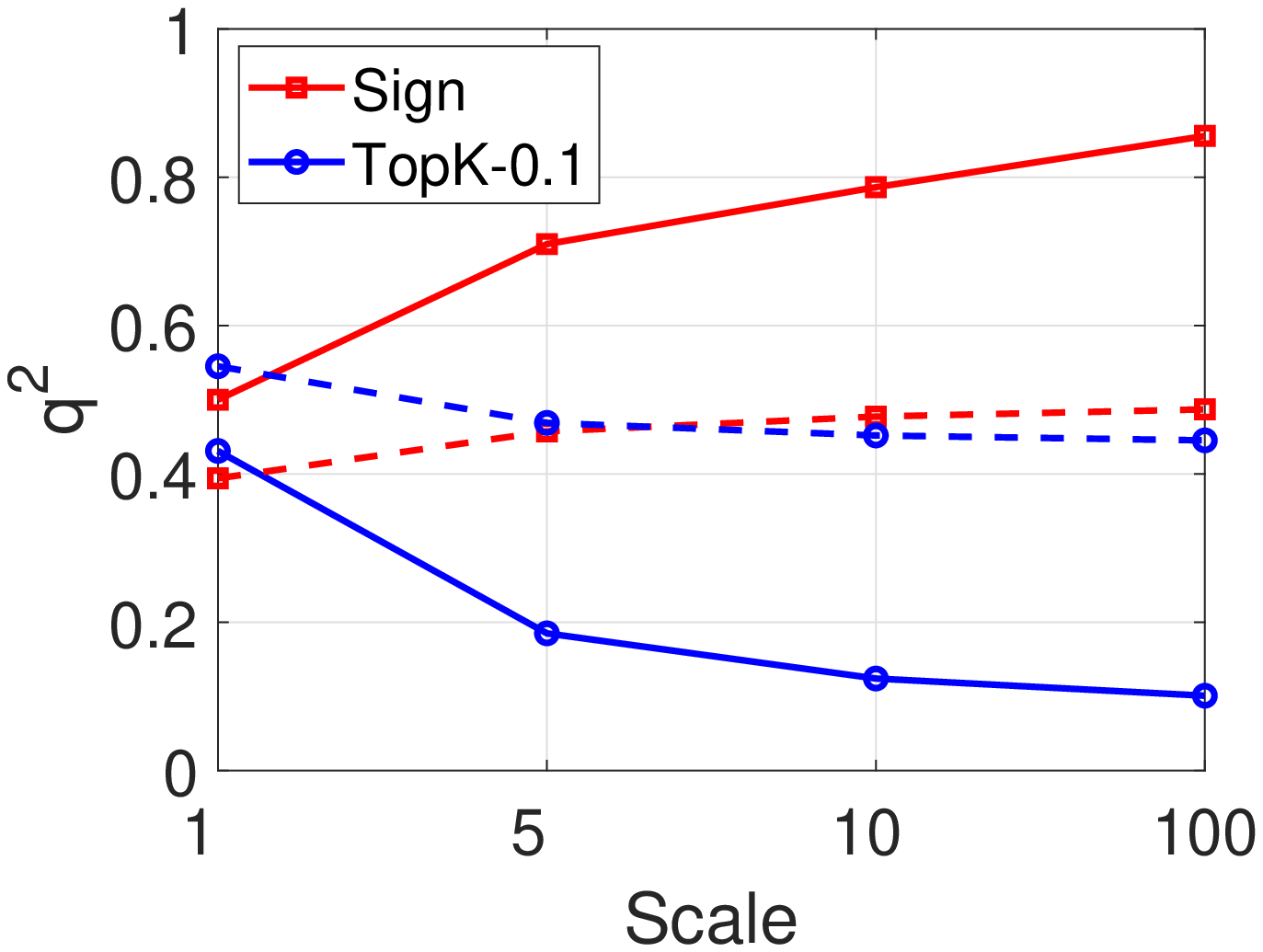}
    }
  \end{center}

\vspace{-0.2in}

  \caption{The compression coefficient $q_{\mathcal A}$ in Assumption~\ref{ass:compress_diff} on simulated gradients. \textbf{TopK} is applied with sparsity $k=0.1$. Left: Gaussian distribution. Right: Laplace distribution. $q_{\mathcal A}^2$ is computed by $q=\frac{\|\delta(x)-x\|^2}{\|x\|^2}$ where $\delta(x)=\frac{1}{n}\sum_{i=1}^n \mathcal C(\del_{t,i}+e_{t,i})$ and $x=\frac{1}{n}\sum_{i=1}^n (\del_{t,i}+e_{t,i})$. The dashed curves are respectively the compression coefficients $q_{\mathcal C}^2$ from Definition~\ref{def:quant}, which is calculated by replacing $\delta(x)=\mathcal C(\frac{1}{n}\sum_{i=1}^n \del_{t,i}+e_{t,i})$. We see that in all cases, $q_{\mathcal A}<1$.}
  \label{fig:simulate q}
\end{figure}

\subsection{Real-world Data}

We report the empirical $q_{\mathcal A}$ values when training CNN on MNIST and FMNIST datasets. The experimental setup is the same as in Section~\ref{sec:experiment}. We present the result in Figure~\ref{fig:cnn q} with $\eta=1$, $\eta_l=0.01$ under the same heterogeneous setting where client data are highly non-iid. The plots for other learning rate combinations and iid data are similar. In particular, we see for both compressors and both datasets, the empirical $q_{\mathcal A}$ is well-bounded below 1 throughout the training process.

\vspace{0.2in}

\begin{figure}[h]
  \begin{center}
   \mbox{\hspace{-0.2in}
    \includegraphics[width=2.3in]{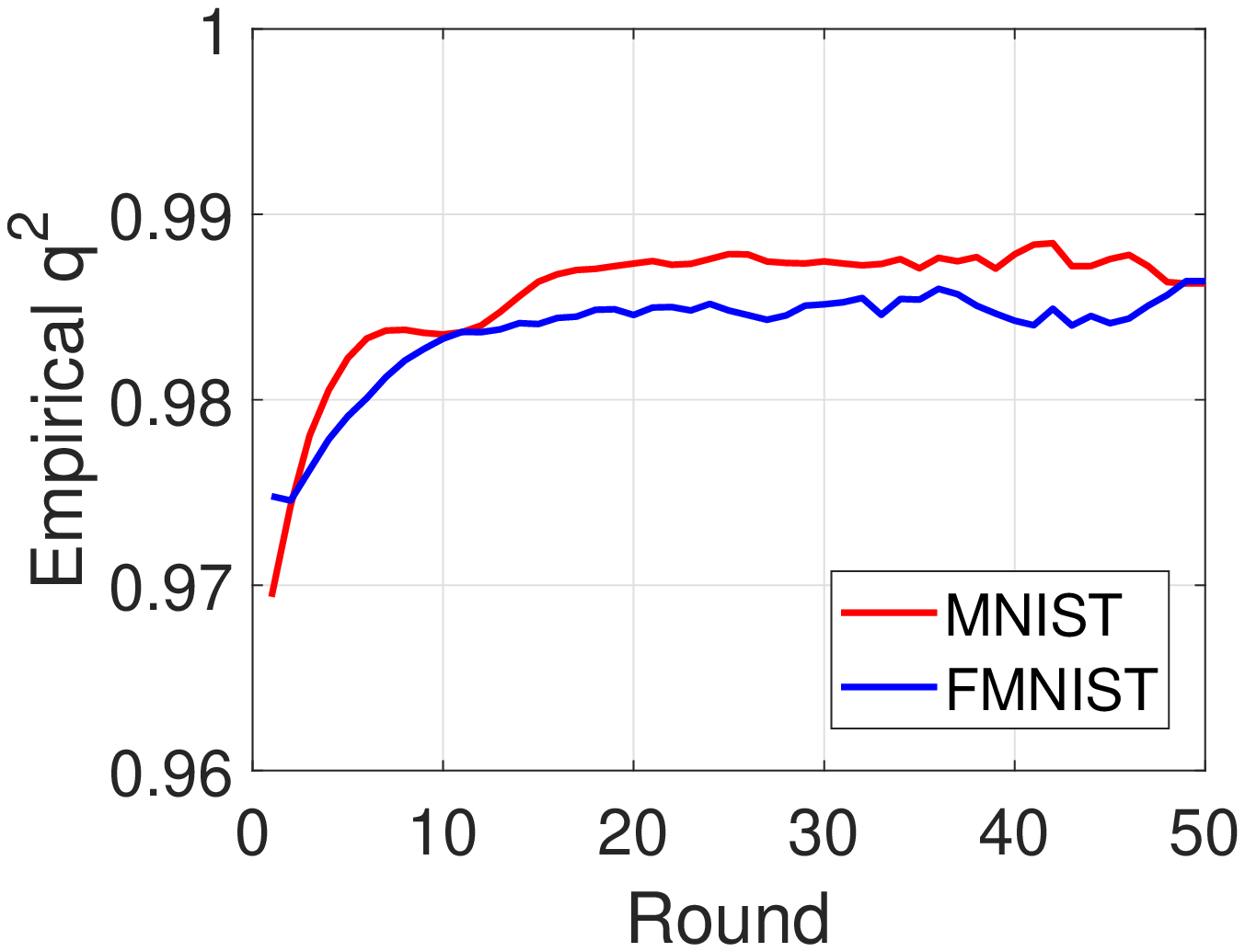}\hspace{-0.1in}
    \includegraphics[width=2.3in]{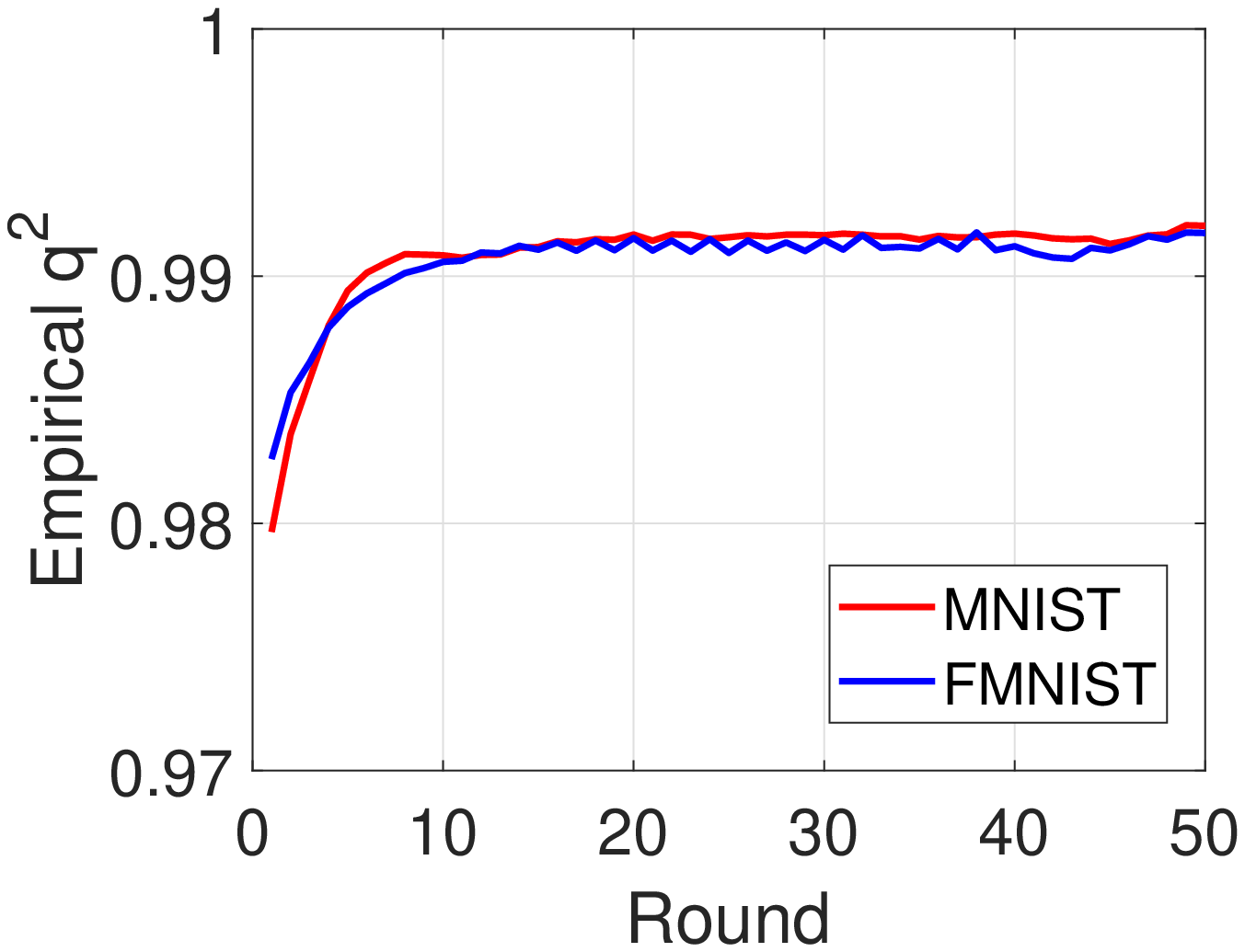}\hspace{-0.1in}
    \includegraphics[width=2.3in]{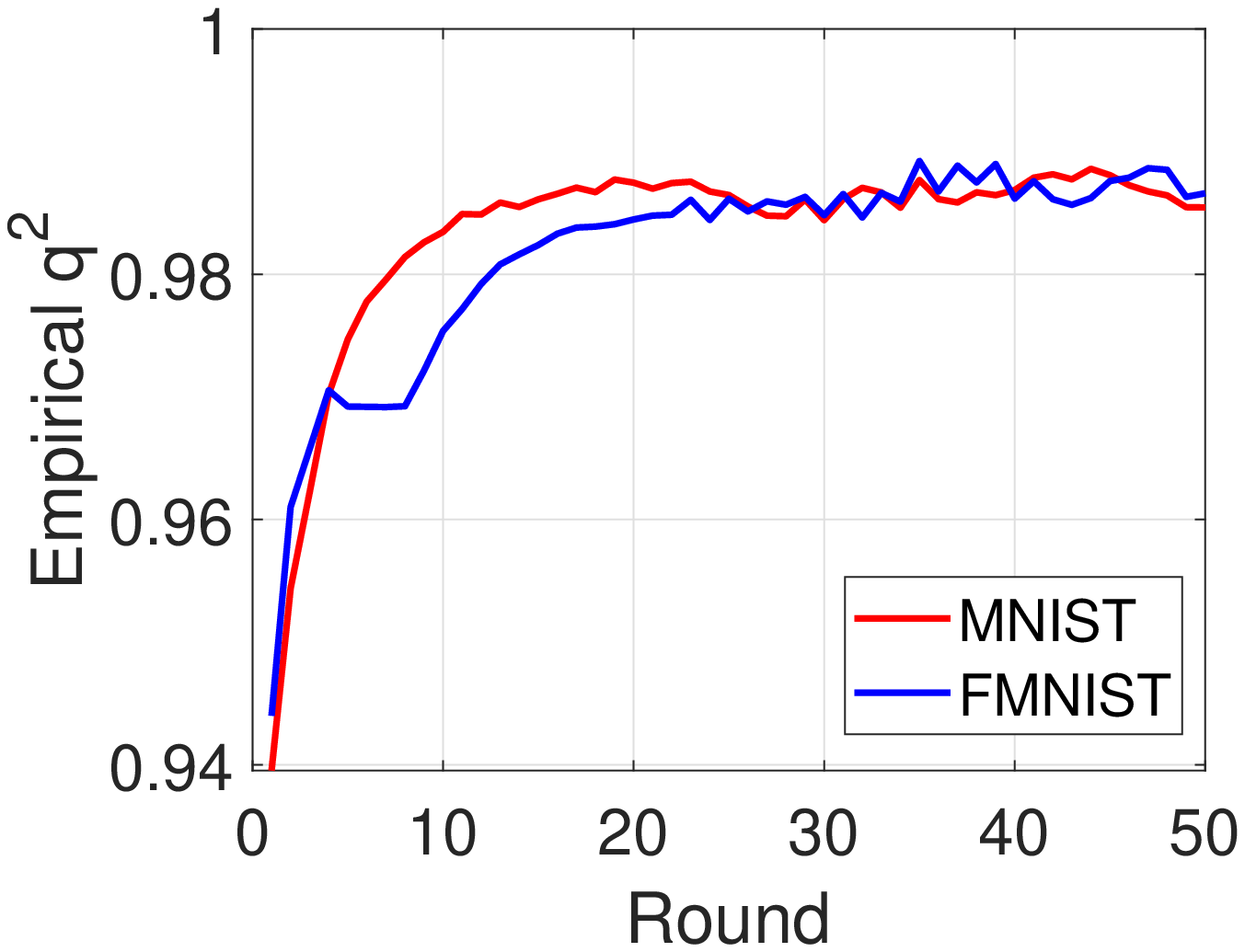}
    }
  \end{center}

\vspace{-0.2in}

  \caption{The compression coefficient $q_{\mathcal A}$ in Assumption~\ref{ass:compress_diff} in our experiments (Section~\ref{sec:experiment}) for CNN trained on MNIST and FMNIST dataset, averaged over multiple runs. $\eta=1$, $\eta_l=0.01$, non-iid client data distribution. Left: \textbf{Sign} compression. Mid: \textbf{TopK} compression with $k=0.01$. Right: \textbf{TopK} compression with $k=0.1$.}
  \label{fig:cnn q}
\end{figure}

\clearpage
\newpage

\section{Proof of Convergence Results}\label{app:proof}

We first present the proof for the more complicated Fed-EF-AMS in Section~\ref{app sec:AMS}, and the proof of Fed-EF-SGD would follow in Section~\ref{app sec:SGD}. Section~\ref{app:lemmas} contains intermediary lemmas, Section~\ref{app sec:partial-simple} provides the analysis of Fed-EF in partial participation and Section~\ref{app sec: no EF} proves the rate of FL directly using biased compression.

\subsection{Proof of Theorem~\ref{theo:rate AMS}: Fed-EF-AMS} \label{app sec:AMS}

\begin{proof}
We first clarify some notations. At round $t$, let the full-precision local model update of the $i$-th worker be $\del_{t,i}$, the error accumulator be $e_{t,i}$, and denote $\widetilde\del_{t,i}=\mathcal C(g_{t,i}+e_{t,i})$. Define $\bar\del_t=\frac{1}{n}\sum_{i=1}^n\del_{t,i}$, $\overline{\widetilde\del}_t=\frac{1}{n}\sum_{i=1}^n \widetilde\del_{t,i}$ and $\bar e_t=\frac{1}{n}\sum_{i=1}^n e_{t,i}$. The second moment computed by the compressed local model updates is denoted as $v_t=\beta_2 v_{t-1}+(1-\beta_2) \overline{\widetilde\del}_t^2$, and $\hat v_t=\max\{\hat v_{t-1}, v_t\}$. Also, the first order moving average sequence
\begin{align*}
m_t=\beta_1 m_{t-1}+(1-\beta_1)\overline{\widetilde \del}_t \quad & \textrm{and} \quad m_t'=\beta_1 m_{t-1}'+(1-\beta_1) \bar \del_t,
\end{align*}
where $m_t'$ represents the first moment moving average sequence using the uncompressed updates. By construction we have $m_t'=(1-\beta_1)\sum_{\tau=1}^t \beta_1^{t-\tau} \bar \del_\tau$.

\vspace{0.1in}
\noindent Our proof will use the following auxiliary sequences: for round $t=1,...,T$,
\begin{align*}
& \mathcal E_{t+1}\eqdef (1-\beta_1)\sum_{\tau=1}^{t+1} \beta_1^{t+1-\tau} \bar e_\tau,\\
&\theta_{t+1}':=\theta_{t+1}-\eta \frac{\mathcal E_{t+1}}{\sqrt{\hat v_t+\epsilon}}.
\end{align*}

Then, we can write the evolution of $\theta_t'$ as
\begin{align*}
    \theta_{t+1}'&=\theta_{t+1}-\eta \frac{\mathcal E_{t+1}}{\sqrt{\hat v_t+\epsilon}}\\
    &=\theta_t-\eta\frac{(1-\beta_1)\sum_{\tau=1}^{t} \beta_1^{t-\tau}\overline{\widetilde \del}_\tau+(1-\beta_1)\sum_{\tau=1}^{t+1} \beta_1^{t+1-\tau}\bar e_\tau}{\sqrt{\hat v_t+\epsilon}}\\
    &=\theta_t-\eta\frac{(1-\beta_1)\sum_{\tau=1}^{t} \beta_1^{t-\tau}(\overline{\widetilde \del}_\tau+ \bar e_{\tau+1})+(1-\beta)\beta_1^t \bar e_1}{\sqrt{\hat v_t+\epsilon}}\\
    &=\theta_t-\eta\frac{(1-\beta_1)\sum_{\tau=1}^{t} \beta_1^{t-\tau} \bar e_\tau}{\sqrt{\hat v_t+\epsilon}}-\eta\frac{m_t'}{\sqrt{\hat v_t+\epsilon}}\\
    &=\theta_t-\eta\frac{\mathcal E_t}{\sqrt{\hat v_{t-1}+\epsilon}}-\eta\frac{m_t'}{\sqrt{\hat v_t+\epsilon}}+\eta(\frac{1}{\sqrt{\hat v_{t-1}+\epsilon}}-\frac{1}{\sqrt{\hat v_t+\epsilon}})\mathcal E_t\\
    &\overset{(a)}{=}\theta_t'-\eta\frac{m_t'}{\sqrt{\hat v_t+\epsilon}}+\eta(\frac{1}{\sqrt{\hat v_{t-1}+\epsilon}}-\frac{1}{\sqrt{\hat v_t+\epsilon}})\mathcal E_t\\
    &\eqdef \theta_t'-\eta \frac{m_t'}{\sqrt{\hat v_t+\epsilon}}+\eta D_t\mathcal E_t,
\end{align*}
where (a) uses the fact of error feedback that for every $i\in[n]$, $\widetilde \del_{t,i}+e_{{t+1,i}}=\del_{t,i}+e_{t,i}$, and $e_{t,1}=0$ at initialization. Further define the virtual iterates:
\begin{align*}
    x_{t+1}\eqdef\theta_{t+1}'-\eta \frac{\beta_1}{1-\beta_1} \frac{m_t'}{\sqrt{\hat v_t+\epsilon}},
\end{align*}
which follows the recurrence:
\begin{align*}
    x_{t+1}&=\theta_{t+1}'-\eta\frac{\beta_1}{1-\beta_1} \frac{m_t'}{\sqrt{\hat v_t+\epsilon}}\\
    &=\theta_t'-\eta\frac{m_t'}{\sqrt{\hat v_t+\epsilon}}-\eta\frac{\beta_1}{1-\beta_1} \frac{m_t'}{\sqrt{\hat v_t+\epsilon}}+\eta D_t\mathcal E_t\\
    &=\theta_t'-\eta \frac{\beta_1 m_{t-1}'+(1-\beta_1)\bar \del_t+\frac{\beta_1^2}{1-\beta_1}m_{t-1}'+\beta_1 \bar \del_t}{\sqrt{\hat v_t+\epsilon}}+\eta D_t\mathcal E_t\\
    &=\theta_t'-\eta\frac{\beta_1}{1-\beta_1}\frac{m_{t-1}'}{\sqrt{\hat v_t+\epsilon}}-\eta\frac{\bar \del_t}{\sqrt{\hat v_t+\epsilon}}+\eta D_t\mathcal E_t\\
    &=x_t-\eta\frac{\bar \del_t}{\sqrt{\hat v_t+\epsilon}}+\eta\frac{\beta_1}{1-\beta_1} D_t m_{t-1}'+\eta D_t\mathcal E_t.
\end{align*}
The general idea is to study the convergence of the sequence $x_t$, and show that the difference between $x_t$ and $\theta_t$ (of interest) is small. First, by the smoothness Assumption~\ref{ass:smooth}, we have
\begin{align*}
    f(x_{t+1})\leq f(x_t)+\langle \nabla f(x_t), x_{t+1}-x_t\rangle+\frac{L}{2}\| x_{t+1}-x_t\|^2.
\end{align*}
Taking expectation w.r.t. the randomness at round $t$ and re-arranging terms, we obtain
\begin{align}
    &\mathbb E[f(x_{t+1})]-f(x_t) \nonumber\\
    &\leq -\eta\mathbb E\Big[\big\langle \nabla f(x_t), \frac{\bar \del_t}{\sqrt{\hat v_t+\epsilon}}\big\rangle \Big]+\eta \mathbb E\Big[\big\langle \nabla f(x_t), \frac{\beta_1}{1-\beta_1}D_tm_{t-1}'+D_t\mathcal E_t\big\rangle\Big] \nonumber\\
    &\hspace{2in} +\frac{\eta^2L}{2}\mathbb E\Big[\|\frac{\bar \del_t}{\sqrt{\hat v_t+\epsilon}}-\frac{\beta_1}{1-\beta_1}D_tm_{t-1}'- D_t\mathcal E_t\|^2\Big] \nonumber\\
    &=\underbrace{-\eta\mathbb E\Big[\big\langle \nabla f(\theta_t), \frac{\bar \del_t}{\sqrt{\hat v_t+\epsilon}}\big\rangle\Big]}_{I}+\underbrace{\eta \mathbb E\Big[\big\langle \nabla f(x_t), \frac{\beta_1}{1-\beta_1}D_tm_{t-1}'+D_t\mathcal E_t\big\rangle\Big]}_{II} \nonumber\\
    &\hspace{0.5in} +\underbrace{\frac{\eta^2L}{2}\mathbb E\Big[\|\frac{\bar \del_t}{\sqrt{\hat v_t+\epsilon}}-\frac{\beta_1}{1-\beta_1}D_tm_{t-1}'- D_t\mathcal E_t\|^2\Big]}_{III}+\underbrace{\eta\mathbb E\Big[\big\langle \nabla f(\theta_t)-\nabla f(x_t), \frac{\bar \del_t}{\sqrt{\hat v_t+\epsilon}} \big\rangle\Big]}_{IV}, \label{eq0}
\end{align}

\vspace{0.1in}
\noindent \textbf{Bounding term I.} We have
\begin{align}
    \bm I&=-\eta\mathbb E\Big[\big\langle \nabla f(\theta_t), \frac{\bar \del_t}{\sqrt{\hat v_{t-1}+\epsilon}}\big\rangle\Big]-\eta\mathbb E\Big[\big\langle \nabla f(\theta_t), (\frac{1}{\sqrt{\hat v_t+\epsilon}}-\frac{1}{\sqrt{\hat v_{t-1}+\epsilon}})\bar \del_t \big\rangle \Big] \nonumber\\
    &\leq -\eta\mathbb E\Big[\big\langle \nabla f(\theta_t), \frac{\bar \del_t}{\sqrt{\hat v_{t-1}+\epsilon}}\big\rangle\Big]+\eta \eta_l K G^2\mathbb E[\|D_t\|_1], \label{eq:I0}
\end{align}
where we use Assumption~\ref{ass:boundgrad} on the stochastic gradient magnitude. The last inequality holds by simply bounding the aggregated local model update by $$\|\bar\del_t\|\leq\frac{1}{n}\sum_{i=1}^n\|\eta_l\sum_{k=1}^K g_{t,i}^{(k)}\|\leq \eta_l KG,$$
and the fact that for any vector in $\mathbb R^d$, the $l_2$ norm is upper bounded by the $l_1$ norm.

Regarding the first term in \eqref{eq:I0}, we have
\begin{align*}
    &-\eta\mathbb E\Big[\langle \nabla f(\theta_t), \frac{\bar \del_t}{\sqrt{\hat v_{t-1}+\epsilon}}\Big]
    =-\eta \mathbb E\Big[\langle \frac{\nabla f(\theta_t)}{\sqrt{\hat v_{t-1}+\epsilon}}, \bar\del_t-\eta_l K\nabla f(\theta_t)+\eta_l K\nabla f(\theta_t)\rangle\Big]\\
    &=-\eta \eta_l K\mathbb E\Big[\frac{\|\nabla f(\theta_t)\|^2}{\sqrt{\hat v_{t-1}+\epsilon}}\Big]+\eta \mathbb E\Big[\langle \frac{\nabla f(\theta_t)}{\sqrt{\hat v_{t-1}+\epsilon}},-\bar\del_t+\eta_l K\nabla f(\theta_t) \rangle\Big]\\
    &\overset{(a)}{\leq }-\frac{\eta \eta_l K}{\sqrt{\frac{4\eta_l^2(1+q^2)^3K^2}{(1-q^2)^2}G^2+\epsilon}}\mathbb E\big[\|\nabla f(\theta_t)\|^2\big]+\eta \big\langle \frac{\nabla f(\theta_t)}{\sqrt{\hat v_{t-1}+\epsilon}},\mathbb E\big[-\frac{1}{n}\sum_{i=1}^n \sum_{k=1}^K \eta_l g_{t,i}^{(k)}+\eta_l K\nabla f(\theta_t)\big] \big\rangle\\
    &\overset{(b)}{=} -\frac{\eta \eta_l K}{\sqrt{\frac{4\eta_l^2(1+q^2)^3K^2}{(1-q^2)^2}G^2+\epsilon}}\mathbb E\big[\|\nabla f(\theta_t)\|^2\big] \\
    &\hspace{1.2in} + \eta\underbrace{\big\langle \frac{\sqrt{\eta_l}\nabla f(\theta_t)}{(\hat v_{t-1}+\epsilon)^{1/4}},\mathbb E\Big[\frac{\sqrt{\eta_l}}{n (\hat v_{t-1}+\epsilon)^{1/4}}(-\sum_{i=1}^n \sum_{k=1}^K \nabla f_i(\theta_{t,i}^{(k)})+K\nabla f(\theta_t))\Big] \big\rangle}_{V},
\end{align*}
where (a) uses Lemma~\ref{lemma:bound v_t} and (b) is due to Assumption~\ref{ass:var} that $g_{t,i}^{(k)}$ is an unbiased estimator of $\nabla f_i(\theta_{t,i}^{(k)})$. To bound term \textbf{V}, we use the similar proof structure as in the proof of Lemma~\ref{lemma:inner-product}. Specifically, we have
\begin{align*}
    \bm V&\leq \frac{\eta_l K}{2\sqrt\epsilon}\mathbb E\big[\|\nabla f(\theta_t)\|^2\big]+\frac{\eta_l}{2K\sqrt\epsilon} \mathbb E\big[\| \frac{1}{n}\sum_{i=1}^n\sum_{k=1}^K (\nabla f_i(\theta_{t,i}^{(k)})-\nabla f_i(\theta_t)) \|^2\big]\\
    &\leq \frac{\eta_l K}{2\sqrt\epsilon}\mathbb E\big[\|\nabla f(\theta_t)\|^2\big]+ \frac{\eta_l}{2nK\sqrt\epsilon} \mathbb E\big[ \sum_{i=1}^n \|\sum_{k=1}^K (\nabla f_i(\theta_{t,i}^{(k)})-\nabla f_i(\theta_t)) \|^2\big]\\
    &\leq \frac{\eta_l K}{2\sqrt\epsilon}\mathbb E\big[\|\nabla f(\theta_t)\|^2\big]+ \frac{\eta_l}{2n\sqrt\epsilon} \mathbb E\big[ \sum_{i=1}^n \sum_{k=1}^K \|\nabla f_i(\theta_{t,i}^{(k)})-\nabla f_i(\theta_t) \|^2\big]\\
    &\leq \frac{\eta_l K}{2\sqrt\epsilon}\mathbb E\big[\|\nabla f(\theta_t)\|^2\big]+ \frac{\eta_l L^2}{2n\sqrt\epsilon} \mathbb E\big[ \sum_{i=1}^n \sum_{k=1}^K \|\theta_{t,i}^{(k)}-\theta_t \|^2\big],
\end{align*}
where the last inequality is a result of the $L$-smoothness assumption on the loss function $f_i(x)$. Applying Lemma~\ref{lemma:consensus} to the consensus error, we can further bound term $V$ by
\begin{align*}
    \bm V&\leq \frac{\eta_l K}{2\sqrt\epsilon}\mathbb E\big[\|\nabla f(\theta_t)\|^2\big]+\frac{\eta_l K L^2}{2\sqrt\epsilon}\big[5\eta_l^2 K(\sigma^2+6K\sigma_g^2)+30\eta_l^2 K^2\mathbb E[\|\nabla f(\theta_t)\|^2]\big]\\
    &\leq \frac{47\eta_l K}{64\sqrt\epsilon}\mathbb E\big[\|\nabla f(\theta_t)\|^2\big]+\frac{5\eta_l^3 K^2 L^2}{2\sqrt\epsilon}(\sigma^2+6K\sigma_g^2),
\end{align*}
when we choose $\eta_l\leq \frac{1}{8KL}$. Further, if we set $\eta_l\leq \frac{\sqrt{15}(1-q^2)\sqrt\epsilon}{14 (1+q^2)^{1.5}KG}$, we have
\begin{align*}
    \frac{4\eta_l^2(1+q^2)^3K^2}{(1-q^2)^2}G^2+\epsilon \leq \frac{60}{196}\epsilon+\epsilon = \frac{64}{49}\epsilon.
\end{align*}

Hence, as $\frac{47}{64}<\frac{3}{4}$, we can establish from (\ref{eq:I0}) that
\begin{align}
    \bm I \leq -\frac{\eta\eta_l K}{8 \sqrt\epsilon} \mathbb E\big[\| \nabla f(\theta_t) \|^2\big] +\frac{5\eta\eta_l^3 K^2 L^2}{2\sqrt\epsilon}(\sigma^2+6K\sigma_g^2) + \eta \eta_l K G^2\mathbb E[\|D_t\|_1].  \label{eq:I}
\end{align}

\vspace{0.1in}
\noindent \textbf{Bounding term II.} By Lemma~\ref{lemma:bound big E_t}, we know that $\|\mathcal E_t\|\leq \frac{2\eta_l qKG}{1-q^2}$, and by Lemma~\ref{lemma:m_t,m_t'}, $\|m_t'\|\leq \eta_l KG$. Thus, we have
\begin{align}
    \bm{II}&\leq\eta \Big(\mathbb E\big[\langle  \nabla f(\theta_t),\frac{\beta_1}{1-\beta_1}D_tm_{t-1}'+D_t\mathcal E_t\rangle\big]+\mathbb E\big[\langle  \nabla f(x_t)-\nabla f(\theta_t),\frac{\beta_1}{1-\beta_1}D_tm_{t-1}'+D_t\mathcal E_t\rangle\big] \Big) \nonumber\\
    &\leq \eta\mathbb E\big[\|\nabla f(\theta_t)\|\|\frac{\beta_1}{1-\beta_1}D_tm_{t-1}'+D_t\mathcal E_t\|\big] \nonumber\\
    &\hspace{1.5in}+\eta^2 \ L \mathbb E\big[\|\frac{\frac{\beta_1}{1-\beta_1}m_{t-1}'+\mathcal E_t}{\sqrt{\hat v_{t-1}+\epsilon}}\| \|\frac{\beta_1}{1-\beta_1}D_tm_{t-1}'+D_t\mathcal E_t\|\big] \nonumber\\
    &\leq \eta\eta_l C_1 K G^2 \mathbb E[\|D_t\|_1]+\frac{\eta^2 \eta_l^2 C_1^2 L K^2 G^2}{\sqrt\epsilon}\mathbb E[\|D_t\|_1],  \label{eq:II}
\end{align}
where $C_1\eqdef \frac{\beta_1}{1-\beta_1}+\frac{2q}{1-q^2}$, and the second inequality is due to the smoothness of $f(\theta)$.


\vspace{0.1in}
\noindent \textbf{Bounding term III.} This term can be bounded as follows:
\begin{align}
    \bm{III}&\leq \eta^2 L\mathbb E\big[\|\frac{\bar\del_t}{\sqrt{\hat v_t+\epsilon}}\|^2\big]+\eta^2 L\mathbb E\big[\|\frac{\beta_1}{1-\beta_1}D_tm_{t-1}'- D_t\mathcal E_t\|^2\big] \nonumber\\
    &\leq \frac{\eta^2 L}{\epsilon}\mathbb E\big[\|\bar\del_t\|^2\big]+\eta^2 L\mathbb E\big[\|D_t(\frac{\beta_1}{1-\beta_1}m_{t-1}'-\mathcal E_t)\|^2\big] \nonumber\\
    &\leq  \frac{\eta^2L(2\eta_l^2 K^2+120\eta_l^4 K^4 L^2)}{\epsilon} \mathbb E\big[\|\nabla f(\theta_t)\|^2\big] +\frac{4\eta^2\eta_l^2 KL}{n\epsilon}\sigma^2 \nonumber\\
    &\hspace{1.2in} +\frac{20\eta^2\eta_l^4K^3L^3}{\epsilon}(\sigma^2+6K\sigma_g^2)+\eta^2 \eta_l^2 C_1^2 L K^2 G^2 \mathbb E[\|D_t\|^2],  \label{eq:III}
\end{align}
where we apply Lemma~\ref{lemma:bound delta} and use similar argument as in bounding term II.

\vspace{0.1in}
\noindent \textbf{Bounding term IV.} Lastly, for term IV, we have for some $\rho>0$,
\begin{align}
    \bm{IV}
    &=\eta\mathbb E\Big[\big\langle \nabla f(\theta_t)-\nabla f(x_t), \frac{\bar\del_t}{\sqrt{\hat v_{t-1}+\epsilon}} \big\rangle\Big]+\eta\mathbb E\Big[\big\langle \nabla f(\theta_t)-\nabla f(x_t), (\frac{1}{\sqrt{\hat v_t+\epsilon}}-\frac{1}{\sqrt{\hat v_{t-1}+\epsilon}})\bar \del_t \big\rangle\Big] \nonumber\\
    &\overset{(a)}{\leq}  \frac{\eta \rho}{2\epsilon}\mathbb E\big[\|\bar\del_t\|^2\big]+\frac{\eta}{2\rho}\mathbb E\big[\|\nabla f(\theta_t)-\nabla f(x_t)\|^2\big]+\eta^2 L\mathbb E\Big[\|\frac{\frac{\beta_1}{1-\beta_1}m_{t-1}'+\mathcal E_t}{\sqrt{\hat v_{t-1}+\epsilon}}\|\|D_t \del_t\| \Big] \nonumber\\
    &\overset{(b)}{\leq} \frac{\rho \eta(\eta_l^2K^2+60\eta_l^4 K^4 L^2) }{\epsilon}\mathbb E\big[\|\nabla f(\theta_t)\|^2\big] + \frac{2\rho\eta \eta_l^2 K}{\epsilon n}\sigma^2 + \frac{10\rho\eta\eta_l^4 K^3L^2}{\epsilon}(\sigma^2+6K\sigma_g^2) \nonumber \\
    &\hspace{1.8in} + \frac{\eta^3 L^2}{2\rho}\mathbb E\Big[\|\frac{\frac{\beta_1}{1-\beta_1}m_{t-1}'+\mathcal E_t}{\sqrt{\hat v_{t-1}+\epsilon}}\|^2 \Big]+\frac{\eta^2\eta_l^2 C_1 LK^2G^2}{\sqrt\epsilon} \mathbb E[\|D_t\| ]  \nonumber \\
    &\leq \frac{\rho \eta\eta_l^2K^2(60\eta_l^2 K^2 L^2+1) }{\epsilon}\mathbb E\big[\|\nabla f(\theta_t)\|^2\big] + \frac{2\rho\eta \eta_l^2 K}{\epsilon n}\sigma^2 + \frac{10\rho\eta\eta_l^4 K^3L^2}{\epsilon}(\sigma^2+6K\sigma_g^2) \nonumber \\
    &\hspace{0.8in} + \frac{\eta^3 L^2}{\rho \epsilon}\Big[ \frac{\beta_1^2}{(1-\beta_1)^2}\mathbb E\big[\|m_t'\|^2\big]+ \mathbb E\big[\|\mathcal E_t\|^2\big]\Big]+\frac{\eta^2\eta_l^2 C_1 LK^2G^2}{\sqrt\epsilon} \mathbb E[\|D_t\|_1],  \label{eq:IV}
\end{align}
where (a) is a consequence of Young's inequality ($\rho$ will be specified later) and the smoothness Assumption~\ref{ass:smooth}, and (b) is based on Lemma~\ref{lemma:bound delta}.

After we have bounded all four terms in (\ref{eq0}), the next step is to gather the ingredients by taking the telescope sum over $t=1,...,T$. For the ease of presentation, we first do this for the third term in~(\ref{eq:IV}). According to Lemma~\ref{lemma:m_t,m_t'} and Lemma~\ref{lemma:bound big E_t}, summing over $t=1,...,T$, we conclude
\begin{align}
    &\sum_{t=1}^T\frac{\eta^3 L^2}{\rho\epsilon}\Big[ \frac{\beta_1^2}{(1-\beta_1)^2}\mathbb E[\|m_t'\|^2]+ \mathbb E[\|\mathcal E_t\|^2]\Big] \nonumber\\
    &\leq \frac{\eta^3 \beta_1^2L^2}{\rho(1-\beta_1)^2\epsilon}\Big[ 2\eta_l^2 K^2(60\eta_l^2 K^2 L^2+1)\sum_{t=1}^T \mathbb E\big[\|\nabla f(\theta_t)\|^2\big] +4\frac{T\eta_l^2 K}{n}\sigma^2+20T\eta_l^4K^3L^2(\sigma^2+6K\sigma_g^2) \Big]  \nonumber\\
    &+ \frac{\eta^3q^2 L^2}{\rho(1-q^2)^2\epsilon}\Big[ 8\eta_l^2 K^2 (60\eta_l^2 K^2 L^2+1) \sum_{t=1}^T \mathbb E\big[\|\nabla f(\theta_\tau)\|^2\big] +\frac{16T\eta_l^2 K}{ n}\sigma^2+80T\eta_l^4 K^3L^2(\sigma^2+6K\sigma_g^2) \Big]  \nonumber \\
    &\leq \frac{2\eta^3 \eta_l^2 C_2 K^2 L^2}{\rho\epsilon}(60 \eta_l^2K^2L^2+1) \sum_{t=1}^T \mathbb E\big[\|\nabla f(\theta_t)\|^2\big]  \nonumber\\
    &\hspace{1.5in} + \frac{4T\eta^3 \eta_l^2 C_2KL^2}{\rho n\epsilon}\sigma^2 + \frac{20 T \eta^3\eta_l^4 C_2 K^3 L^4 }{\rho\epsilon}(\sigma^2+6K\sigma_g^2), \label{eq:IV error}
\end{align}
with $C_2\eqdef \frac{\beta_1^2}{(1-\beta_1)^2}+\frac{4q^2}{(1-q^2)^2}$.

\vspace{0.1in}
\noindent \textbf{Putting together.} We are in the position to combine pieces together to get our final result by integrating \eqref{eq:I}, \eqref{eq:II}, \eqref{eq:III}, \eqref{eq:IV} and \eqref{eq:IV error} into \eqref{eq0} and taking the telescoping sum over $t=1,...,T$. After re-arranging terms, when $\eta_l\leq \min\big\{\frac{1}{8KL}, \frac{(1-q^2)\sqrt\epsilon}{4 (1+q^2)^{1.5}KG}\big\}$, we have
\begin{align}
    &\mathbb E[f(x_{T+1})-f(x_1)]  \nonumber\\
    &\leq -\frac{\eta\eta_l K}{8 \sqrt\epsilon} \sum_{t=1}^T\mathbb E\big[\| \nabla f(\theta_t) \|^2\big] +\frac{5T\eta\eta_l^3 K^2 L^2}{2\sqrt\epsilon}(\sigma^2+6K\sigma_g^2) + \eta \eta_l K G^2\sum_{t=1}^T\mathbb E[\|D_t\|_1]  \nonumber\\
    & + \eta\eta_l C_1 K G^2 \sum_{t=1}^T\mathbb E[\|D_t\|_1]+\frac{\eta^2 \eta_l^2 C_1^2 L K^2 G^2}{\sqrt\epsilon}\sum_{t=1}^T\mathbb E[\|D_t\|_1] \nonumber \\
    & + \frac{\eta^2L(2\eta_l^2 K^2+120\eta_l^4 K^4 L^2)}{\epsilon} \sum_{t=1}^T\mathbb E\big[\|\nabla f(\theta_t)\|^2\big] +\frac{4T\eta^2\eta_l^2 KL}{n\epsilon}\sigma^2 \nonumber\\
    &\hspace{1.2in} +\frac{20T\eta^2\eta_l^4K^3L^3}{\epsilon}(\sigma^2+6K\sigma_g^2)+\eta^2 \eta_l^2 C_1^2 L K^2 G^2 \sum_{t=1}^T\mathbb E[\|D_t\|^2] \nonumber \\
    & + \frac{\rho \eta\eta_l^2K^2(60\eta_l^2 K^2 L^2+1) }{\epsilon}\sum_{t=1}^T\mathbb E\big[\|\nabla f(\theta_t)\|^2\big] + \frac{2T\rho\eta \eta_l^2 K}{\epsilon n}\sigma^2 + \frac{10T\rho\eta\eta_l^4 K^3L^2}{\epsilon}(\sigma^2+6K\sigma_g^2) \nonumber \\
    & + \frac{2\eta^3 \eta_l^2 C_2 K^2 L^2}{\rho\epsilon}(60 \eta_l^2K^2L^2+1) \sum_{t=1}^T \mathbb E\big[\|\nabla f(\theta_t)\|^2\big] + \frac{4T\eta^3 \eta_l^2 C_2KL^2}{\rho n\epsilon}\sigma^2 \nonumber\\
    &\hspace{1.5in} + \frac{20 T \eta^3\eta_l^4 C_2 K^3 L^4 }{\rho\epsilon}(\sigma^2+6K\sigma_g^2) +\frac{\eta^2\eta_l^2 C_1 LK^2G^2}{\sqrt\epsilon} \sum_{t=1}^T\mathbb E[\|D_t\|_1] \nonumber \\
    &= \Upsilon_1 \cdot \sum_{t=1}^T\mathbb E\big[\| \nabla f(\theta_t) \|^2\big] + \Upsilon_2\cdot (\sigma^2+6K\sigma_g^2) + \Upsilon_3\cdot \sigma^2 \nonumber \\
    &\hspace{1.5in} + \Upsilon_4\cdot\sum_{t=1}^T\mathbb E[\|D_t\|_1] +\eta^2 \eta_l^2 C_1^2 L K^2 G^2 \sum_{t=1}^T\mathbb E[\|D_t\|^2],  \label{eq:final1}
\end{align}
where
\begin{align}
    \Upsilon_1&= -\frac{\eta\eta_l K}{8 \sqrt\epsilon}+\frac{\eta^2L(2\eta_l^2 K^2+120\eta_l^4 K^4 L^2)}{\epsilon} \nonumber\\
    &\hspace{1.2in} +\frac{\rho \eta\eta_l^2K^2(60\eta_l^2 K^2 L^2+1) }{\epsilon}+\frac{2\eta^3 \eta_l^2 C_2 K^2 L^2}{\rho\epsilon}(60 \eta_l^2K^2L^2+1) \nonumber\\
    &\leq  -\frac{\eta\eta_l K}{8 \sqrt\epsilon}+\frac{2\eta^2\eta_l^2 K^2 L}{\epsilon}+\frac{120\eta^2\eta_l^4 K^4 L^3}{\epsilon} +\frac{2\rho \eta\eta_l^2K^2}{\epsilon}+\frac{4\eta^3 \eta_l^2 C_2 K^2 L^2}{\rho\epsilon}, \label{eq:Upsilon_1}\\
    \Upsilon_2&= \frac{5T\eta\eta_l^3 K^2 L^2}{2\sqrt\epsilon}+\frac{20T\eta^2\eta_l^4K^3L^3}{\epsilon}+\frac{10T\rho\eta\eta_l^4 K^3L^2}{\epsilon}+\frac{20 T \eta^3\eta_l^4 C_2 K^3 L^4 }{\rho\epsilon}, \nonumber\\
    \Upsilon_3&= \frac{4T\eta^2\eta_l^2 KL}{n\epsilon}+\frac{2T\rho\eta \eta_l^2 K}{n\epsilon}+ \frac{4T\eta^3 \eta_l^2 C_2KL^2}{\rho n\epsilon}, \nonumber\\
    \Upsilon_4&= \eta\eta_l (C_1+1) K G^2+\frac{\eta^2 \eta_l^2 C_1^2 L K^2 G^2}{\sqrt\epsilon}+\frac{\eta^2\eta_l^2 C_1 LK^2G^2}{\sqrt\epsilon},  \nonumber
\end{align}
where to bound $\Upsilon_1$ we use the fact that $\eta_l\leq \frac{1}{8KL}$. We now look at the upper bound (\ref{eq:Upsilon_1}) of $\Upsilon_1$ which contains 5 terms. In the following, we choose $\rho\equiv L\eta$ in (\ref{eq:IV}) and (\ref{eq:IV error}). Suppose we choose $\epsilon<1$. Then, when the local learning rate satisfies
\begin{align*}
    \eta_l&\leq \frac{1}{K}\min \Big\{ \frac{1}{8L}, \frac{(1-q^2)\sqrt\epsilon}{4(1+q^2)^{1.5}G}, \frac{\sqrt\epsilon}{128\eta L},  \frac{\sqrt\epsilon}{256C_2\eta L}, \big( \frac{\sqrt{\epsilon}}{7680\eta}\big)^{1/3}\frac{1}{L} \Big\} \\
    &\leq \frac{\sqrt\epsilon}{8KL}\min \Big\{ \frac{1}{\sqrt\epsilon}, \frac{2(1-q^2)L}{(1+q^2)^{1.5}G}, \frac{1}{\max\{16,32C_2\}\eta},   \frac{1}{3\eta^{1/3}} \Big\},
\end{align*}
each of the last four terms can be bounded by $\frac{\eta\eta_l K}{48\sqrt\epsilon}$. Thus, under this learning rate setting,
$$\Upsilon_1\leq -\frac{\eta\eta_l K}{16\sqrt\epsilon}.$$
Taking the above into (\ref{eq:final1}), we arrive at
\begin{align*}
    \frac{\eta\eta_l K}{16\sqrt\epsilon} \sum_{t=1}^T\mathbb E\big[\| \nabla f(\theta_t) \|^2\big]&\leq f(x_1)-\mathbb E[f(x_{T+1})]+ \Upsilon_2 \cdot (\sigma^2+6K\sigma_g^2) \\
    &\hspace{0.8in}+ \Upsilon_3\cdot \sigma^2  + \Upsilon_4\cdot \frac{d}{\sqrt\epsilon}+\frac{\eta^2 \eta_l^2 C_1^2 L K^2 G^2d}{\epsilon},
\end{align*}
where Lemma~\ref{lemma:bound difference} on the difference sequence $D_t$ is applied. Consequently, we have
\begin{align*}
    \frac{1}{T}\sum_{t=1}^T\mathbb E\big[\| \nabla f(\theta_t) \|^2\big]&\lesssim \frac{f(x_1)-\mathbb E[f(x_{T+1})]}{\eta\eta_l TK} + \widetilde\Upsilon_2 \cdot (\sigma^2+6K\sigma_g^2)+ \widetilde\Upsilon_3\cdot\sigma^2 \\
    &\hspace{0.7in} + \frac{(C_1+1)  G^2d}{T\sqrt\epsilon}+\frac{2\eta \eta_l C_1^2 L K G^2d}{T\epsilon} +\frac{\eta\eta_l C_1^2 L K G^2d}{T\epsilon} \\
    &\leq \frac{f(x_1)-\mathbb E[f(x_{T+1})]}{\eta\eta_l TK} + \widetilde\Upsilon_2 \cdot (\sigma^2+6K\sigma_g^2)+  \widetilde\Upsilon_3\cdot\sigma^2  \\
    &\hspace{1.8in} + \frac{(C_1+1)  G^2d}{T\sqrt\epsilon}+\frac{3\eta \eta_l C_1^2 L K G^2d}{T\epsilon},
\end{align*}
where we make simplification at the second inequality using the fact that $C_1\leq C_1^2$ since $C_1\geq 1$. Moreover, $\widetilde\Upsilon_2$ and $\widetilde\Upsilon_3$ is defined as (recall that we have chosen $\rho\equiv L\eta$)
\begin{align*}
    \widetilde\Upsilon_2&= \frac{5\eta_l^2 K L^2}{2\sqrt\epsilon}+\frac{20\eta\eta_l^3 K^2L^3}{\epsilon}+\frac{10\eta\eta_l^3K^2L^3}{\epsilon}+\frac{20\eta\eta_l^3C_2K^2L^3}{\epsilon}\\
    &\leq \frac{5\eta_l^2 K L^2}{2\sqrt\epsilon}+\frac{\eta\eta_l^3(30+20C_2) K^2L^3}{\epsilon},\\
    \widetilde\Upsilon_3&=\frac{4\eta\eta_l L}{n\epsilon}+\frac{2\eta \eta_lL}{n\epsilon}+\frac{4\eta\eta_lC_2L}{n\epsilon}\leq \frac{\eta\eta_l L (6+4C_2)}{n\epsilon}.
\end{align*}
Finally, to connect the virtual iterates $x_t$ with the actual iterates $\theta_t$, note that $x_1=\theta_1$, and $f(x_{T+1})\geq f(\theta^*)$ since $\theta^*=\arg\min_\theta f(\theta)$. Replacing $\widetilde\Upsilon_2$ and $\widetilde\Upsilon_3$ with above upper bounds, this eventually leads to the bound
\begin{align*}
    \frac{1}{T}\sum_{t=1}^T\mathbb E\big[\| \nabla f(\theta_t) \|^2\big]&\lesssim  \frac{f(\theta_1)- f(\theta^*)}{\eta\eta_l TK} +\Big[ \frac{5\eta_l^2 K L^2}{2\sqrt\epsilon}+\frac{\eta\eta_l^3(30+20C_1^2) K^2L^3}{\epsilon} \Big] (\sigma^2+6K\sigma_g^2)   \\
    &\hspace{0.5in} + \frac{\eta\eta_l L (6+4C_1^2)}{n\epsilon}\sigma^2+ \frac{(C_1+1)  G^2d}{T\sqrt\epsilon}+\frac{3\eta \eta_l C_1^2 L K G^2d}{T\epsilon},
\end{align*}
which gives the desired result. Here we use the fact that $C_2\leq C_1^2$. This completes the proof.
\end{proof}



\subsection{Proof of Theorem~\ref{theo:rate SGD}: Fed-EF-SGD}  \label{app sec:SGD}

\begin{proof}

Now, we prove the variant of Fed-EF with SGD as the central server update rule. The proof follows the same routine as the one for Fed-EF-AMS, but is simpler since there are no moving average terms that need to be handled. Note that for this algorithm, we do not need Assumption~\ref{ass:boundgrad} that the stochastic gradients are uniformly bounded.

\vspace{0.1in}
\noindent For Fed-EF-SGD, consider the virtual sequence
\begin{align}
    x_{t+1} &= \theta_{t+1}-\eta \bar e_{t+1} \nonumber\\
    &=\theta_t-\eta\overline{\widetilde\del}_t -\eta\bar e_{t+1} \nonumber\\
    &=\theta_t-\frac{\eta}{n}\sum_{i=1}^n (\tilde\del_{t,i}+e_{t+1,i}) \nonumber\\
    &=\theta_t-\eta\bar\del_t-\eta\bar e_t \nonumber\\
    &=x_t -\eta\bar\del_t, \label{eq:virtual-SGD}
\end{align}
where the second last equality follows from the update rule that $\tilde\del_{t,i}+e_{t+1,i}=\del_{t,i}+e_{t,i}$ for all $i\in [n]$ and $t\in [T]$.

By the smoothness Assumption~\ref{ass:smooth}, we have
\begin{align*}
    f(x_{t+1})\leq f(x_t)+\langle \nabla f(x_t), x_{t+1}-x_t\rangle+\frac{L}{2}\| x_{t+1}-x_t\|^2.
\end{align*}
Taking expectation w.r.t. the randomness at round $t$ gives
\begin{align}
    &\mathbb E[f(x_{t+1})]-f(x_t) \nonumber\\
    &\leq -\eta\mathbb E\big[\langle \nabla f(x_t), \bar \del_t\rangle\big] +\frac{\eta^2L}{2}\mathbb E\big[\|\bar \del_t\|^2\big] \nonumber\\
    &=-\eta\mathbb E\big[\langle \nabla f(\theta_t), \bar \del_t\rangle\big] +\frac{\eta^2L}{2}\mathbb E\big[\|\bar \del_t\|^2\big]+\eta\mathbb E\big[\langle \nabla f(\theta_t)-\nabla f(x_t),\bar\del_t \rangle\big].  \label{sgd-eq0}
\end{align}
We can bound the first term in (\ref{sgd-eq0}) using similar technique as bounding term \textbf{I} in the proof of Fed-EF-AMS. Specifically, we have
\begin{align*}
    -\eta\mathbb E\big[\langle \nabla f(\theta_t), \bar \del_t\rangle\big]&=-\eta\mathbb E\big[\langle \nabla f(\theta_t), \bar\del_t -\eta_l K\nabla f(\theta_t)+\eta_l K\nabla f(\theta_t) \rangle\big] \\
    &=-\eta\eta_l K\mathbb E\big[\| \nabla f(\theta_t) \|^2\big] + \eta \mathbb E\big[\langle \nabla f(\theta_t),-\bar\del_t+\eta_l K\nabla f(\theta_t) \rangle\big].
\end{align*}
With $\eta_l\leq \frac{1}{8KL}$, applying Lemma~\ref{lemma:inner-product}, we have
\begin{align*}
    -\eta\mathbb E\big[\langle \nabla f(\theta_t), \bar \del_t\rangle\big]&\leq -\eta\eta_l K\mathbb E\big[\| \nabla f(\theta_t) \|^2\big]+
    \frac{3\eta\eta_l K}{4}\mathbb E\big[\|\nabla f(\theta_t)\|^2\big]+\frac{5\eta\eta_l^3 K^2 L^2}{2}(\sigma^2+6K\sigma_g^2) \\
    &=-\frac{\eta\eta_l K}{4}\mathbb E\big[\|\nabla f(\theta_t)\|^2\big]+\frac{5\eta\eta_l^3 K^2 L^2}{2}(\sigma^2+6K\sigma_g^2).
\end{align*}
The second term in (\ref{sgd-eq0}) can be bounded using Lemma~\ref{lemma:bound delta} as
\begin{align*}
    \frac{\eta^2L}{2}\mathbb E\big[\|\bar\del_t\|^2\big]&\leq \eta^2\eta_l^2 K^2 L(60\eta_l^2 K^2 L^2+1)\mathbb E\big[\|\nabla f(\theta_t)\|^2\big] \\
    &\hspace{1in} +\frac{2\eta^2\eta_l^2 KL}{n}\sigma^2+10\eta^2\eta_l^4K^3L^3(\sigma^2+6K\sigma_g^2).
\end{align*}
The last term in (\ref{sgd-eq0}) can be bounded similarly as \textbf{VI} in Fed-EF-AMS by
\begin{align}
    &\eta\mathbb E\big[\langle \nabla f(\theta_t)-\nabla f(x_t),\bar\del_t \rangle\big] \nonumber\\
    &\leq \frac{\eta\rho}{2}\mathbb E\big[\|\bar\del_t\|^2\big]+\frac{\eta}{2\rho} \mathbb E\big[\| \nabla f(\theta_t)-\nabla f(x_t) \|^2\big] \nonumber\\
    &\overset{(a)}{\leq} \frac{\eta^2}{2}\mathbb E\big[\|\bar\del_t\|^2\big]+\frac{\eta^2 L^2}{2} \mathbb E\big[\| \bar e_t \|^2\big]  \label{eq:sgd1} \\
    &\overset{(b)}{\leq} \frac{\eta^2 L}{2}\Big[ 2\eta_l^2 K^2(60\eta_l^2 K^2 L^2+1)\mathbb E\big[\|\nabla f(\theta_t)\|^2\big] +4\frac{\eta_l^2 K}{n}\sigma^2+20\eta_l^4K^3L^2(\sigma^2+6K\sigma_g^2) \Big] \nonumber\\
    &+\frac{\eta^2L}{2}\Big[ \frac{4q^2\eta_l^2 K^2 (60\eta_l^2 K^2 L^2+1)}{1-q^2}\sum_{\tau=1}^t (\frac{1+q^2}{2})^{t-\tau} \mathbb E\big[\|\nabla f(\theta_{\tau})\|^2\big] \nonumber \\
    &\hspace{1.8in} + \frac{16\eta_l^2 q^2K}{(1-q^2)^2 n}\sigma^2+\frac{80\eta_l^4q^2 K^3L^2}{(1-q^2)^2}(\sigma^2+6K\sigma_g^2) \Big],  \nonumber
\end{align}
where (a) uses Young's inequality and (b) uses Lemma~\ref{lemma:bound delta} and Lemma~\ref{lemma:bound e_t}. When taking telescoping sum of this term over $t=1,...,T$, again using the geometric summation trick, we further obtain
\begin{align*}
    &\eta\sum_{t=1}^T\mathbb E\big[\langle \nabla f(\theta_t)-\nabla f(x_t),\bar\del_t \rangle\big]\\
    &\leq \eta^2\eta_l^2C_1 K^2L(60\eta_l^2 K^2 L^2+1)\sum_{t=1}^T\mathbb E\big[\|\nabla f(\theta_t)\|^2\big] \\
    &\hspace{1.5in} +\frac{2T\eta^2\eta_l^2C_1 KL}{n}\sigma^2+10T\eta^2\eta_l^4C_1K^3L^3(\sigma^2+6K\sigma_g^2),
\end{align*}
where $C_1=1+\frac{4q^2}{(1-q^2)^2}$. Now, taking the summation over all terms in (\ref{sgd-eq0}), we get
\begin{align*}
    \mathbb E[f(x_{t+1})]-f(x_1)&\leq \Big(-\frac{\eta\eta_l K}{4} + \eta^2\eta_l^2(C_1+1) K^2L(60\eta_l^2 K^2 L^2+1)\Big)\sum_{t=1}^T\mathbb E\big[\|\nabla f(\theta_t)\|^2\big]
    &\hspace{1.5in} +\frac{2T\eta^2\eta_l^2C_1 KL}{n}\sigma^2 \\
    &\hspace{0.5in} +\frac{2T\eta^2\eta_l^2(C_1+1) KL}{n}\sigma^2+10T\eta^2\eta_l^4(C_1+1)K^3L^3(\sigma^2+6K\sigma_g^2).
\end{align*}
Since $\eta_l\leq \frac{1}{8KL}$, we know that $60\eta_l^2 K^2 L^2+1<2$. Therefore, provided that the local learning rate is such that
\begin{align*}
    \eta_l\leq  \frac{1}{2KL\cdot\max\{4,\eta(C_1+1)\}},
\end{align*}
we have
\begin{align*}
    \frac{\eta\eta_lK}{8}\sum_{t=1}^T\mathbb E\big[\|\nabla f(\theta_t)\|^2\big]&\leq f(x_1)-\mathbb E[f(x_{t+1})]  +\frac{2T\eta^2\eta_l^2(C_1+1) KL}{n}\sigma^2 \\
    &\hspace{1.2in} +10T\eta^2\eta_l^4(C_1+1)K^3L^3(\sigma^2+6K\sigma_g^2),
\end{align*}
leading to
\begin{align*}
    \frac{1}{T}\sum_{t=1}^T\mathbb E\big[\|\nabla f(\theta_t)\|^2\big]&\lesssim \frac{f(x_1)-\mathbb E[f(x_{t+1})]}{\eta\eta_l TK}+\frac{2\eta\eta_l(C_1+1) L}{n}\sigma^2 \\
    &\hspace{1.2in} +10\eta\eta_l^3(C_1+1)K^2L^3(\sigma^2+6K\sigma_g^2)\\
    &\leq \frac{f(\theta_1)-f(\theta^*)}{\eta\eta_l TK}+\frac{2\eta\eta_l(C_1+1) L}{n}\sigma^2 \\
    &\hspace{1.2in} +10\eta\eta_l^3(C_1+1)K^2L^3(\sigma^2+6K\sigma_g^2),
\end{align*}
which concludes the proof.
\end{proof}

\subsection{Intermediate Lemmas}\label{app:lemmas}

In our analysis, we will make use of the following lemma on the consensus error. Note that this is a general result holding for algorithms (both Fed-EF-SGD and Fed-EF-AMS) with local SGD steps.

\begin{Lemma}[\cite{reddi2021adaptive}] \label{lemma:consensus}
For $\eta_l\leq \frac{1}{8LK}$, for any round $t$, local step $k\in[K]$ and client $i\in[n]$, under Assumption~\ref{ass:smooth} to Assumption~\ref{ass:var}, it holds that
\begin{align*}
    \mathbb E\big[\|\theta_{t,i}^{(k)}-\theta_t\|^2\big]\leq 5\eta_l^2 K(\sigma^2+6K\sigma_g^2)+30\eta_l^2 K^2\mathbb E\big[\|\nabla f(\theta_t)\|^2\big].
\end{align*}
\end{Lemma}

We then state some results that bound several key ingredients in our analysis.

\begin{Lemma} \label{lemma:bound delta}
Recall $\bar\del_t=\frac{1}{n}\sum_{i=1}^n\del_{t,i}$. Under Assumption~\ref{ass:smooth} to Assumption~\ref{ass:var}, for $\forall t$, the following bounds hold:

1. Bound by local gradients:
\begin{align*}
    \mathbb E\big[\|\bar\del_t\|^2\big]&\leq  \frac{\eta_l^2}{n^2}\mathbb E\big[\| \sum_{i=1}^n \sum_{k=1}^K \nabla f_i(\theta_{t,i}^{(k)})\|^2\big] + \frac{\eta_l^2 K}{n}\sigma^2.
\end{align*}

2. Bound by global gradient:
\begin{align*}
    \mathbb E\big[\|\bar\del_t\|^2\big]\leq (2\eta_l^2 K^2+120\eta_l^4 K^4 L^2)\mathbb E\big[\|\nabla f(\theta_t)\|^2\big] +4\frac{\eta_l^2 K}{n}\sigma^2+20\eta_l^4K^3L^2(\sigma^2+6K\sigma_g^2).
\end{align*}

\end{Lemma}

\begin{proof}
By definition, we have
\begin{align*}
    \mathbb E\big[\|\bar\del_t\|^2\big]&= \mathbb E\big[\|\frac{1}{n}\sum_{i=1}^n \sum_{k=1}^K \eta_l g_{t,i}^{(k)}\|^2\big]\\
    &\leq \frac{\eta_l^2}{n^2}\mathbb E\big[\|\sum_{i=1}^n \sum_{k=1}^K (g_{t,i}^{(k)}-\nabla f_i(\theta_{t,i}^{(k)})\|^2\big] +\frac{\eta_l^2}{n^2}\mathbb E\big[\| \sum_{i=1}^n \sum_{k=1}^K \nabla f_i(\theta_{t,i}^{(k)})\|^2\big] \\
    &\leq \frac{\eta_l^2 K}{n}\sigma^2 + \frac{\eta_l^2}{n^2}\mathbb E\big[\| \sum_{i=1}^n \sum_{k=1}^K \nabla f_i(\theta_{t,i}^{(k)})\|^2\big],
\end{align*}
where the second line is due to the variance decomposition, and the last inequality uses Assumption~\ref{ass:var} on independent and unbiased stochastic gradients. This proves the first part. For the second part,
\begin{align*}
    \mathbb E\big[\|\bar\del_t\|^2\big]&= \mathbb E\big[\|\frac{1}{n}\sum_{i=1}^n \sum_{k=1}^K \eta_l g_{t,i}^{(k)}-K\eta_l \nabla f(\theta_t) +K\eta_l \nabla f(\theta_t)\|^2\big]\\
    &\leq 2\eta_l^2 K^2 \mathbb E\big[\|\nabla f(\theta_t)\|^2\big] + 2\eta_l^2 \mathbb E\big[\|\frac{1}{n}\sum_{i=1}^n \sum_{k=1}^K g_{t,i}^{(k)}-\frac{K}{n}\sum_{i=1}^n \nabla f_i(\theta_t)\|^2\big]\\
    &= 2\eta_l^2 K^2 \mathbb E\big[\|\nabla f(\theta_t)\|^2\big] + \frac{2\eta_l^2}{n^2} \mathbb E\big[\|\sum_{i=1}^n \sum_{k=1}^K (g_{t,i}^{(k)}-\nabla f_i(\theta_t))\|^2\big] \\
    &\leq 2\eta_l^2 K^2 \mathbb E\big[\|\nabla f(\theta_t)\|^2\big] + \frac{2\eta_l^2}{n^2} \mathbb E\big[\|\sum_{i=1}^n \sum_{k=1}^K (g_{t,i}^{(k)}-\nabla f_i(\theta_{t,i}^{(k)})+\nabla f_i(\theta_{t,i}^{(k)})- \nabla f_i(\theta_t))\|^2\big] \\
    &\leq 2\eta_l^2 K^2 \mathbb E\big[\|\nabla f(\theta_t)\|^2\big] + \frac{2\eta_l^2}{n^2} \underbrace{\mathbb E\big[\|\sum_{i=1}^n \sum_{k=1}^K (g_{t,i}^{(k)}-\nabla f_i(\theta_{t,i}^{(k)})+\nabla f_i(\theta_{t,i}^{(k)})- \nabla f_i(\theta_t))\|^2\big]}_{A}.
\end{align*}
The expectation $A$ can be further bounded as
\begin{align*}
    A&\leq 2\mathbb E\big[\|\sum_{i=1}^n \sum_{k=1}^K (g_{t,i}^{(k)}-\nabla f_i(\theta_{t,i}^{(k)}))\|^2\big]+2\mathbb E\big[\|\sum_{i=1}^n \sum_{k=1}^K (\nabla f_i(\theta_{t,i}^{(k)})- \nabla f_i(\theta_t))\|^2\big] \\
    & \overset{(a)}{\leq} 2nK \sigma^2 + 2nK\sum_{i=1}^n \sum_{k=1}^K \mathbb E\big[\| \nabla f_i(\theta_{t,i}^{(k)})- \nabla f_i(\theta_t) \|^2\big] \\
    &\overset{(b)}{\leq} 2nK \sigma^2 + 2nKL^2\sum_{i=1}^n \sum_{k=1}^K \mathbb E\big[\| \theta_{t,i}^{(k)}- \theta_t \|^2\big] \\
    &\overset{(c)}{\leq} 60\eta_l^2n^2 K^4 L^2\mathbb E\big[\|\nabla f(\theta_t)\|^2\big]+2nK\sigma^2+10\eta_l^2n^2K^3L^2(\sigma^2+6K\sigma_g^2),
\end{align*}
where (a) is implied by Assumption~\ref{ass:var} that each local stochastic gradient $g_{t,i}^{(k)}$ can be written as $g_{t,i}^{(k)}=\nabla f_i(\theta_{t,i}^{(k)})+\xi_{t,i}^{(k)}$, where $\xi_{t,i}^{k}$ is a zero-mean random noise with bounded variance $\sigma^2$, and all the noises for $t \in [T], i\in [n], k\in [K]$ are independent. The inequality (b) is due to the smoothness Assumption~\ref{ass:smooth}, and (c) follows from Lemma~\ref{lemma:consensus}. Therefore, we obtain
\begin{align*}
    \mathbb E[\|\bar\del_t\|^2]&\leq (2\eta_l^2 K^2+120\eta_l^4 K^4 L^2)\mathbb E[\|\nabla f(\theta_t)\|^2] +4\frac{\eta_l^2 K}{n}\sigma^2+20\eta_l^4K^3L^2(\sigma^2+6K\sigma_g^2),
\end{align*}
which completes the proof of the second claim.
\end{proof}

\begin{Lemma}  \label{lemma:inner-product}
Under Assumption \ref{ass:smooth} and Assumption~\ref{ass:var}, when $\eta_l\leq\frac{1}{8KL}$, Fed-EF-SGD admits
\begin{align*}
    \mathbb E\big[\langle \nabla f(\theta_t),-\bar\del_t+\eta_l K\nabla f(\theta_t) \rangle\big]\leq  \frac{3\eta_l K}{4}\mathbb E\big[\|\nabla f(\theta_t)\|^2\big]+\frac{5\eta_l^3 K^2 L^2}{2}(\sigma^2+6K\sigma_g^2).
\end{align*}
\end{Lemma}

\begin{proof}
It holds that
\begin{align*}
    &\mathbb E\big[\langle \nabla f(\theta_t),-\bar\del_t+\eta_l K\nabla f(\theta_t) \rangle\big] \\
    &=\big\langle \nabla f(\theta_t),\mathbb E\big[-\frac{1}{n}\sum_{i=1}^n \sum_{k=1}^K \eta_l g_{t,i}^{(k)}+\eta_l K\nabla f(\theta_t)\big] \big\rangle \\
    &= \big\langle \sqrt{\eta_l}\nabla f(\theta_t),\sqrt{\eta_l} \mathbb E\big[-\frac{1}{n}\sum_{i=1}^n \sum_{k=1}^K  \nabla f_i(\theta_t^{(k)})+ K\nabla f(\theta_t)\big] \big\rangle \\
    &\overset{(a)}{\leq} \frac{\eta_l K}{2}\mathbb E\big[\|\nabla f(\theta_t)\|^2\big]+\frac{\eta_l}{2K} \mathbb E\big[\| \frac{1}{n}\sum_{i=1}^n\sum_{k=1}^K (\nabla f_i(\theta_{t,i}^{(k)})-\nabla f_i(\theta_t)) \|^2\big] \\
    &\leq \frac{\eta_l K}{2}\mathbb E\big[\|\nabla f(\theta_t)\|^2\big]+\frac{\eta_l}{2n}  \sum_{i=1}^n\sum_{k=1}^K \mathbb E\big[\|\nabla f_i(\theta_{t,i}^{(k)})-\nabla f_i(\theta_t) \|^2\big] \\
    &\overset{(b)}{\leq} \frac{\eta_l K}{2}\mathbb E\big[\|\nabla f(\theta_t)\|^2\big]+\frac{\eta_l L^2}{2n}  \sum_{i=1}^n\sum_{k=1}^K \mathbb E\big[\|\theta_{t,i}^{(k)}-\theta_t \|^2\big] \\
    &\overset{(c)}{\leq} \frac{\eta_l K}{2}\mathbb E\big[\|\nabla f(\theta_t)\|^2\big]+\frac{\eta_l KL^2}{2}  \Big[5\eta_l^2 K(\sigma^2+6K\sigma_g^2)+30\eta_l^2 K^2\mathbb E\big[\|\nabla f(\theta_t)\|^2\Big]
\end{align*}
where (a) is due to $\langle a,b\rangle\leq \frac{\alpha}{2} a^2+\frac{1}{2\alpha}b^2$ for any $a,b\in\mathbb R$ and $\alpha>0$, , (b) is a consequence of Assumption~\ref{ass:smooth}, and (c) is due to Lemma~\ref{lemma:consensus}. If $\eta_l\leq\frac{1}{8KL}$, we have that $\eta_l^2\leq\frac{1}{64K^2L^2}$, bounding the last term by $\frac{15}{64}\eta_l K\mathbb E\big[\|\nabla f(\theta_t)\|^2\big]$. Hence, we obtain
\begin{align*}
    \mathbb E\big[\langle \nabla f(\theta_t),-\bar\del_t+\eta_l K\nabla f(\theta_t) \rangle\big]\leq  \frac{47\eta_l K}{64}\mathbb E\big[\|\nabla f(\theta_t)\|^2\big]+\frac{5\eta_l^3 K^2 L^2}{2}(\sigma^2+6K\sigma_g^2),
\end{align*}
where the proof is completed since $\frac{47}{64}<\frac{3}{4}$.
\end{proof}

\begin{Lemma} \label{lemma:m_t,m_t'}
Under Assumption~\ref{ass:smooth} to Assumption~\ref{ass:var} we have:
\begin{align*}
    &\|m_t'\|\leq \eta_l KG, \quad \text{for}\ \forall t,\\
    &\sum_{t=1}^T\mathbb E\big[\|m_t'\|^2\big]\leq (2\eta_l^2 K^2+120\eta_l^4 K^4 L^2)\sum_{t=1}^T \mathbb E\big[\|\nabla f(\theta_t)\|^2\big] + \\
    &\hspace{2in} +4\frac{T\eta_l^2 K}{n}\sigma^2+20T\eta_l^4K^3L^2(\sigma^2+6K\sigma_g^2) .
\end{align*}
\end{Lemma}

\begin{proof}
For the first part, by Assumption~\ref{ass:boundgrad} we know that
\begin{align*}
    \|m_t'\|&=(1-\beta_1)\|\sum_{\tau=1}^t \beta_1^{t-\tau} \bar \del_t\|\\
    &=(1-\beta_1)\sum_{\tau=1}^t \beta_1^{t-\tau} \frac{\eta_l}{n}\sum_{i=1}^n\sum_{k=1}^K \|g_{t,i}^{(k)}\|\\
    &\leq \eta_l KG.
\end{align*}
For the second claim, by Lemma~\ref{lemma:bound delta} we know that
\begin{align*}
    \mathbb E\big[\|\bar\del_t\|^2\big]&\leq (2\eta_l^2 K^2+120\eta_l^4 K^4 L^2)\mathbb E\big[\|\nabla f(\theta_t)\|^2 \big] +4\frac{\eta_l^2 K}{n}\sigma^2+20\eta_l^4K^3L^2(\sigma^2+6K\sigma_g^2).
\end{align*}
Let $\bar \del_{t,j}$ denote the $j$-th coordinate of $\bar \del_t$. By the updating rule of Fed-EF, we have
\begin{align*}
    \mathbb E\big[\|m_t'\|^2\big]&=\mathbb E\big[\|(1-\beta_1)\sum_{\tau=1}^t\beta_1^{t-\tau} \bar \del_\tau\|^2\big]\\
    &\leq (1-\beta_1)^2\sum_{j=1}^d \mathbb E\big[(\sum_{\tau=1}^t\beta_1^{t-\tau} \bar \del_{\tau,j})^2\big]\\
    &\overset{(a)}{\leq} (1-\beta_1)^2\sum_{j=1}^d \mathbb E\big[(\sum_{\tau=1}^t\beta_1^{t-\tau})(\sum_{\tau=1}^t\beta_1^{t-\tau} \bar \del_{\tau,j}^2)\big]\\
    &\leq (1-\beta_1)\sum_{\tau=1}^t \beta_1^{t-\tau}\mathbb E\big[\|\bar \del_\tau\|^2\big]\\
    &\leq (2\eta_l^2 K^2+120\eta_l^4 K^4 L^2) (1-\beta_1)\sum_{\tau=1}^t \beta_1^{t-\tau}\mathbb E\big[\|\nabla f(\theta_t)\|^2\big] \\
    &\hspace{2in} +4\frac{\eta_l^2 K}{n}\sigma^2+20\eta_l^4K^3L^2(\sigma^2+6K\sigma_g^2) ,
\end{align*}
where (a) is due to Cauchy-Schwartz inequality. Summing over $t=1,...,T$, we obtain
\begin{align*}
    \sum_{t=1}^T\mathbb E\big[\|m_t'\|^2\big]&\leq (2\eta_l^2 K^2+120\eta_l^4 K^4 L^2) (1-\beta)\sum_{t=1}^T\sum_{\tau=1}^t\beta_1^{t-\tau} \mathbb E\big[\|\nabla f(\theta_t)\|^2\big] + \\
    &\hspace{2in} +4\frac{T\eta_l^2 K}{n}\sigma^2+20T\eta_l^4K^3L^2(\sigma^2+6K\sigma_g^2) \\
    &\leq (2\eta_l^2 K^2+120\eta_l^4 K^4 L^2)\sum_{t=1}^T \mathbb E\big[\|\nabla f(\theta_t)\|^2\big] + \\
    &\hspace{2in} +4\frac{T\eta_l^2 K}{n}\sigma^2+20T\eta_l^4K^3L^2(\sigma^2+6K\sigma_g^2) ,
\end{align*}
which concludes the proof.
\end{proof}

\begin{Lemma} \label{lemma:bound e_t}
Under Assumption~\ref{ass:var}, we have for $\forall t$ and each local worker $\forall i\in [n]$,
\begin{align*}
    &\|e_{t,i}\|^2\leq \frac{4\eta_l^2 q^2K^2G^2}{(1-q^2)^2},\ \forall t,\\
    &\mathbb E[\|\bar e_{t+1}\|^2]\leq \frac{4q^2\eta_l^2 K^2 (60\eta_l^2 K^2 L^2+1)}{1-q^2}\sum_{\tau=1}^t (\frac{1+q^2}{2})^{t-\tau} \mathbb E\big[\|\nabla f(\theta_{\tau})\|^2\big] \\
    &\hspace{1.5in} + \frac{16\eta_l^2 q^2K}{(1-q^2)^2 n}\sigma^2+\frac{80\eta_l^4q^2 K^3L^2}{(1-q^2)^2}(\sigma^2+6K\sigma_g^2).
\end{align*}
\end{Lemma}

\begin{proof}
To prove the second claim, we start by using Assumption~\ref{ass:compress_diff} and Young's inequality to get
\begin{align}
    \|\bar e_{t+1}\|^2&=\|\bar \del_t+\bar e_t-\frac{1}{n}\sum_{i=1}^n\mathcal C(\del_{t,i}+e_{t,i})\|^2 \nonumber\\
    &\leq q^2\|\bar\del_{t}+\bar e_{t}\|^2 \nonumber\\
    &\leq q^2(1+\rho)\|\bar e_{t}\|^2+q^2(1+\frac{1}{\rho})\|\bar\del_{t,i}\|^2 \nonumber\\
    &\leq \frac{1+q^2}{2}\|\bar e_{t}\|^2 + \frac{2q^2}{1-q^2}\|\bar\del_{t}\|^2, \label{eq:e_t 0}
\end{align}
where \eqref{eq:e_t 0} is derived by choosing $\rho=\frac{1-q^2}{2q^2}$ and the fact that $q<1$. Now by recursion and the initialization $e_{1,i}=0$, $\forall i$, we have
\begin{align*}
    \mathbb E\big[\|\bar e_{t+1}\|^2\big]&\leq \frac{2q^2}{1-q^2} \sum_{\tau=1}^t (\frac{1+q^2}{2})^{t-\tau} \mathbb E\big[\|\bar\del_{\tau}\|^2\big]  \\
    &\leq \frac{4q^2\eta_l^2 K^2 (60\eta_l^2 K^2 L^2+1)}{1-q^2}\sum_{\tau=1}^t (\frac{1+q^2}{2})^{t-\tau} \mathbb E\big[\|\nabla f(\theta_{\tau})\|^2\big] \\
    &\hspace{1.5in} + \frac{16\eta_l^2 q^2K}{(1-q^2)^2 n}\sigma^2+\frac{80\eta_l^4q^2 K^3L^2}{(1-q^2)^2}(\sigma^2+6K\sigma_g^2), \nonumber
\end{align*}
which proves the second argument, where we use Lemma~\ref{lemma:consensus} to bound the local model update. In addition, we know that $\|\del_t\|\leq \eta_l KG$ by Assumption~\ref{ass:boundgrad} for any $t$.

The absolute bound $\|e_{t,i}\|^2\leq \frac{4\eta_l^2 q_{\mathcal C}^2K^2G^2}{(1-q_{\mathcal C}^2)^2}$ follows from \eqref{eq:e_t 0} by a similar recursion argument used on local error $e_{t,i}$, and the fact that $q_{\mathcal C}\leq \max\{q_{\mathcal C}, q_{\mathcal A}\}=q$.
\end{proof}

\begin{Lemma} \label{lemma:bound big E_t}
For the moving average error sequence $\mathcal E_t$, it holds that
\begin{align*}
    & \|\mathcal E_{t}\|^2\leq \frac{4\eta_l^2 q^2K^2G^2}{(1-q^2)^2},\quad \text{for}\ \forall t,\\
    &\sum_{t=1}^T \mathbb E\big[\|\mathcal E_t\|^2\big]\leq \frac{8q^2\eta_l^2 K^2 (60\eta_l^2 K^2 L^2+1)}{(1-q^2)^2} \sum_{t=1}^T \mathbb E\big[\|\nabla f(\theta_\tau)\|^2\big] \\
    &\hspace{1.3in}+\frac{16T\eta_l^2 q^2K}{(1-q^2)^2 n}\sigma^2+\frac{80T\eta_l^4q^2 K^3L^2}{(1-q^2)^2}(\sigma^2+6K\sigma_g^2).
\end{align*}
\end{Lemma}

\begin{proof}
The first argument can be easily deduced by the definition of $\mathcal E_t$ that
\begin{align*}
    \|\mathcal E_{t}\|&=(1-\beta_1)\|\sum_{\tau=1}^t\beta_1^{t-\tau}\bar e_t\|\\
    &\leq \|e_{t,i}\|\leq \frac{2\eta_l qKG}{1-q^2}.
\end{align*}
Denote the quantity
$$K_{t}\eqdef \sum_{\tau=1}^t (\frac{1+q^2}{2})^{t-\tau} \mathbb E\big[\|\nabla f(\theta_\tau)\|^2\big].$$
By the same technique as in the proof of Lemma~\ref{lemma:m_t,m_t'}, denoting $\bar e_{t,j}$ as the $j$-th coordinate of $\bar e_{t}$, we can bound the accumulated error sequence by
\begin{align*}
    \mathbb E\big[\|\mathcal E_t\|^2\big]&=\mathbb E\big[\|(1-\beta_1)\sum_{\tau=1}^t\beta_1^{t-\tau} \bar e_\tau\|^2\big]\\
    &\leq (1-\beta_1)^2\sum_{j=1}^d \mathbb E\big[(\sum_{\tau=1}^t\beta_1^{t-\tau} \bar e_{\tau,j})^2\big]\\
    &\overset{(a)}{\leq} (1-\beta_1)^2\sum_{j=1}^d \mathbb E\big[(\sum_{\tau=1}^t\beta_1^{t-\tau})(\sum_{\tau=1}^t\beta_1^{t-\tau} \bar e_{\tau,j}^2)\big]\\
    &\leq (1-\beta_1)\sum_{\tau=1}^t \beta_1^{t-\tau}\mathbb E\big[\|\bar e_\tau\|^2\big]\\
    &\overset{(b)}{\leq} \frac{16\eta_l^2 q^2K}{(1-q^2)^2 n}\sigma^2+\frac{80\eta_l^4q^2 K^3L^2}{(1-q^2)^2}(\sigma^2+6K\sigma_g^2) \\
    &\hspace{1.3in} +\frac{4(1-\beta_1)q^2\eta_l^2 K^2 (60\eta_l^2 K^2 L^2+1)}{1-q^2} \sum_{\tau=1}^t \beta_1^{t-\tau} K_{\tau},
\end{align*}
where (a) is due to Cauchy-Schwartz inequality and (b) is a result of Lemma~\ref{lemma:bound e_t}. Summing over $t=1,...,T$ and using the technique of geometric series summation leads to
\begin{align*}
    \sum_{t=1}^T \mathbb E\big[\|\mathcal E_t\|^2\big]&\leq \frac{16T\eta_l^2 q^2K}{(1-q^2)^2 n}\sigma^2+\frac{80T\eta_l^4q^2 K^3L^2}{(1-q^2)^2}(\sigma^2+6K\sigma_g^2) \\
    &\hspace{1in}+ \frac{4(1-\beta_1)q^2\eta_l^2 K^2 (60\eta_l^2 K^2 L^2+1)}{1-q^2} \sum_{t=1}^T \sum_{\tau=1}^t \beta_1^{t-\tau} K_{\tau}\\
    &\leq \frac{16T\eta_l^2 q^2K}{(1-q^2)^2 n}\sigma^2+\frac{80T\eta_l^4q^2 K^3L^2}{(1-q^2)^2}(\sigma^2+6K\sigma_g^2) \\
    &\hspace{1in}+ \frac{4q^2\eta_l^2 K^2 (60\eta_l^2 K^2 L^2+1)}{1-q^2} \sum_{t=1}^T \sum_{\tau=1}^t (\frac{1+q^2}{2})^{t-\tau} \mathbb E\big[\|\nabla f(\theta_\tau)\|^2\big] \\
    &\leq \frac{16T\eta_l^2 q^2K}{(1-q^2)^2 n}\sigma^2+\frac{80T\eta_l^4q^2 K^3L^2}{(1-q^2)^2}(\sigma^2+6K\sigma_g^2) \\
    &\hspace{1in}+ \frac{8q^2\eta_l^2 K^2 (60\eta_l^2 K^2 L^2+1)}{(1-q^2)^2} \sum_{t=1}^T \mathbb E\big[\|\nabla f(\theta_\tau)\|^2\big].
\end{align*}
The desired result is obtained.
\end{proof}

\newpage
\begin{Lemma} \label{lemma:bound v_t}
It holds that $\forall t\in [T]$, $\forall i\in [d]$, $\hat v_{t,i}\leq \frac{4\eta_l^2(1+q^2)^3K^2}{(1-q^2)^2}G^2$.
\end{Lemma}

\begin{proof}
For any $t$, by Lemma~\ref{lemma:bound e_t} and Assumption~\ref{ass:boundgrad} we have
\begin{align*}
    \|\widetilde \del_t\|^2&=\|\mathcal C(\del_t+e_t)\|^2\\
    &\leq \|\mathcal C(\del_t+e_t)-(\del_t+e_t)+(\del_t+e_t)\|^2\\
    &\leq 2(q^2+1)\|\del_t+e_t\|^2\\
    &\leq 4(q^2+1)(\eta_l^2K^2G^2+\frac{4\eta_l^2 q^2K^2G^2}{(1-q^2)^2})\\
    &=\frac{4\eta_l^2(1+q^2)^3K^2}{(1-q^2)^2}G^2.
\end{align*}
Consider the updating rule of $\hat v_t=\max\{v_t,\hat v_{t-1}\}$. We know that there exists a $j\in[t]$ such that $\hat v_t=v_j$. Thus, we have
\begin{align*}
    \hat v_{t,i}=(1-\beta_2)\sum_{\tau=1}^j \beta_2^{j-\tau} \tilde g_{t,i}^2\leq \frac{4\eta_l^2(1+q^2)^3K^2}{(1-q^2)^2}G^2,
\end{align*}
which proves the claim.
\end{proof}

The next Lemma is analogue to Lemma 5 in \cite{li2022distributed}.

\begin{Lemma}  \label{lemma:bound difference}
Let $D_t\eqdef \frac{1}{\sqrt{\hat v_{t-1}+\epsilon}}-\frac{1}{\sqrt{\hat v_t+\epsilon}}$ be defined as above. Then,
\begin{align*}
    &\sum_{t=1}^T \|D_t\|_1 \leq \frac{d}{\sqrt\epsilon},\quad  \sum_{t=1}^T \|D_t\|^2 \leq \frac{d}{\epsilon}.
\end{align*}
\end{Lemma}

\begin{proof}
By the updating rule of Fed-EF-AMS, $\hat v_{t-1}\leq \hat v_t$ for $\forall t$. Therefore, by the initialization $\hat v_0=0$, we have
\begin{align*}
    \sum_{t=1}^T \|D_t\|_1 &=\sum_{t=1}^T \sum_{i=1}^d (\frac{1}{\sqrt{\hat v_{t-1,i}+\epsilon}}-\frac{1}{\sqrt{\hat v_{t,i}+\epsilon}})\\
    &=\sum_{i=1}^d (\frac{1}{\sqrt{\hat v_{0,i}+\epsilon}}-\frac{1}{\sqrt{\hat v_{T,i}+\epsilon}})\\
    &\leq \frac{d}{\sqrt\epsilon}.
\end{align*}
For the sum of squared $l_2$ norm, note the fact that for $a\geq b>0$, it holds that
\begin{equation*}
    (a-b)^2\leq (a-b)(a+b)=a^2-b^2.
\end{equation*}
Thus,
\begin{align*}
    \sum_{t=1}^T \|D_t\|^2&=\sum_{t=1}^T \sum_{i=1}^d (\frac{1}{\sqrt{\hat v_{t-1,i}+\epsilon}}-\frac{1}{\sqrt{\hat v_{t,i}+\epsilon}})^2\\
    &\leq \sum_{t=1}^T \sum_{i=1}^d (\frac{1}{\hat v_{t-1,i}+\epsilon}-\frac{1}{\hat v_{t,i}+\epsilon})\\
    &\leq \frac{d}{\epsilon},
\end{align*}
which gives the desired result.
\end{proof}

\subsection{Proof of Theorem~\ref{theo:partial-simple}: Partial Participation} \label{app sec:partial-simple}

\begin{proof}
We can use a similar proof structure as previous analysis for full participation Fed-EF-SGD as in Section~\ref{app sec:SGD}. Like before, we first define the following virtual iterates:
\begin{align}
    x_{t+1} = \theta_{t+1}-\eta \frac{1}{m}\sum_{i=1}^n e_{t+1,i} &=  \theta_t-\eta\overline{\widetilde\del}_{t, \mathcal M_t}- \eta\frac{1}{m}\sum_{i\in \mathcal M_t}e_{t+1,i}-\eta\frac{1}{m}\sum_{i\notin \mathcal M_t}e_{t+1,i} \nonumber\\
    &=\theta_t-\eta\bar\del_{t,\mathcal M_t} -\eta\frac{1}{m}\sum_{i\in \mathcal M_t}e_{t,i}-\eta\frac{1}{m}\sum_{i\notin \mathcal M_t}e_{t,i} \label{eq:virtual-SGD-partial-key}\\
    &= \theta_t-\eta\bar\del_{t,\mathcal M_t} -\eta \frac{1}{m}\sum_{i=1}^n e_{t,i}  \nonumber\\
    &=x_t -\eta\bar\del_{t,\mathcal M_t}. \nonumber
\end{align}
Here, (\ref{eq:virtual-SGD-partial-key}) follows from the partial participation setup where there is no error accumulation for inactive clients. The smoothness of loss functions implies
\begin{align*}
    f(x_{t+1})\leq f(x_t)+\langle \nabla f(x_t), x_{t+1}-x_t\rangle+\frac{L}{2}\| x_{t+1}-x_t\|^2.
\end{align*}
Taking expectation w.r.t. the randomness at round $t$, we have
\begin{align}
    &\mathbb E[f(x_{t+1})]-f(x_t) \nonumber\\
    &\leq -\eta\mathbb E\big[\langle \nabla f(x_t), \bar \del_{t,\mathcal M_t}\rangle \big] +\frac{\eta^2L}{2}\mathbb E\big[\|\bar \del_{t,\mathcal M_t}\|^2\big] \nonumber\\
    &=\underbrace{-\eta\mathbb E\big[\langle\nabla f(\theta_t),\bar\del_{t,\mathcal M_t}\rangle\big]}_{I} +  +\underbrace{\frac{\eta^2L}{2}\mathbb E\big[\|\bar \del_{t,\mathcal M_t}\|^2\big]}_{II}+\underbrace{\eta\mathbb E\big[\langle\nabla f(x_t)-\nabla f(\theta_t),\bar\del_{t,\mathcal M_t}\rangle\big]}_{III}. \label{eq0:partial_simple}
\end{align}
Note that the expectation is also with respect to the randomness in the client sampling procedure. For the first term, we may adopt the similar idea of the proof of Lemma~\ref{lemma:inner-product}. Since the client sampling is random, we have $\mathbb E[\bar\del_{t,\mathcal M_t}]=\mathbb E[\bar\del_{t}]$. Thus,
\begin{align*}
    \bm I&=-\eta\mathbb E\big[\langle\nabla f(\theta_t),\bar\del_{t,\mathcal M_t}\rangle\big] \\
    &=-\eta\mathbb E\big[\langle\nabla f(\theta_t),\bar\del_{t}-\eta_l K\nabla f(\theta_t)+\eta_l K\nabla f(\theta_t)\rangle\big] \\
    &=-\eta\eta_l K\mathbb E\big[\| \nabla f(\theta_t) \|^2\big]+\eta\eta_l \big\langle \sqrt K \nabla f(\theta_t), -\frac{1}{n\sqrt K}\sum_{i=1}^n\sum_{k=1}^K \big(\nabla f_i(\theta_{t,i}^{(k)}) + \nabla f(\theta_t)\big) \big\rangle \\
    &\overset{(a)}{=} -\eta\eta_l K\mathbb E\big[\| \nabla f(\theta_t) \|^2\big]+\eta\eta_l \mathbb E\Big[ \frac{K}{2}\|\nabla f(\theta_t) \|^2\\
    &\hspace{1in} +\frac{1}{2Kn^2}\|\sum_{i=1}^n\sum_{k=1}^K (\nabla f_i(\theta_{t,i}^{(k)}) + \nabla f(\theta_t))\|^2 - \frac{1}{2Kn^2}\|\sum_{i=1}^n\sum_{k=1}^K \nabla f_i(\theta_{t,i}^{(k)})\|^2\Big]\\
    &\overset{(b)}{\leq} -\eta\eta_l K\mathbb E\big[\| \nabla f(\theta_t) \|^2\big]+ \frac{\eta\eta_lK}{2}\mathbb E\big[\|\nabla f(\theta_t) \|^2\big]\\
    &\hspace{0.3in} +\frac{\eta\eta_l KL^2}{2}\Big[5\eta_l^2 K(\sigma^2+6K\sigma_g^2)+30\eta_l^2 K^2\mathbb E\big[\|\nabla f(\theta_t)\|^2\big]\Big] - \frac{\eta\eta_l}{2Kn^2}\mathbb E\big[\|\sum_{i=1}^n\sum_{k=1}^K \nabla f_i(\theta_{t,i}^{(k)})\|^2\big]\\
    &\leq -\frac{\eta\eta_l K}{4}\mathbb E\big[\| \nabla f(\theta_t) \|^2\big]+\frac{5\eta\eta_l^3 K^2 L^2}{2}(\sigma^2+6K\sigma_g^2)- \frac{\eta\eta_l}{2Kn^2}\mathbb E\big[\|\sum_{i=1}^n\sum_{k=1}^K \nabla f_i(\theta_{t,i}^{(k)})\|^2\big],
\end{align*}
when $\eta_l\leq \frac{1}{8KL}$, where (a) is a result of the fact that $\langle z_1, z_2\rangle=\|z_1\|^2+\|z_2\|^2-\|z_1-z_2\|^2$, (b) is because of Lemma~\ref{lemma:consensus}.

\vspace{0.1in}
\noindent For term II, we have
\begin{align*}
    \bm{II}&\leq \frac{\eta^2\eta_l^2 KL}{2m}\sigma^2 + \frac{\eta^2\eta_l^2L}{2n(n-1)}\mathbb E\big[\|\sum_{i=1}^n\sum_{k=1}^K \nabla f_i(\theta_{t,i}^{(k)})\|\big]^2\\
    &\hspace{0.2in} + C'\Big[\frac{3\eta^2\eta_l^2K^2L(30\eta_l^2K^2L^2+1)}{2m} \mathbb E\big[\|\nabla f(\theta_t)\|^2\big]+\frac{15\eta^2\eta_l^4K^3L^3}{2m}(\sigma^2+6K\sigma_g^2)+\frac{3\eta^2\eta_l^2K^2L}{2m}\sigma_g^2 \Big],
\end{align*}
with $C'=\frac{n-m}{n-1}$. Furthermore, we have that

\begin{align*}
    \bm{III}&\leq 2\eta^2L \mathbb E\big[\|\frac{1}{m}\sum_{i=1}^n e_{t,i}\|^2 \big]+\frac{\eta^2L}{2} \mathbb E\big[ \|\bar\del_{t,\mathcal M_t}\|^2 \big].
\end{align*}
The second term is the same as term II. We now bound the first term. Denote $\tilde e_{t,i}=e_{t,i}+\del_{t,i}-\widetilde\del_{t,i}$, we have

\begin{align*}
    \mathbb E\big[\|\frac{1}{m}\sum_{i=1}^n e_{t+1,i}\|^2 \big]&=\frac{1}{m^2}\mathbb E_{\bm e_t}\Big[\underbrace{\mathbb E_{\mathcal M_t}\big[\| \sum_{i=1}^n \mathbbm 1\{i \in \mathcal M_t\} \tilde e_{t,i}+ \sum_{i=1}^n \mathbbm 1\{i \notin \mathcal M_t\} e_{t,i} \|^2  \big| \bm e_t}_{A} \big]\Big].
\end{align*}
By the updating rule of $e_{t,i}$, the inner expectation, conditional on $\bm e_t=(e_{t,1},..., e_{t,n})^T$, can be computed as
\begin{align*}
    A&= \frac{m}{n}\sum_{i=1}^n \|\tilde e_{i,t}\|^2 + \frac{m(m-1)}{n(n-1)}\sum_{i\neq j}^{n}\tilde e_{t,i}\tilde e_{t,j} +\frac{n-m}{n}\sum_{i=1}^n \|e_{i,t}\|^2 \\
    &\hspace{1.2in} + \frac{(n-m)(n-m-1)}{n(n-1)}\sum_{i\neq j}^{n}e_{t,i} e_{t,j} + \frac{m(n-m)}{n(n-1)}\sum_{i\neq j}^{n} \tilde e_{t,i}e_{t,j} \\
    &=\frac{m}{n} \| \sum_{i=1}^n\tilde e_{t,i} \|^2-\frac{m(n-m)}{n(n-1)}\sum_{i\neq j}^{n}\tilde e_{t,i}\tilde e_{t,j}+\frac{n-m}{n} \|\sum_{i=1}^n e_{i,t}\|^2 \\
    &\hspace{1.5in} - \frac{m(n-m)}{n(n-1)}\sum_{i\neq j}^{n}e_{t,i} e_{t,j} + \frac{m(n-m)}{n(n-1)}\sum_{i\neq j}^{n} \tilde e_{t,i}e_{t,j}\\
    &=\frac{m}{n} \| \sum_{i=1}^n\tilde e_{t,i} \|^2+\frac{n-m}{n} \|\sum_{i=1}^n e_{i,t}\|^2-\frac{m(n-m)}{n(n-1)}\sum_{i\neq j}^{n} (\tilde e_{t,i}\tilde e_{t,j}+ e_{t,i}e_{t,j}-\tilde e_{t,i}e_{t,j}) \\
    &=\frac{m}{n} \| \sum_{i=1}^n\tilde e_{t,i} \|^2+\frac{n-m}{n} \|\sum_{i=1}^n e_{i,t}\|^2 \\
    &\hspace{1.2in} -\frac{m(n-m)}{n(n-1)}\|\sum_{i=1}^n (\tilde e_{t,i}-e_{t,i})\|^2+\frac{m(n-m)}{n(n-1)}\sum_{i=1}^n(\|\tilde e_{t,i}\|^2 + \|e_{t,i}\|^2) \\
    &\leq \frac{m}{n} \| \sum_{i=1}^n\tilde e_{t,i} \|^2+\frac{n-m}{n} \|\sum_{i=1}^n e_{t,i}\|^2+\frac{m(n-m)}{n(n-1)}\sum_{i=1}^n(\|\tilde e_{t,i}\|^2 + \|e_{t,i}\|^2).
\end{align*}

Therefore, by Definition~\ref{def:quant} and Assumption~\ref{ass:compress_diff} we obtain
\begin{align*}
    &\mathbb E\big[\|\frac{1}{m}\sum_{i=1}^ne_{t+1,i}\|^2\big]\\
    &\leq \frac{mq^2}{n}\mathbb E[\|\frac{1}{m}\sum_{i\in\mathcal G}(e_{t,i}+\del_{t,i})\|^2] + \frac{n-m}{n}\mathbb E[\|\frac{1}{m} \sum_{i=1}^ne_{t,i} \|^2] \\
    &\hspace{1.4in} +\frac{(n-m)}{mn(n-1)}\sum_{i=1}^n ((2q^2+1)\|e_{t,i}\|^2+2q^2\|\del_{t,i}\|^2) \\
    &\leq \frac{m(1+\gamma)q^2+(n-m)}{n} \mathbb E[\|\frac{1}{m}\sum_{i=1}^ne_{t,i}\|^2+\frac{m(1+1/\gamma)q^2}{n}\mathbb E[\|\frac{1}{m}\sum_{i=1}^n\del_{t,i}\|^2]\\
    &\hspace{1.4in} +\frac{(n-m)}{mn(n-1)}\sum_{i=1}^n ((2q^2+1)\|e_{t,i}\|^2+2q^2\|\del_{t,i}\|^2).
\end{align*}
We have by Lemma~\ref{lemma:consensus} that
\begin{align*}
    \mathbb E[\|e_{t,i}\|^2]&\leq \frac{20q^2\eta_l^2 K}{(1-q^2)^2}(\sigma^2+6K\sigma_g^2)+\frac{60\eta_l^2 q^2K^2}{1-q^2}\sum_{\tau=1}^t (\frac{1+q^2}{2})^{t-\tau} \mathbb E\big[\|\nabla f(\theta_{\tau})\|^2\big],\\
    \mathbb E[\|\del_{t,i}\|^2]&\leq 5\eta_l^2 K(\sigma^2+6K\sigma_g^2)+30\eta_l^2 K^2\mathbb E\big[\|\nabla f(\theta_t)\|^2\big],
\end{align*}
which implies
\begin{align*}
    &\frac{(n-m)}{mn(n-1)}\sum_{i=1}^n ((2q^2+1)\|e_{t,i}\|^2+2q^2\|\del_{t,i}\|^2)\\
    &\leq \frac{(n-m)}{m(n-1)}\Big[ \frac{70q^2\eta_l^2 K}{(1-q^2)^2}(\sigma^2+6K\sigma_g^2) +\frac{180\eta_l^2 q^2K^2}{1-q^2}\sum_{\tau=1}^t (\frac{1+q^2}{2})^{t-\tau} \mathbb E\big[\|\nabla f(\theta_{\tau})\|^2\big] \\
    &\hspace{1.5in} + \frac{60\eta_l^2 q^2 K^2}{(1-q^2)^2}\mathbb E\big[\|\nabla f(\theta_t)\|^2\big] \Big].
\end{align*}
Recall $q=\max\{q_{\mathcal A},q_{\mathcal C}\}$. Let $\gamma=(1-q^2)/2q^2$. We have
\begin{align*}
    &\frac{m(1+\gamma)q^2+(n-m)}{n}=1-\frac{(1-q^2)m}{2n}<1, \\
    &\frac{m(1+1/\gamma)q^2}{n}=\frac{m(1+q^2)q^2}{n(1-q^2)}\leq \frac{2mq^2}{n(1-q^2)}.
\end{align*}
By the recursion argument used before, applying Lemma~\ref{lemma:bound delta} (adjusted by an $n^2/m^2$ factor) we obtain
\begin{align*}
    &\mathbb E\big[\|\frac{1}{m}\sum_{i=1}^ne_{t+1,i}\|^2\big]\\
    &\leq \frac{2mq^2}{n(1-q^2)}\sum_{\tau=1}^t\big( 1-\frac{(1-q^2)m}{2n} \big)^{t-\tau} \frac{n^2}{m^2}\big[ \frac{\eta_l^2}{n^2}\mathbb E\big[\| \sum_{i=1}^n \sum_{k=1}^K \nabla f_i(\theta_{\tau,i}^{(k)})\|^2\big] + \frac{\eta_l^2 K}{n}\sigma^2 \big]\\
    &+\frac{2n(n-m)}{(1-q^2)m^2(n-1)}\Big[ \frac{70q^2\eta_l^2 K}{(1-q^2)^2}(\sigma^2+6K\sigma_g^2) +\frac{180\eta_l^2 q^2K^2}{1-q^2}\sum_{\tau=1}^t (\frac{1+q^2}{2})^{t-\tau} \mathbb E\big[\|\nabla f(\theta_{\tau})\|^2\big] \\
    &\hspace{2in} + \frac{60\eta_l^2 q^2 K^2}{(1-q^2)^2}\mathbb E\big[\|\nabla f(\theta_t)\|^2\big] \Big]\\
    &\leq \frac{2\eta_l^2q^2}{(1-q^2)mn} \sum_{\tau=1}^t\big( 1-\frac{(1-q^2)m}{2n} \big)^{t-\tau}\mathbb E\big[\| \sum_{i=1}^n \sum_{k=1}^K \nabla f_i(\theta_{\tau,i}^{(k)})\|^2\big] + \frac{4\eta_l^2q^2Kn}{(1-q^2)^2m^2}\sigma^2\\
    &\hspace{0.2in}+\frac{280\eta_l^2q^2(n-m)K}{(1-q^2)^3m^2}(\sigma^2+6K\sigma_g^2) + \frac{720\eta_l^2q^2(n-m)K^2}{(1-q^2)^2m^2}\sum_{\tau=1}^t (\frac{1+q^2}{2})^{t-\tau} \mathbb E\big[\|\nabla f(\theta_{\tau})\|^2\big]\\
    &\hspace{2in} + \frac{240\eta_l^2q^2(n-m)K^2}{(1-q^2)^3m^2}\mathbb E\big[ \|\nabla f(\theta_t)\|^2 \big].
\end{align*}
Summing over $t=1,...,T$ gives
\begin{align*}
    &\sum_{t=1}^T\mathbb E\big[\|\frac{1}{m}\sum_{i=1}^ne_{t+1,i}\|^2\big]\\
    &\leq \frac{4\eta_l^2q^2}{(1-q^2)^2m^2}\sum_{t=1}^T\mathbb E\big[\| \sum_{i=1}^n \sum_{k=1}^K \nabla f_i(\theta_{t,i}^{(k)})\|^2\big] + \frac{4T\eta_l^2q^2Kn}{(1-q^2)^2m^2}\sigma^2\\
    &\hspace{0.3in}+\frac{280T\eta_l^2q^2(n-m)K}{(1-q^2)^3m^2}(\sigma^2+6K\sigma_g^2) + \frac{1680\eta_l^2q^2(n-m)K^2}{(1-q^2)^3m^2}\sum_{t=1}^T \mathbb E\big[\|\nabla f(\theta_{t})\|^2\big].
\end{align*}
Now we turn back to (\ref{eq0:partial_simple}). By taking the telescoping sum, we have
\begin{align*}
    &\mathbb E[f(x_{t+1})]-f(x_1)\\
    &\leq -\frac{\eta\eta_l K}{4}\sum_{t=1}^T\mathbb E\big[\| \nabla f(\theta_t) \|^2\big]+\frac{5T\eta\eta_l^3 K^2 L^2}{2}(\sigma^2+6K\sigma_g^2)- \frac{\eta\eta_l}{2Kn^2}\sum_{t=1}^T\mathbb E\big[\|\sum_{i=1}^n\sum_{k=1}^K \nabla f_i(\theta_{t,i}^{(k)})\|^2\big] \\
    &\hspace{0.5in} +\frac{T\eta^2\eta_l^2 KL}{m}\sigma^2 + \frac{\eta^2\eta_l^2L}{n(n-1)}\sum_{t=1}^T\mathbb E\big[\|\sum_{i=1}^n\sum_{k=1}^K \nabla f_i(\theta_{t,i}^{(k)})\|\big]^2\\
    &+ \frac{3\eta^2\eta_l^2C'K^2L(30\eta_l^2K^2L^2+1)}{m} \sum_{t=1}^T\mathbb E\big[\|\nabla f(\theta_t)\|^2\big]+\frac{15T\eta^2\eta_l^4C'K^3L^3}{m}(\sigma^2+6K\sigma_g^2)+\frac{3T\eta^2\eta_l^2C'K^2L}{m}\sigma_g^2  \\
    & + \frac{8\eta^2\eta_l^2q^2L}{(1-q^2)^2m^2}\sum_{t=1}^T\mathbb E\big[\| \sum_{i=1}^n \sum_{k=1}^K \nabla f_i(\theta_{t,i}^{(k)})\|^2\big] + \frac{8T\eta^2\eta_l^2q^2KLn}{(1-q^2)^3m^2}\sigma^2\\
    &\hspace{0.2in}+\frac{560T\eta^2\eta_l^2q^2(n-m)KL}{(1-q^2)^3m^2}(\sigma^2+6K\sigma_g^2) + \frac{3360\eta^2\eta_l^2q^2(n-m)K^2L}{(1-q^2)^3m^2}\sum_{t=1}^T \mathbb E\big[\|\nabla f(\theta_{t})\|^2\big].
\end{align*}
When the learning rate is chosen such that
\begin{align*}
    \eta_l\leq \min\Big\{ \frac{1}{6},\frac{m}{96C'\eta},\frac{m^2}{53760(n-m) C_1\eta}, \frac{1}{4\eta}, \frac{1}{32C_1\eta}\Big\}\frac{1}{KL},
\end{align*}
we can get
\begin{align*}
     \frac{1}{T}\sum_{t=1}^T\mathbb E\big[\| \nabla f(\theta_t) \|^2\big]&\lesssim \frac{f(\theta_1)-f(\theta^*)}{\eta\eta_l TK}+ \Big[ \frac{\eta\eta_l L}{m}+\frac{8\eta\eta_l C_1Ln}{m^2} \Big]\sigma^2 + \frac{3\eta\eta_lC'KL}{m}\sigma_g^2\\
    &\hspace{0.1in} + \Big[ \frac{5\eta_l^2KL^2}{2}+\frac{15\eta\eta_l^3C'K^2L^3}{m}+\frac{560\eta\eta_lC_1(n-m)L}{m^2} \Big](\sigma^2+6K\sigma_g^2),
\end{align*}
where $C_1=q^2/(1-q^2)^3$. Denote $B=n/m$. When choosing $\eta=\Theta(\sqrt{Km})$, $\eta_l=\Theta(\frac{1}{K\sqrt{TB}})$, the rate can be further bounded by
\begin{align*}
    \frac{1}{T}\sum_{t=1}^T\mathbb E\big[\| \nabla f(\theta_t) \|^2\big]&=\mathcal O\Big( \frac{\sqrt B (f(\theta_1)-f(\theta^*))}{\sqrt{TKm}} + (\frac{1}{\sqrt{TKmB}} +\frac{\sqrt B}{\sqrt{TKm}})\sigma^2+\frac{\sqrt{K}}{\sqrt{TmB}}\sigma_g^2\\
    &\hspace{0.3in} +(\frac{1}{TKB}+\frac{1}{T^{3/2}B^{3/2}\sqrt{Km}}+\frac{\sqrt B}{\sqrt{TKm}})(\sigma^2+6K\sigma_g^2) \Big),
\end{align*}
which can be further simplified by ignoring smaller terms as
\begin{align*}
    \frac{1}{T}\sum_{t=1}^T\mathbb E\big[\| \nabla f(\theta_t) \|^2\big]&=\mathcal O\Big(\frac{\sqrt n}{\sqrt m}\big( \frac{f(\theta_1)-f(\theta^*)}{\sqrt{TKm}} + \frac{1}{\sqrt{TKm}}\sigma^2 +\frac{\sqrt{K}}{\sqrt{Tm}}\sigma_g^2 \big)\Big).
\end{align*}
This completes the proof.
\end{proof}

\subsection{Convergence of Directly Using Biased Compressors}  \label{app sec: no EF}

\begin{proof}

We first clarify some (slightly modified) notations. Since we use the biased compression directly, there are no error compensation terms as in previous analysis. The update rule is simply
\begin{align*}
    \theta_{t+1} = \theta_t-\eta \overline{\widetilde\del}_t,
\end{align*}
where $\overline{\widetilde\del}_t=\frac{1}{n}\sum_{i=1}^n \widetilde\del_{t,i}\eqdef \frac{1}{n}\sum_{i=1}^n \mathcal C(\del_{t,i})$ is the average of compressed local model updates. Denote $b_{t,i}=\widetilde\del_{t,i}-\del_{t,i}$ as the difference (bias) between the true local model update of client $i$ in round $t$, and $\bar b=\frac{1}{n}\sum_{i=1}^n b_{t,i}$. By Assumption~\ref{ass:compress_diff}, we have
\begin{align}
    \mathbb E[\|\bar b_t\|^2]&=\mathbb E\big[\|\frac{1}{n}\sum_{i=1}^n \widetilde\del_{t,i}-\frac{1}{n}\sum_{i=1}^n \del_{t,i}\|^2 \big] \leq q^2\mathbb E[\|\bar\del_t\|^2].   \label{eqn:bound-b}
\end{align}
Our analysis starts with the smoothness Assumption~\ref{ass:smooth}, where
\begin{align*}
    f(\theta_{t+1})\leq f(\theta_t)+\langle \nabla f(\theta_t), \theta_{t+1}-\theta_t\rangle+\frac{L}{2}\| \theta_{t+1}-\theta_t\|^2.
\end{align*}
Taking expectation w.r.t. the randomness at round $t$ gives
\begin{align}
    &\mathbb E[f(\theta_{t+1})]-f(\theta_t)\leq -\eta\mathbb E\big[\langle \nabla f(\theta_t), \overline{\widetilde\del}_t \rangle\big] +\frac{\eta^2L}{2}\mathbb E\big[\|\overline{\widetilde\del}_t\|^2\big]. \label{sgd-eq0:nocomp}
\end{align}
The second term in (\ref{sgd-eq0:nocomp}) admits the following:
\begin{align*}
    \frac{\eta^2L}{2}\mathbb E\big[\|\overline{\widetilde\del}_t\|^2\big]&=\frac{\eta^2L}{2}\mathbb E\big[\|\bar\del_t+\bar b_t\|^2\big] \\
    &\leq (1+q^2)\eta^2 L\mathbb E[\|\bar\del_t\|^2] \\
    &\leq (1+q^2)\eta^2 L\Big[ (2\eta_l^2 K^2+120\eta_l^4 K^4 L^2)\mathbb E\big[\|\nabla f(\theta_t)\|^2\big]\\
    &\hspace{1.4in} +4\frac{\eta_l^2 K}{n}\sigma^2+20\eta_l^4K^3L^2(\sigma^2+6K\sigma_g^2) \Big],
\end{align*}
where the last inequality uses Lemma~\ref{lemma:bound delta}. We can bound the first term in (\ref{sgd-eq0:nocomp}) by
\begin{align*}
    &-\eta\mathbb E\big[\langle \nabla f(\theta_t), \overline{\widetilde\del}_t \rangle\big]\\
    &=-\eta\mathbb E\big[\langle \nabla f(\theta_t), \bar\del_t + \bar b_t -\eta_l K\nabla f(\theta_t)+\eta_l K\nabla f(\theta_t) \rangle\big] \\
    &=-\eta\eta_l K\mathbb E\big[\| \nabla f(\theta_t) \|^2\big] + \eta \underbrace{\mathbb E\big[\langle \nabla f(\theta_t),-\bar\del_t +\eta_l K\nabla f(\theta_t) \rangle\big]}_{\textbf{VI}}+ \eta \underbrace{\mathbb E\big[\langle \nabla f(\theta_t),-\bar b_t \rangle\big]}_{\textbf{VII}}.
\end{align*}
For the second term in the above, by Lemma~\ref{lemma:inner-product}, with $\eta_l\leq \frac{1}{8KL}$, we have
\begin{align*}
    \textbf{VI} &\leq
    \frac{3\eta_l K}{4}\mathbb E\big[\|\nabla f(\theta_t)\|^2\big]+\frac{5\eta_l^3 K^2 L^2}{2}(\sigma^2+6K\sigma_g^2).
\end{align*}
Regarding term \textbf{VII}, by Assumption~\ref{ass:compress_diff}, Young's inequality and (\ref{eqn:bound-b}), we obtain
\begin{align*}
    \textbf{VII} &\leq \frac{\eta_lK}{16} \mathbb E[\|\nabla f(\theta_t)\|^2]+\frac{16q^2}{\eta_l K} \mathbb E[\|\bar\del_t\|^2] \\
    &\leq \frac{\eta_lK}{16} \mathbb E[\|\nabla f(\theta_t)\|^2]+ 16q^2\Big[ (2\eta_l K+120\eta_l^3 K^3 L^2)\mathbb E\big[\|\nabla f(\theta_t)\|^2\big]\\
    &\hspace{2.2in} +4\frac{\eta_l }{n}\sigma^2+20\eta_l^3K^2L^2(\sigma^2+6K\sigma_g^2) \Big].
\end{align*}
When the learning rates admit $\eta_l\leq \frac{1}{8KL\max\{1,8(1+q^2)\eta\}}$ and $q\leq \frac{1}{32}$, taking the summation over all terms in (\ref{sgd-eq0:nocomp}) we get
\begin{align*}
    &\mathbb E[f(\theta_{t+1})]-f(\theta_t)\\
    &\leq -\frac{\eta\eta_l K}{8}\mathbb E\big[\|\nabla f(\theta_t)\|^2\big]
    +\frac{4\eta^2\eta_l^2(1+q^2) KL}{n}\sigma^2 +20\eta^2\eta_l^4(1+q^2)K^3L^3(\sigma^2+6K\sigma_g^2)\\
    &\hspace{2in} + \frac{64\eta\eta_lq^2}{n}\sigma^2+(320q^2+3)\eta\eta_l^3K^2L^2(\sigma^2+6K\sigma_g^2).
\end{align*}
We now take the telescope summation from round $1$ to $t$ and re-organize terms to obtain
\begin{align*}
    \frac{1}{T}\sum_{t=1}^T\mathbb E\big[\|\nabla f(\theta_t)\|^2\big]
    &\lesssim \frac{f(\theta_1)-\mathbb E[f(\theta_{t+1})]}{\eta\eta_l TK}+\frac{4\eta\eta_l(1+q^2) L}{n}\sigma^2 + \frac{64q^2}{Kn}\sigma^2 \\
    &+ 20\eta\eta_l^3(1+q^2)K^2L^3(\sigma^2+6K\sigma_g^2)+(320q^2+3)\eta_l^2K L^2(\sigma^2+6K\sigma_g^2).
\end{align*}
Set $\eta_l=\Theta(\frac{1}{K\sqrt T})$ and $\eta=\Theta(\sqrt{Kn})$, we have
\begin{align*}
    \frac{1}{T}\sum_{t=1}^T\mathbb E\big[\|\nabla f(\theta_t)\|^2\big]=\mathcal O\Big( \frac{1+q^2}{\sqrt{TKn}}+\frac{1+q^2}{TK}(\sigma^2+K\sigma_g^2)+\frac{q^2\sigma^2}{Kn} \Big).
\end{align*}
This completes the proof. Note that if we consider unbiased compressors, i.e., $\mathbb E[b_{t,i}|\del_{t,i}]=0$, $\forall t,i$, then term \textbf{VII} equals zero which eliminates the bias term in the final convergence rate.
\end{proof}

\end{document}